\documentclass[10pt,journal,compsoc]{IEEEtran}

\ifCLASSOPTIONcompsoc
  \usepackage[nocompress]{cite}
\else
  \usepackage{cite}
\fi
\usepackage{wrapfig}
\usepackage{amsmath}
\usepackage{amsfonts}
\usepackage{bm}
\usepackage{subcaption}
\usepackage{mathtools}
\usepackage{amsthm}
\usepackage{multirow}
\usepackage{xcolor}
\usepackage{booktabs}
\usepackage{nicefrac}
\usepackage{url}
\usepackage{hyperref}
\usepackage{amssymb}
\usepackage{algorithm}
\usepackage{algorithmic}
\usepackage{enumitem}

\ifCLASSINFOpdf
\else
\fi
\usepackage{array}

\usepackage{stfloats}
\newcommand{\embV}{\mathbf{V}}

\newcommand{\embh}{\mathbf{h}}
\newcommand{\embH}{\mathbf{H}}
\newcommand{\embm}{\mathbf{m}}
\newcommand{\embg}{\mathbf{g}}
\newcommand{\embX}{\mathbf{X}}
\newcommand{\embx}{\mathbf{x}}

\newcommand{\embv}{\mathbf{v}}

\newcommand{\hisH}{\overline{\mathbf{H}}}
\newcommand{\hish}{\overline{\mathbf{h}}}
\newcommand{\hisV}{\overline{\mathbf{V}}}
\newcommand{\hisv}{\overline{\mathbf{v}}}

\newcommand{\vecx}{\mathbf{x}}
\newcommand{\vecy}{\mathbf{y}}
\newcommand{\vecd}{\mathbf{d}}
\newcommand{\vecM}{\mathbf{M}}

\newcommand{\temH}{\widehat{\embH}}
\newcommand{\temh}{\widehat{\embh}}
\newcommand{\temV}{\widehat{\embV}}
\newcommand{\temv}{\widehat{\embv}}

\newcommand{\rmd}{\,\mathrm{d}}

\newcommand{\thetarec}{\theta^{\Diamond}}

\newcommand{\update}{u}
\newcommand{\aggregate}{\oplus}
\newcommand{\loss}{\mathcal{L}}
\newcommand{\inbatch}{\mathcal{V}_{\mathcal{B}}}

\newcommand{\jacobian}{\mathbf{J}}
\newcommand{\dropout}{\mathbf{Dropout}}
\newcommand{\neighbor}[1]{\mathcal{N}(#1)}
\newcommand{\kneighbor}[2]{\mathcal{N}^{#2}(#1)}

\newcommand{\acc}[2]{#1\textcolor{black!70!white}{\scriptsize{$\pm$#2}}}
\newcommand{\mr}[2]{\multirow{#1}{*}{#2}}
\newcommand{\mc}[3]{\multicolumn{#1}{#2}{#3}}
\newcommand{\udfsection}[1]{\noindent\textbf{#1}\, }

\newtheorem{assumption}{Assumption}
\newtheorem{lemma}{Lemma}

\newtheorem{theorem}{Theorem}
\newtheorem{remark}{Remark}

\newif\ifproof\prooftrue

\newif\ifupdate\updatefalse 
\newcommand{\modify}[2]{\ifupdate{#1}\else{\color{red}#2}\fi}

\newif\ifupdateok\updateoktrue
\newcommand{\modifyok}[2]{\ifupdate{#1}\else{\color{black}#2}\fi}

\newif\ifswitch\switchtrue
\begin{document}
%
\title{Provably Convergent Subgraph-wise Sampling for Fast GNN Training}

%
%
%
%

\author{Jie~Wang,~\IEEEmembership{Senior Member,~IEEE,} Zhihao~Shi,~Xize~Liang,~Defu~Lian,~\IEEEmembership{Member,~IEEE,}~Shuiwang~Ji,~\IEEEmembership{Fellow,~IEEE,}Bin~Li,~\IEEEmembership{Member,~IEEE,}~Enhong~Chen,~\IEEEmembership{Fellow,~IEEE,}and~Feng~Wu,~\IEEEmembership{Fellow,~IEEE}
\IEEEcompsocitemizethanks{
\IEEEcompsocthanksitem Jie Wang is with MoE Key Laboratory of Brain-inspired Intelligent Perception and Cognition, University of Science and Technology of China, Hefei 230027, China. (Corresponding author, e-mail: jiewangx@ustc.edu.cn).
\IEEEcompsocthanksitem J. Wang, Z. Shi, X. Liang, B.~Li, and F. Wu are with: a) CAS Key Laboratory of Technology in GIPAS, University of Science and Technology of China, Hefei 230027, China; b) Institute of Artificial Intelligence, Hefei Comprehensive National Science Center, Hefei 230091, China. E-mail: jiewangx@ustc.edu.cn, zhihaoshi@mail.ustc.edu.cn, xizeliang@mail.ustc.edu.cn, binli@ustc.edu.cn, fengwu@ustc.edu.cn.
\IEEEcompsocthanksitem Zhihao Shi, Xize Liang, Defu Lian, Enhong Chen, Bin Li, and Feng Wu are with University of Science and Technology of China, Hefei 230027, China. E-mail: \{zhihaoshi,xizeliang\}@mail.ustc.edu.cn, \{liandefu,cheneh,binli,fengwu\}@ustc.edu.cn.
\IEEEcompsocthanksitem S. Ji is with the Department of Computer Science and Engineering, Texas A\&M University, College Station, TX77843 USA. E-mail: sji@tamu.edu.
}

\thanks{Manuscript received April 19, 2005; revised August 26, 2015.}}

%
%

\markboth{Journal of \LaTeX\ Class Files,~Vol.~14, No.~8, August~2015}%
{Shell \MakeLowercase{\textit{et al.}}: Bare Advanced Demo of IEEEtran.cls for IEEE Computer Society Journals}
%



\IEEEtitleabstractindextext{%
\begin{abstract}
Subgraph-wise sampling---a promising class of mini-batch training techniques for graph neural networks (GNNs)---is critical for real-world applications.
During the message passing (MP) in GNNs, subgraph-wise sampling methods discard messages outside the mini-batches in backward passes to avoid the well-known \textit{neighbor explosion} problem, i.e., the exponentially increasing dependencies of nodes with the number of MP iterations.
However, discarding messages may sacrifice the gradient estimation accuracy, posing significant challenges to their convergence analysis and convergence speeds.
To address this challenge, we propose a novel subgraph-wise sampling method with a convergence guarantee, namely \textbf{L}ocal \textbf{M}essage \textbf{C}ompensation (LMC).
To the best of our knowledge, LMC is the \textit{first} subgraph-wise sampling method with provable convergence.
The key idea is to retrieve the discarded messages in backward passes based on a message passing formulation of backward passes.
By efficient and effective compensations for the discarded messages in both forward and backward passes, LMC computes accurate mini-batch gradients and thus accelerates convergence.
Moreover, LMC is applicable to various MP-based GNN architectures, including convolutional GNNs (finite message passing iterations with different layers) and recurrent GNNs (infinite message passing iterations with a shared layer).
Experiments on large-scale benchmarks demonstrate that LMC is significantly faster than state-of-the-art subgraph-wise sampling methods.
\end{abstract}

}

\maketitle

\IEEEdisplaynontitleabstractindextext

%
\IEEEpeerreviewmaketitle

\ifCLASSOPTIONcompsoc
\IEEEraisesectionheading{\section{Introduction}\label{sec:introduction}}
\else
\section{Introduction}
\label{sec:introduction}
\fi

%
%
%
%

\IEEEPARstart{G}{raph} neural networks (GNNs) are powerful frameworks that iteratively generate node representations depending on the structures of graphs and any feature information we might have \cite{grl}.
Recently, GNNs have achieved great success in many real-world applications involving graph-structured data, such as search engines \cite{gnn_recommendation}, recommendation systems \cite{gnn_social}, materials engineering \cite{gnn_material}, and molecular property prediction \cite{gnn_mol1, gnn_mol2}.

Many GNNs use the message passing framework \cite{mpnn}. Message passing (MP) follows an iterative \textit{aggregation} and \textit{update} scheme.
At each MP iteration, GNNs aggregate messages from each node's neighborhood and then update node embeddings based on aggregation results.
Many graph neural networks for node property prediction are categorized into convolutional graph neural networks (ConvGNNs) and recurrent graph neural networks (RecGNNs) \cite{comprehensive}.
ConvGNNs \cite{gcn, gat, gin} perform $T$ MP iterations with different graph convolutional layers to learn node embeddings. 
However, despite their simplicity, the finite MP iterations cannot capture the dependency with a range longer than $T$-hops away from any given node \cite{ignn}.
Inspired by \cite{idl, deq}, many researchers have focused on RecGNNs recently, which approximate infinite MP iterations recurrently using a shared MP layer.
Many works \cite{recgnn1, contraction, fdgnn, ignn} demonstrate that RecGNNs can effectively capture long-range dependencies.

However, when the graph scale is large, the iterative MP scheme poses challenges to the training of GNNs.
To scale deep learning models (such as GNNs) to arbitrarily large-scale data with constrained GPU memory, a widely used approach is to approximate full-batch gradients with mini-batch gradients.
Nevertheless, in scenarios involving graph data, this approach may incur prohibitive computational costs for the loss across a mini-batch of nodes and the corresponding mini-batch gradients due to the notorious \textit{neighbor explosion} problem.
Specifically, the embedding of a node at the $k$-th MP iteration recursively depends on the embeddings of its neighbors at the $(k-1)$-th MP iteration.
Thus, the computational complexity increases exponentially with the number of MP iterations.

To deal with the neighbor explosion problem, various sampling 
techniques have recently emerged to reduce the number of nodes involved in message passing \cite{dlg}.
For instance, node-wise \cite{graphsage, vrgcn} and layer-wise \cite{fastgcn, ladies, adapt} sampling methods recursively sample neighbors over MP iterations to estimate node embeddings and corresponding mini-batch gradients.
Different from the recursive fashion, subgraph-wise sampling methods \cite{cluster_gcn, graphsaint, gas, shadow_gnn} adopt a cheap and simple one-shot sampling fashion, i.e., sampling the same subgraph constructed based on a mini-batch for different MP iterations.
By getting rid of messages outside the mini-batches, subgraph-wise sampling methods restrict the computation to the mini-batches, leading to the complexity growing linearly with the number of MP iterations.
Moreover, by directly running GNNs on the subgraphs constructed by mini-batches, we can apply subgraph-wise sampling methods to a wide range of GNN architectures without additional designs \cite{gas}.
Due to the above advantages, subgraph-wise sampling has been increasingly popular recently.

Despite the empirical success of subgraph-wise sampling methods, discarding messages outside the mini-batch decreases the gradient estimation accuracy, posing significant challenges to their convergence behaviors as follows. First, recent works \cite{vrgcn, mvs} demonstrate that the inaccurate mini-batch gradients seriously slow down the convergence speeds of GNNs.
Second, in Sections \ref{sec:smallbatchsize_conv} and \ref{sec:smallbatchsize_rec}, we demonstrate that under small batch sizes, which is commonly seen when the GPU memory is limited, many subgraph-wise sampling methods are hard to resemble full-batch performance. These problems seriously limit the real-world applications of GNNs.

In this paper, we propose a novel subgraph-wise sampling method with a convergence guarantee, namely \textbf{L}ocal \textbf{M}essage \textbf{C}ompensation (LMC), whose efficient and effective compensations can reduce the bias of mini-batch gradients and thus accelerates convergence. To the best of our knowledge, LMC is the {\it first} subgraph-wise sampling method with provable convergence.
Specifically, by \modifyok{developing unbiased mini-batch gradients for the one-shot sampling fashion \modify{}{and proposing the compensation mechanism}, \modify{we propose a bias-variance decomposition of the gradient computation errors, i.e., the bias from the discarded messages and the variance of the unbiased mini-batch gradients}{we derive the bias-variance decomposition of gradient computation errors, and show that the bias terms can tend to be arbitrarily small with appropriate learning rates}
(see Theorems \ref{thm:grad_error_conv} and \ref{thm:grad_error_rec}).}{formulating backward passes as message passing, we first develop unbiased mini-batch gradients for the one-shot sampling fashion (see Theorem \ref{thm:unbiased}).}
Then, during the approximation of unbiased mini-batch gradients, we retrieve the messages discarded by existing subgraph-wise sampling methods based on \modifyok{a message passing formulation of backward passes}{the aforementioned message passing formulation of backward passes}.
Finally, to avoid the exponentially growing time and memory consumption, we propose efficient and effective compensations for the discarded messages with the historical information in previous iterations.
An appealing feature of the compensation mechanism is that it can effectively reduce the biases of mini-batch gradients, leading to accurate gradient estimation and fast convergence speeds.
We further show that LMC converges to first-order stationary points of GNNs.
Notably, the convergence of LMC stems from the interactions between mini-batch nodes and their $1$-hop neighbors, without the need for recursive dependencies on neighborhoods far away from the mini-batches.

Another appealing feature is that LMC is applicable to MP-based GNN architectures, including ConvGNNs and RecGNNs \cite{comprehensive}.
The major challenges of ConvGNNs (finite message passing iterations with different layers) and RecGNNs (infinite message passing iterations with a shared layer) are layer-wise accumulated errors from the staleness of historical information and the convergence of \modifyok{infinite message passing iterations without a priori limitations on the number of MP layers}{recurrent embeddings without a priori limitations on the number of MP iterations} \cite{ignn, comprehensive, eignn}, respectively.
We implement LMC for the two aforementioned models respectively to tackle the problems. Experiments on large-scale benchmark tasks demonstrate that LMC significantly outperforms state-of-the-art subgraph-wise sampling methods in terms of efficiency. Moreover, under small batch sizes, LMC outperforms the baselines and resembles the prediction performance of full-batch methods.

An earlier version of this paper has been published at ICLR 2023 \cite{lmc}, which mainly focuses on ConvGNNs. This journal manuscript significantly extends the conference version to cover another popular GNN architectures with infinite layers, i.e., RecGNNs.
\modifyok{}{The major contribution of this extension is that we show the convergence of LMC under infinite MP iterations.}
To demonstrate that LMC significantly accelerates the training of RecGNNs, we provide detailed derivations, rigorous theoretical analysis, and extensive experiments in Sections \ref{subsubsec:lmc4rec_methodology}, \ref{subsubsec:lmc4rec_theoretical}, and \ref{sec:exp_recgnns}, respectively.

\section{Related Work}

In this section, we discuss some works related to our proposed method.

\subsection{Scalable Training for Graph Neural Networks}

In this section, we discuss the scalable training for graph neural networks including subgraph-wise and recursive (node-wise and layer-wise) sampling methods that are most related to our proposed method. Please refer to Appendix \ref{appendix:add_related} for other scalable training methods.

\udfsection{Subgraph-wise Sampling Methods.} Subgraph-wise sampling methods sample a mini-batch and then construct the same subgraph based on it for different MP layers \cite{dlg}.
For example, Cluster-GCN \cite{cluster_gcn} and GraphSAINT \cite{graphsaint} construct the subgraph induced by a sampled mini-batch. They encourage connections between the sampled nodes by graph clustering methods (e.g., METIS \cite{metis1} and Graclus \cite{graclus}), edge, node, or random-walk based samplers.
GNNAutoScale (GAS) \cite{gas} and MVS-GNN \cite{mvs} use historical embeddings to generate messages outside a sampled subgraph, maintaining the expressiveness of the original GNNs.
Subgraph-wise sampling methods are applicable to both ConvGNNs and RecGNNs by directly running them on the subgraphs constructed by the sampled mini-batches \cite{gas}.

\udfsection{Recursive Graph Sampling Methods.}
Both node-wise and layer-wise sampling methods recursively sample neighbors over MP layers and then construct different computation graphs for each MP iteration.
Node-wise sampling methods \cite{graphsage, vrgcn} aggregate messages from a small subset of sampled neighborhoods at each MP layer to decrease the bases in the exponentially increasing dependencies.
To avoid the exponentially growing computation, layer-wise sampling methods \cite{fastgcn, ladies, adapt} independently sample nodes for each MP layer and then use importance sampling to reduce variance, resulting in a constant sample size in each MP layer.
Many node-wise and layer-wise sampling methods are designed only for ConvGNNs. Because node-wise and layer-wise sampling methods may introduce random noise to the message passing iterations of RecGNNs, it is difficult to solve the fixed-point equations robustly (see Eqs. \eqref{eqn:transformation_rec} and \eqref{eqn:mpeq_grad}), posing challenges to the robust training for RecGNNs.

\subsection{Historical Values as an Affordable Approximation.} The historical values are affordable approximations of the exact values in practice. However, they suffer from frequent data transfers to/from the GPU and the staleness problem. For example, in node-wise sampling, VR-GCN \cite{vrgcn} uses historical embeddings to reduce the variance from neighbor sampling \cite{graphsage}.
GAS \cite{gas} proposes a concurrent mini-batch execution to transfer the active historical embeddings to and from the GPU, leading to comparable runtime with the standard full-batch approach.
GraphFM-IB and GraphFM-OB \cite{graphfm} apply a momentum step on historical embeddings for node-wise and subgraph-wise sampling methods with historical embeddings, respectively, to alleviate the staleness problem.
Both LMC and GraphFM-OB use the node embeddings in mini-batches to alleviate the staleness problem of the node embeddings outside mini-batches.
We discuss the main differences between LMC and GraphFM-OB in Appendix \ref{sec:diff_graphfm}.

\section{Preliminaries}
We introduce notations in Section \ref{sec:notations}. Then, we introduce convolutional graph neural networks and recurrent graph neural networks in Section \ref{sec:convgnn}.

\subsection{Notations}\label{sec:notations}
A graph {\small $\mathcal{G}=(\mathcal{V}, \mathcal{E})$} is defined by a set of nodes {\small$\mathcal{V}=\{v_1,v_2,\dots,v_n\}$} and a set of edges {\small $\mathcal{E}$} among these nodes. The set of nodes consists of labeled nodes {\small $\mathcal{V}_{L}$} and unlabeled nodes {\small $\mathcal{V}_{U}:=\mathcal{V} \setminus \mathcal{V}_{L}$}. The label of a node $v_i \in \mathcal{V}_L$ is $y_i$. Let {\small $(v_i,v_j)\in\mathcal{E}$} denote an edge going from node {\small $v_i\in\mathcal{V}$} to node {\small $v_j\in\mathcal{V}$}, {\small $\neighbor{v_i}=\{v_j\in\mathcal{V}| (v_i,v_j)\in\mathcal{E}\}$} denote the neighborhood of node {\small $v_i$}, and {\small $\overline{\mathcal{N}}(v_i)$} denote {\small$\neighbor{v_i} \cup \{v_i\}$}. We assume that {\small $\mathcal{G}$} is undirected, i.e., {\small $v_j \in \neighbor{v_i} \Leftrightarrow v_i \in \neighbor{v_j}$}. Let {\small $\neighbor{\mathcal{S}} = \{v\in\mathcal{V}| (v_i,v_j)\in\mathcal{E},v_i\in\mathcal{S}\}$} denote the neighborhoods of a set of nodes {\small $\mathcal{S}$} and {\small $\overline{\mathcal{N}}(\mathcal{S})=\neighbor{\mathcal{S}} \cup \mathcal{S}$}. For a positive integer {\small $L$}, {\small$[L]$} denotes {\small$\{1,\ldots,L\}$}. Let the boldface character {\small $\embx_{i} \in \mathbb{R}^{d_x}$} denote the feature of node {\small $v_i$} with dimension {\small $d_x$}. Let {\small $\embh_i\in\mathbb{R}^d$} be the {\small $d$}-dimensional embedding of the node {\small $v_i$}. Let {\small $\embX = (\embx_1,\embx_2,\dots,\embx_n) \in \mathbb{R}^{d_x \times n}$} and {\small $\embH  = (\embh_1,\embh_2,\dots,\embh_n) \in \mathbb{R}^{d \times n}$}. We also denote the embeddings of a set of nodes {\small $\mathcal{S}=\{v_{i_k}\}_{k=1}^{|\mathcal{S}|}$} by {\small $\embH_{\mathcal{S}}  = (\embh_{i_1}, \embh_{i_2}, \dots, \embh_{i_{|\mathcal{S}|}}) \in \mathbb{R}^{d \times |\mathcal{S}|}$}. For a {\small $p \times q$} matrix {\small $\mathbf{A}\in\mathbb{R}^{p\times q}$}, {\small $\Vec{\mathbf{A}} \in \mathbb{R}^{pq}$} denotes the vectorization of {\small $\mathbf{A}$}, i.e., {\small $\Vec{\mathbf{A}}_{i+(j-1)p} = \mathbf{A}_{ij}$}. We denote the {\small$j$}-th columns of {\small$\mathbf{A}$} by {\small $\mathbf{A}_{j}$}.




\subsection{Graph Neural Networks}\label{sec:convgnn}
For semi-supervised node-level prediction, Graph Neural Networks (GNNs) aim to learn node embeddings $\embH$ with parameters $\Theta$ by minimizing the objective function $\mathcal{L}=\frac{1}{|\mathcal{V}_L|}\sum_{i\in\mathcal{V}_L}\ell_{w}(\embh_i, y_i)$ such that
\begin{align}
\embH = \mathcal{GNN}(\embX, \mathcal{E};\Theta),
\end{align}
where $\ell_{w}$ is the composition of an output layer with parameters $w$ and a loss function.

GNNs follow the message passing framework in which vector messages are exchanged between nodes and updated using neural networks. Most GNNs are categorized into convolutional graph neural networks (ConvGNNs) and recurrent graph neural networks (RecGNNs) based on whether different message passing iterations share the same parameters \cite{comprehensive}.
Notably, using the same parameters potentially enable infinite message passing iterations \cite{ignn}.

An {\small$L$}-layer ConvGNN performs {\small$L$} message passing iterations with different parameters {\small$\Theta=(\theta^{l})_{l=1}^L$} to generate the final node embeddings {\small$\embH=\embH^{L}$} as
\begin{align}
    \embH^{l}=f_{\theta^{l}}(\embH^{l-1};\embX),\,\,l\in[L], \label{eqn:transformation_conv}
\end{align}
where {\small$\embH^{0}=\embX$} and {\small$f_{\theta^l}$} is the {\small$l$}-th message passing layer with parameters {\small$\theta^{l}$}.

Different from ConvGNNs, RecGNNs recurrently use the message passing layer \eqref{eqn:transformation_conv} with the same parameter $\theta^{l}=\thetarec, l \in \mathbb{N}^*$ until node embeddings $\embH^{l}$ converge to the fixed point
\begin{align}
    \embH^{\Diamond}=f_{\thetarec}(\embH^{\Diamond};\embX), \label{eqn:transformation_rec}
\end{align}
as the final node embeddings $\embH = \embH^{\Diamond}$.
Recent work \cite{ignn} shows that the recurrent embeddings of message passing \eqref{eqn:transformation_rec} converges if the well-posedness property holds (see Appendix \ref{sec:well-posedness}).
Compared with ConvGNNs, RecGNNs enjoy cheaper memory costs because of the shared parameters in different message passing iterations and can model dependencies between nodes that are any hops apart \cite{ignn}.
However, the training of RecGNNs is more unstable and inefficient than ConvGNNs due to the challenge in solving the fixed point equations \cite{ignn, eignn}.

The message passing layer $f_{\theta^{l}}$ follows an \textit{aggregation} and \textit{update} scheme, i.e.,
\begin{align}
    &\embh_i^l  =\update_{\theta^{l}}(\embh^{l-1}_i  , \embm^{l-1}_{\neighbor{v_i}}  ,\embx_i); \nonumber\\
    &\embm^{l-1}_{\neighbor{v_i}}  =\aggregate_{\theta^{l}}( \{g_{\theta^{l}}(\embh_j^{l-1}  ) \mid v_j\in\neighbor{v_i}\})\label{eqn:mpeq_update},
\end{align}
where $g_{\theta^{l}}$ is the function generating \textit{individual messages} for each neighborhood of $v_i$ of the $l$-th message passing iteration, $\aggregate_{\theta^{l}}$ is the aggregation function mapping from a set of messages to the final message $\embm^{l-1}_{\neighbor{v_i}}$, and $\update_{\theta^{l}}$ is the update function that combines previous node embedding $\embh_i^{l-1}$, message $\embm_{\neighbor{v_i}}^{l-1}$, and features $\embx_i$ to update node embeddings.
For the consistency of the notations of ConvGNNs and RecGNNs, we denote the parameters of RecGNNs in Eq. \eqref{eqn:mpeq_update} by $(\theta^l)_{l=1}^\infty$ with $\theta^l = \thetarec,\, l \in \mathbb{N}^{*}$.

\begin{figure}[t]
\centering 
\begin{subfigure}{0.50\textwidth}
  \includegraphics[width=260pt]{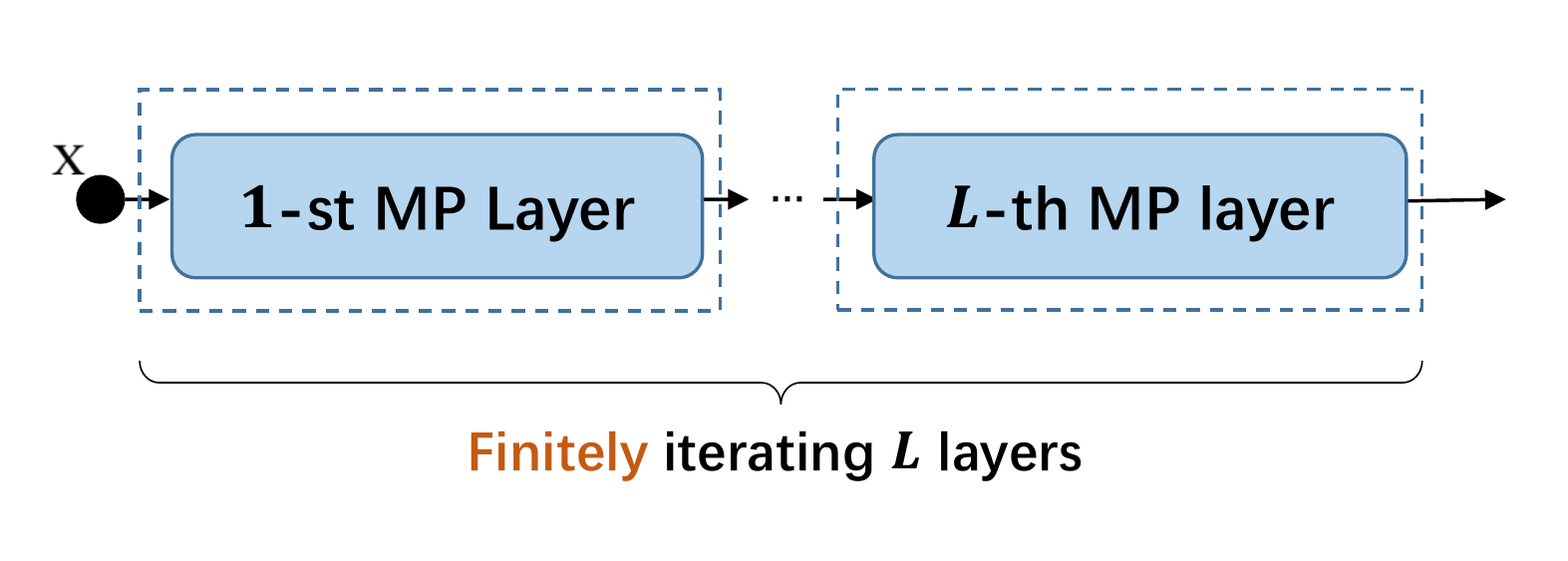}
  \caption{ConvGNNs}\label{subfig:ConvGNN}
\end{subfigure}\hfil
\begin{subfigure}{0.50\textwidth}
  \includegraphics[width=260pt]{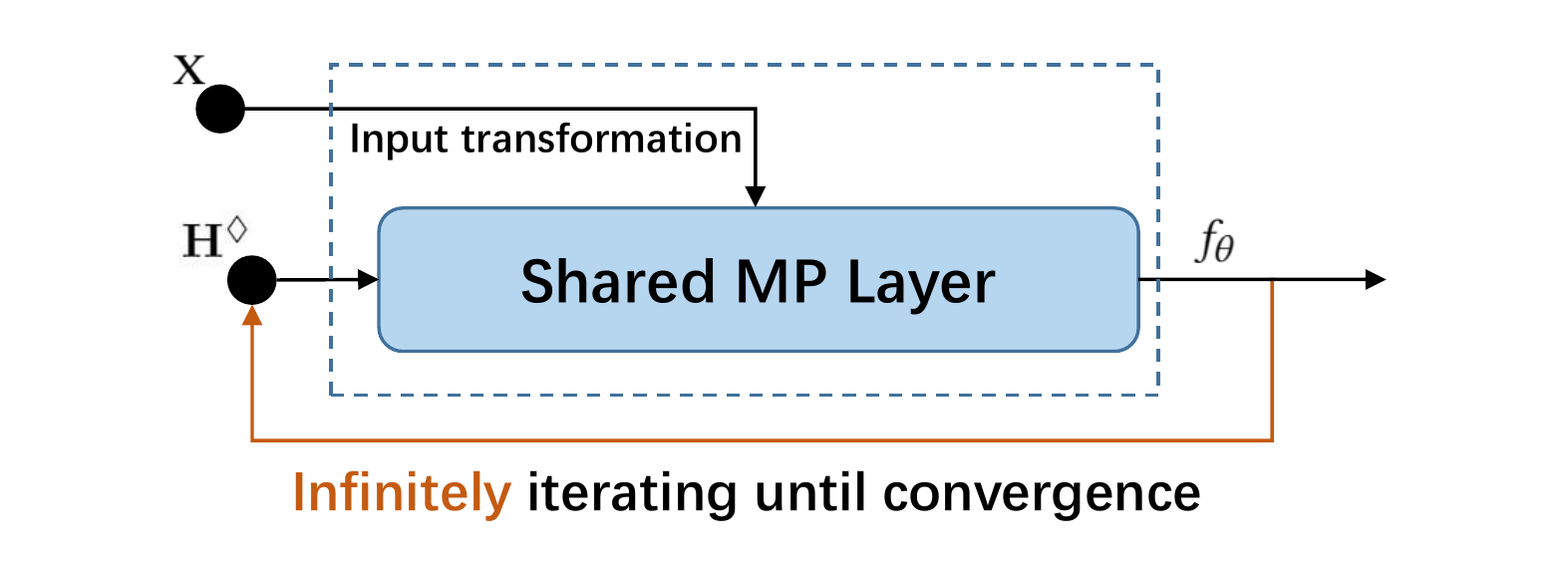}
    \caption{RecGNNs}\label{subfig:RecGNN}
\end{subfigure}\hfil 
\caption{
The architectures of ConvGNNs and RecGNNs. We denote message passing by \textit{MP}.} \label{fig:model}
\end{figure}

\section{Message Passing in Backward Passes}\label{sec:grad}

We introduce the gradients of ConvGNNs and RecGNNs in Sections \ref{sec:grad_convgnn} and \ref{sec:grad_recgnn}, respectively, and then formulate the corresponding backward passes as message passing.
Finally, in Section \ref{sec:naive_sgd}, we introduce an SGD variant---backward SGD, which provides unbiased gradient estimations based on the message passing formulation of backward passes.

\subsection{Backward Passes of ConvGNNs}\label{sec:grad_convgnn}

The gradient {\small$\nabla_{w} \mathcal{L}$} is easy to compute and we hence introduce the chain rule to compute {\small$\nabla_{\Theta} L$} in this section, where {\small$\Theta=(\theta^{l})_{l=1}^{L}$}. Let {\small$\embV^{l}\triangleq \nabla_{\embH^{l}} \mathcal{L}$} for {\small$l\in[L]$} be auxiliary variables. It is easy to compute {\small$\vec{\embV}^{L} = \nabla_{\vec{\embH}^{L}} \mathcal{L} = \nabla_{\vec{\embH}} \mathcal{L}$}. By the chain rule, we iteratively compute {\small$\embV^{l}$} based on {\small$\embV^{l+1}$} as
\begin{align}
    \vec{\embV}^{l}=\vec{\phi}_{\theta^{l+1}}(\embV^{l+1})\triangleq(\nabla_{\vec{\embH}^{l}}\vec{f}_{\theta^{l+1}}) \vec{\embV}^{l+1}, \label{eqn:auxiliary_recursion}
\end{align}
and
\begin{align}
    \embV^{l}=\phi_{\theta^{l+1}} \circ \cdots \circ
    \phi_{\theta^{L}}(\embV^{L}). \label{eqn:auxiliary_compute}
\end{align}
Then, we compute the gradient {\small$\nabla_{\theta^{l}} \mathcal{L}= (\nabla_{\theta^l} \vec{f}_{\theta^l})\vec{\embV}^{l},\,l\in[L]$} by autograd packages for vector-Jacobian product.

\udfsection{Message Passing Formulation of Backward Passes for ConvGNNs.} Combining Eqs. \eqref{eqn:auxiliary_recursion} and \eqref{eqn:mpeq_update} leads to
\begin{align}
    \embV^{l}_i=\sum_{v_j\in\neighbor{v_i}}(\nabla_{\embh^{l}_i} \update_{\theta^{l+1}}(\embh^{l}_j, \embm^{l}_{\neighbor{v_j}}  ,\embx_j))\embV^{l+1}_j,\,\, i\in[n], \label{eqn:mpeq_auxiliary}
\end{align}
where \modify{{\small$\embV_k^l$} is the {\small$k$}-th column of {\small$\embV^l$} and}{} {\small$\embm^{l}_{\neighbor{v_j}}$} is a function of {\small$\embh^{l}_i$} defined in Eq. \eqref{eqn:mpeq_update}. Eq. \eqref{eqn:mpeq_auxiliary} uses {\small$(\nabla_{\embh^{l}_i} \update_{\theta^{l+1}}(\embh^{l}_j, \embm^{l}_{\neighbor{v_j}}  ,\embx_j))\embV_j^{l+1}$}, sum aggregation, and the identity mapping as the generation function, the aggregation function, and the update function, respectively.

\subsection{Backward Passes of RecGNNs}\label{sec:grad_recgnn}

We compute the gradients of RecGNNs through implicit differentiation \cite{ignn, deq}.
Similar to ConvGNNs, the gradients {\small$\nabla_w \loss$} and {\small$\nabla_{\embH^{\Diamond} } \loss$} are easy to compute.
By the chain rule and implicit differentiation, we can compute the Jacobian of parameters {\small$\thetarec$} by {\small$\frac{\partial \loss}{\partial \thetarec} = \frac{\partial \loss}{\partial \Vec{\embH}^{\Diamond} } \frac{\partial \Vec{\embH}^{\Diamond} }{\partial \thetarec} = \frac{\partial \loss}{\partial \Vec{\embH}^{\Diamond} }(I - \jacobian_{\Vec{\embH}^{\Diamond}}\Vec{f}_{\thetarec})^{-1}\frac{\partial \Vec{f}_{\thetarec}}{\partial \thetarec}$}.
As the inverse of {\small$I - \jacobian_{\Vec{\embH}^{\Diamond}}\Vec{f}_{\thetarec}$} is intractable, we compute auxiliary variables {\small$(\Vec{\embV}^{\Diamond})^{\top} = \frac{\partial \loss}{\partial \Vec{\embH}^{\Diamond} }(I - \jacobian_{\Vec{\embH}^{\Diamond}}\Vec{f}_{\thetarec})^{-1}$} by solving the fixed-point equations
\begin{align}
    \Vec{\embV}^{\Diamond} = \Vec{\phi}_{\thetarec,w}(\embV^{\Diamond}) \,\,\triangleq\,\,  (\nabla_{\Vec{\embH}^{\Diamond}}\Vec{f}_{\thetarec})\Vec{\embV}^{\Diamond}  + \nabla_{\Vec{\embH}^{\Diamond}}\loss \label{eqn:mpeq_grad}.
\end{align}
If the well-posedness property \cite{ignn} holds (see Appendix \ref{sec:well-posedness}), we solve the fixed-point equations \eqref{eqn:mpeq_grad} by iterative solvers. We then use autograd packages to compute the vector-Jacobian product {\small$\nabla_{\thetarec} \loss = (\nabla_{\thetarec} \vec{f}_{\thetarec})\Vec{\embV}^{\Diamond}$}.

\udfsection{Message Passing Formulation of Backward Passes for RecGNNs.} Combining Eqs. \eqref{eqn:mpeq_grad} and \eqref{eqn:mpeq_update} leads to
\begin{align}
    &\embV_{i}^{\Diamond} = \sum_{v_j \in\neighbor{v_i}} (\nabla_{\embh_i^{\Diamond}}\update_{\thetarec}(\embh_j^{\Diamond} ,\embm_{\neighbor{v_j}}^{\Diamond}  ,\embx_j) ) \embV_{j}^{\Diamond}+ \nabla_{\embh_i^{\Diamond}} \loss \label{eqn:mpeq_grad_node}
\end{align}
for $i\in [n]$, where {\small $\embV_k^{\Diamond}$} is the $k$-th column of {\small $\embV^{\Diamond}$} and {\small $\embm_{\neighbor{v_j}}^{\Diamond}$} is a function of {\small$\embh_i^{\Diamond}$} defined in Eq. \eqref{eqn:mpeq_update}.
Eq. \eqref{eqn:mpeq_grad_node} uses {\small$(\nabla_{\embh_i^{\Diamond}}\update_{\thetarec}(\embh_j^{\Diamond},\embm_{\neighbor{v_j}}^{\Diamond}  ,\embx_j) ) \embV_{j}^{\Diamond}$}, sum aggregation, and {\small ${\rm update}_{i}(\cdot)=\cdot + \nabla_{\embh_i^{\Diamond}} \loss$} as the generation function, the aggregation function, and the update function, respectively.

\subsection{Backward SGD}  \label{sec:naive_sgd}

In this section, we develop an SGD variant---backward SGD, which provides unbiased gradient estimations based on the message passing formulation of backward passes.
Backward SGD helps retrieve the messages
discarded by existing subgraph-wise sampling methods (see Theorems \ref{thm:grad_error_conv} and \ref{thm:grad_error_rec}).

Given a sampled mini-batch $\mathcal{V}_{\mathcal{B}}$, suppose that we have computed exact node embeddings and auxiliary variables of nodes in $\mathcal{V}_{\mathcal{B}}$, i.e., $(\embH^{l}_{\mathcal{V}_{\mathcal{B}}}, \embV^{l}_{\mathcal{V}_{\mathcal{B}}})_{l=1}^{L}$ for ConvGNNs or $(\embH^{\Diamond}_{\mathcal{V}_{\mathcal{B}}}, \embV^{\Diamond}_{\mathcal{V}_{\mathcal{B}}})$ for RecGNNs.
To simplify the analysis, we assume that $\mathcal{V}_{\mathcal{B}}$ is uniformly sampled from $\mathcal{V}$ and the corresponding set of labeled nodes $\mathcal{V}_{L_{\mathcal{B}}} := \mathcal{V}_{\mathcal{B}} \cap \mathcal{V}_L$ is uniformly sampled from $\mathcal{V}_L$.
When the sampling is not uniform, we use the normalization technique \cite{graphsaint} to enforce the assumption (please refer to Appendix \ref{sec:normalization}).

First, backward SGD computes the mini-batch gradient $\mathbf{g}_{w}(\mathcal{V}_\mathcal{B})$ for parameters ${w}$ by the derivative of mini-batch loss $\mathcal{L}_{\mathcal{V}_\mathcal{B}}=\frac{1}{|\mathcal{V}_{L_{\mathcal{B}}}|}\sum_{v_j\in\mathcal{V}_{L_\mathcal{B}}}\ell_{w}(\embh_j,y_j)$, i.e.,
\begin{align}
    \mathbf{g}_{w}(\mathcal{V}_\mathcal{B})=\frac{1}{|\mathcal{V}_{L_{\mathcal{B}}}|}\sum_{v_j\in\mathcal{V}_{L_\mathcal{B}}} \nabla_{w} \ell_{w}(\embh_j,y_j).\label{eqn:mini-batch_grad_w}
\end{align}
Then, for ConvGNNs, backward SGD computes the mini-batch gradient $\mathbf{g}_{\theta^{l}}(\mathcal{V}_\mathcal{B})$ for parameters $\theta^{l}$, $l \in [L]$ as
\begin{align}
    &\mathbf{g}_{\theta^{l}}(\mathcal{V}_\mathcal{B})=\frac{|\mathcal{V}|}{|\mathcal{V}_\mathcal{B}|}\sum_{v_j\in\mathcal{V}_\mathcal{B}}(\nabla_{\theta^{l}}u_{\theta^l}(\embh^{l-1}_j, \embm^{l-1}_{\neighbor{v_j}}, \embx_j))\embV^{l}_j,\label{eqn:mini-batch_grad_theta}
\end{align}
For RecGNNs, backward SGD replaces $\theta^{l}$, $\embH^{l-1}$, $\embV^{l}$ with $\theta^{\Diamond}$, $\embH^{\Diamond}$, $\embV^{\Diamond}$ in Eq. \eqref{eqn:mini-batch_grad_theta} to compute $\mathbf{g}_{\theta^{\Diamond}}(\mathcal{V}_{\mathcal{B}})$.

Note that the mini-batch gradients $\mathbf{g}_{\theta^l}(\mathcal{V}_{\mathcal{B}})$ for different $l\in[L]$ are based on the same mini-batch $\mathcal{V}_{\mathcal{B}}$, which facilitates designing subgraph-wise sampling methods\modify{based on backward SGD.}{.}
Another appealing feature of backward SGD is that the mini-batch gradients $\mathbf{g}_w(\mathcal{V}_{\mathcal{B}})$, $\mathbf{g}_{\theta^{\Diamond}}(\mathcal{V}_{\mathcal{B}})$, and $\mathbf{g}_{\theta^l}(\mathcal{V}_{\mathcal{B}})$, $l\in [L]$ are unbiased, as shown in the following theorem. Please see Appendix \ref{appendix:proof_thm1} for the detailed proof.

\begin{theorem}\label{thm:unbiased}
    Suppose that a subgraph {\small $\inbatch$} is uniformly sampled from {\small $\mathcal{V}$} and the corresponding labeled nodes {\small$\mathcal{V}_{L_{\mathcal{B}}} = \inbatch \cap \mathcal{V}_{L}$} is uniformly sampled from {\small$\mathcal{V}_{L}$}. Then the mini-batch gradients {\small $\mathbf{g}_w(\mathcal{V}_{\mathcal{B}})$}, {\small $\mathbf{g}_{\theta^{\Diamond}}(\mathcal{V}_{\mathcal{B}})$}, and {\small $\mathbf{g}_{\theta^l}(\mathcal{V}_{\mathcal{B}})$}, {\small$l\in [L]$} in Eqs. \eqref{eqn:mini-batch_grad_w} and \eqref{eqn:mini-batch_grad_theta} are unbiased.
\end{theorem}

\begin{figure*}[t]
\begin{subfigure}{0.3\textwidth}
  \includegraphics[width=180pt]{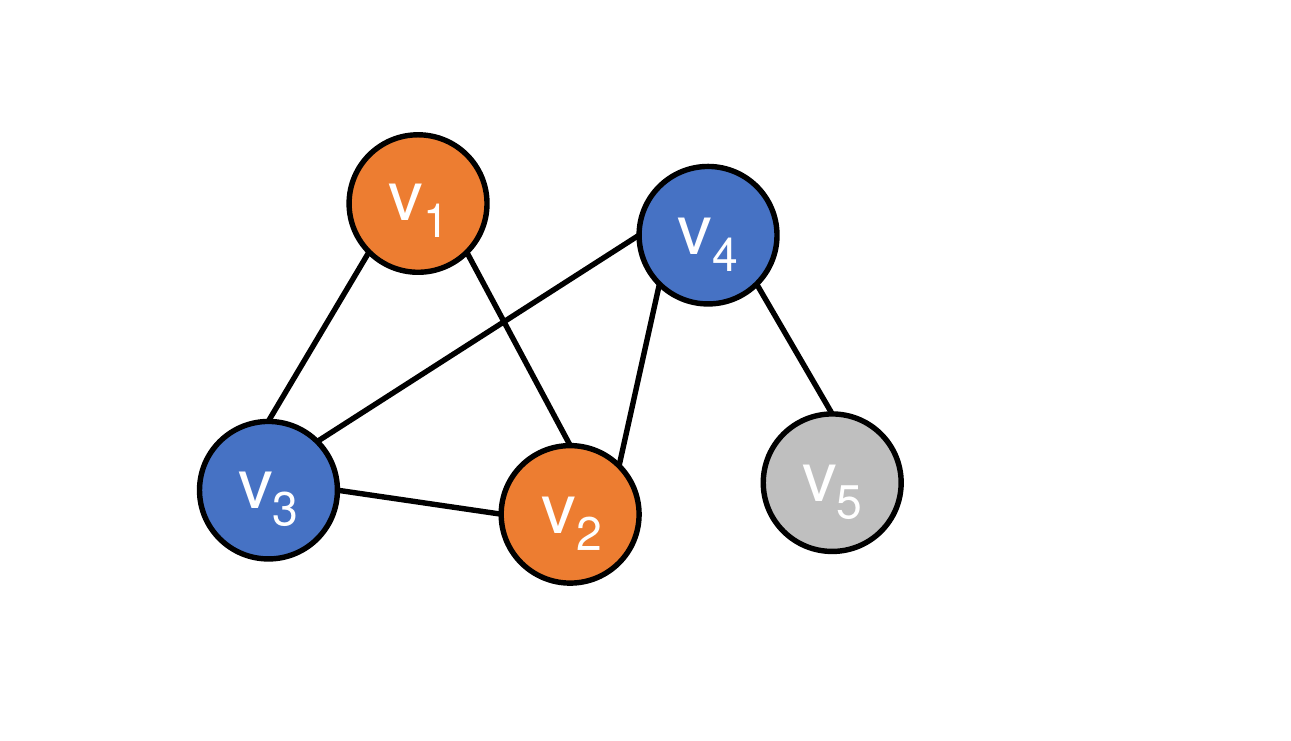}
  \caption{Original graph}\label{subfig:graph_conv}
\end{subfigure}\hfil
\begin{subfigure}{0.3\textwidth}
  \includegraphics[width=180pt]{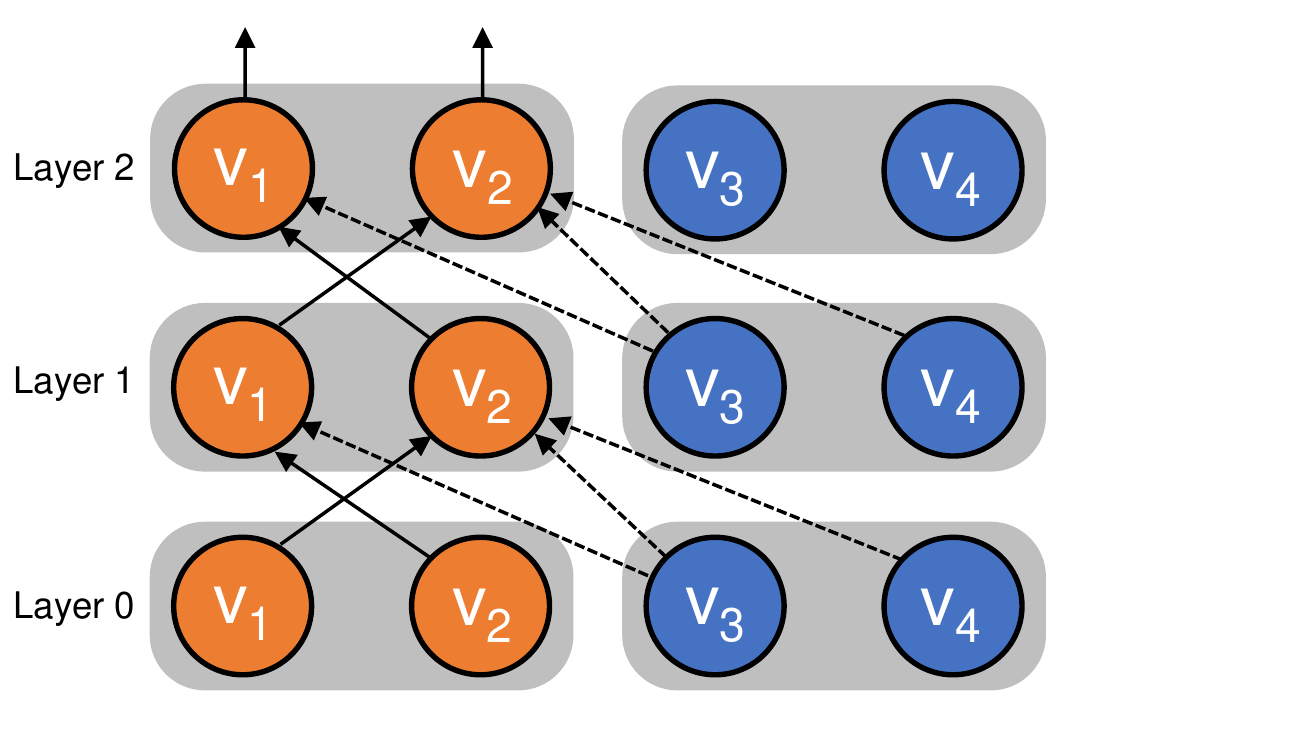}
    \caption{Forward passes of GAS}\label{subfig:gas_forward}
\end{subfigure}\hfil 
\begin{subfigure}{0.33\textwidth}
  \includegraphics[width=180pt]{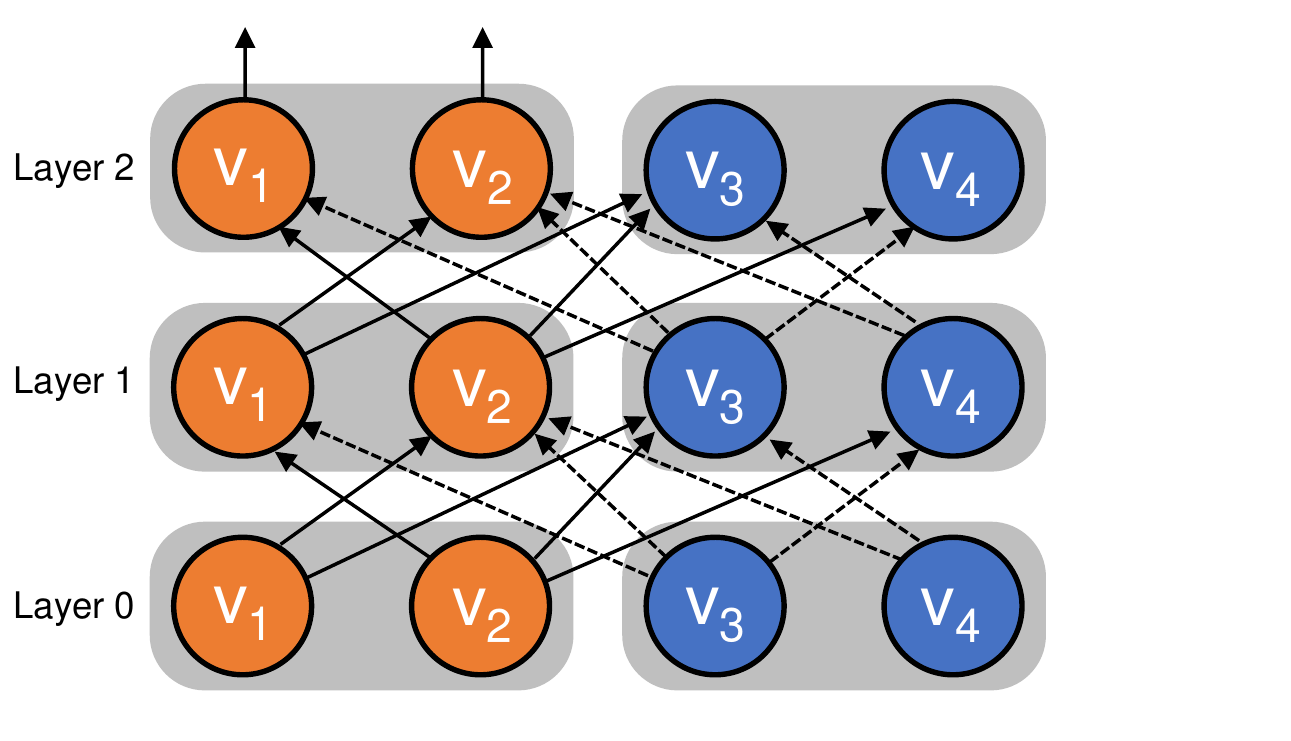}
  \caption{Forward passes of LMC}\label{subfig:lmc_forward}
\end{subfigure}\hfil
\begin{subfigure}{0.33\textwidth}
  \includegraphics[width=180pt]{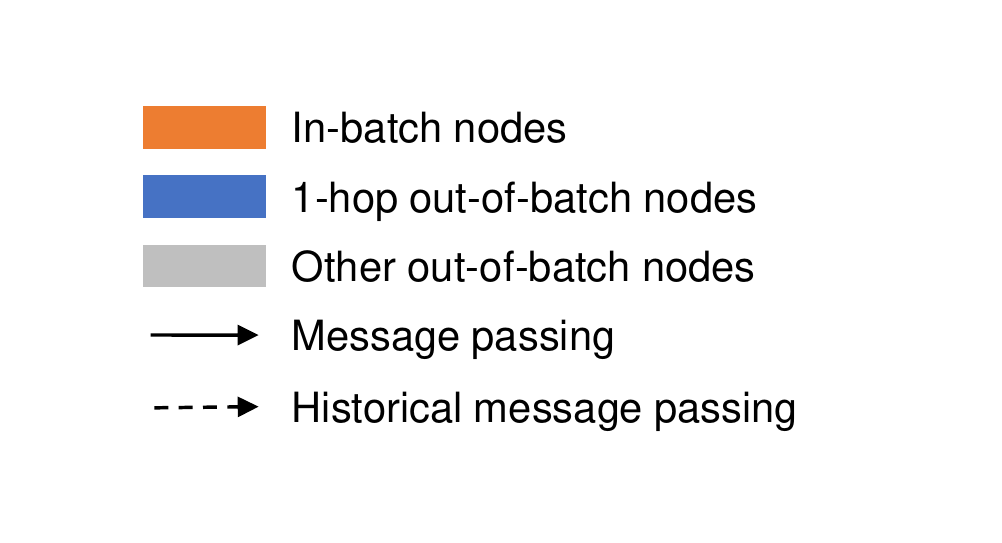}
\end{subfigure}\hfil
\begin{subfigure}{0.33\textwidth}
  \includegraphics[width=180pt]{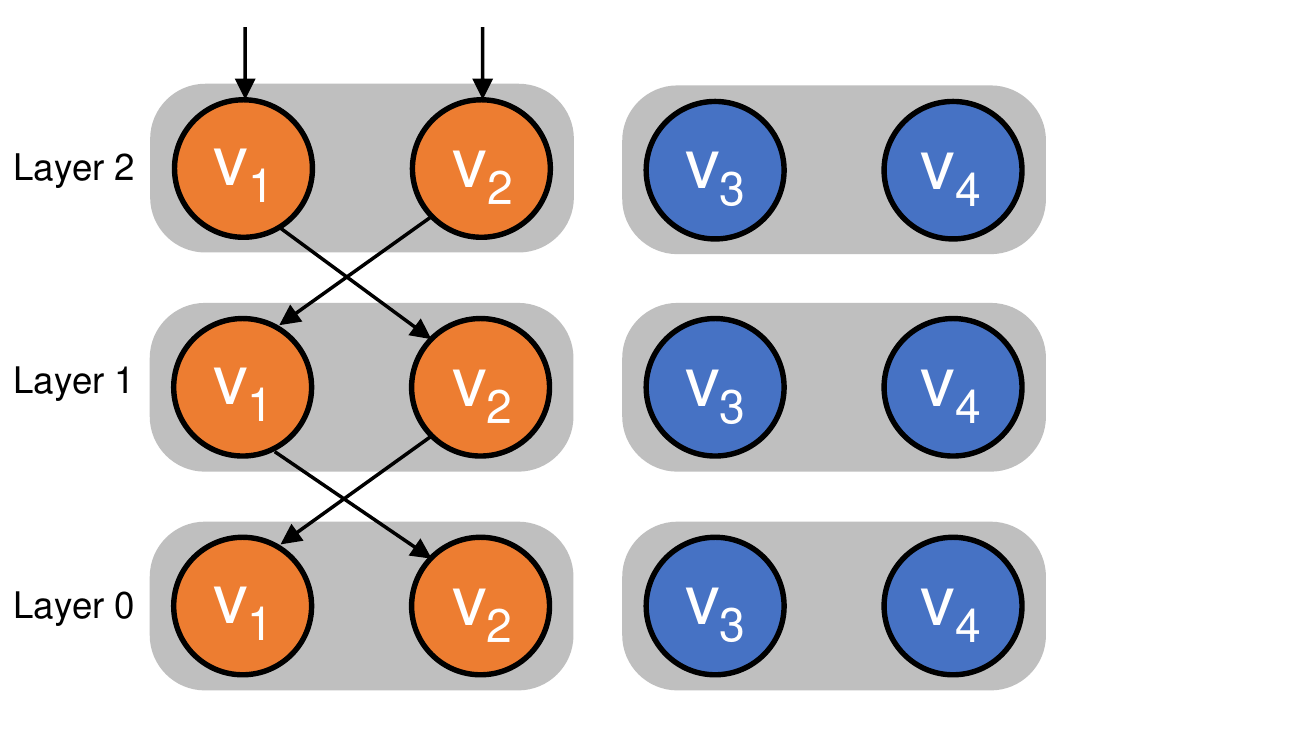}
    \caption{Backward passes of GAS}\label{subfig:gas_backward}
\end{subfigure}\hfil
\begin{subfigure}{0.33\textwidth}
  \includegraphics[width=180pt]{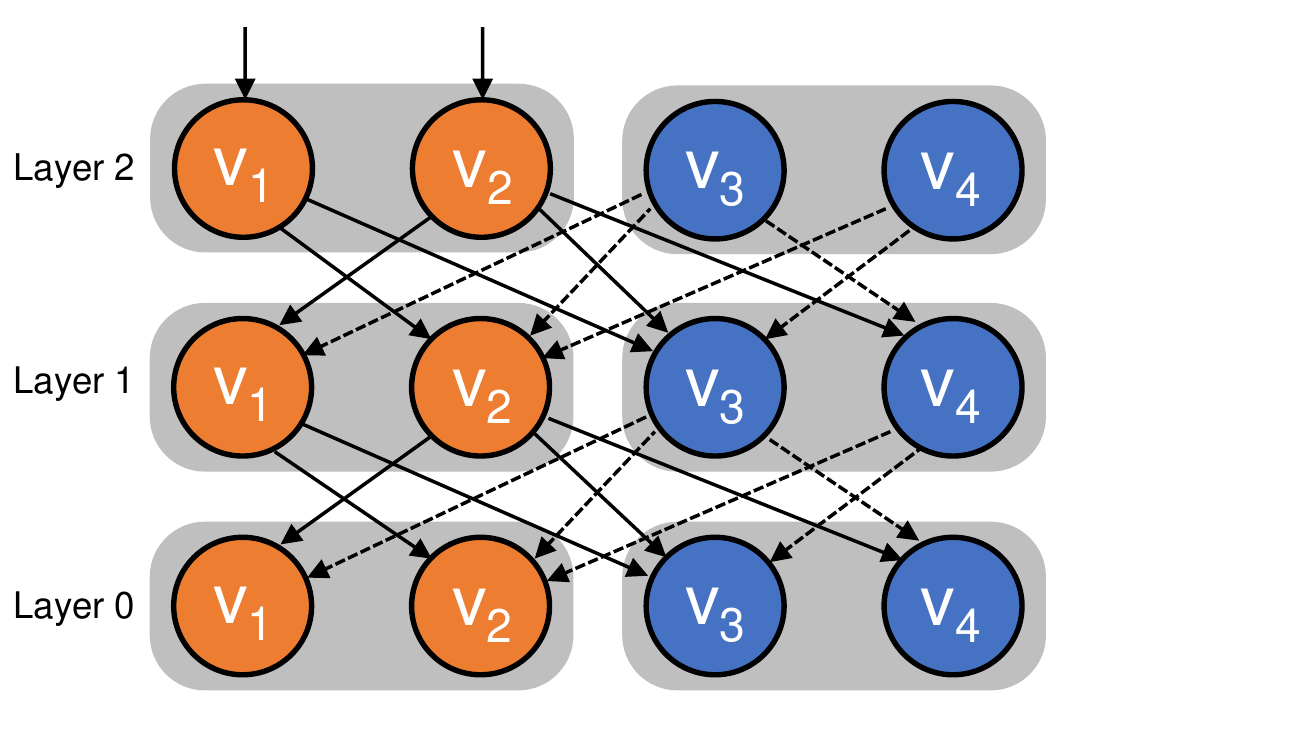}
    \caption{Backward passes of LMC}\label{subfig:lmc_backward}
\end{subfigure}\hfil
\caption{
Comparison of LMC4Conv with GAS \cite{gas}. (a) shows the original graph with in-batch nodes, 1-hop out-of-batch nodes, and other out-of-batch nodes in orange, blue, and grey, respectively. (b) and (d) show the computation graphs of forward passes and backward passes of GAS, respectively. (c) and (e) show the computation graphs of forward passes and backward passes of LMC4Conv, respectively.} \label{fig:computation}
\end{figure*}

\section{Local Message Compensation}
The exact mini-batch gradients {\small$\mathbf{g}_{w}(\mathcal{V}_\mathcal{B})$}, {\small$\mathbf{g}_{\theta^{\Diamond}}(\mathcal{V}_{\mathcal{B}})$}, and {\small$\mathbf{g}_{\theta^l}(\mathcal{V}_\mathcal{B}),\,l\in[L]$} computed by backward SGD depend on exact embeddings and auxiliary variables of nodes in the mini-batch {\small$\mathcal{V}_{\mathcal{B}}$} rather than the whole graph. However, backward SGD is not scalable, as the exact  {\small$(\embH^l_{\inbatch},\embV^l_{\inbatch})_{l=1}^L$} and {\small$(\embH^{\Diamond}_{\inbatch},\embV^{\Diamond}_{\inbatch})$} are expensive to compute due to the {\it neighbor explosion} problem. To deal with this problem, we develop a novel and scalable subgraph-wise sampling method for ConvGNNs and RecGNNs, namely \textbf{L}ocal \textbf{M}essage \textbf{C}ompensation (LMC). We introduce the methodologies of LMC for ConvGNNs (LMC4Conv) and LMC for RecGNNs (LMC4Rec) in Sections \ref{subsubsec:lmc4conv_methodology} and \ref{subsubsec:lmc4rec_methodology}, respectively. Further, we show that LMC converges to first-order stationary points of ConvGNNs and RecGNNs in Sections \ref{subsubsec:lmc4conv_theoretical} and \ref{subsubsec:lmc4rec_theoretical}, respectively.

\subsection{LMC for ConvGNNs}\label{subsec:lmc4conv}

In Algorithm \ref{alg:lmc4conv} and Section \ref{subsubsec:lmc4conv_theoretical}, we denote a vector or a matrix in the $l$-th layer at the $k$-th iteration by $(\cdot)^{l,k}$, but elsewhere we omit the superscript $k$ and simply denote it by $(\cdot)^{l}$.

\subsubsection{Methodology}\label{subsubsec:lmc4conv_methodology}

We store {\it historical embeddings and auxiliary variables} {\small$(\hisH^{l},\hisV^{l})_{l=1}^L$} to provide an affordable approximation.
At each training iteration, we sample a mini-batch of nodes {\small$\inbatch$} and update historical embeddings and auxiliary variables of nodes in the mini-batch, i.e., {\small$(\hisH^{l}_{\inbatch},\hisV^{l}_{\inbatch})_{l=1}^L$}.
We use updated {\small$(\hisH^{l}_{\inbatch},\hisV^{l}_{\inbatch})_{l=1}^L$} to approximate exact $(\embH^{l}_{\mathcal{V}_{\mathcal{B}}}, \embV^{l}_{\mathcal{V}_{\mathcal{B}}})_{l=1}^{L}$ and then use Eqs. \eqref{eqn:mini-batch_grad_w}, \eqref{eqn:mini-batch_grad_theta} to compute mini-batch gradients {\small$\widetilde{\mathbf{g}}_w$} and {\small$\widetilde{\mathbf{g}}_{\theta^{l}},\,l\in[L]$}.

\udfsection{Update of {\small$(\hisH^{l}_{\inbatch})_{l=1}^L$}.} In forward passes, we initialize the historical embeddings for {\small$l=0$} as {\small$\hisH^0_{\inbatch} = \embX_{\inbatch}$}.
At the {\small$l$}-th layer, we update the historical embedding of each node {\small$v_i \in \inbatch$} as
\begin{align}
    & \overline{\embh}_i^l = \update_{\theta^l}(\overline{\embh}_i^{l-1}, \overline{\embm}_{\neighbor{v_i}}^{l-1}, \embx_i);\nonumber\\
    & \overline{\embm}_{\neighbor{v_i}}^{l-1} = \aggregate_{\theta^l}( \{ g_{\theta^l}(\overline{\embh}_j^{l-1}) \mid v_j \in \neighbor{v_i} \cap \inbatch \}\nonumber\\
    &\quad\quad\quad\quad\quad\quad\quad \cup \{ g_{\theta^l}(\widehat{\embh}_j^{l-1}) \mid v_j \in \neighbor{v_i} \setminus \inbatch \} ). \label{eqn:mini_mpeq_update}
\end{align}
We call {\small$\mathbf{C}_f^{l} \triangleq \aggregate_{\theta^{l}}(\{ g_{\theta^{l}}(\widehat{\embh}_j^{l-1}) \mid v_j\in\neighbor{v_i}\setminus \inbatch \})$} the {\it local message compensation} in the {\small$l$}-th layer in forward passes, where {\small $\widehat{\embh}_j^{l-1}$} is the {\it temporary embedding} of node $v_j$ outside the mini-batch.
We denote the temporary embeddings of nodes outside the mini-batch by {\small$\widehat{\embH}^{l}_{\neighbor{\inbatch} \setminus \inbatch}$}.

\udfsection{Computation of {\small$\widehat{\embH}^{l}_{\neighbor{\inbatch} \setminus \inbatch}$}.} We initialize the temporary and historical embeddings for {\small$l=0$} as {\small$\widehat{\embH}^0_{\neighbor{\inbatch} \setminus \inbatch} = \hisH_{\neighbor{\inbatch} \setminus \inbatch}^0 = \embX_{\neighbor{\inbatch} \setminus \inbatch}$}.
We compute the temporary embedding of each neighbor {\small$v_i \in \neighbor{\inbatch} \setminus \inbatch$} as
\begin{align}
    &\widehat{\embh}_i^{l} = (1-\beta_{i}) \overline{\embh}_i^{l} + \beta_{i} \widetilde{\embh}_i^{l},\label{eqn:temp_compute}
\end{align}
where {\small$\beta_{i}\in [0,1]$} is the convex combination coefficient for node $v_i$, and
\begin{align}
    &\widetilde{\embh}_i^l = \update_{\theta^l}(\widehat{\embh}_i^{l-1}, \overline{\embm}_{\neighbor{v_i}}^{l-1}, \embx_i);\nonumber\\ 
    &\overline{\embm}_{\neighbor{v_i}}^{l-1} = \aggregate_{\theta^l} (\{ g_{\theta^l}(\overline{\embh}_j^{l-1}) \mid v_j\in\neighbor{v_i} \cap \inbatch \}\nonumber\\
    &\quad\quad\quad\quad\cup \{g_{\theta^l}(\widehat{\embh}_j^{l-1}) \mid v_j \in \neighbor{\inbatch}\cap\neighbor{v_i} \setminus \inbatch \} ).
    \label{eqn:subexact_compute}
\end{align}
Notably, the neighbors of node {\small$v_i \in \neighbor{\inbatch} \setminus \inbatch$} may contain the nodes in {\small$\kneighbor{\inbatch}{2} \setminus \neighbor{\inbatch}$}, which we prune in the computation of messages {\small $\overline{\embm}_{\neighbor{v_i}}^{l-1}$} to avoid the neighbor explosion problem.
Thus, we call $\widetilde{\embh}_i^l$ {\it incomplete up-to-date embeddings}.
The effectiveness of the convex combination is based on a observation that {\small$\neighbor{\inbatch}$} covers most the $1$-hop neighbors of most nodes {\small$v_i \in \neighbor{\inbatch} \setminus \inbatch$} if $|\inbatch|$ is large.
For these nodes, the pruning errors are very small and we can tune $\beta_i$ to encourage them to be close to the incomplete up-to-date embeddings. We provide the selection of $\beta_i$ in Appendix \ref{sec:selection_beta}.

\udfsection{Update of {\small$(\hisV^{l}_{\inbatch})_{l=1}^L$}.}
In backward passes, we initialize the {\it historical auxiliary variables} for {\small$l=L$} as {\small$\hisV^L_{\inbatch} = \nabla_{\hisH_{\inbatch}} \loss$.} We update the historical auxiliary variable of each {\small$v_i\in\inbatch$} as
\begin{align}
    \hisV^{l}_i=&\sum_{v_j\in\neighbor{v_i} \cap \inbatch}(\nabla_{\embh^{l}_i} \update_{\theta^{l+1}}(\overline{\embh}^{l}_j, \overline{\embm}^{l}_{\neighbor{v_j}}  ,\embx_j))\overline{\embV}^{l+1}_j\nonumber\\
    &+ \sum_{v_j \in \neighbor{v_i} \setminus \inbatch } (\nabla_{\embh^{l}_i} \update_{\theta^{l+1}}(\widehat{\embh}^{l}_j, \overline{\embm}^{l}_{\neighbor{v_j}}  ,\embx_j))\widehat{\embV}^{l+1}_j, \label{eqn:mini_mpeq_auxiliary}
\end{align}
where 
{\small$\overline{\embh}_j^l$}, {\small$\overline{\embm}^l_{\neighbor{v_j}}$}, and {\small$\widehat{\embh}_j^l$}
are computed as shown in Eqs. \eqref{eqn:mini_mpeq_update}--\eqref{eqn:subexact_compute}.
We call {\small$\mathbf{C}_b^{l} \triangleq \sum_{v_j \in \neighbor{v_i} \setminus \inbatch } (\nabla_{\embh^{l}_i} \update_{\theta^{l+1}}(\widehat{\embh}^{l}_j, \overline{\embm}^{l}_{\neighbor{v_j}}  ,\embx_j))\widehat{\embV}^{l+1}_j$} the {\it local message compensation} in the {\small$l$}-th layer in backward passes, where {\small $\widehat{\embV}_j^{l-1}$} is the {\it temporary auxiliary variable} of node $v_j$ outside the mini-batch.
We denote the temporary auxiliary variables of nodes outside the mini-batch by {\small$\widehat{\embV}^{l}_{\neighbor{\inbatch} \setminus \inbatch}$}.

\udfsection{Computation of {\small$\widehat{\embV}^{l}_{\neighbor{\inbatch} \setminus \inbatch}$}.}
We initialize the {\it hitorical auxiliary variables} for {\small$l=L$} as {\small$\widehat{\embV}^L_{\neighbor{\inbatch} \setminus \inbatch} = \hisV^L_{\neighbor{\inbatch} \setminus \inbatch} =\nabla_{\widehat{\embH}_{\neighbor{\inbatch} \setminus \inbatch}} \loss$}. We compute the temporary auxiliary variable of each neighbor {\small$v_i\in\neighbor{\inbatch} \setminus \inbatch$} as
\begin{align}
    \widehat{\embV}_i^l = (1-\beta_{i}) \overline{\embV}_i^l + \beta_{i} \widetilde{\embV}_i^l, \label{eqn:temp_compute_auxiliary}
\end{align}
where {\small$\beta_{i}$} is the convex combination coefficient used in Eq. \eqref{eqn:temp_compute}, and
\begin{align}
    \widetilde{\embV}^{l}_i={}&\sum_{ v_j\in\neighbor{v_i} \cap \inbatch}(\nabla_{\embh^{l}_i} \update_{\theta^{l+1}}(\overline{\embh}^{l}_j, \overline{\embm}^{l}_{\neighbor{v_j}}  ,\embx_j))\overline{\embV}^{l+1}_j\nonumber\\
    + &\sum_{v_j \in \neighbor{\inbatch}\cap\neighbor{v_i} \setminus \inbatch } (\nabla_{\embh^{l}_i} \update_{\theta^{l+1}}(\widehat{\embh}^{l}_j, \overline{\embm}^{l}_{\neighbor{v_j}}  ,\embx_j))\widehat{\embV}^{l+1}_j. \label{eqn:subexact_compute_auxiliary}
\end{align}
Similar to forward passes, we prune the nodes in {\small$\kneighbor{\inbatch}{2} \setminus \neighbor{\inbatch}$} in the computation of messages {\small $\overline{\embm}_{\neighbor{v_i}}^{l-1}$} to avoid the neighbor explosion problem.

\udfsection{Mini-batch Gradients of LMC4Conv.} 
Combining Eqs. \eqref{eqn:mini-batch_grad_w} and \eqref{eqn:mini-batch_grad_theta} with {\small$(\overline{\embh}_j^l, \overline{\embm}^l_{\neighbor{v_j}}, \overline{\embV}_j^l)_{l=0}^{L}$} leads to
\begin{align}
    &\widetilde{\mathbf{g}}_{w}(\mathcal{V}_\mathcal{B})=\frac{1}{|\mathcal{V}_{L_{\mathcal{B}}}|}\sum_{v_j\in\mathcal{V}_{L_\mathcal{B}}} \nabla_{w} \ell_{w}(\hish_j,y_j),\label{eqn:grad_conv_w}\\
    &\widetilde{\mathbf{g}}_{\theta^{l}}(\mathcal{V}_\mathcal{B})=\frac{|\mathcal{V}|}{|\mathcal{V}_\mathcal{B}|}\sum_{v_j\in\mathcal{V}_\mathcal{B}}(\nabla_{\theta^{l}}u_{\theta^l}(\hish^{l-1}_j, \overline{\embm}^{l-1}_{\neighbor{v_j}}, \embx_j))\hisV^{l}_j.\label{eqn:grad_conv_theta}
\end{align}

\udfsection{Time Complexity.} Notice that the total size of Eqs. \eqref{eqn:mini_mpeq_update}--\eqref{eqn:subexact_compute_auxiliary} is linear with {\small$|\neighbor{\inbatch}|$} rather than the size of the whole graph.
Suppose that the maximum neighborhood size is {\small$n_{\max}$} and the number of layers is {\small$L$}, then the time complexity in forward and backward passes is {\small$\mathcal{O}( L(n_{\max}|\inbatch|d+|\inbatch| d^2) )$}.

\udfsection{Space Complexity.} LMC4Conv additionally stores the historical node embeddings {\small$\overline{\embH}^l$} and auxiliary variables {\small$\overline{\embV}^l$} for {\small$l\in[L]$}. As pointed out in \cite{gas}, we can store the majority of historical values in RAM or hard drive storage rather than GPU memory. Thus, the active historical values in forward and backward passes employ {\small$\mathcal{O}(n_{\max} L|\inbatch| d)$ and {\small$
\mathcal{O}(n_{\max} L|\inbatch| d)$}} GPU memory, respectively. As the time and memory complexity are independent of the size of the whole graph, i.e., {\small$|\mathcal{V}|$}, LMC4Conv is scalable. We summarize the computational complexity in Appendix \ref{appendix:complexity}.

Fig. \ref{fig:computation} shows the message passing mechanisms of GAS \cite{gas} and LMC4Conv. Compared with GAS, LMC4Conv proposes compensation messages between in-batch nodes and their 1-hop neighbors simultaneously in forward and backward passes, which corrects the bias of mini-batch gradients and thus accelerates convergence.

\begin{algorithm}[H]
    \caption{LMC4Conv}
    \label{alg:lmc4conv}
    \begin{algorithmic}[1]
        \STATE {\bfseries Input:} 
        The learning rate {\small$\eta$} and the convex combination coefficients {\small$(\beta_i)_{i=1}^n$}.
        \STATE Partition {\small$\mathcal{V}$} into {\small$B$} parts {\small$(\mathcal{V}_{b})_{b=1}^B$} 
        \FOR{{\small$k = 1, \dots, N$}}
            \STATE Randomly sample {\small$\mathcal{V}_{b_k}$} from {\small$(\mathcal{V}_b)_{b=1}^B$}
            \STATE Initialize {\small$\hisH^{0,k}_{\neighbor{\mathcal{V}_{b_k}}} = \embX_{\neighbor{\mathcal{V}_{b_k}}}$}
            \STATE {\small Initialize $\widehat{\embH}^{0,k}_{\neighbor{\mathcal{V}_{b_k}} \setminus \mathcal{V}_{b_k}} = \embX_{\neighbor{\mathcal{V}_{b_k}} \setminus \mathcal{V}_{b_k}}$}
            \FOR{{\small$l=1,\dots,L$}}
                \STATE Update {\small$\overline{\embH}^{l,k}_{\mathcal{V}_{b_k}}$} by Eq. \eqref{eqn:mini_mpeq_update}
                \STATE Compute {\small$\widehat{\embH}^{l,k}_{\neighbor{\mathcal{V}_{b_k}} \setminus \mathcal{V}_{b_k}}$} by Eqs. \eqref{eqn:temp_compute} and \eqref{eqn:subexact_compute}
            \ENDFOR
            \STATE Initialize {\small$\hisV^L_{\mathcal{V}_{b_k}} = \nabla_{\hisH_{\mathcal{V}_{b_k}}} \loss$}
            \STATE {\small Initialize $\widehat{\embV}^L_{\neighbor{\mathcal{V}_{b_k}} \setminus \mathcal{V}_{b_k}} = \hisV^L_{\neighbor{\mathcal{V}_{b_k}} \setminus \mathcal{V}_{b_k}} =\nabla_{\widehat{\embH}_{\neighbor{\mathcal{V}_{b_k}} \setminus \mathcal{V}_{b_k}}} \loss$}
            \FOR{{\small$l=L-1,\dots,1$}}
                \STATE Update {\small$\overline{\embV}^{l,k}_{\mathcal{V}_{b_k}}$} by Eq. \eqref{eqn:mini_mpeq_auxiliary}
                \STATE Compute {\small$\widehat{\embV}^{l,k}_{\neighbor{\mathcal{V}_{b_k}} \setminus \mathcal{V}_{b_k}}$} by Eqs. \eqref{eqn:temp_compute_auxiliary} and \eqref{eqn:subexact_compute_auxiliary}
            \ENDFOR
            \STATE Compute {\small$\widetilde{\mathbf{g}}_w^k$} and {\small$\widetilde{\mathbf{g}}_{\theta^{l}}^k,\,l\in[L]$} by Eqs. \eqref{eqn:grad_conv_w} and \eqref{eqn:grad_conv_theta}
            \STATE Update parameters by\\
            \quad\quad\quad {\small$w^k = w^{k-1} - \eta \widetilde{\mathbf{g}}_w^k$}\\
            \quad\quad\quad {\small$\theta^{l,k} = \theta^{l,k-1} - \eta\widetilde{\mathbf{g}}_{\theta^{l}}^k,\,l\in[L]$}
        \ENDFOR
    \end{algorithmic}
\end{algorithm}

\par Algorithm \ref{alg:lmc4conv} summarizes LMC4Conv. We add a superscript {\small$k$} for each value to indicate that it is the value at the {\small$k$}-th iteration. At the preprocessing step, we partition {\small$\mathcal{V}$} into {\small$B$} parts {\small$(\mathcal{V}_b)_{b=1}^B$}.
At the {\small$k$}-th training step, LMC4Conv first randomly samples a subgraph constructed by {\small$\mathcal{V}_{b_k}$}. 
Then, LMC4Conv updates the stored historical node embeddings {\small$\overline{\embH}^{l,k}_{\mathcal{V}_{b_k}}$} in the order of {\small$l=1,\ldots,L$} by Eqs. \eqref{eqn:mini_mpeq_update}--\eqref{eqn:subexact_compute}, and the stored historical auxiliary variables {\small$\overline{\embV}^{l,k}_{\mathcal{V}_{b_k}}$} in the order of {\small$l=L-1,\ldots,1$} by Eqs. \eqref{eqn:mini_mpeq_auxiliary}--\eqref{eqn:subexact_compute_auxiliary}. By the random update, the historical values get close to the exact up-to-date values. Finally, for {\small$l\in[L]$} and {\small$v_j\in\mathcal{V}_{b_k}$}, by replacing {\small$\embh^{l,k}_j$}, {\small$\embm^{l,k}_{\neighbor{v_j}}$} and {\small$\embV^{l,k}_j$} in Eqs. \eqref{eqn:mini-batch_grad_w} and \eqref{eqn:mini-batch_grad_theta} with {\small$\overline{\embh}^{l,k}_j$}, {\small$\overline{\embm}^{l,k}_{\neighbor{v_j}}$}, and {\small$\overline{\embV}^{l,k}_j$}, respectively, LMC4Conv computes mini-batch gradients {\small$\widetilde{\mathbf{g}}_{w}^k,\widetilde{\mathbf{g}}_{\theta^1}^k,\ldots,\widetilde{\mathbf{g}}_{\theta^L}^k$} to update parameters {\small$w,\theta^1,\ldots,\theta^L$}.

\begin{figure*}[t]
\begin{subfigure}{0.33\textwidth}
  \includegraphics[width=170pt]{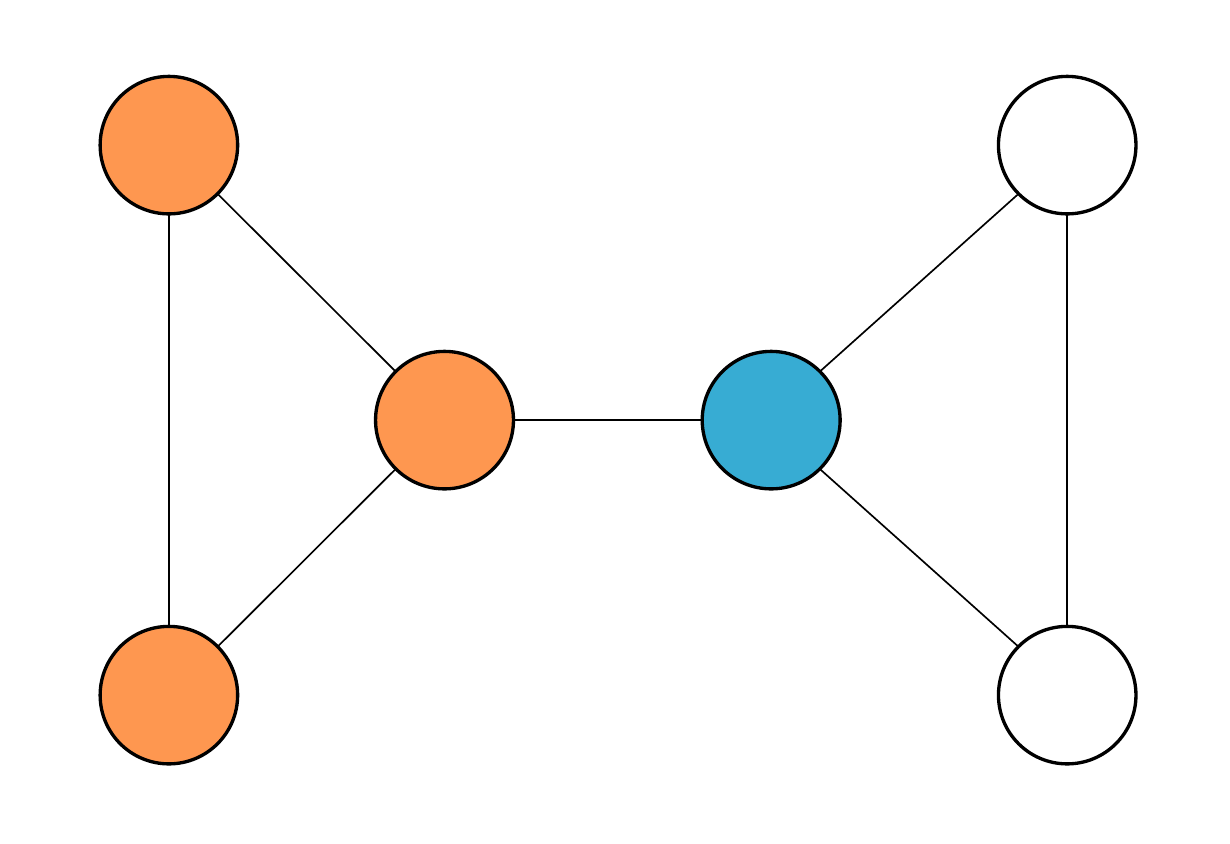}
  \caption{Graph}\label{subfig:graph_rec}
\end{subfigure}\hfil
\begin{subfigure}{0.33\textwidth}
  \includegraphics[width=170pt]{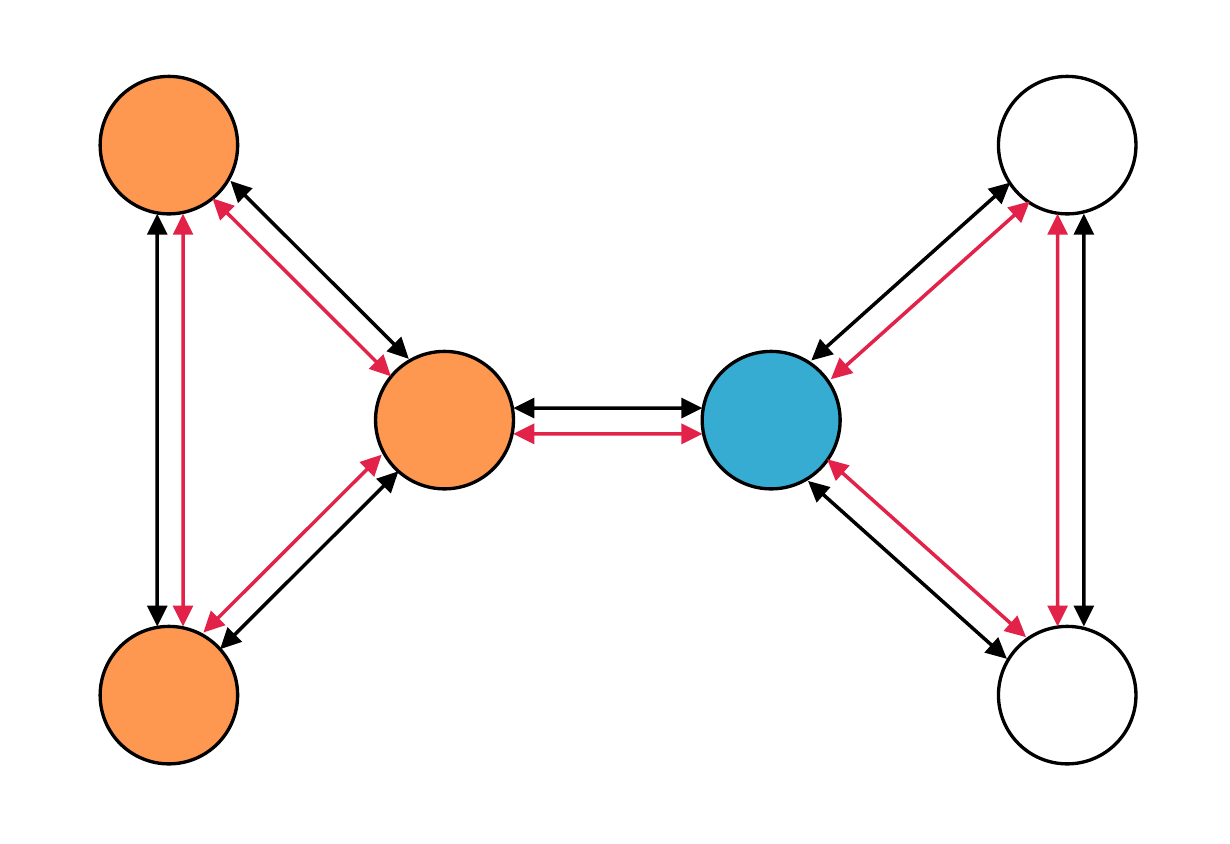}
    \caption{Backward SGD}\label{subfig:naive_sgd}
\end{subfigure}\hfil 
\begin{subfigure}{0.33\textwidth}
  \includegraphics[width=170pt]{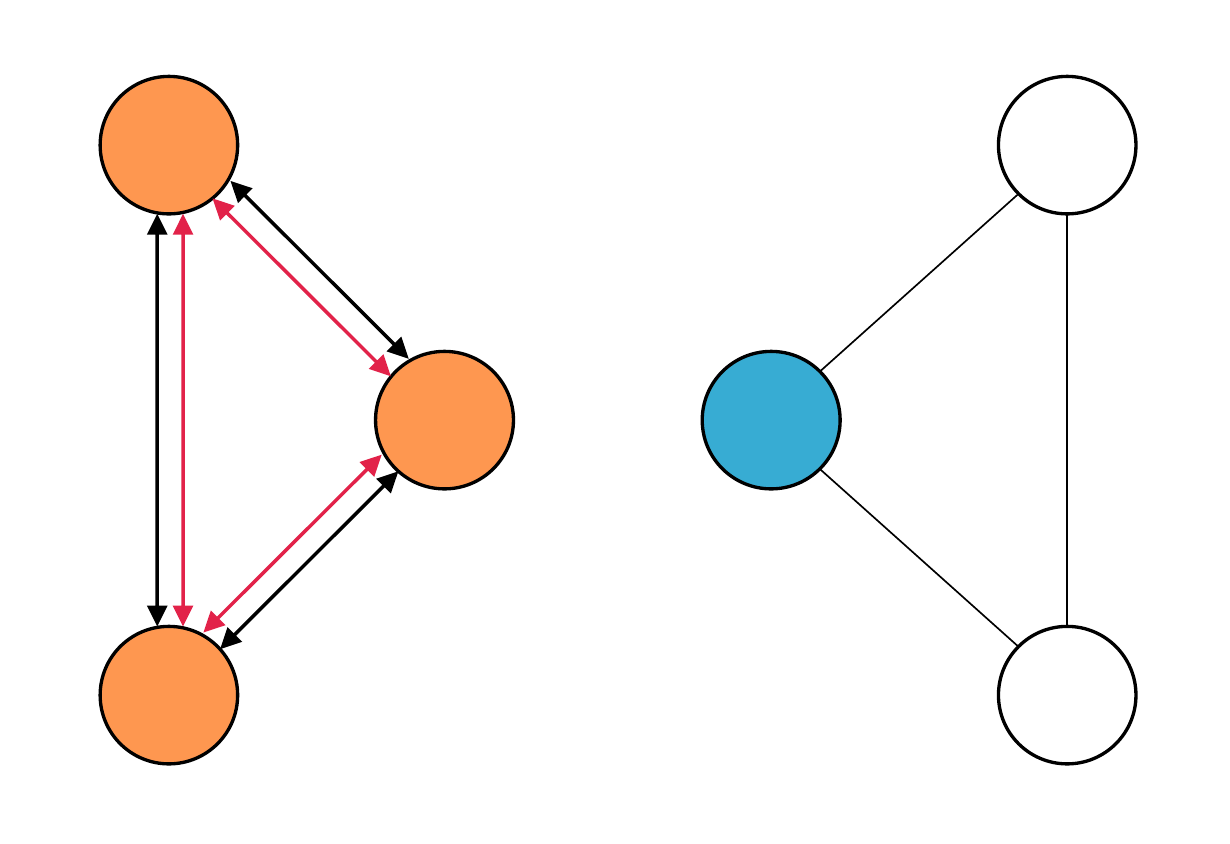}
  \caption{Cluster-GCN}\label{subfig:cluster}
\end{subfigure}\hfil
\begin{subfigure}{0.33\textwidth}
  \includegraphics[width=170pt]{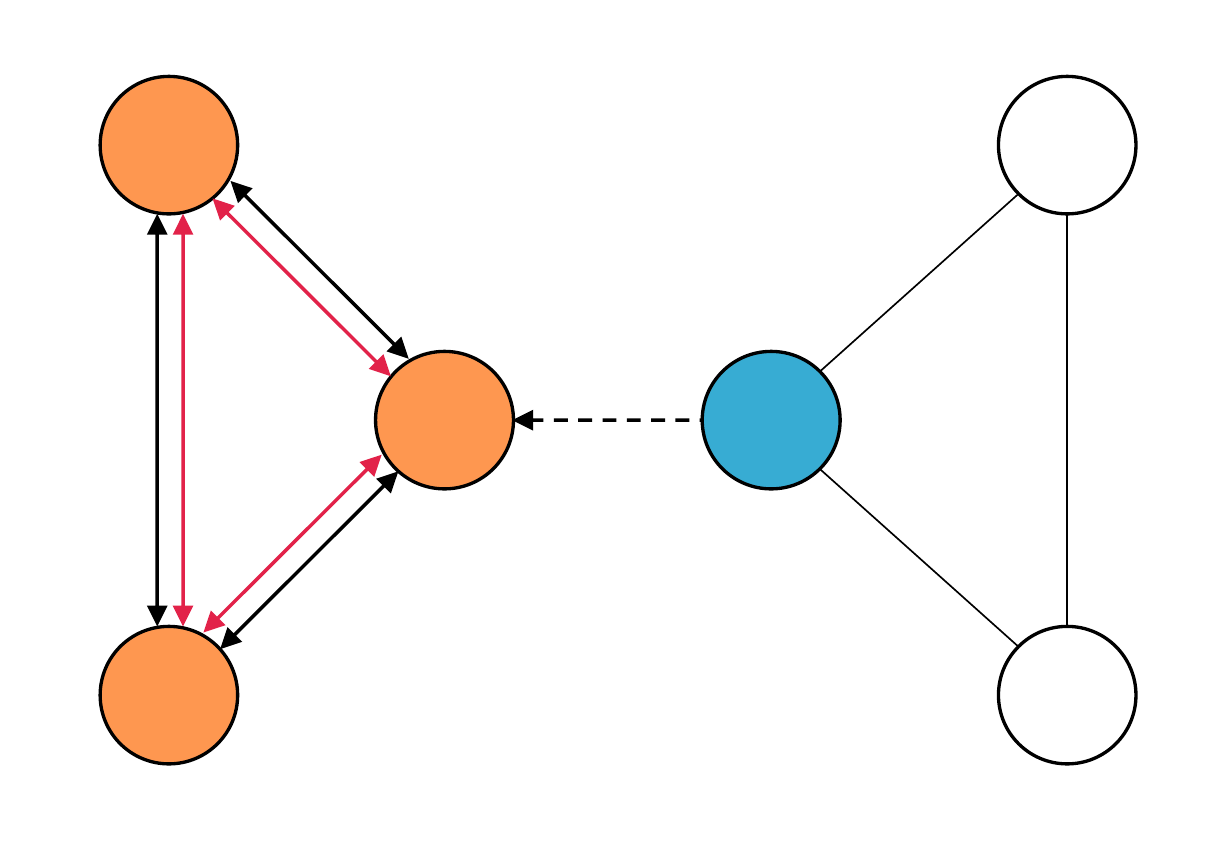}
    \caption{GAS}\label{subfig:gas}
\end{subfigure}\hfil
\begin{subfigure}{0.33\textwidth}
  \includegraphics[width=170pt]{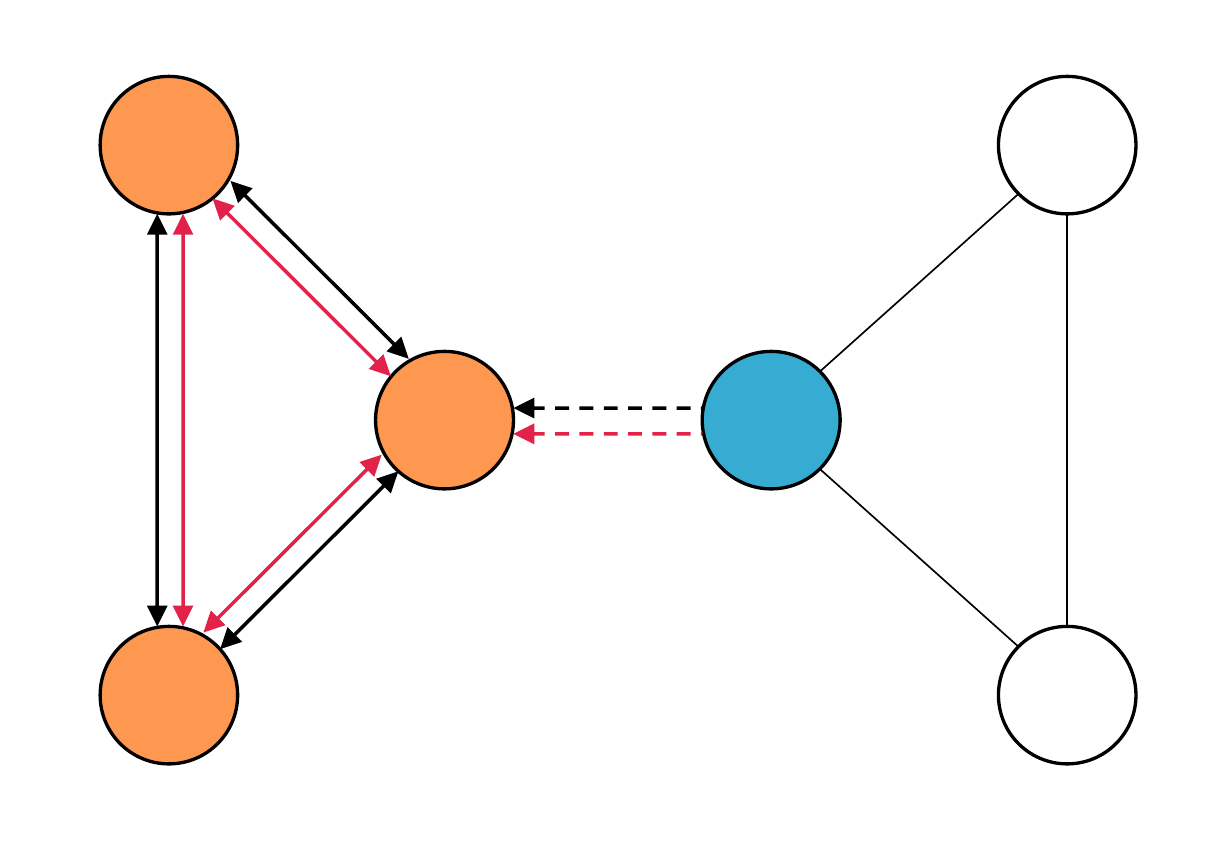}
    \caption{LMC}\label{subfig:amp}
\end{subfigure}\hfil
\begin{subfigure}{0.33\textwidth}
  \includegraphics[width=200pt]{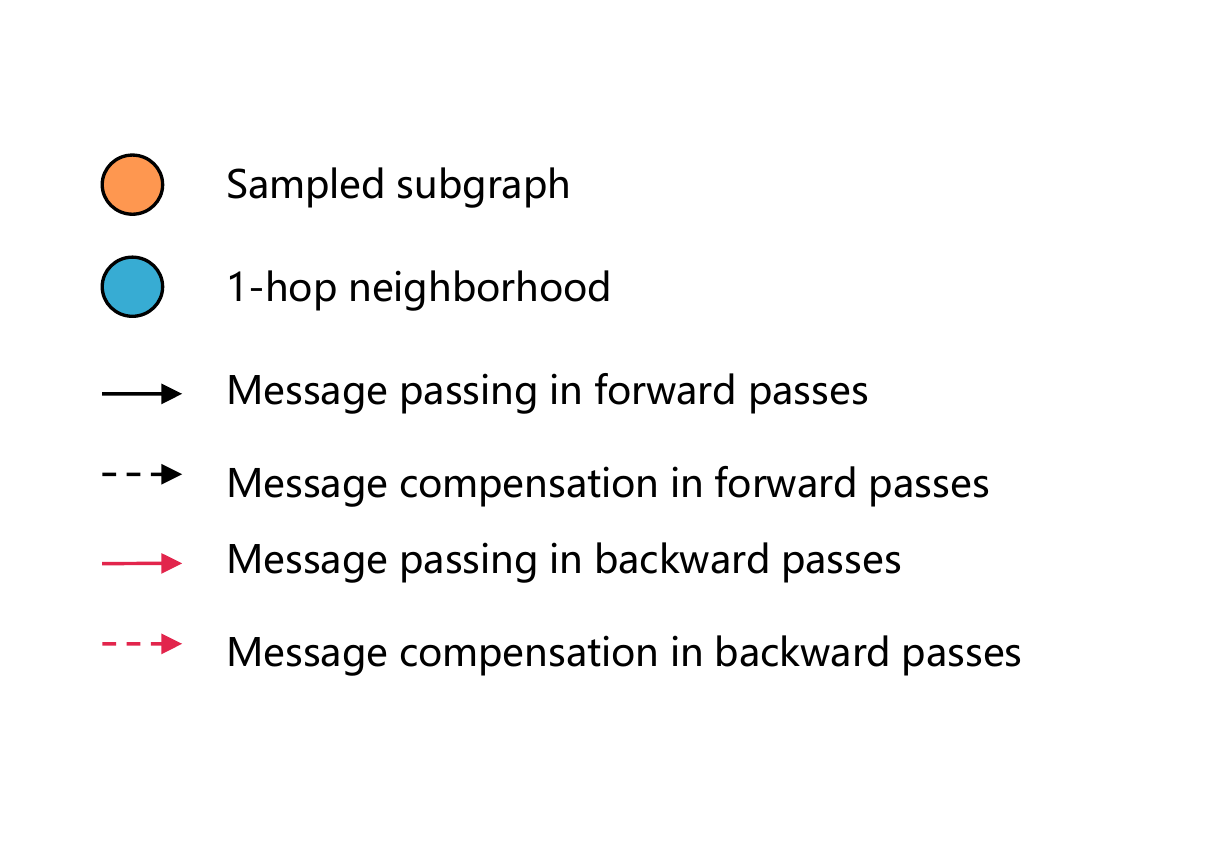}
\end{subfigure}\hfil
\caption{
Message passing of backward SGD, Cluster-GCN \cite{cluster_gcn}, GAS \cite{gas}, and LMC for RecGNNs.} \label{fig:message_flow}
\vspace{-4mm}
\end{figure*}

\subsubsection{Theoretical Results}\label{subsubsec:lmc4conv_theoretical}

In this subsection, we provide the theoretical analysis of LMC4Conv. Theorem \ref{thm:grad_error_conv} shows that the biases of mini-batch gradients computed by LMC4Conv can tend to an arbitrarily small value by setting a proper learning rate and convex combination coefficients. Then, Theorem \ref{thm:convergence_conv} shows that LMC4Conv converges to first-order stationary points of ConvGNNs. We provide detailed proofs of the theorems in Appendix \ref{appendix:proof_conv}. In the theoretical analysis, we use the following assumptions.
\begin{assumption}\label{assmp:proof} 
    Assume that (1) at the {\small$k$}-th iteration, a batch of nodes {\small$\mathcal{V}_{\mathcal{B}}^k$} is uniformly sampled from {\small$\mathcal{V}$} and the corresponding labeled node set {\small$\mathcal{V}_{L_{\mathcal{B}}}^{k}=\mathcal{V}_{\mathcal{B}}^k \cap \mathcal{V}_L$} is uniformly sampled from {\small$\mathcal{V}_L$}, (2) functions {\small$f_{\theta^{l}}$}, {\small$\phi_{\theta^{l}}$}, {\small$\nabla_{w}\mathcal{L}$}, {\small$\nabla_{\theta^{l}}\mathcal{L}$}, {\small$\nabla_w\ell_{w}$}, and {\small$\nabla_{\theta^l} u_{\theta^l}$} are {\small$\gamma$}-Lipschitz with {\small$\gamma>1$}, {\small$\forall\, l\in[L]$}, (3) norms {\small$\|\embH^{l,k}\|_F$}, {\small$\|\hisH^{l,k}\|_F$}, {\small$\|\temH^{l,k}\|_F$}, {\small$\|\widetilde{\embH}^{l,k}\|_F$}, {\small$\|\embV^{l,k}\|_F$}, {\small$\|\hisV^{l,k}\|_F$}, {\small$\|\temV^{l,k}\|_F$}, {\small$\|\widetilde{\embV}^{l,k}\|_F$}, {\small$\|\nabla_w\loss\|_2$}, {\small$\|\nabla_{\theta^l}\loss\|_2$}, {\small$\|\widetilde{\mathbf{g}}_{\theta^l}\|_2$}, and {\small$\|\widetilde{\mathbf{g}}_{w}\|_2$} are bounded by {\small$G>1$}, {\small$\forall\, l\in[L],\,k\in\mathbb{N}^*$}.
\end{assumption}

\begin{theorem}\label{thm:grad_error_conv}
    For any $k\in\mathbb{N}^*$ and $l\in[L]$, the expectations of $\|\Delta_w^k\|_2^2 \triangleq \|\widetilde{\mathbf{g}}_w(w^k) - \nabla_w\loss(w^k)\|_2^2$ and $\|\Delta_{\theta^l}^k\|_2^2 \triangleq \|\widetilde{\mathbf{g}}_{\theta^l}(\theta^{l,k}) - \nabla_{\theta^l}\loss(\theta^{l,k})\|_2^2$ have the bias-variance decomposition
    \begin{align*}
        &\mathbb{E}[\|\Delta^k_w\|_2^2] = ({\rm Bias}(\widetilde{\mathbf{g}}_w(w^k)))^2+ {\rm Var}(\widetilde{\mathbf{g}}_w(w^k)),\\
        &\mathbb{E}[\|\Delta^k_{\theta^l}\|_2^2] = ({\rm Bias}(\widetilde{\mathbf{g}}_{\theta^l}(\theta^{l,k})))^2+ {\rm Var}(\widetilde{\mathbf{g}}_{\theta^l}(\theta^{l,k})),
    \end{align*}
    where
    \begin{align*}
        &{\rm Bias}(\widetilde{\mathbf{g}}_w(w^k)) = \|\mathbb{E}[\widetilde{\mathbf{g}}_w(w^k)] - \nabla_w\mathcal{L}(w^k)\|_2,\\
        &{\rm Var}(\widetilde{\mathbf{g}}_w(w^k)) = \mathbb{E}[\|\mathbb{E}[\widetilde{\mathbf{g}}_w(w^k)] - \widetilde{\mathbf{g}}_w(w^k)\|_2^2],\\
        &{\rm Bias}(\widetilde{\mathbf{g}}_{\theta^l}(\theta^{l,k})) = \|\mathbb{E}[\widetilde{\mathbf{g}}_{\theta^l}(\theta^{l,k})] - \nabla_{\theta^l}\mathcal{L}(\theta^{l,k})\|_2,\\
        &{\rm Var}(\widetilde{\mathbf{g}}_{\theta^l}(\theta^{l,k})) = \mathbb{E}[\|\mathbb{E}[\widetilde{\mathbf{g}}_{\theta^l}(\theta^{l,k})] - \widetilde{\mathbf{g}}_{\theta^l}(\theta^{l,k})\|_2^2].
    \end{align*}
    Suppose that Assumption \ref{assmp:proof} holds, then with {\small$\eta = \mathcal{O}(\varepsilon^2)$} and {\small$\beta_{i}=\mathcal{O}(\varepsilon^2)$}, {\small$i\in[n]$}, there exist {\small$C>0$} and {\small$\rho\in(0,1)$} such that for any $k\in\mathbb{N}^*$ and $l\in[L]$, the bias terms can be bounded as
    \begin{align*}
        &{\rm Bias}(\widetilde{\mathbf{g}}_w(w^k))\leq C\varepsilon + C\rho^{\frac{k-1}{2}},\\
        &{\rm Bias}(\widetilde{\mathbf{g}}_{\theta^l}(\theta^{l,k}))\leq C\varepsilon + C\rho^{\frac{k-1}{2}}.
    \end{align*}
\end{theorem}

\begin{theorem}\label{thm:convergence_conv}
    Suppose that Assumption \ref{assmp:proof} holds. Besides, assume that the optimal value {\small$\loss^*=\inf_{w,\Theta}\loss(w,\Theta)$} is bounded by {\small$G$}. Then, with {\small$\eta=\mathcal{O}(\varepsilon^4)$}, {\small$\beta_{i}=\mathcal{O}(\varepsilon^4)$}, {\small$i\in[n]$}, and {\small$N=\mathcal{O}(\varepsilon^{-6})$}, LMC4Conv ensures to find an {\small$\varepsilon$}-stationary solution such that $\mathbb{E}[\|\nabla_{w,\Theta}\loss(w^R,\Theta^R)\|_2] \leq \varepsilon$ after running for $N$ iterations, where $R$ is uniformly selected from $[N]$ and $\Theta^R=(\theta^{l,R})_{l=1}^L$.
\end{theorem}

\subsection{LMC for RecGNNs}\label{subsec:lmc4rec}

In this section, we extend the idea for ConvGNNs to RecGNNs.
Similarly, LMC4Rec first efficiently estimates $\embH^{\Diamond}_{\mathcal{V}_{\mathcal{B}}}$ and $\embV^{\Diamond}_{\mathcal{V}_{\mathcal{B}}}$ based on {\it historical values} and {\it temporary values}, and then compute the mini-batch gradients as shown in Eqs. \eqref{eqn:mini-batch_grad_w} and \eqref{eqn:mini-batch_grad_theta}. We show that LMC4Rec converges to first-order stationary points of RecGNNs in Section \ref{subsubsec:lmc4rec_theoretical}. In Algorithm \ref{alg:lmc4rec} and Section \ref{subsubsec:lmc4rec_theoretical}, we denote a vector and a matrix at the $k$-th iteration by $(\cdot)^{\Diamond,k}$, but elsewhere we omit the superscipt $k$ and simply denote it by $(\cdot)^{\Diamond}$.

\subsubsection{Methodology}\label{subsubsec:lmc4rec_methodology}

\label{sec:cm_hist}

We store {\it historical embeddings and auxiliary variables} {\small$(\hisH^{\Diamond},\hisV^{\Diamond})$} to provide an affordable approximation. At each training iteration, we sample a mini-batch of nodes $\mathcal{V}_{\mathcal{B}}$, and update historical embeddings and auxiliary variables of nodes in the mini-batch, i.e., {\small$(\hisH^{\Diamond}_{\inbatch},\hisV^{\Diamond}_{\inbatch})_{l=1}^L$}.
Unlike LMC4Conv, LMC4Rec further computes the {\it temporary embeddings and auxiliary variables} $(\temH^{\Diamond}_{\mathcal{V}_{\mathcal{B}}}, \temV^{\Diamond}_{\mathcal{V}_{\mathcal{B}}})$ according to the recurrent architecture of RecGNNs and use them to approximate exact $(\embH^{\Diamond}_{\mathcal{V}_{\mathcal{B}}}, \embV^{\Diamond}_{\mathcal{V}_{\mathcal{B}}})$.

\udfsection{Update of {\small$\hisH^{\Diamond}_{\inbatch}$}.} In forward passes, we first update the historical embedding of each node $v_i\in\mathcal{V}_{\mathcal{B}}$ as
\begin{align}
    &\hish^{\Diamond}_i \leftarrow u_{\theta^{\Diamond}}(\hish^{\Diamond}_i, \overline{\embm}^{\Diamond}_{\neighbor{v_i}}, \embx_i);\nonumber\\
    & \overline{\embm}^{\Diamond}_{\neighbor{v_i}}=\aggregate_{\theta^{\Diamond}}(\{g_{\theta^{\Diamond}}(\hish_j^{\Diamond}) \mid v_j \in \neighbor{v_i}\})\label{eqn:his_update_rec},
\end{align}
We use the left arrow {\small$\leftarrow$} in Eq. \eqref{eqn:his_update_rec} to emphasize that the step is not to solve the fixed-point equation, but instead uses the historical embeddings in previous training iterations to update {\small$\hish^{\Diamond}_i$}.

\udfsection{Computation of {\small$\widehat{\embH}^{\Diamond}_{ \inbatch}$}.} Then, we compute temporary embeddings of nodes in $\mathcal{V}_{\mathcal{B}}$, i.e., $\temH^{\Diamond}_{\mathcal{V}_{\mathcal{B}}}$, by using iterative solvers to solve the local fixed-point equations
\begin{align}
    &\temh^{\Diamond}_i  =\update_{\theta^{\Diamond}}(\temh^{\Diamond}_i, \widehat{\embm}^{\Diamond}_{\neighbor{v_i}}  ,\embx_i);\nonumber\\ 
    &\widehat{\embm}_{\neighbor{v_i}}^{\Diamond} = \aggregate_{\theta^{\Diamond}}(\{g_{\theta^{\Diamond}}(\temh^{\Diamond}_j  ) \mid v_j \in \neighbor{v_i} \cap \mathcal{V}_{\mathcal{B}}\}\nonumber\\  
    &\quad\quad\quad\quad\quad\quad\quad \cup \{g_{\theta^{\Diamond}}(\hish^{\Diamond}_j  ) \mid v_j \in \neighbor{v_i} \setminus \mathcal{V}_{\mathcal{B}}\})\label{eqn:temp_rec}
\end{align}
We call {\small$\mathbf{C}_f^{\Diamond} \triangleq \aggregate_{\thetarec}(\{g_{\thetarec}(\hish_j^{\Diamond}) \} \mid v_j \in \neighbor{v_i} \setminus \mathcal{V}_{\mathcal{B}})$} the {\it local message compensation} in forward passes. The fixed point {\small$\temh_i^{\Diamond}  $} to Eq. \eqref{eqn:temp_rec} is an approximation solution of Eq. \eqref{eqn:mpeq_update}.

\udfsection{Update of {\small$\hisV^{\Diamond}_{\inbatch}$}.} In backward passes, we first update the historical auxiliary variable of each node {\small$v_i\in\mathcal{V}_{\mathcal{B}}$} as
\begin{align}
    \hisV^{\Diamond}_i \leftarrow &\sum_{v_j \in \neighbor{v_i}} (\nabla_{\embh_i^{\Diamond}} u_{\theta^{\Diamond}}(\hish^{\Diamond}_j, \overline{\embm}^{\Diamond}_{\neighbor{v_j}}, \embx_j)) \hisV^{\Diamond}_j \nonumber\\
    &+ \nabla_{\embh^{\Diamond}_i} \loss(\hish^{\Diamond}_i), \label{eqn:his_aux_update_rec}
\end{align}
where {\small$\overline{\embm}^{\Diamond}_{\neighbor{v_j}}$} is computed as shown in Eq. \eqref{eqn:his_update_rec}.

\udfsection{Computation of {\small$\widehat{\embV}^{\Diamond}_{ \inbatch}$}.} Then, we compute temporary auxiliary variables of nodes in {\small$\mathcal{V}_{\mathcal{B}}$}, i.e., {\small$\temV^{\Diamond}_{\mathcal{V}_{\mathcal{B}}}$}, by using iterative solvers to solve the local fixed-point equations
\begin{align}
    \temV_{i}^{\Diamond}  = &\sum_{v_j \in\neighbor{v_i}\cap \inbatch} (\nabla_{\embh_i^{\Diamond}}\update_{\theta^{\Diamond}}(\temh_j^{\Diamond} ,\widehat{\embm}_{\neighbor{v_j}}^{\Diamond},\embx_j) ) \temV_{j}^{\Diamond} \nonumber\\
    &+ \sum_{v_j \in\neighbor{v_i}\backslash \inbatch} (\nabla_{\embh_i^{\Diamond}}\update_{\theta^{\Diamond}}(\hish_j^{\Diamond} ,\widehat{\embm}_{\neighbor{v_j}}^{\Diamond}  ,\embx_j) ) \hisV_{j}^{\Diamond} \nonumber\\
    &+ \nabla_{\embh_i^{\Diamond}} \loss(\temh^{\Diamond}_i), \label{eqn:temp_aux_rec} 
\end{align}
where {\small$\temh^{\Diamond}_j$} and {\small$\widehat{\embm}^{\Diamond}_{\neighbor{v_j}}$} are computed as shown in Eq. \eqref{eqn:temp_rec}. We call {\small$\mathbf{C}_b^{\Diamond}\triangleq \sum_{v_j \in\neighbor{v_i}\backslash \inbatch} (\nabla_{\embh_i^{\Diamond}}\update_{\theta^{\Diamond}}(\hish_j^{\Diamond} ,\widehat{\embm}_{\neighbor{v_j}}^{\Diamond}  ,\embx_j) ) \hisV_{j}^{\Diamond}$} the {\it local message compensation} in backward passes.

\udfsection{Mini-batch Gradients of LMC4Rec.} Combining Eqs. \eqref{eqn:mini-batch_grad_w} and \eqref{eqn:mini-batch_grad_theta} with {\small$(\widehat{\embh}_j^{\Diamond}, \widehat{\embm}^{\Diamond}_{\neighbor{v_j}}, \widehat{\embV}_j^{\Diamond})$} leads to
\begin{align}
    &\widetilde{\mathbf{g}}_{w}(\mathcal{V}_\mathcal{B})=\frac{1}{|\mathcal{V}_{L_{\mathcal{B}}}|}\sum_{v_j\in\mathcal{V}_{L_\mathcal{B}}} \nabla_{w} \ell_{w}(\widehat{\embh}_j,y_j),\label{eqn:grad_rec_w}\\
    &\widetilde{\mathbf{g}}_{\theta^{\Diamond}}(\mathcal{V}_\mathcal{B})=\frac{|\mathcal{V}|}{|\mathcal{V}_\mathcal{B}|}\sum_{v_j\in\mathcal{V}_\mathcal{B}}(\nabla_{\theta^{l}}u_{\theta^l}(\widehat{\embh}^{\Diamond}_j, \widehat{\embm}^{\Diamond}_{\neighbor{v_j}}, \embx_j))\widehat{\embV}^{\Diamond}_j.\label{eqn:grad_rec_theta} 
\end{align}

\udfsection{Time Complexity.} Notice that the size of Eq. \eqref{eqn:temp_rec} is linear with {\small$\mathcal{V}_{\mathcal{B}}$} rather than the whole graph. Suppose that the maximum neighborhood size is {\small$n_{\max}$} and the number of iterations is {\small$L$}, then the time complexity in forward passes is {\small$\mathcal{O}(L(n_{\max}|\inbatch|d+|\inbatch| d^2) )$}. In backward passes, LMC4Rec generates compensation messages from {\small$\widehat{\embm}_{\neighbor{v_j}}^{\Diamond}$}, {\small $v_j \in \neighbor{v_i}\setminus \inbatch$}, which depends on $2$-hop neighborhoods {\small$\kneighbor{\inbatch}{2}$}. To save computational costs, we compute the compensation term once in a single backward pass and reuse it in iterative solvers. Therefore, the time complexity in backward passes is {\small $\mathcal{O}(n_{\max}^2|\inbatch|d + Ln_{\max}|\inbatch|d )$}. As {\small$L$} is large in RecGNNs, the time complexity becomes {\small$\mathcal{O}( Ln_{\max}|\inbatch|d )$} by ignoring the term {\small$n_{\max}^2|\inbatch|d$}.

\udfsection{Space Complexity.} Similar to LMC4Conv, LMC4Rec additionally stores the historical node embeddings {\small$\hisH^{\Diamond} $} and auxiliary variables {\small$\hisV^{\Diamond}$}. For RecGCN, the active histories in forward and backward passes employ {\small$\mathcal{O}(n_{\max} |\inbatch| d)$} GPU memory, respectively. As the time and memory complexity are independent of the size of the whole graph, i.e., {\small$|\mathcal{V}|$}, LMC4Rec is scalable. We summarize the computational complexity in Appendix \ref{appendix:complexity}.

Fig. \ref{fig:message_flow} shows the message passing of backward SGD, Cluster-GCN \cite{cluster_gcn}, GAS \cite{gas}, and LMC4Rec. Compared with other subgraph-sampling methods (Cluster-GCN and GAS), LMC4Rec proposes compensation messages simultaneously in forward and backward passes, leading to similar convergence behaviors with few additional computational costs. Based on the modification, we show that LMC4Rec converges to stationary points of RecGNNs in Section \ref{subsubsec:lmc4rec_theoretical}.

\begin{algorithm}[H]
    \caption{LMC4Rec}
    \label{alg:lmc4rec}
    \begin{algorithmic}[1]
        \STATE {\bfseries Input:} 
        The learning rate {\small$\eta$}.
        \STATE Partition {\small$\mathcal{V}$} into {\small$B$} parts {\small$(\mathcal{V}_{b})_{b=1}^B$}
        \FOR{{$k = 1, \dots, N$}}
            \STATE Randomly sample {\small$\mathcal{V}_{b_k}$} from {\small$(\mathcal{V}_b)_{b=1}^B$}
            \STATE Update {\small $\hisH^{\Diamond}_{\mathcal{V}_{b_k}}$} by Eq. \eqref{eqn:his_update_rec}
            \STATE Compute {\small$\temH^{\Diamond}_{\mathcal{V}_{b_k}}$} by solving Eq. \eqref{eqn:temp_rec}
            \STATE Update {\small $\hisV^{\Diamond}_{\mathcal{V}_{b_k}}$} by Eq. \eqref{eqn:his_aux_update_rec}
            \STATE Compute {\small$\temV^{\Diamond}_{\mathcal{V}_{b_k}}$} by solving Eq. \eqref{eqn:temp_aux_rec}
            \STATE Compute {\small $\widetilde{\mathbf{g}}_w^k$} and {\small $\widetilde{\mathbf{g}}_{\thetarec}^k$} by Eqs. \eqref{eqn:grad_rec_w} and \eqref{eqn:grad_rec_theta}
            \STATE Update parameters by\\
            \quad\quad\quad {\small$w^k = w^{k-1} - \eta \widetilde{\mathbf{g}}_w^k$}\\
            \quad\quad\quad {\small$\theta^{\Diamond,k} = \theta^{\Diamond,k-1} - \eta\widetilde{\mathbf{g}}_{\theta^{\Diamond}}^k$}
        \ENDFOR
    \end{algorithmic}
\end{algorithm}

Algorithm \ref{alg:lmc4rec} summarizes LMC4Rec.
We add a superscript {\small$k$} for each value to indicate that it is the value at the {\small$k$}-th iteration.
At the preprocessing step, we partition {\small$\mathcal{V}$} into {\small$B$} parts {\small$(\mathcal{V}_b)_{b=1}^B$}.
At the $k$-th training step, LMC4Rec first randomly samples a subgraph constructed by {\small$\mathcal{V}_{b_k}$}.
Notice that we sample more subgraphs to build a large graph in experiments, whose convergence analysis is consistent to that of sampling a single subgraph.
Then,  LMC4Rec updates the stored historical node embeddings $\hisH^{\Diamond,k}_{\mathcal{V}_{b_k}}$ by Eq. \eqref{eqn:his_update_rec} and auxiliary variables $\hisV^{\Diamond,k}_{\mathcal{V}_{b_k}}$ by Eq. \eqref{eqn:his_aux_update_rec}.
By the random update, the historical values get close to the exact up-to-date values.
Next, to approximate the exact node embeddings and auxiliary variables of $\mathcal{V}_{b_k}$ for computing mini-batch gradients, LMC4Rec recurrently aggregates the local messages from $\mathcal{V}_{b_k}$ and compensation messages from $\neighbor{\mathcal{V}_{b_k}} \setminus \mathcal{V}_{b_k}$ in Eqs. \eqref{eqn:temp_rec} and \eqref{eqn:temp_aux_rec}.
Finally, by replacing {\small$\embh_j^{\Diamond, k}$}, {\small$\embm^{\Diamond, k}_{\neighbor{v_j}}$}, and {\small$\embV_j^{\Diamond, k}$} in Eqs. \eqref{eqn:mini-batch_grad_w} and \eqref{eqn:mini-batch_grad_theta} with {\small$\temh_j^{\Diamond, k}$}, {\small $\widehat{\embm}_{\neighbor{v_j}}^{\Diamond, k}$}, and {\small$\temV_j^{\Diamond, k}$}, respectively, LMC4Rec computes mini-batch gradients {\small$\widetilde{\mathbf{g}}_{w}^k$} and {\small$\widetilde{\mathbf{g}}_{\thetarec}^k$} to update parameters {\small$w$} and {\small$\theta^{\Diamond}$}.

\subsubsection{Theoretical Results} \label{subsubsec:lmc4rec_theoretical}
In this section, we provide the theoretical analysis of LMC4Rec.
Theorem \ref{thm:grad_error_rec} shows that the biases of mini-batch gradients computed by LMC4Rec can tend to an arbitrarily small value by setting a proper learning rate. Then, Theorem \ref{thm:convergence_rec} shows that LMC4Rec converges to first-order stationary points of RecGNNs. We provide detailed proofs of the theorems in Appendix \ref{appendix:proof_rec}. In the theoretical analysis, we use the following assumptions.
\begin{assumption}\label{assmp:proof_rec}
    Assume that (1) at the {\small$k$}-th iteration, a batch of nodes {\small$\mathcal{V}_{\mathcal{B}}^k$} is uniformly sampled from {\small$\mathcal{V}$} and the corresponding labeled node set {\small$\mathcal{V}_{L_{\mathcal{B}}}^{k}=\mathcal{V}_{\mathcal{B}}^k \cap \mathcal{V}_L$} is uniformly sampled from {\small$\mathcal{V}_L$}, (2) functions {\small$f_{\theta^{\Diamond}}$}, {\small$\phi_{w,\theta^{\Diamond}}$}, {\small$\nabla_{w}\mathcal{L}$}, {\small$\nabla_{\theta^{\Diamond}}\mathcal{L}$}, {\small$\nabla_w\ell_{w}$}, and {\small$\nabla_{\theta^{\Diamond}} u_{\theta^{\Diamond}}$} are {\small$\gamma$}-Lipschitz with {\small$\gamma\in[0,1)$}, (3) norms {\small$\|\embH^{\Diamond,k}\|_F$}, {\small$\|\hisH^{\Diamond,k}\|_F$}, {\small$\|\temH^{\Diamond,k}\|_F$}, {\small$\|\embV^{\Diamond,k}\|_F$}, {\small$\|\hisV^{\Diamond,k}\|_F$}, {\small$\|\temV^{\Diamond,k}\|_F$}, {\small$\|\nabla_w\loss\|_2$}, {\small$\|\nabla_{\theta^{\Diamond}}\loss\|_2$}, {\small$\|\widetilde{\mathbf{g}}_{\theta^{\Diamond}}\|_2$}, and {\small$\|\widetilde{\mathbf{g}}_{w}\|_2$} are bounded by {\small$G>0$}, {\small$\forall\,k\in\mathbb{N}^*$}.
\end{assumption}

\begin{remark}
   The assumption that $f_{\theta^{\Diamond}}$ is $\gamma$-Lipschitz with $\gamma\in[0,1)$ is the standard contraction assumption \cite{contraction} for the fixed-point equation $\embH^{\Diamond} = f_{\theta^{\Diamond}}(\embH^{\Diamond};\embX)$, which enforces the existence and uniqueness of the fixed-point.
\end{remark}

\begin{theorem}\label{thm:grad_error_rec}
    For any $k\in\mathbb{N}^*$, the expectations of $\|\Delta_w^k\|_2^2 \triangleq \|\widetilde{\mathbf{g}}_{w}(w^{k}) - \nabla_w\loss(w^k)\|_2^2$ and $\|\Delta_{\theta^{\Diamond}}^k\|_2^2 \triangleq \|\widetilde{\mathbf{g}}_{\theta^{\Diamond}}(\theta^{k,\Diamond}) - \nabla_{\theta^{\Diamond}}\loss(\theta^{{\Diamond},k})\|_2^2$ have the bias-variance decomposition
    \begin{align*}
        &\mathbb{E}[\|\Delta_{w}^k\|_2^2] = ({\rm Bias}(\widetilde{\mathbf{g}}_w(w^k)))^2+{\rm Var}(\widetilde{\mathbf{g}}_w(w^k)),\\
        &\mathbb{E}[\|\Delta_{\theta^{\Diamond}}^k\|_2^2] = ({\rm Bias}(\widetilde{\mathbf{g}}_{\theta^{\Diamond}}(\theta^{\Diamond,k})))^2+{\rm Var}(\widetilde{\mathbf{g}}_{\theta^{\Diamond}}(\theta^{\Diamond,k})),
    \end{align*}
    where
    \begin{align*}
        &{\rm Bias}(\widetilde{\mathbf{g}}_w(w^k)) = \|\mathbb{E}[\widetilde{\mathbf{g}}_w(w^k)] - \nabla_w\mathcal{L}_w(w^k)\|_2,\\
        &{\rm Var}(\widetilde{\mathbf{g}}_w(w^k)) = \mathbb{E}[\|\mathbb{E}[\widetilde{\mathbf{g}}_w(w^k)] - \widetilde{\mathbf{g}}_w(w^k)\|_2^2],\\
        &{\rm Bias}(\widetilde{\mathbf{g}}_{\theta^{\Diamond}}(\theta^{\Diamond,k}))= \|\mathbb{E}[\widetilde{\mathbf{g}}_{\theta^{\Diamond}}(\theta^{\Diamond,k})] - \nabla_{\theta^{\Diamond}}\mathcal{L}_{\theta^{\Diamond}}(\theta^{\Diamond,k})\|_2,\\
        &{\rm Var}(\mathbf{g}_{\theta^{\Diamond}}(\theta^{\Diamond,k}))= \mathbb{E}[\|\mathbb{E}[\widetilde{\mathbf{g}}_{\theta^{\Diamond}}(\theta^{\Diamond,k})] - \widetilde{\mathbf{g}}_{\theta^{\Diamond}}(\theta^{\Diamond,k})\|_2^2].
    \end{align*}
    Suppose that Assumption \ref{assmp:proof_rec} holds, then with $\eta = \mathcal{O}(\varepsilon)$, there exist $C>0$ and $\rho\in(0,1)$ such that for any $k\in\mathbb{N}^*$, the bias terms are bounded as
    \begin{align*}
        &{\rm Bias}(\widetilde{\mathbf{g}}_w(w^k)) \leq C\varepsilon + C\rho^{k-1},\\
        &{\rm Bias}(\widetilde{\mathbf{g}}_{\theta^{\Diamond}}(\theta^{{\Diamond},k})) \leq C\varepsilon + C\rho^{k-1}.
    \end{align*}
\end{theorem}

\begin{theorem}\label{thm:convergence_rec}
    Suppose that Assumption \ref{assmp:proof_rec} holds. Besides, assume that the optimal value $\mathcal{L}^*=\inf_{w,\theta^{\Diamond}} \loss(w,\theta^{\Diamond})$ is bounded by $G$. Then, with $\eta=\mathcal{O}(\varepsilon^2)$ and $N=\mathcal{O}(\varepsilon^{-4})$, LMC4Rec ensures to find an $\varepsilon$-stationary solution such that $\mathbb{E}[\|\nabla_{w,\theta^{\Diamond}}\loss(w^R,\theta^{\Diamond,R})\|_2] \leq \varepsilon$ after running for $N$ iterations, where $R$ is uniformly selected from $[N]$.
\end{theorem}

\section{Experiments} \label{sec:exp}

To demonstrate that LMC is an effective and widely-applicable algorithm, we conduct extensive experiments on large-scale benchmark tasks for both ConvGNNs (\modifyok{}{see Section \ref{sec:exp_convgnns}) and RecGNNs (see Section \ref{sec:exp_recgnns})}. The experiments are carried out on a single GeForce RTX 2080 Ti (11 GB).

\subsection{Experiments with ConvGNNs (Finite Number of MP Layers)}\label{sec:exp_convgnns}

We first introduce experimental settings in Section \ref{sec:setting_conv}. We then evaluate the convergence and the efficiency of LMC4Conv in Sections \ref{sec:scalable_conv} and \ref{sec:smallbatchsize_conv}. Finally, we conduct ablation studies about the proposed compensations in Section \ref{sec:ablation}.

\subsubsection{Experimental Settings} \label{sec:setting_conv}

\udfsection{Datasets.} Some recent works \cite{ogb} have indicated that many frequently-used graph datasets are too small compared with graphs in real-world applications.
Therefore, we evaluate LMC4Conv on four large datasets, PPI, REDDIT , FLICKR\cite{graphsage}, and Ogbn-arxiv \cite{ogb}.
These datasets contain thousands or millions of nodes/edges and have been widely used in previous works \cite{gas, graphsaint, graphsage, cluster_gcn, vrgcn, fastgcn}.
For more details, please refer to Appendix \ref{sec:dataset}.

\udfsection{Baselines and Implementation Details.} In terms of prediction performance, our baselines include node-wise sampling methods (GraphSAGE \cite{graphsage} and VR-GCN \cite{vrgcn}), layer-wise sampling methods (FastGCN \cite{fastgcn}, \modify{}{DROPEDGE \cite{dropedge}}, and LADIES\cite{ladies}), subgraph-wise sampling methods (Cluster-GCN \cite{cluster_gcn}, GraphSAINT \cite{graphsaint}, FM \cite{graphfm} and GAS \cite{gas}), a precomputing method (SIGN \cite{sign}).
By noticing that GAS achieves the state-of-the-art prediction performance (Table \ref{tab:largegraph}) among the baselines, we further compare the efficiency of LMC with GAS \cite{gas} and Cluster-GCN \cite{cluster_gcn}, another subgraph-wise sampling method using METIS partition. We implement LMC4Conv and Cluster-GCN based on the codes and toolkits of GAS \cite{gas} to ensure a fair comparison. For other implementation details, please refer to Appendix \ref{sec:implementation}.

\udfsection{Hyperparameters.} To ensure a fair comparison, we follow the data splits, training pipeline, and most hyperparameters in \cite{gas} except for the additional hyperparameters in LMC4Conv such as $\beta_i$. We use the grid search to find the best $\beta_i$ (see Appendix \ref{sec:selection_beta} for more details).

\begin{figure*}[ht]
\centering 
\begin{subfigure}{1.0\linewidth}
  \includegraphics[width=\linewidth]{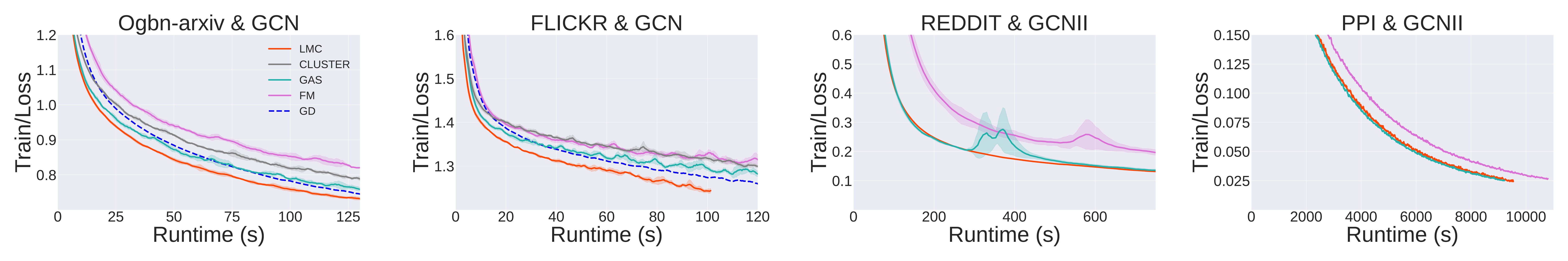}
  \caption{Training loss}\label{subfig:train_loss_runtime}
\end{subfigure}\hfil 
\begin{subfigure}{1.0\linewidth}
  \includegraphics[width=\linewidth]{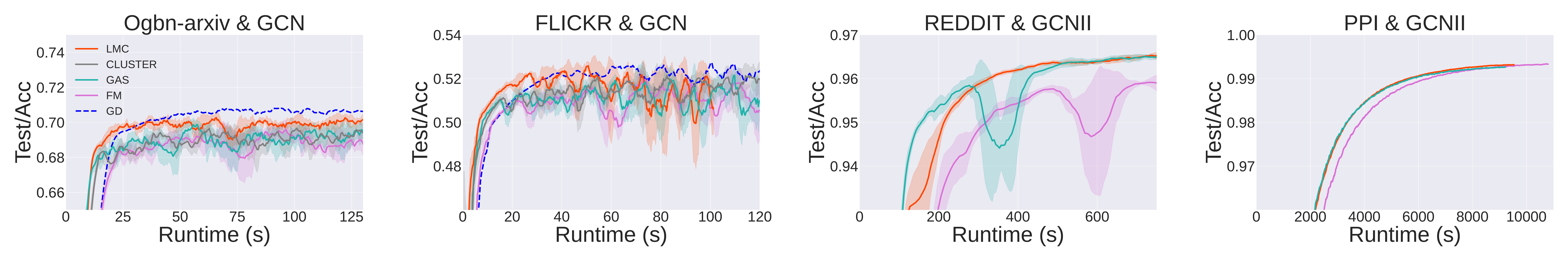}
    \caption{Testing accuracy}\label{subfig:test_acc_runtime}
\end{subfigure}\hfil
\caption{
Testing accuracy and training loss w.r.t. runtimes (s). The LMC in the figures refers to LMC4Conv.} \label{fig:runtime}
\end{figure*}

\subsubsection{LMC4Conv is Fast without Sacrificing Accuracy}\label{sec:scalable_conv}

Table \ref{tab:largegraph} reports the prediction performance of LMC4Conv and the baselines.
We report the mean and the standard deviation by running each experiment five times for GAS, FM, and LMC4Conv.
LMC4Conv, FM, and GAS all resemble full-batch performance on all datasets, while other baselines may fail, especially on the FLICKR dataset.
Moreover, LMC4Conv, FM, and GAS with deep ConvGNNs, e.g., GCNII \cite{gcnii}, outperform other baselines on all datasets.

\begin{table}[ht]
    \centering
      \caption{%
      \textbf{Prediction performance on large graph datasets.} OOM denotes the out-of-memory issue. Bold font indicates the best result and underline indicates the second best result.
        }\label{tab:largegraph}
    \setlength{\tabcolsep}{2.0mm}
    \resizebox{1.0\linewidth}{!}{
    \begin{tabular}{llcccc}
    \toprule
    \mc{2}{l}{\footnotesize{\textbf{\#\,nodes}}} & \footnotesize{230K} & \footnotesize{57K} & \footnotesize{89K} & \footnotesize{169K} \\[-0.1cm]
    \mc{2}{l}{\footnotesize{\textbf{\#\,edges}}} & \footnotesize{11.6M} & \footnotesize{794K} & \footnotesize{450K} & \footnotesize{1.2M} \\[-0.05cm]
    \mc{2}{l}{\mr{2}{\textbf{Method}}} & \mr{2}{\textsc{Reddit}} & \mr{2}{\textsc{PPI}} & \mr{2}{\textsc{Flickr}} & \texttt{ogbn-} \\
    & & & & & \texttt{arxiv} \\
    \midrule
    \mc{2}{l}{\textsc{GraphSAGE}}   & 95.40          & 61.20          & 50.10          & 71.49          \\
    \mc{2}{l}{\textsc{VR-GCN}}      & 94.50          & 85.60          & ---            & ---            \\
    \mc{2}{l}{\textsc{FastGCN}}     & 93.70          & ---            & 50.40          & ---            \\
    \mc{2}{l}{\textsc{LADIES}}      & 92.80          & ---            & ---            & ---            \\
    \mc{2}{l}{\textsc{Cluster-GCN}} & 96.60          & \underline{99.36}          & 48.10          & ---            \\
    \mc{2}{l}{\textsc{GraphSAINT}}  & 97.00          & \textbf{99.50}          & 51.10          & ---            \\
    \mc{2}{l}{\textsc{SIGN}}        & 96.80          & 97.00          & 51.40          & ---            \\
    \mc{2}{l}{\modify{}{\textsc{DROPEDGE}}}        & \underline{97.02}          & ---          & ---          & ---            \\
    \midrule
    \mr{2}{\rotatebox{90}{\footnotesize{\,GD}}}
    & ~~\textsc{GCN}   & 95.43 & 97.58 & 53.73 & 71.64 \\
    & ~~\textsc{GCNII} & OOM   & OOM   & 55.28 & \textbf{72.83} \\
    \midrule
    \mr{2}{\rotatebox{90}{\footnotesize{GAS}}}
    & ~~\textsc{GCN}                 & \acc{95.35}{0.01} & \acc{98.91}{0.03} & \acc{53.44}{0.11} & \acc{71.54}{0.19}\\
    & ~~\textsc{GCNII}               & \acc{96.73}{0.04} & \acc{\underline{99.36}}{0.02} & \textbf{\acc{55.42}{0.27}} & \acc{72.50}{0.28}\\
    \midrule
    \mr{2}{\rotatebox{90}{\footnotesize{FM}}}
    & ~~\textsc{GCN}                 & \acc{95.27}{0.03} & \acc{98.91}{0.01} & \acc{53.48}{0.17} & \acc{71.49}{0.33}\\
    & ~~\textsc{GCNII}               & \acc{96.52}{0.06} & \acc{99.34}{0.03} & \acc{54.68}{0.27} & \acc{72.54}{0.27}\\
    \midrule
    \mr{2}{\rotatebox{90}{\footnotesize{\textbf{LMC}}}}
    & ~~\textsc{GCN}                 & \acc{95.44}{0.02} & \acc{98.87}{0.04} & \acc{53.80}{0.14} & \acc{71.44}{0.23}\\
    & ~~\textsc{GCNII}               & \textbf{\acc{96.88}{0.03}} & \acc{99.32}{0.01}    & \acc{\underline{55.36}}{0.49}    & \acc{\underline{72.76}}{0.22}\\
    \bottomrule
  \end{tabular}
    }
\end{table}

\begin{table*}[ht]\centering
      \caption{%
      \textbf{ Efficiency of CLUSTER, GAS, and LMC4Conv.}
      }\label{tab:memory}
    \setlength{\tabcolsep}{1.9mm}
    \resizebox{\linewidth}{!}{
    \begin{tabular}{lcccccc}
    \toprule
    \mr{2}{\textbf{Dataset} \& \textbf{ConvGNN}} & \mc{3}{c}{\textbf{Epochs}} & \mc{3}{c}{\textbf{Runtime} (s)} \\
    & {\small CLUSTER} & {\small GAS} & {\small LMC4Conv} &  {\small CLUSTER} & {\small GAS} & {\small LMC4Conv}   \\
    \midrule
    Ogbn-arxiv \& GCN   & \acc{211.0}{75.6} & \acc{176.0}{43.4} & \textbf{\acc{124.4}{34.2}} &\acc{108}{39}& \acc{79}{20} & \textbf{\acc{55}{15}}  \\
    FLICKR \& GCN       &\acc{379.2}{29.0}& \acc{389.4}{21.2} & \textbf{\acc{334.2}{18.6}} &\acc{127}{10}& \acc{117}{7} & \textbf{\acc{85}{5}}  \\
    REDDIT \& GCN       & \acc{239.0}{26.2} & \acc{372.4}{55.2} & \textbf{\acc{166.8}{20.9}} &\acc{516}{57}& \acc{790}{114} & \textbf{\acc{381}{47}}  \\
    PPI \& GCN          &\acc{428.0}{45.8}& \acc{293.6}{11.9} & \textbf{\acc{290.2}{28.5}} &\acc{359}{38}& \textbf{\acc{179}{9}} & \textbf{\acc{179}{18}}  \\
    \midrule
    Ogbn-arxiv \& GCNII & --- & \acc{234.8}{30.2}  & \textbf{\acc{197.4}{34.7}} & --- & \acc{218}{28}  & \textbf{\acc{178}{31}} \\
    FLICKR \& GCNII     & --- & \textbf{\acc{352}{54}} & \acc{356}{64} & --- & \textbf{\acc{465}{71}} & \acc{475}{85}  \\
    \bottomrule
  \end{tabular}
    }
\end{table*}

As LMC4Conv, FM, and GAS share a similar prediction performance, we additionally compare the convergence speed of LMC4Conv, FM, GAS, and Cluster-GCN, another subgraph-wise sampling method using METIS partition, in Fig. \ref{fig:runtime} and Table \ref{tab:memory}. We use a sliding window to smooth the convergence curve in Fig. \ref{fig:runtime} as the accuracy on test data is unstable.
The solid curves correspond to the mean, and the shaded regions correspond to values within plus or minus one standard deviation of the mean.
Table \ref{tab:memory} reports the number of epochs, the runtime to reach the full-batch accuracy in Table \ref{tab:largegraph}, and the GPU memory.
As shown in Table \ref{tab:memory} and Fig. \ref{subfig:train_loss_runtime}, LMC4Conv is significantly faster than GAS, especially with a speed-up of 2x on the REDDIT dataset.
Notably, the test accuracy of LMC4Conv is more stable than GAS, and thus the smooth test accuracy of LMC4Conv outperforms GAS in Fig. \ref{subfig:test_acc_runtime}.
Although GAS finally resembles full-batch performance in Table \ref{tab:largegraph} by selecting the best performance on the valid data,
it may fail to resemble under small batch sizes due to its unstable process (see Section \ref{sec:smallbatchsize_conv}).
Another appealing feature of LMC4Conv is that it shares comparable GPU memory costs with GAS, and thus avoiding the neighbor explosion problem.
FM is slower than other methods, as they additionally update historical embeddings in the storage for the nodes outside the mini-batches. Please see Appendix \ref{sec:exp_epochtime} for the comparison in terms of training time per epoch.

To further illustrate the convergence of LMC4Conv, we compare the errors of mini-batch gradients computed by Cluster-GCN, GAS, and LMC4Conv. At epoch training step, we record the relative errors {\small$\|\widetilde{\mathbf{g}}_{\theta^{l}} - \nabla_{\theta^{l}} \loss\|_2\,/\,\| \nabla_{\theta^{l}} \loss\|_2$}, where $\nabla_{\theta^{l}} \loss$ is the full-batch gradient for the parameters $\theta^{l}$ at the $l$-th MP layer and the $\widetilde{\mathbf{g}}_{\theta^{l}}$ is a mini-batch gradient.
To avoid the randomness of the full-batch gradient $\nabla_{\theta^{l}} \loss$, we set the dropout rate as zero.
We report average relative errors during training in Fig. \ref{fig:error_conv}.
LMC4Conv enjoys the smallest estimated errors in the experiments.

\begin{figure}[ht]
\centering 
\begin{subfigure}{0.33\linewidth}
  \includegraphics[width=\linewidth]{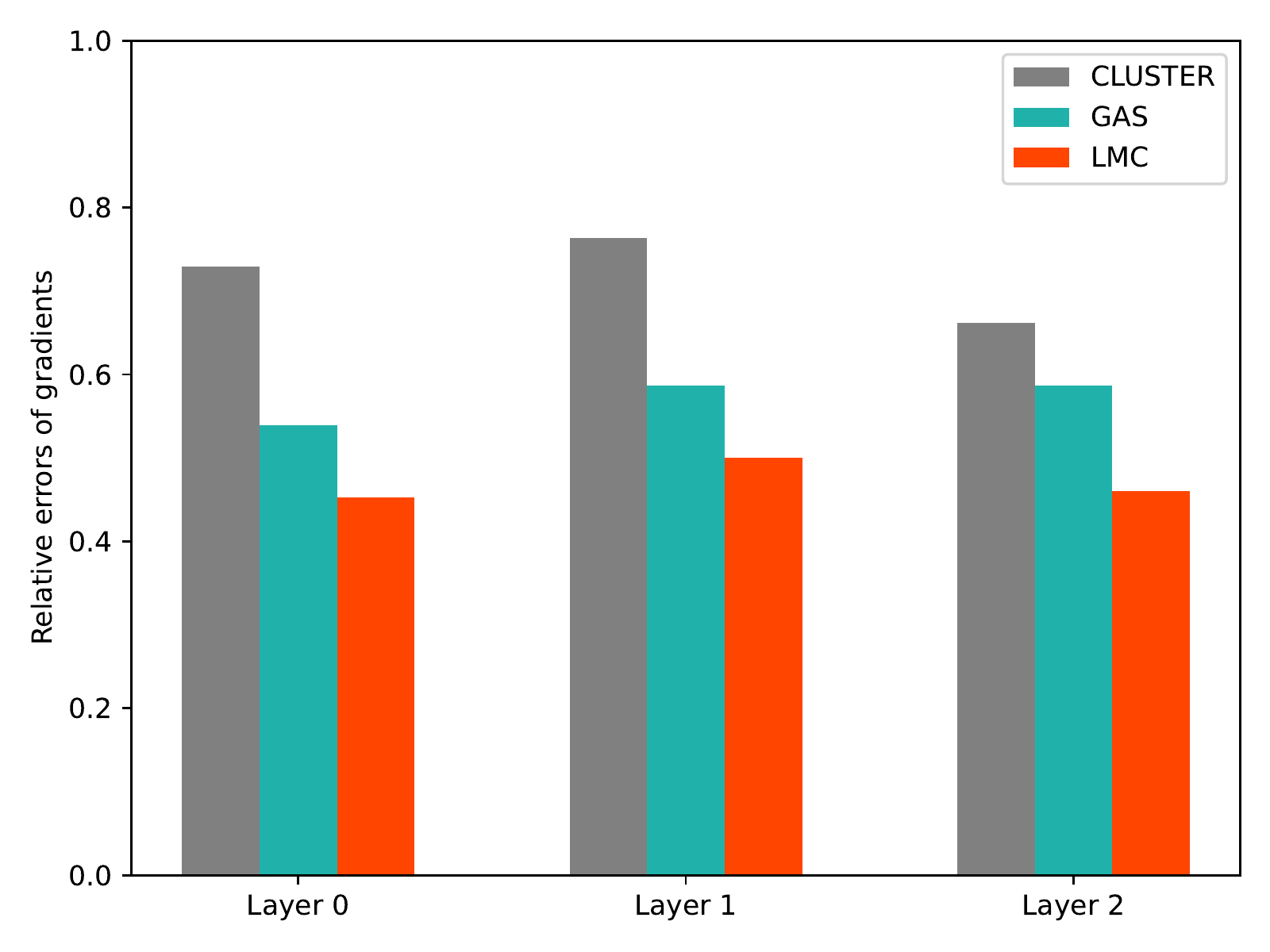}
  \caption{Ogbn-arxiv}\label{subfig:error_arxiv_conv}
\end{subfigure}\hfil 
\begin{subfigure}{0.33\linewidth}
  \includegraphics[width=\linewidth]{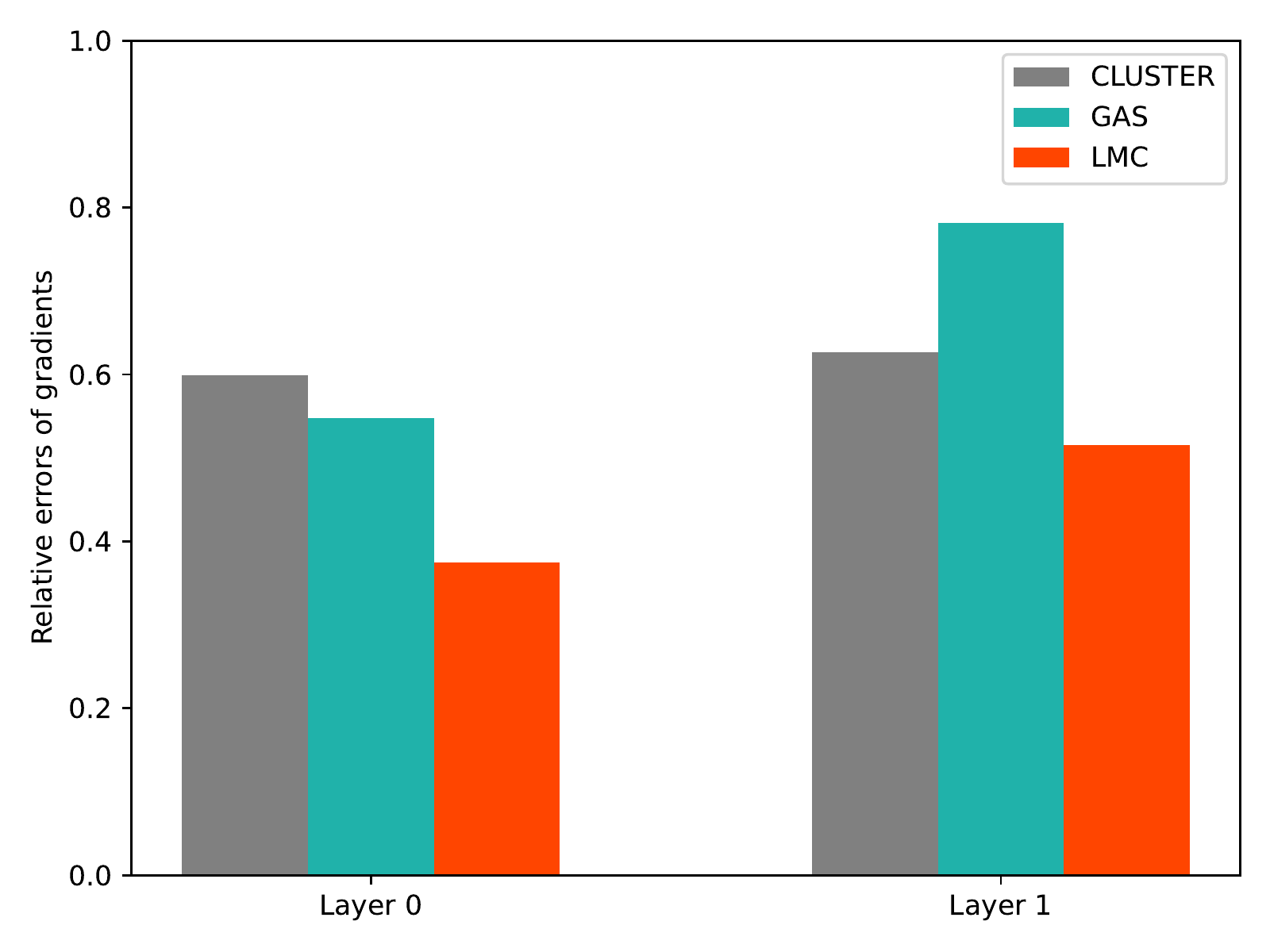}
    \caption{FLICKR}\label{subfig:error_flickr_conv}
\end{subfigure}\hfil
\begin{subfigure}{0.33\linewidth}
  \includegraphics[width=\linewidth]{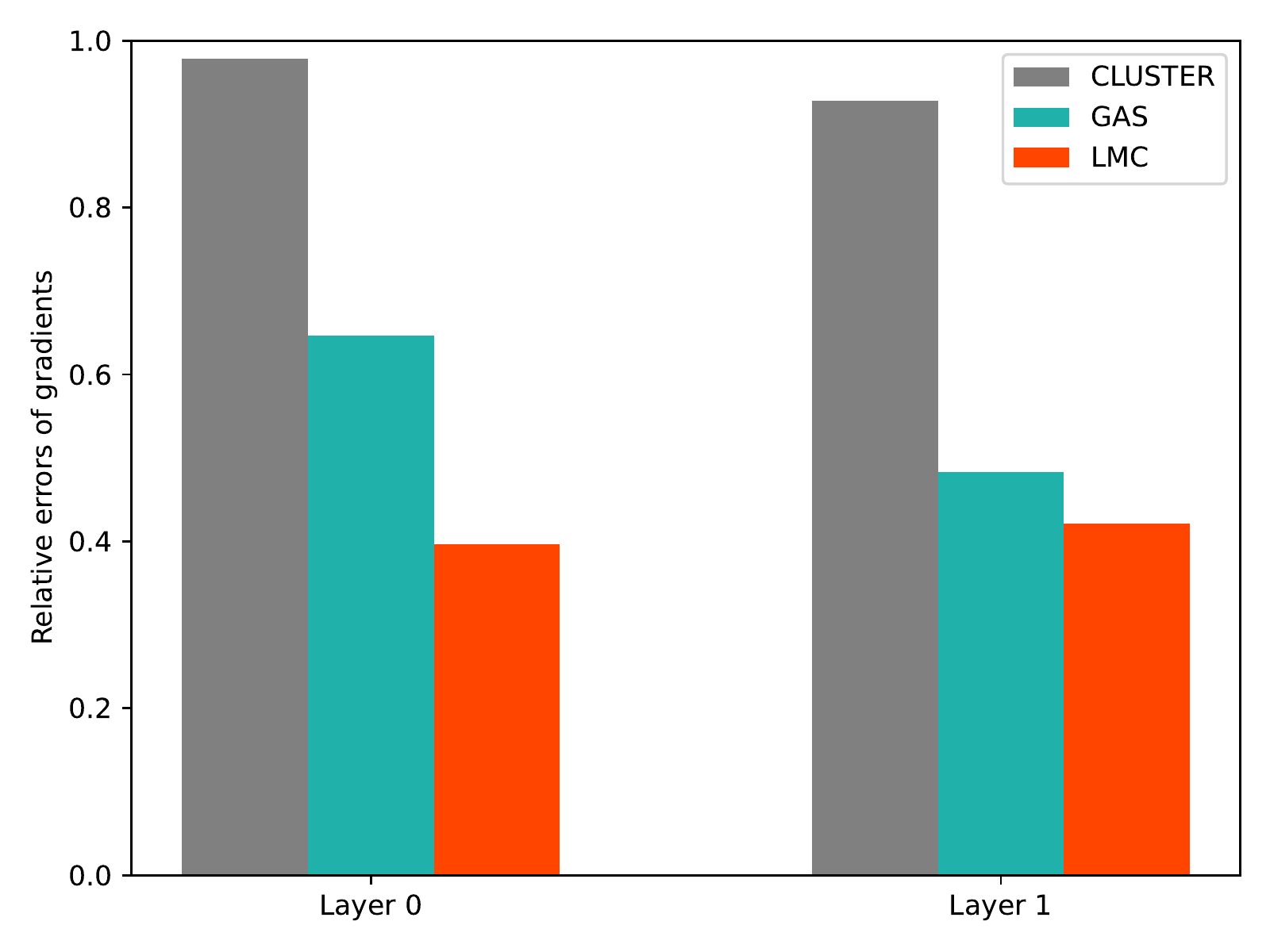}
    \caption{REDDIT}\label{subfig:error_reddit}
\end{subfigure}\hfil
\caption{The average relative estimated errors of mini-batch gradients computed by Cluster-GCN, GAS, and LMC4Conv for GCN.} \label{fig:error_conv}
\end{figure}

\modify{}{
\subsection{LMC4Conv keeps all the available neighborhood information with low memory costs}


Following Table 3 in \cite{gas}, we report the GPU memory consumption of different methods and the proportion of reserved messages {\small $\sum_{k=1}^{b} \|\widetilde{\mathbf{A}}_{\mathcal{V}_{b_k}}^{alg}\|_0/\|\widetilde{\mathbf{A}}\|_0$} in Table \ref{tab:memory_consumption}, where {\small $\widetilde{\mathbf{A}}$} is the adjacency matrix of full-batch GCN, {\small $\widetilde{\mathbf{A}}_{\mathcal{V}_{b_k}}^{alg}$} is the adjacency matrix used in a subgraph-wise method $alg$ (e.g., Cluster-GCN, GAS, and LMC4Conv), and {\small$\|\cdot\|_0$} denotes the {\small$\ell_0$}-norm. {\it Default} indicates the default batch size used in the codes and toolkits of GAS \cite{gas}.

LMC is more memory-efficient and faster than the full-batch GD---which is provably convergent by keeping all the available neighborhood information in both forward and backward passes---especially with a memory consumption reduction of 79\% under batch size = 1 on the REDDIT dataset. However, existing subgraph-wise sampling methods discard messages to save memory, which poses significant challenges to their convergence behaviors, e.g., convergence speeds and the robustness to batch sizes.

\begin{table*}[h]\centering
    \caption{GPU memory consumption (MB) and the proportion of reserved messages (\%) in forward and backward passes of GD, CLUSTER, GAS, and LMC4Conv for training GCN. {\it Default} indicates the default batch size used in the codes and toolkits of GAS \cite{gas}. }\label{tab:memory_consumption}
    \setlength{\tabcolsep}{1.9mm}
    \resizebox{\linewidth}{!}{
    \begin{tabular}{cccccc}
    \toprule
    \textbf{Batch size} & \textbf{Methods} & Ogbn-arxiv & FLICKR & REDDIT & PPI\\
    \midrule
    \mc{2}{c}{Full-batch GD} & 681 MB/{\bf 100}\%/{\bf 100}\% & 411 MB/{\bf 100}\%/{\bf 100}\% & 2067 MB/{\bf 100}\%/{\bf 100}\% & 605 MB/{\bf 100}\%/{\bf 100}\%\\
    \midrule
    \mr{3}{1} & CLUSTER & {\bf 177} MB/\,\,\,67\%/\,\,\,67\% & {\bf 138} MB/\,\,\,57\%/\,\,\,57\% & {\bf 428} MB/\,\,\,35\%/\,\,\,35\% & {\bf 189} MB/\,\,\,90\%/\,\,\,90\%\\
    & GAS & 178 MB/{\bf 100}\%/\,\,\,67\% & 168 MB/{\bf 100}\%/\,\,\,57\% & 482 MB/{\bf 100}\%/\,\,\,35\% & 190 MB/{\bf 100}\%/\,\,\,90\%\\
    & {\bf LMC4Conv} & 207 MB/{\bf 100}\%/{\bf 100}\% & 177 MB /{\bf 100}\%/{\bf 100}\% & 610 MB/{\bf 100}\%/{\bf 100}\% & 197 MB/{\bf 100}\%/{\bf 100}\%\\
    \midrule
    \mr{3}{\rm Default} & CLUSTER & 424 MB/\,\,\,83\%/\,\,\,83\% & 310 MB/\,\,\,77\%/\,\,\,77\% & 1193 MB/\,\,\,65\%/\,\,\,65\% & 212 MB/\,\,\,91\%/\,\,\,91\%\\
    & GAS & 452 MB/{\bf 100}\%/\,\,\,83\% & 375 MB/{\bf 100}\%/\,\,\,77\% & 1508 MB/{\bf 100}\%/\,\,\,65\% & 214 MB/{\bf 100}\%/\,\,\,91\%\\
    & {\bf LMC4Conv} & 557 MB/{\bf 100}\%/{\bf 100}\% & 376 MB/{\bf 100}\%/{\bf 100}\% & 1829 MB/{\bf 100}\%/{\bf 100}\% & 267 MB/{\bf 100}\%/{\bf 100}\%\\
    \bottomrule
  \end{tabular}
    }
\end{table*}

}

\subsubsection{LMC4Conv is Robust in terms of Batch Sizes}\label{sec:smallbatchsize_conv}

An appealing feature of mini-batch training methods is to be able to avoid the out-of-memory issue by decreasing the batch size. Thus, we evaluate the prediction performance of LMC4Conv on Ogbn-arxiv datasets with different batch sizes (numbers of clusters).
We conduct experiments under different sizes of sampled clusters per mini-batch.
We run each experiment with the same epoch and search learning rates in the same set.
We report the best prediction accuracy in Table \ref{tab:batchsize_conv}.
LMC4Conv outperforms GAS under small batch sizes (batch size = $1$ or $2$) and achieve a comparable performance with GAS (batch size = $5$ or $10$).

\begin{table}[ht]
  \centering
  \caption{\textbf{Performance under different batch sizes on the Ogbn-arxiv dataset.}
  }\label{tab:batchsize_conv}
  \setlength{\tabcolsep}{8pt}
  \resizebox{1.0\linewidth}{!}{%
  \begin{tabular}{ccccc}
    \toprule
    \mr{2}{Batch size} & \mc{2}{l}{\quad\quad\,\,\, GCN} & \mc{2}{l}{\quad\quad \,\,\,GCNII} \\
     & GAS & LMC4Conv & GAS & LMC4Conv \\
    \midrule
    1 & 70.56 & \textbf{71.65} & 71.34 & \textbf{72.11} \\ 
    2 & 71.11 & \textbf{71.89} & 72.25 & \textbf{72.55} \\ 
    5 & \textbf{71.99} & 71.84 & 72.23 & \textbf{72.87}   \\
    10& 71.60 & \textbf{72.14} & \textbf{72.82} & 72.80 \\
    \bottomrule
  \end{tabular}
  }
\end{table}

\begin{figure*}[t]
\centering 
\begin{subfigure}{0.5\linewidth}
  \includegraphics[width=\linewidth]{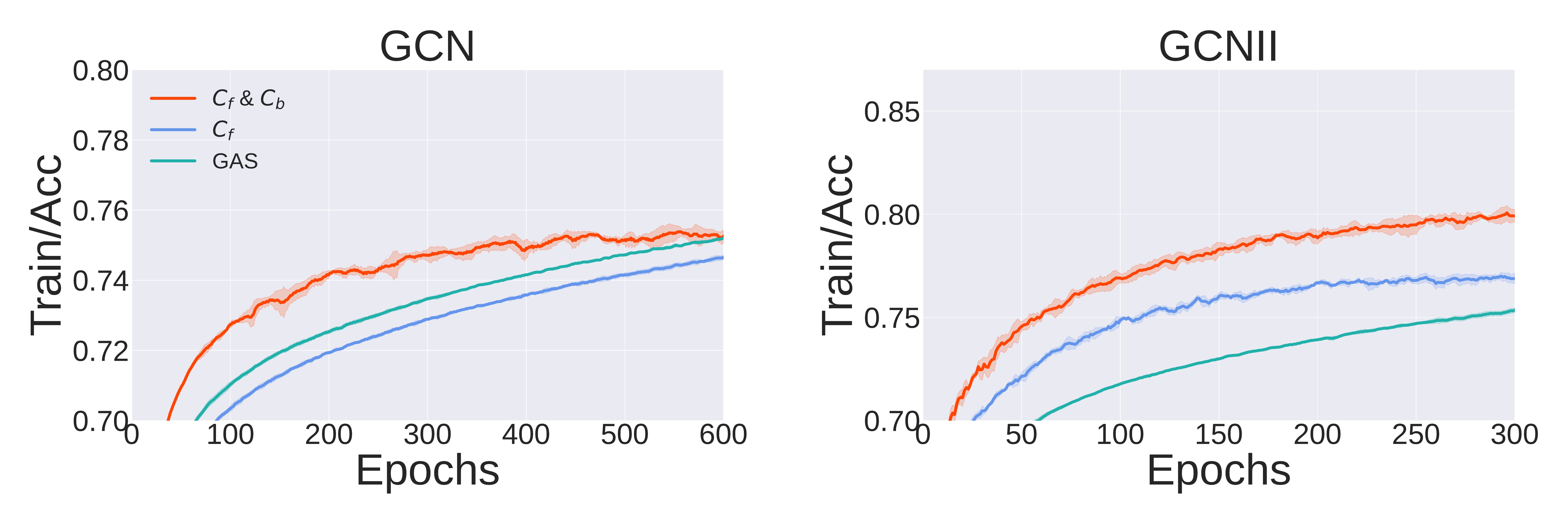}
  \caption{Under small batch sizes}\label{subfig:ablation_small}
\end{subfigure}\hfil 
\begin{subfigure}{0.5\linewidth}
  \includegraphics[width=\linewidth]{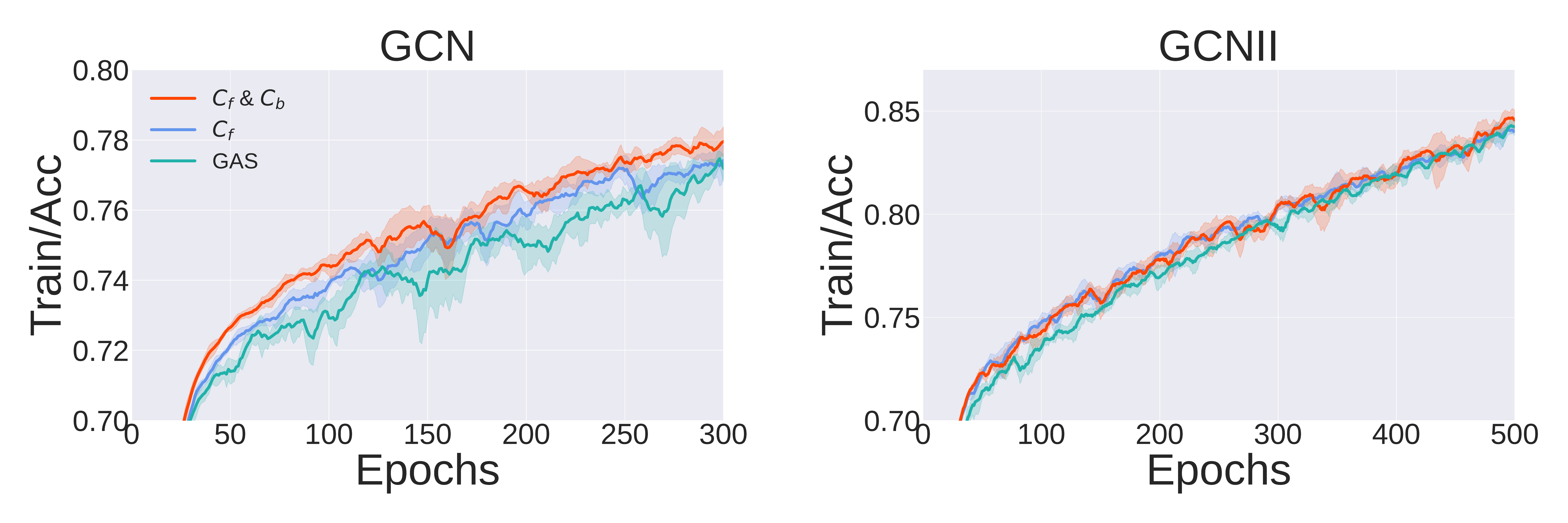}
    \caption{Under large batch sizes}\label{subfig:ablation_large}
\end{subfigure}\hfil
\caption{The improvement of the compensations on the Ogbn-arixv dataset.} \label{fig:ablation}
\end{figure*}

\begin{figure*}[b]
\centering 
\begin{subfigure}{1.0\linewidth}
  \includegraphics[width=\linewidth]{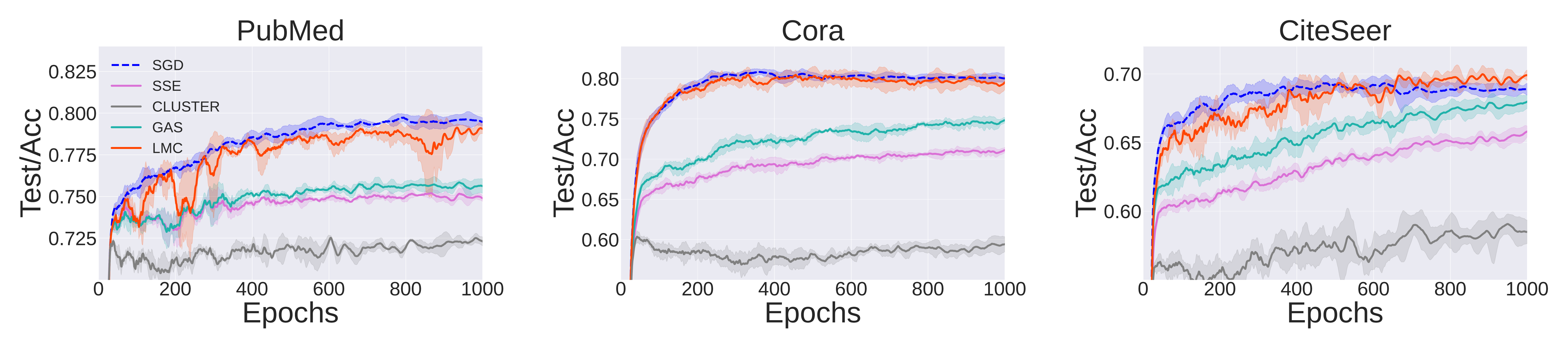}
    \caption{Testing accuracy}\label{subfig:test_acc_random}
\end{subfigure}\hfil
\begin{subfigure}{1.0\linewidth}
  \includegraphics[width=\linewidth]{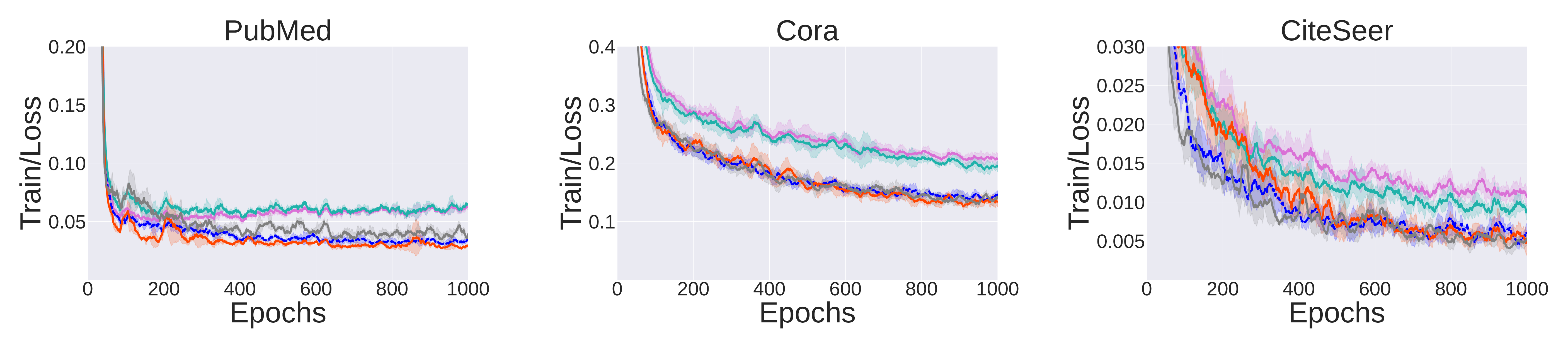}
  \caption{Training loss}\label{subfig:train_loss_random}
\end{subfigure}\hfil 
\caption{
Testing accuracy and training loss w.r.t. the number of epochs under random partitions.} \label{fig:convergence_random}
\end{figure*}

\subsubsection{Ablation}\label{sec:ablation}
The improvement of LMC4Conv is due to two parts: the compensation in forward passes $\mathbf{C}_f^l$ and the compensation in back passes $\mathbf{C}_b^l$.
Compared with GAS, the compensation in forward passes $\mathbf{C}_f^l$ additionally combines the incomplete up-to-date messages.
Fig. \ref{fig:ablation} show the convergence curves of LMC4Conv using both $\mathbf{C}_f^l$ and $\mathbf{C}_b^l$ (denoted by $\mathbf{C}_f$\&$\mathbf{C}_b$), LMC4Conv using only $\mathbf{C}_f^l$ (denoted by $\mathbf{C}_f$), and GAS on the Ogbn-arixv dataset.
Under small batch sizes, the improvement mainly comes from $\mathbf{C}_b^l$ and the incomplete up-to-date messages in forward passes may hurt the performance. This is because the mini-batch and the union of their neighbors are hard to contain most neighbors of out-of-batch nodes when the batch size is small.
Thus, the compensation in back passes $\mathbf{C}_b^l$ is the most important component by correcting the bias of the mini-batch gradients.
Under large batch sizes, the improvement is due to $\mathbf{C}_f^l$, as the large batch sizes decrease the discarded messages and improve the accuracy of the mini-batch gradients.
Notably, $\mathbf{C}_b^l$ still slightly improves the performance.

\subsection{Experiments with RecGNNs (Infinite Number of MP Layers)} \label{sec:exp_recgnns}

We first introduce experimental settings in Section \ref{sec:setting_rec}. We then evaluate the convergence of LMC4Rec in Section \ref{sec:exp_fast}. Finally, we demonstrate that LMC4Rec can efficiently scale RecGNNs to large-scale graphs in Sections \ref{sec:scalable_rec} and \ref{sec:smallbatchsize_rec}.

\subsubsection{Experimental Settings} \label{sec:setting_rec}

We evaluate LMC4Rec for RecGCN (see Appendix \ref{sec:recgcn} for details) on three small datasets, including Cora, Citeseer, and PubMed from Planetoid \cite{planetoid}, and five large datasets, PPI, Reddit \cite{graphsage}, AMAZON \cite{amazon}, Ogbn-arxiv, and Ogbn-products \cite{ogb}. On the three small datasets, we evaluate the convergence performance and the ability to learn long-range dependencies. On the five large datasets, we evaluate the efficiency of scalable methods. We use the same data splits, training pipeline, subgraph partitions, and GNN structures for a fair comparison. We report the embedding dimensions, the learning rates, the number of partitions, and sampled clusters per mini-batch for each dataset in Table \ref{tab:hyperparameters} in Appendix \ref{sec:hyperparameters_recgcn}. We provide implementation details for important baselines such as GAS and Cluster-GCN in Appendix \ref{sec:implement_baseline}. For other implementation details, please refer to Appendix \ref{sec:recgcn}.

\begin{figure*}[t]
\centering 
\begin{subfigure}{1.0\linewidth}
  \includegraphics[width=\linewidth]{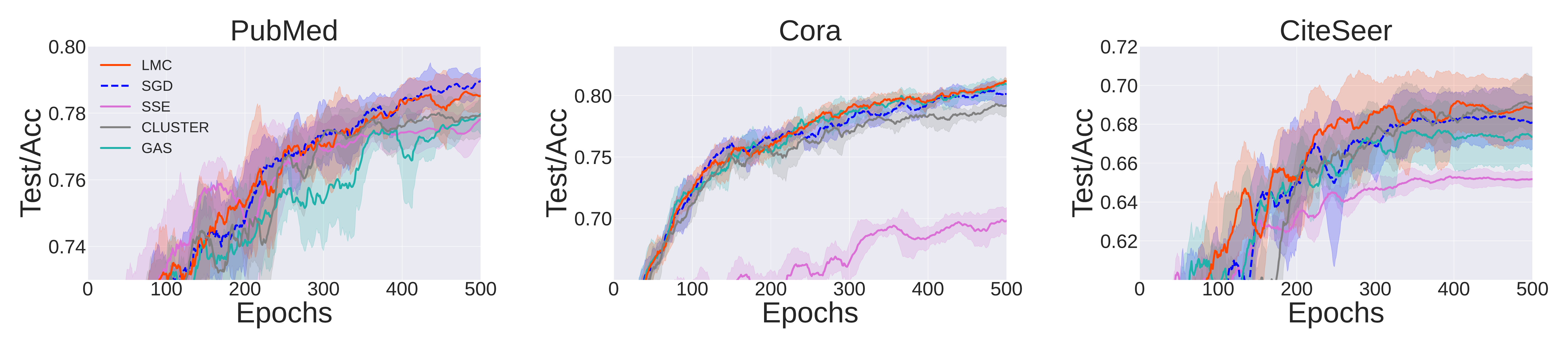}
    \caption{Testing accuracy}\label{subfig:test_acc_metis}
\end{subfigure}\hfil
\begin{subfigure}{1.0\linewidth}
  \includegraphics[width=\linewidth]{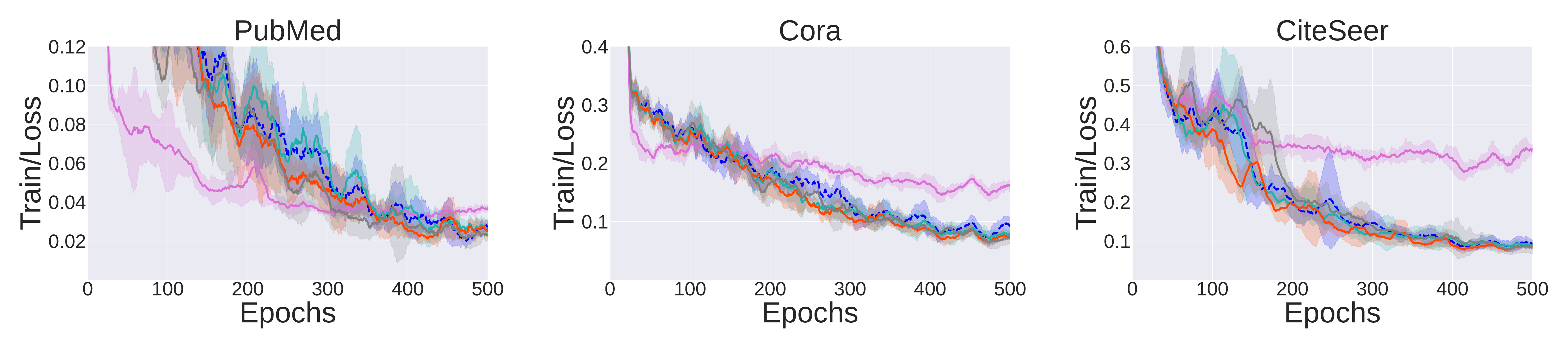}
  \caption{Training loss}\label{subfig:train_loss_metis}
\end{subfigure}\hfil 
\caption{
Testing accuracy and training loss w.r.t. the number of epochs under METIS partitions.} \label{fig:convergence_metis}
\end{figure*}

\subsubsection{LMC4Rec is Convergent}\label{sec:exp_fast}

To evaluate the convergence of LMC4Rec, we implement backward SGD introduced in Section \ref{sec:naive_sgd} as an exact baseline. We compare LMC4Rec with backward SGD, the state-of-the-art mini-batch training method for RecGNNs (SSE \cite{sse}), and subgraph-wise sampling methods (Cluster-GCN \cite{cluster_gcn} and GAS \cite{gas}). We plot the training loss and the testing accuracy on the small datasets under random and METIS partitions in Figs. \ref{fig:convergence_random} and \ref{fig:convergence_metis} respectively. We provide the runtime in Appendix \ref{sec:acc_runtime_small_rec}. We run each experiment with five different random seeds. The solid curves correspond to the mean, and the shaded regions correspond to values within plus or minus one standard deviation of the mean. Although LMC4Rec uses mini-batch gradients based on historical information to update parameters, the convergence behavior of LMC4Rec is comparable with that of backward SGD, which computes the exact mini-batch gradients. Moreover, LMC4Rec outperforms other scalable algorithms for RecGNNs in terms of  convergence speed. Although the running time of LMC4Rec each epoch is slightly more than state-of-the-art subgraph-wise sampling methods, LMC4Rec can accelerate the training of RecGCN due to the improvement of the convergence speed, as shown in Section \ref{sec:scalable_rec}.

Figs. \ref{fig:convergence_random} and \ref{fig:convergence_metis} demonstrate that LMC4Rec is robust in terms of different subgraph partitions. Random partitions randomly select a set of nodes to construct a subgraph, resulting in a set of unconnected nodes in a sampled subgraph. Thus, under random partitions, the compensation messages are critical to recovering the external information out of the subgraph for subgraph-wise sampling methods. Besides, METIS partitions help minimize the ignored inter-connectivity among subgraphs by constructing subgraphs with connected nodes, which hides the weakness of ignoring the external information. However, METIS can not remove the whole inter-connectivity among subgraphs. When the inter-connectivity among subgraphs is important, these subgraph-wise sampling methods also suffer from sub-optimal performance under METIS partitions (see Section \ref{sec:scalable_rec}). 
SSE is slower than LMC4Rec on all datasets, as the gradient computed by SSE is the first-order approximated solution to Eq. \eqref{eqn:mpeq_grad}, which is very different from the exact gradient.  By noticing that LMC4Rec also benefits from METIS partitions\footnote{METIS increases the range of learning rates for LMC4Rec and hence accelerates convergence.}, we use METIS partitions in the experiments in Sections \ref{sec:scalable_rec}.

\begin{figure}[ht]
\centering 
    \begin{subfigure}{.5\linewidth}
        \includegraphics[width=0.97\linewidth]{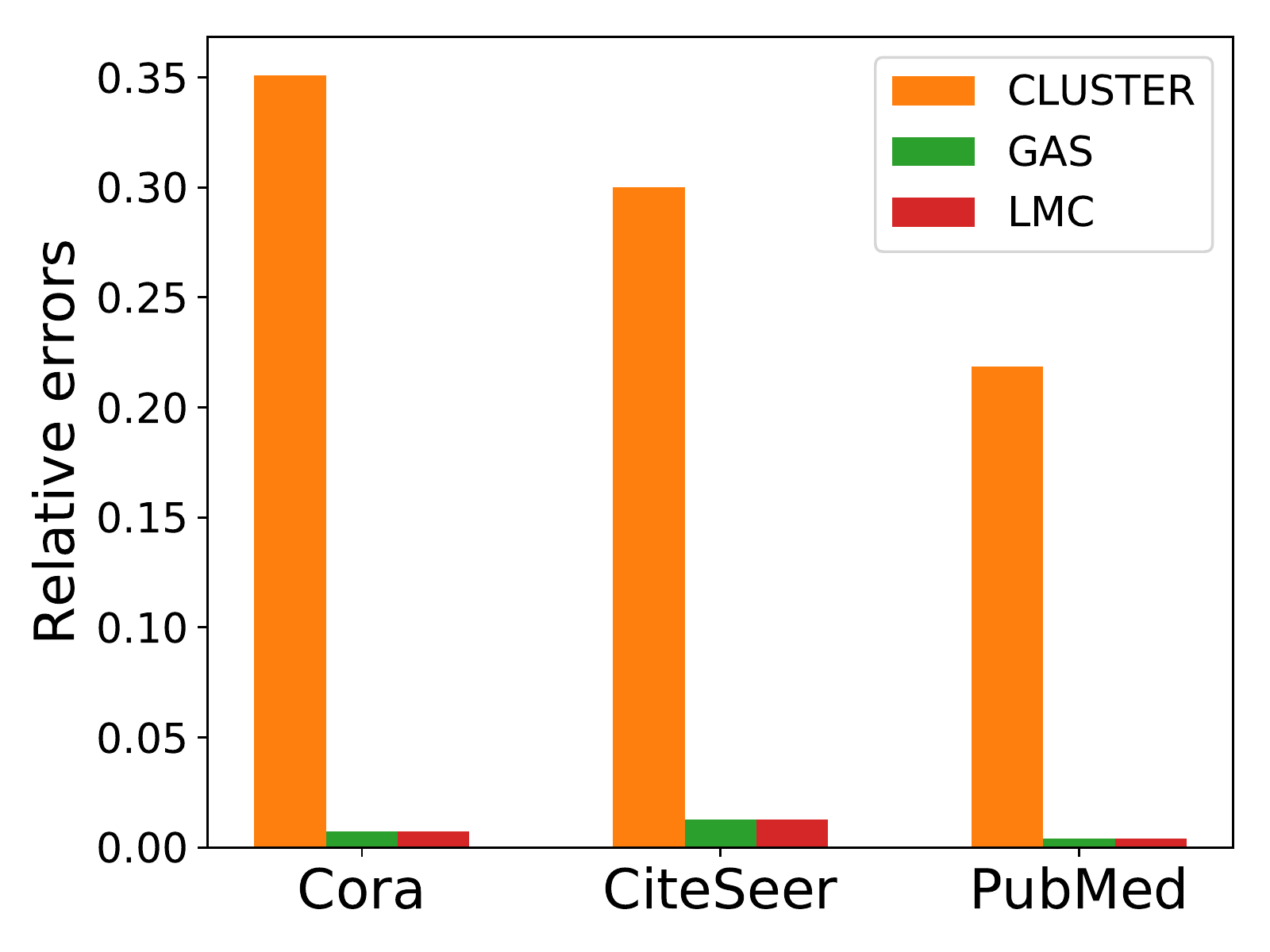}
        \caption{$\hisH_{\inbatch}$ \& random}\label{subfig:random_herror}
    \end{subfigure}\hfil
    \begin{subfigure}{.5\linewidth}
      \includegraphics[width=0.97\linewidth]{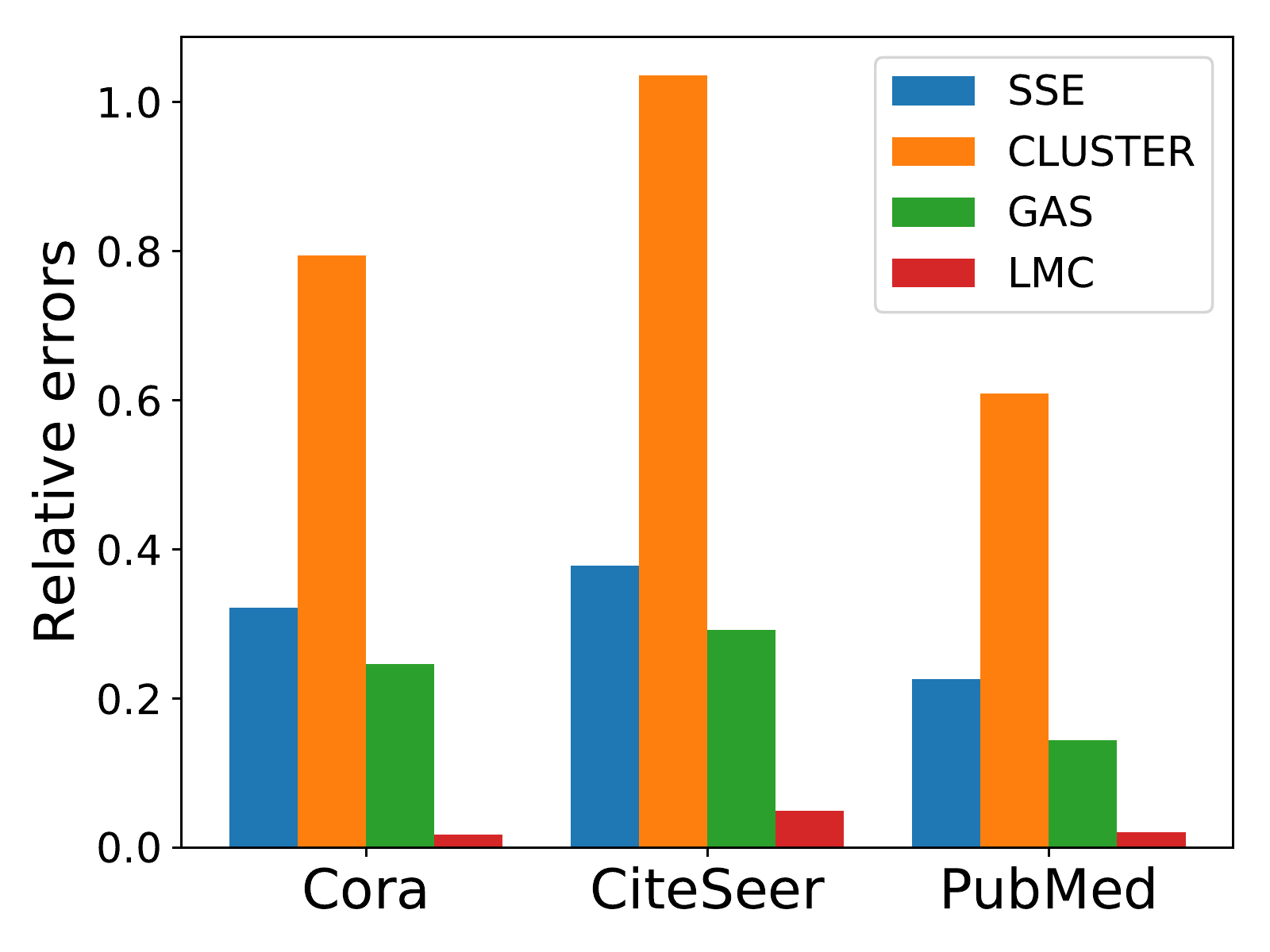}
        \caption{$\overline{\embV}_{\inbatch}$ \& random}\label{subfig:random_graderror}
    \end{subfigure}\hfil 
    \begin{subfigure}{.5\linewidth}
      \includegraphics[width=0.97\linewidth]{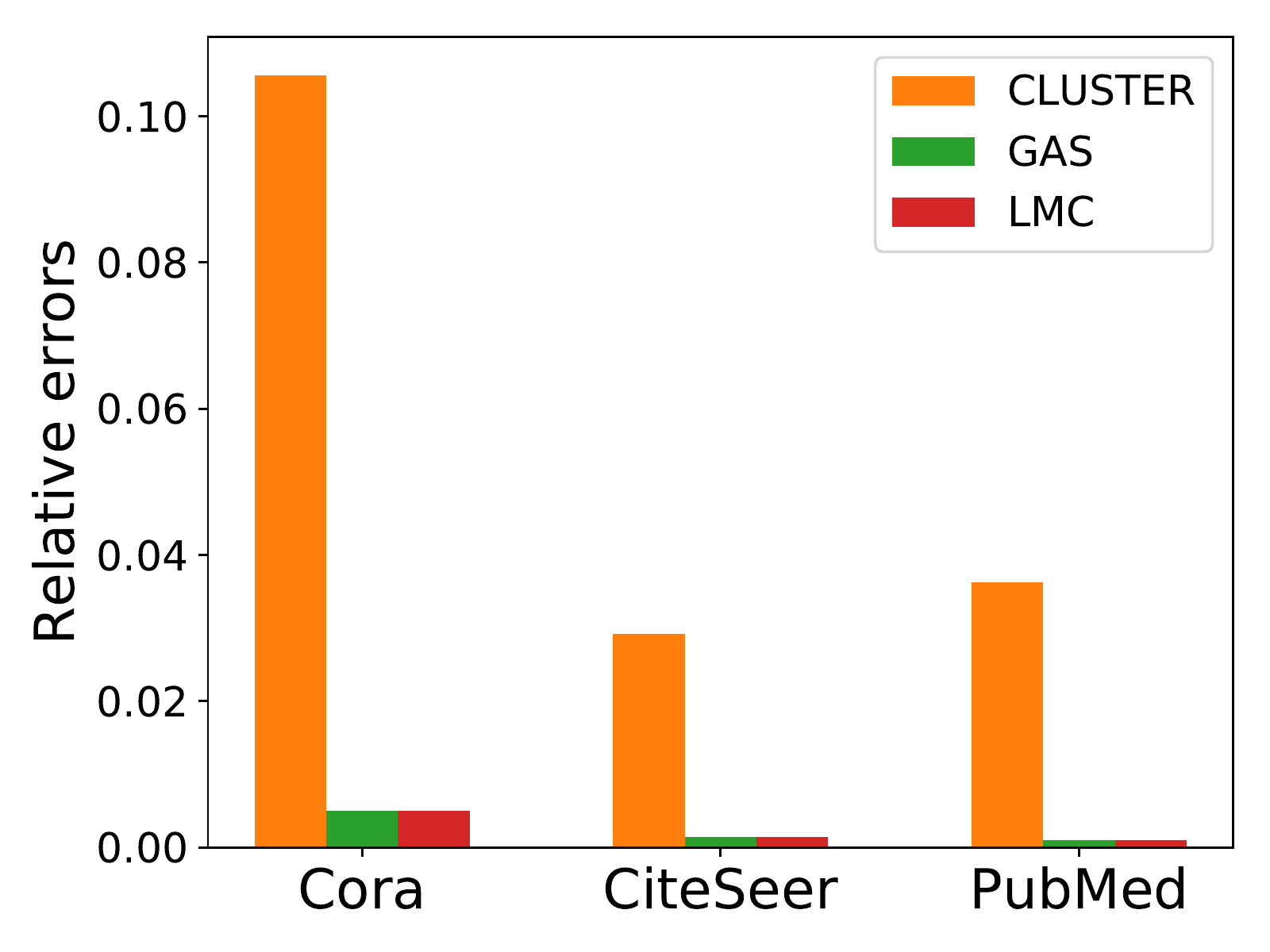}
      \caption{$\hisH_{\inbatch}$ \& METIS}\label{subfig:metis_herror}
    \end{subfigure}\hfil
    \begin{subfigure}{.5\linewidth}
      \includegraphics[width=0.97\linewidth]{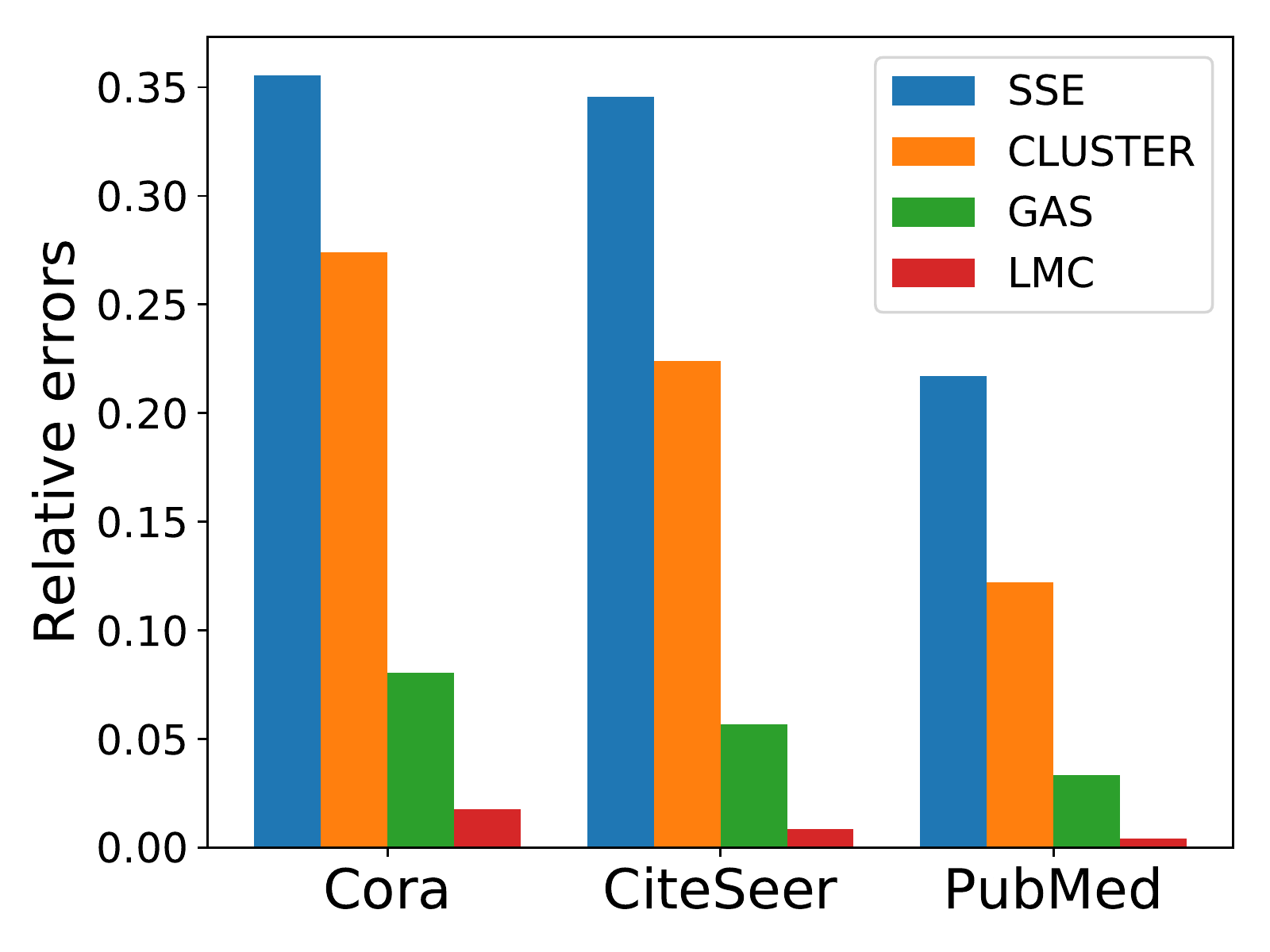}
        \caption{$\overline{\embV}_{\inbatch}$ \& METIS}\label{subfig:metis_graderror}
    \end{subfigure}\hfil 
    \caption{
    The average relative estimated errors of node embeddings $\hisH_{\inbatch}$ and auxiliary variables $\overline{\embV}_{\inbatch}$ under random and METIS partitions. 
    } \label{fig:error_rec}
\end{figure}

To further illustrate the convergence of LMC4Rec, we compare estimated errors of node embeddings and auxiliary variables computed by SSE, Cluster-GCN, GAS, and LMC4Rec. For a subgraph $\inbatch$, we use the node embeddings and auxiliary variables computed by backward SGD as the exact values. We report average relative estimated errors of node embeddings {$\|\hisH_{\inbatch} - \embH_{\inbatch}\|_F\,/\,\|\embH_{\inbatch}\|_F$} and auxiliary variables {$ \|\overline{\embV}_{\inbatch} - \embV_{\inbatch} \|_F\,/\,\|\embV_{\inbatch} \|_F$} in Fig. \ref{fig:error_rec}. LMC4Rec enjoys the smallest estimated errors of gradients in the experiments.

\begin{table*}[t]
    \centering
      \caption{%
      \textbf{Efficiency of \textsc{RecGCN} with Cluster-GCN, GAS, and LMC4Rec.}
      }\label{tab:runtime}
    \setlength{\tabcolsep}{2mm}{
    \begin{tabular}{lccccccccc}
    \toprule
    \mr{2}{\textbf{Dataset}} & \mc{3}{c}{\textbf{Epochs}} & \mc{3}{c}{\textbf{Runtime} (s)} & \mc{3}{c}{\textbf{Memory} (MB)} \\
    & {\small Cluster-GCN} & {\small GAS} & {\small LMC4Rec} & {\small Cluster-GCN} & {\small GAS} & {\small LMC4Rec} & {\small Cluster-GCN} & {\small GAS} & {\small LMC4Rec} \\
    \midrule
    PPI & 1000 & 353 & \textbf{301} & 36020 & 13873 & \textbf{13245}  & \textbf{1999} & 3517 & 4103 \\
    REDDIT & 1000 & 119 & \textbf{102}  & 5487 & 1138 & \textbf{999}  & \textbf{5159} & 7347 & 8009 \\
    AMAZON & 1000 & 607 & \textbf{485} & 6662 & 5291 & \textbf{4414}  & \textbf{1925} & 1947 & 1955   \\
    Ogbn-arxiv & 1000 & 353  & \textbf{332}& 3865 & 2148  & \textbf{2035} & \textbf{1865}  & 3453 & 3453 \\
    Ogbn-products & 1000 & 319 & \textbf{270} & 30018 & 18560 & \textbf{16450} & \textbf{2345} & 8151 & 8295 \\
    \bottomrule
  \end{tabular}
    }
\end{table*}

\subsubsection{LMC4Rec is Fast}\label{sec:scalable_rec}

In Section \ref{sec:exp_fast}, we demonstrate that LMC4Rec slightly outperforms GAS \cite{gas} in terms of convergence speed under METIS partitions. However, as LMC4Rec additionally computes local message compensation for gradients $\mathbf{C}_b$, whether the improvement of convergence speed accelerates the training of RecGNNs is unclear. We thus compare the runtime of LMC4Rec, Cluster-GCN, and GAS on the five large datasets. Table \ref{tab:runtime} reports the GPU memory, the number of epochs and time to reach given 98.5\%, 96\%, 87\%, 72\%, and 77\% accuracy\footnote{Cluster-GCN does not reach the given accuracy on the datasets (see accuracy vs. runtime in Appendix \ref{sec:acc_runtime_large_rec}).} on PPI, REDDIT, AMAZON, Ogbn-arxiv, and Ogbn-products, respectively (see Appendix \ref{sec:acc_runtime_large_rec} for the convergence curve). As the additional computation of LMC4Rec is very cheap, LMC4Rec is faster than Cluster-GCN and GAS on the five large datasets due to the improvement of convergence speed.

\subsubsection{LMC4Rec is Robust in terms of Batch Sizes}\label{sec:smallbatchsize_rec}

We evaluate the prediction performance of LMC4Rec on AMAZON and Ogbn-arxiv datasets with different batch sizes (numbers of clusters). We set the number of partitions to 80 and 400 on Ogbn-arxiv and AMAZON, respectively, and conduct experiments under different sizes of sampled clusters per mini-batch. We run each experiment with the same epoch and report the best prediction accuracy in Table \ref{tab:batchsize_rec}. The performance of LMC4Rec outperforms GAS in each experiment, especially on the AMAZON dataset with small batch sizes. By noticing that AMAZON is a standard dataset to evaluate the ability to learn long-range dependencies \cite{ignn}, LMC4Rec can retrieve the inter-connectivity among sampled subgraphs discarded by subgraph partitions to help extract long-range patterns, while GAS, the state-of-the-art subgraph-wise sampling method, may fail. We include more experiments in terms of prediction performance in Appendix \ref{sec:prediction_large_rec}.

\begin{table}[ht]
  \centering
  \caption{\textbf{Performance under different batch sizes.}
  }\label{tab:batchsize_rec}
  \setlength{\tabcolsep}{8pt}
  \resizebox{1.0\linewidth}{!}{%
  \begin{tabular}{ccccc}
    \toprule
    \mr{2}{Batch size} & \mc{2}{l}{\quad\,\, AMAZON} & \mc{2}{l}{\quad\,\,Ogbn-arxiv} \\
     & GAS & LMC4Rec & GAS & LMC4Rec \\
    \midrule
    1 & 83.04 & \textbf{87.03} & 72.28 & \textbf{72.52} \\
    2 & 84.90 & \textbf{87.35} & 72.10 & \textbf{72.12} \\ 
    5 & 87.48 & \textbf{87.71} & 72.19 & \textbf{72.38}   \\ 
    10& 87.34 & \textbf{87.83} & 71.92 & \textbf{72.64} \\ 
    \bottomrule
  \end{tabular}
  }
\end{table}

\section{Conclusion}\label{sec:conclusion}

In this paper, we propose a novel and scalable subgraph-wise sampling with provable convergence for GNNs, namely local message compensation (LMC), based on the message passing formulation of backward passes. We show that LMC converges to stationary points of GNNs. To the best of our knowledge, LMC is the first subgraph-wise sampling method with convergence guarantees. Experiments demonstrate that LMC significantly outperforms state-of-the-art training methods in terms of efficiency without sacrificing accuracy on large-scale benchmark tasks.

\bibliographystyle{IEEEtran}
\bibliography{LMC}

\begin{thebibliography}{10}
\providecommand{\url}[1]{#1}
\csname url@samestyle\endcsname
\providecommand{\newblock}{\relax}
\providecommand{\bibinfo}[2]{#2}
\providecommand{\BIBentrySTDinterwordspacing}{\spaceskip=0pt\relax}
\providecommand{\BIBentryALTinterwordstretchfactor}{4}
\providecommand{\BIBentryALTinterwordspacing}{\spaceskip=\fontdimen2\font plus
\BIBentryALTinterwordstretchfactor\fontdimen3\font minus
  \fontdimen4\font\relax}
\providecommand{\BIBforeignlanguage}[2]{{%
\expandafter\ifx\csname l@#1\endcsname\relax
\typeout{** WARNING: IEEEtran.bst: No hyphenation pattern has been}%
\typeout{** loaded for the language `#1'. Using the pattern for}%
\typeout{** the default language instead.}%
\else
\language=\csname l@#1\endcsname
\fi
#2}}
\providecommand{\BIBdecl}{\relax}
\BIBdecl

\bibitem{grl}
W.~L. Hamilton, ``Graph representation learning,'' \emph{Synthesis Lectures on
  Artificial Intelligence and Machine Learning}, vol.~14, no.~3, pp. 1--159,
  2020.

\bibitem{gnn_recommendation}
S.~Brin and L.~Page, ``The anatomy of a large-scale hypertextual web search
  engine,'' in \emph{Proceedings of the Seventh International Conference on
  World Wide Web 7}, ser. WWW7.\hskip 1em plus 0.5em minus 0.4em\relax NLD:
  Elsevier Science Publishers B. V., 1998, p. 107–117.

\bibitem{gnn_social}
W.~Fan, Y.~Ma, Q.~Li, Y.~He, E.~Zhao, J.~Tang, and D.~Yin, ``Graph neural
  networks for social recommendation,'' in \emph{The World Wide Web
  Conference}, ser. WWW '19, 2019, p. 417–426.

\bibitem{gnn_material}
J.~Gostick, M.~Aghighi, J.~Hinebaugh, T.~Tranter, M.~A. Hoeh, H.~Day,
  B.~Spellacy, M.~H. Sharqawy, A.~Bazylak, A.~Burns, W.~Lehnert, and A.~Putz,
  ``Openpnm: A pore network modeling package,'' \emph{Computing in Science \&
  Engineering}, vol.~18, no.~04, pp. 60--74, jul 2016.

\bibitem{gnn_mol1}
N.~P. Moloi and M.~M. Ali, ``An iterative global optimization algorithm for
  potential energy minimization,'' \emph{Comput. Optim. Appl.}, vol.~30, no.~2,
  pp. 119--132, 2005.

\bibitem{gnn_mol2}
S.~Kearnes, K.~McCloskey, M.~Berndl, V.~Pande, and P.~Riley, ``Molecular graph
  convolutions: Moving beyond fingerprints,'' \emph{Journal of Computer-Aided
  Molecular Design}, vol.~30, 08 2016.

\bibitem{mpnn}
\BIBentryALTinterwordspacing
J.~Gilmer, S.~S. Schoenholz, P.~F. Riley, O.~Vinyals, and G.~E. Dahl, ``Neural
  message passing for quantum chemistry,'' in \emph{Proceedings of the 34th
  International Conference on Machine Learning}, ser. Proceedings of Machine
  Learning Research, D.~Precup and Y.~W. Teh, Eds., vol.~70.\hskip 1em plus
  0.5em minus 0.4em\relax PMLR, 06--11 Aug 2017, pp. 1263--1272. [Online].
  Available: \url{https://proceedings.mlr.press/v70/gilmer17a.html}
\BIBentrySTDinterwordspacing

\bibitem{comprehensive}
Z.~Wu, S.~Pan, F.~Chen, G.~Long, C.~Zhang, and S.~Y. Philip, ``A comprehensive
  survey on graph neural networks,'' \emph{IEEE transactions on neural networks
  and learning systems}, vol.~32, no.~1, pp. 4--24, 2020.

\bibitem{gcn}
T.~N. Kipf and M.~Welling, ``Semi-supervised classification with graph
  convolutional networks,'' in \emph{{ICLR} (Poster)}.\hskip 1em plus 0.5em
  minus 0.4em\relax OpenReview.net, 2017.

\bibitem{gat}
P.~Veli{\v{c}}kovi{\'{c}}, G.~Cucurull, A.~Casanova, A.~Romero, P.~Li{\`{o}},
  and Y.~Bengio, ``{Graph Attention Networks},'' in \emph{International
  Conference on Learning Representations}, 2018.

\bibitem{gin}
K.~Xu, W.~Hu, J.~Leskovec, and S.~Jegelka, ``How powerful are graph neural
  networks?'' in \emph{International Conference on Learning Representations},
  2019.

\bibitem{ignn}
\BIBentryALTinterwordspacing
F.~Gu, H.~Chang, W.~Zhu, S.~Sojoudi, and L.~El~Ghaoui, ``Implicit graph neural
  networks,'' in \emph{Advances in Neural Information Processing Systems},
  H.~Larochelle, M.~Ranzato, R.~Hadsell, M.~F. Balcan, and H.~Lin, Eds.,
  vol.~33.\hskip 1em plus 0.5em minus 0.4em\relax Curran Associates, Inc.,
  2020, pp. 11\,984--11\,995. [Online]. Available:
  \url{https://proceedings.neurips.cc/paper/2020/file/8b5c8441a8ff8e151b191c53c1842a38-Paper.pdf}
\BIBentrySTDinterwordspacing

\bibitem{idl}
L.~El~Ghaoui, F.~Gu, B.~Travacca, A.~Askari, and A.~Tsai, ``Implicit deep
  learning,'' \emph{SIAM Journal on Mathematics of Data Science}, vol.~3,
  no.~3, pp. 930--958, 2021.

\bibitem{deq}
\BIBentryALTinterwordspacing
S.~Bai, J.~Z. Kolter, and V.~Koltun, ``Deep equilibrium models,'' in
  \emph{Advances in Neural Information Processing Systems}, vol.~32.\hskip 1em
  plus 0.5em minus 0.4em\relax Curran Associates, Inc., 2019. [Online].
  Available:
  \url{https://proceedings.neurips.cc/paper/2019/file/01386bd6d8e091c2ab4c7c7de644d37b-Paper.pdf}
\BIBentrySTDinterwordspacing

\bibitem{recgnn1}
M.~Gori, G.~Monfardini, and F.~Scarselli, ``A new model for learning in graph
  domains,'' in \emph{Proceedings. 2005 IEEE International Joint Conference on
  Neural Networks, 2005.}, vol.~2, 2005, pp. 729--734 vol. 2.

\bibitem{contraction}
F.~Scarselli, M.~Gori, A.~C. Tsoi, M.~Hagenbuchner, and G.~Monfardini, ``The
  graph neural network model,'' \emph{IEEE transactions on neural networks},
  vol.~20, no.~1, pp. 61--80, 2008.

\bibitem{fdgnn}
C.~Gallicchio and A.~Micheli, ``Fast and deep graph neural networks,'' in
  \emph{Proceedings of the AAAI Conference on Artificial Intelligence},
  vol.~34, no.~04, 2020, pp. 3898--3905.

\bibitem{dlg}
Y.~Ma and J.~Tang, \emph{Deep Learning on Graphs}.\hskip 1em plus 0.5em minus
  0.4em\relax Cambridge University Press, 2021.

\bibitem{graphsage}
W.~Hamilton, Z.~Ying, and J.~Leskovec, ``Inductive representation learning on
  large graphs,'' in \emph{Advances in Neural Information Processing Systems},
  2017, p. 1025–1035.

\bibitem{vrgcn}
J.~Chen, J.~Zhu, and L.~Song, ``Stochastic training of graph convolutional
  networks with variance reduction,'' in \emph{Proceedings of the 35th
  International Conference on Machine Learning}, ser. Proceedings of Machine
  Learning Research, J.~Dy and A.~Krause, Eds., vol.~80.\hskip 1em plus 0.5em
  minus 0.4em\relax PMLR, 10--15 Jul 2018, pp. 942--950.

\bibitem{fastgcn}
\BIBentryALTinterwordspacing
J.~Chen, T.~Ma, and C.~Xiao, ``Fast{GCN}: Fast learning with graph
  convolutional networks via importance sampling,'' in \emph{International
  Conference on Learning Representations}, 2018. [Online]. Available:
  \url{https://openreview.net/forum?id=rytstxWAW}
\BIBentrySTDinterwordspacing

\bibitem{ladies}
\BIBentryALTinterwordspacing
D.~Zou, Z.~Hu, Y.~Wang, S.~Jiang, Y.~Sun, and Q.~Gu, ``Layer-dependent
  importance sampling for training deep and large graph convolutional
  networks,'' in \emph{Advances in Neural Information Processing Systems},
  H.~Wallach, H.~Larochelle, A.~Beygelzimer, F.~d\textquotesingle
  Alch\'{e}-Buc, E.~Fox, and R.~Garnett, Eds., vol.~32.\hskip 1em plus 0.5em
  minus 0.4em\relax Curran Associates, Inc., 2019. [Online]. Available:
  \url{https://proceedings.neurips.cc/paper/2019/file/91ba4a4478a66bee9812b0804b6f9d1b-Paper.pdf}
\BIBentrySTDinterwordspacing

\bibitem{adapt}
\BIBentryALTinterwordspacing
W.~Huang, T.~Zhang, Y.~Rong, and J.~Huang, ``Adaptive sampling towards fast
  graph representation learning,'' in \emph{Advances in Neural Information
  Processing Systems}, S.~Bengio, H.~Wallach, H.~Larochelle, K.~Grauman,
  N.~Cesa-Bianchi, and R.~Garnett, Eds., vol.~31.\hskip 1em plus 0.5em minus
  0.4em\relax Curran Associates, Inc., 2018. [Online]. Available:
  \url{https://proceedings.neurips.cc/paper/2018/file/01eee509ee2f68dc6014898c309e86bf-Paper.pdf}
\BIBentrySTDinterwordspacing

\bibitem{cluster_gcn}
W.-L. Chiang, X.~Liu, S.~Si, Y.~Li, S.~Bengio, and C.-J. Hsieh, ``Cluster-gcn:
  An efficient algorithm for training deep and large graph convolutional
  networks,'' in \emph{Proceedings of the 25th ACM SIGKDD International
  Conference on Knowledge Discovery \& Data Mining}, 2019, pp. 257--266.

\bibitem{graphsaint}
\BIBentryALTinterwordspacing
H.~Zeng, H.~Zhou, A.~Srivastava, R.~Kannan, and V.~Prasanna, ``Graphsaint:
  Graph sampling based inductive learning method,'' in \emph{International
  Conference on Learning Representations}, 2020. [Online]. Available:
  \url{https://openreview.net/forum?id=BJe8pkHFwS}
\BIBentrySTDinterwordspacing

\bibitem{gas}
\BIBentryALTinterwordspacing
M.~Fey, J.~E. Lenssen, F.~Weichert, and J.~Leskovec, ``Gnnautoscale: Scalable
  and expressive graph neural networks via historical embeddings,'' in
  \emph{Proceedings of the 38th International Conference on Machine Learning},
  ser. Proceedings of Machine Learning Research, M.~Meila and T.~Zhang, Eds.,
  vol. 139.\hskip 1em plus 0.5em minus 0.4em\relax PMLR, 18--24 Jul 2021, pp.
  3294--3304. [Online]. Available:
  \url{https://proceedings.mlr.press/v139/fey21a.html}
\BIBentrySTDinterwordspacing

\bibitem{shadow_gnn}
\BIBentryALTinterwordspacing
H.~Zeng, M.~Zhang, Y.~Xia, A.~Srivastava, A.~Malevich, R.~Kannan, V.~Prasanna,
  L.~Jin, and R.~Chen, ``Decoupling the depth and scope of graph neural
  networks,'' in \emph{Advances in Neural Information Processing Systems},
  A.~Beygelzimer, Y.~Dauphin, P.~Liang, and J.~W. Vaughan, Eds., 2021.
  [Online]. Available: \url{https://openreview.net/forum?id=_IY3_4psXuf}
\BIBentrySTDinterwordspacing

\bibitem{mvs}
\BIBentryALTinterwordspacing
W.~Cong, R.~Forsati, M.~Kandemir, and M.~Mahdavi, ``Minimal variance sampling
  with provable guarantees for fast training of graph neural networks,'' in
  \emph{Proceedings of the 26th ACM SIGKDD International Conference on
  Knowledge Discovery \& Data Mining}, ser. KDD '20.\hskip 1em plus 0.5em minus
  0.4em\relax New York, NY, USA: Association for Computing Machinery, 2020, p.
  1393–1403. [Online]. Available:
  \url{https://doi.org/10.1145/3394486.3403192}
\BIBentrySTDinterwordspacing

\bibitem{eignn}
\BIBentryALTinterwordspacing
J.~Liu, K.~Kawaguchi, B.~Hooi, Y.~Wang, and X.~Xiao, ``{EIGNN}: Efficient
  infinite-depth graph neural networks,'' in \emph{Advances in Neural
  Information Processing Systems}, A.~Beygelzimer, Y.~Dauphin, P.~Liang, and
  J.~W. Vaughan, Eds., 2021. [Online]. Available:
  \url{https://openreview.net/forum?id=blzTEKKRIcV}
\BIBentrySTDinterwordspacing

\bibitem{lmc}
\BIBentryALTinterwordspacing
Z.~Shi, X.~Liang, and J.~Wang, ``{LMC}: Fast training of {GNN}s via subgraph
  sampling with provable convergence,'' in \emph{International Conference on
  Learning Representations}, 2023. [Online]. Available:
  \url{https://openreview.net/forum?id=5VBBA91N6n}
\BIBentrySTDinterwordspacing

\bibitem{metis1}
G.~Karypis and V.~Kumar, ``A fast and high quality multilevel scheme for
  partitioning irregular graphs,'' \emph{SIAM Journal on scientific Computing},
  vol.~20, no.~1, pp. 359--392, 1998.

\bibitem{graclus}
\BIBentryALTinterwordspacing
I.~S. Dhillon, Y.~Guan, and B.~Kulis, ``Weighted graph cuts without
  eigenvectors a multilevel approach,'' \emph{IEEE Trans. Pattern Anal. Mach.
  Intell.}, vol.~29, no.~11, p. 1944–1957, nov 2007. [Online]. Available:
  \url{https://doi.org/10.1109/TPAMI.2007.1115}
\BIBentrySTDinterwordspacing

\bibitem{graphfm}
\BIBentryALTinterwordspacing
H.~Yu, L.~Wang, B.~Wang, M.~Liu, T.~Yang, and S.~Ji, ``{G}raph{FM}: Improving
  large-scale {GNN} training via feature momentum,'' in \emph{Proceedings of
  the 39th International Conference on Machine Learning}, ser. Proceedings of
  Machine Learning Research, K.~Chaudhuri, S.~Jegelka, L.~Song, C.~Szepesvari,
  G.~Niu, and S.~Sabato, Eds., vol. 162.\hskip 1em plus 0.5em minus 0.4em\relax
  PMLR, 17--23 Jul 2022, pp. 25\,684--25\,701. [Online]. Available:
  \url{https://proceedings.mlr.press/v162/yu22g.html}
\BIBentrySTDinterwordspacing

\bibitem{ogb}
W.~Hu, M.~Fey, M.~Zitnik, Y.~Dong, H.~Ren, B.~Liu, M.~Catasta, and J.~Leskovec,
  ``Open graph benchmark: Datasets for machine learning on graphs,'' in
  \emph{Advances in Neural Information Processing Systems}, 2020, pp.
  22\,118--22\,133.

\bibitem{dropedge}
\BIBentryALTinterwordspacing
Y.~Rong, W.~Huang, T.~Xu, and J.~Huang, ``Dropedge: Towards deep graph
  convolutional networks on node classification,'' in \emph{International
  Conference on Learning Representations}, 2020. [Online]. Available:
  \url{https://openreview.net/forum?id=Hkx1qkrKPr}
\BIBentrySTDinterwordspacing

\bibitem{sign}
\BIBentryALTinterwordspacing
E.~Rossi, F.~Frasca, B.~Chamberlain, D.~Eynard, M.~M. Bronstein, and F.~Monti,
  ``Sign: Scalable inception graph neural networks,'' \emph{CoRR}, vol.
  abs/2004.11198, 2020. [Online]. Available:
  \url{https://arxiv.org/abs/2004.11198}
\BIBentrySTDinterwordspacing

\bibitem{gcnii}
\BIBentryALTinterwordspacing
M.~Chen, Z.~Wei, Z.~Huang, B.~Ding, and Y.~Li, ``Simple and deep graph
  convolutional networks,'' in \emph{Proceedings of the 37th International
  Conference on Machine Learning}, ser. Proceedings of Machine Learning
  Research, H.~D. III and A.~Singh, Eds., vol. 119.\hskip 1em plus 0.5em minus
  0.4em\relax PMLR, 13--18 Jul 2020, pp. 1725--1735. [Online]. Available:
  \url{http://proceedings.mlr.press/v119/chen20v.html}
\BIBentrySTDinterwordspacing

\bibitem{planetoid}
\BIBentryALTinterwordspacing
Z.~Yang, W.~Cohen, and R.~Salakhudinov, ``Revisiting semi-supervised learning
  with graph embeddings,'' in \emph{Proceedings of The 33rd International
  Conference on Machine Learning}, ser. Proceedings of Machine Learning
  Research, M.~F. Balcan and K.~Q. Weinberger, Eds., vol.~48.\hskip 1em plus
  0.5em minus 0.4em\relax New York, New York, USA: PMLR, 20--22 Jun 2016, pp.
  40--48. [Online]. Available:
  \url{https://proceedings.mlr.press/v48/yanga16.html}
\BIBentrySTDinterwordspacing

\bibitem{amazon}
J.~Yang and J.~Leskovec, ``Defining and evaluating network communities based on
  ground-truth,'' \emph{Knowledge and Information Systems}, vol.~42, no.~1, pp.
  181--213, 2015.

\bibitem{sse}
\BIBentryALTinterwordspacing
H.~Dai, Z.~Kozareva, B.~Dai, A.~Smola, and L.~Song, ``Learning steady-states of
  iterative algorithms over graphs,'' in \emph{Proceedings of the 35th
  International Conference on Machine Learning}, ser. Proceedings of Machine
  Learning Research, J.~Dy and A.~Krause, Eds., vol.~80.\hskip 1em plus 0.5em
  minus 0.4em\relax PMLR, 10--15 Jul 2018, pp. 1106--1114. [Online]. Available:
  \url{http://proceedings.mlr.press/v80/dai18a.html}
\BIBentrySTDinterwordspacing

\bibitem{glove}
\BIBentryALTinterwordspacing
J.~Pennington, R.~Socher, and C.~Manning, ``{G}lo{V}e: Global vectors for word
  representation,'' in \emph{Proceedings of the 2014 Conference on Empirical
  Methods in Natural Language Processing ({EMNLP})}.\hskip 1em plus 0.5em minus
  0.4em\relax Doha, Qatar: Association for Computational Linguistics, Oct.
  2014, pp. 1532--1543. [Online]. Available:
  \url{https://aclanthology.org/D14-1162}
\BIBentrySTDinterwordspacing

\bibitem{mag}
K.~Wang, Z.~Shen, C.~Huang, C.-H. Wu, Y.~Dong, and A.~Kanakia, ``Microsoft
  academic graph: When experts are not enough,'' \emph{Quantitative Science
  Studies}, vol.~1, no.~1, pp. 396--413, 2020.

\bibitem{word2vec}
T.~Mikolov, I.~Sutskever, K.~Chen, G.~Corrado, and J.~Dean, ``Distributed
  representations of words and phrases and their compositionality,'' in
  \emph{Advances in Neural Information Processing Systems 26}, 2013.

\bibitem{pytorch}
A.~Paszke, S.~Gross, F.~Massa, A.~Lerer, J.~Bradbury, G.~Chanan, T.~Killeen,
  Z.~Lin, N.~Gimelshein, L.~Antiga, A.~Desmaison, A.~Kopf, E.~Yang, Z.~DeVito,
  M.~Raison, A.~Tejani, S.~Chilamkurthy, B.~Steiner, L.~Fang, J.~Bai, and
  S.~Chintala, ``Pytorch: An imperative style, high-performance deep learning
  library,'' in \emph{Advances in Neural Information Processing Systems 32},
  2019, pp. 8024--8035.

\bibitem{pyg}
M.~Fey and J.~E. Lenssen, ``Fast graph representation learning with {PyTorch
  Geometric},'' in \emph{ICLR Workshop on Representation Learning on Graphs and
  Manifolds}, 2019.

\bibitem{dropout}
\BIBentryALTinterwordspacing
N.~Srivastava, G.~Hinton, A.~Krizhevsky, I.~Sutskever, and R.~Salakhutdinov,
  ``Dropout: A simple way to prevent neural networks from overfitting,''
  \emph{Journal of Machine Learning Research}, vol.~15, no.~56, pp. 1929--1958,
  2014. [Online]. Available:
  \url{http://jmlr.org/papers/v15/srivastava14a.html}
\BIBentrySTDinterwordspacing

\bibitem{deq_gcn}
\BIBentryALTinterwordspacing
G.~Li, M.~M{\"u}ller, B.~Ghanem, and V.~Koltun, ``Training graph neural
  networks with 1000 layers,'' in \emph{Proceedings of the 38th International
  Conference on Machine Learning}, ser. Proceedings of Machine Learning
  Research, M.~Meila and T.~Zhang, Eds., vol. 139.\hskip 1em plus 0.5em minus
  0.4em\relax PMLR, 18--24 Jul 2021, pp. 6437--6449. [Online]. Available:
  \url{https://proceedings.mlr.press/v139/li21o.html}
\BIBentrySTDinterwordspacing

\end{thebibliography}


%

\ifCLASSOPTIONcompsoc
  \section*{\modifyok{}{Acknowledgments}}
\else
  \section*{Acknowledgment}
\fi

\modifyok{}{The authors would like to thank all the anonymous reviewers for their insightful comments. This work was supported in part by National Nature Science Foundations of China grants U19B2026, U19B2044, 61836011, 62021001, and 61836006, and the Fundamental Research Funds for the Central Universities grant WK3490000004.}

\ifCLASSOPTIONcaptionsoff
  \newpage
\fi

\begin{IEEEbiography}[{\includegraphics[width=1in,height=1.25in,clip,keepaspectratio]{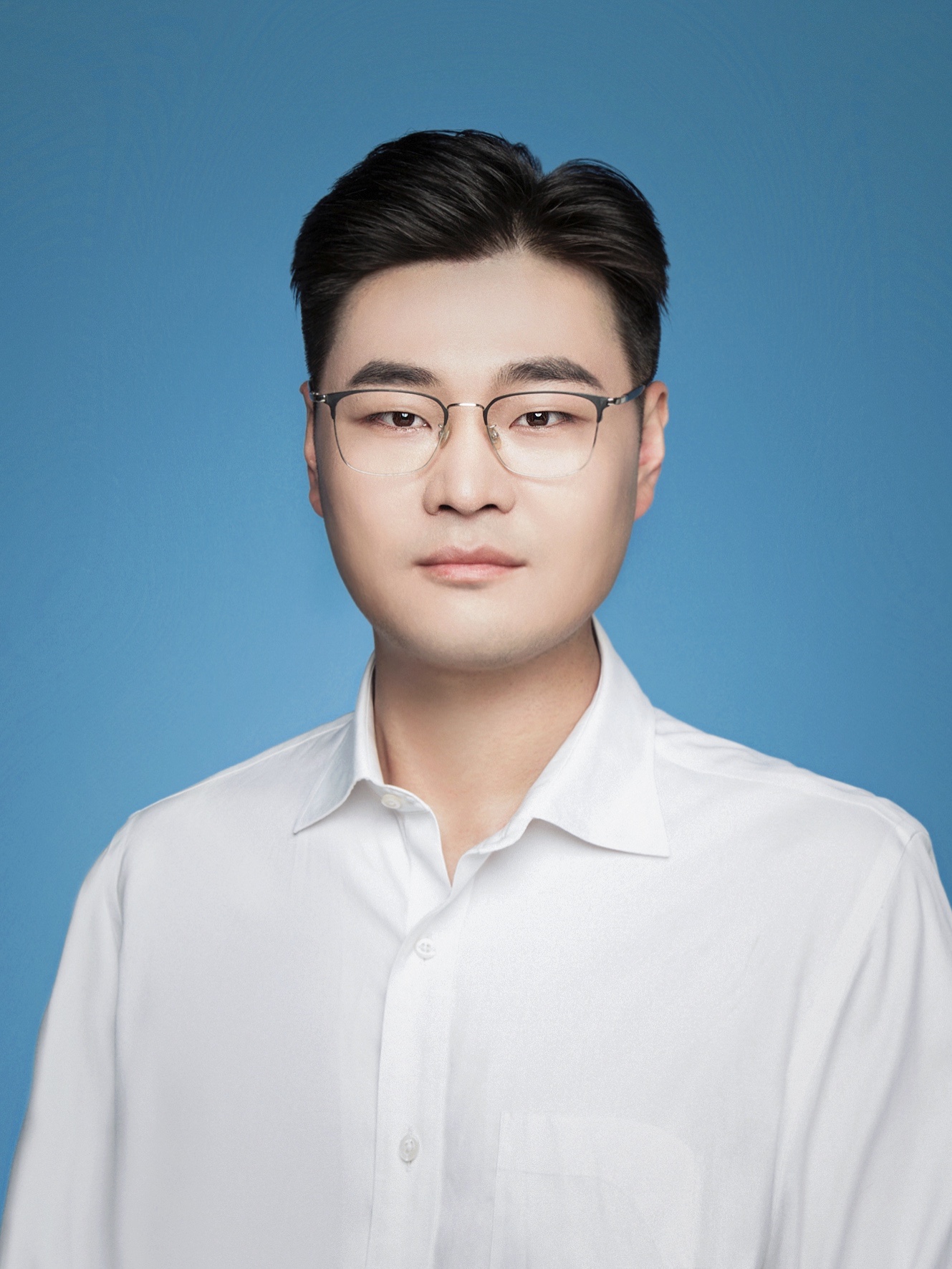}}]{Jie Wang}
  received the B.Sc. degree in electronic information science and technology from University of Science and Technology of China, Hefei, China, in 2005, and the Ph.D. degree in computational science from the Florida \mbox{State} University, Tallahassee, FL, in 2011. He is currently a professor in the Department of Electronic Engineering and Information Science at University of Science and Technology of China, Hefei, China. His research interests include reinforcement learning, knowledge graph, large-scale optimization, deep learning, etc.  He is a senior member of IEEE.
\end{IEEEbiography}

\begin{IEEEbiography}[{\includegraphics[width=1in,height=1.25in,clip,keepaspectratio]{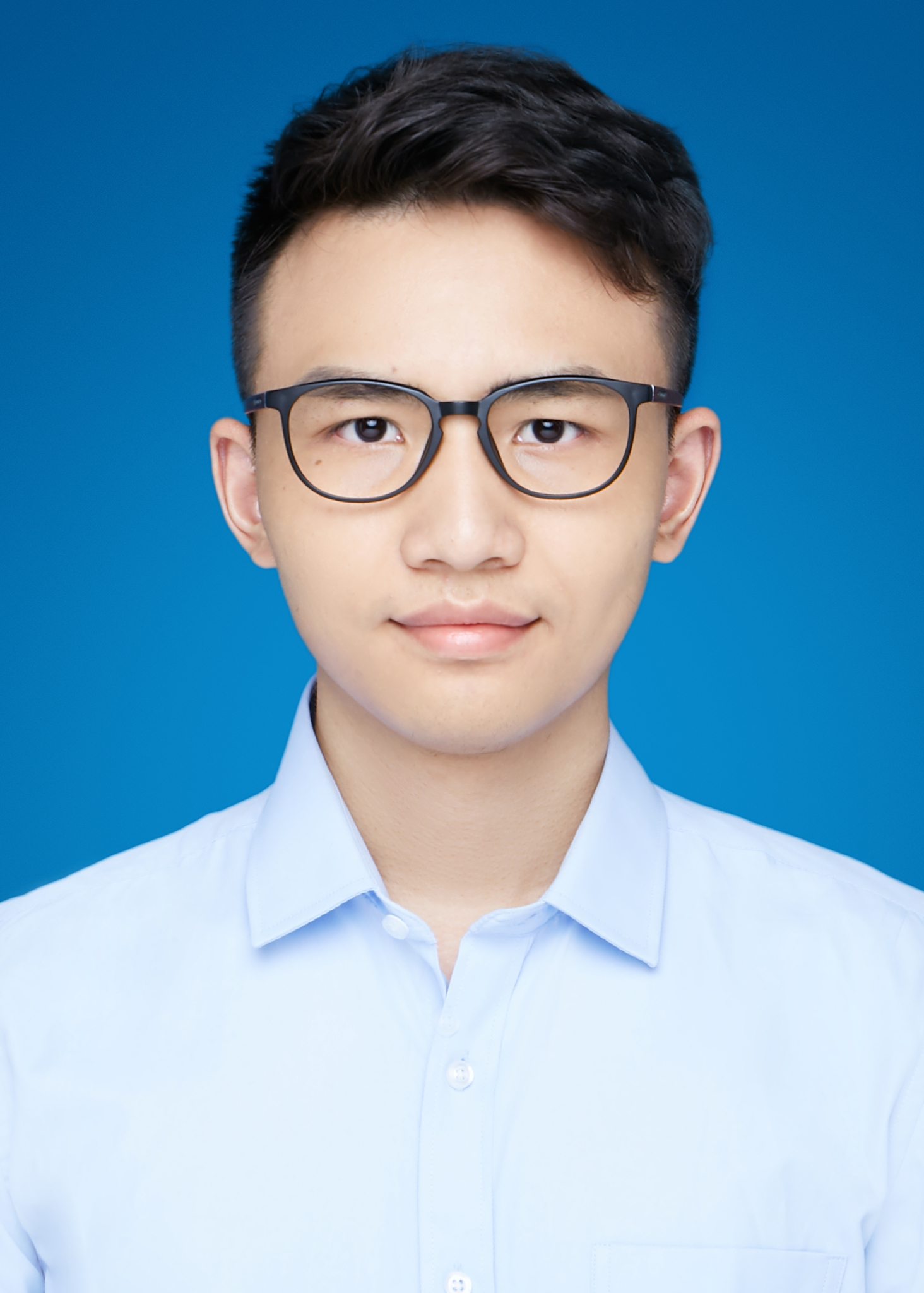}}]{Zhihao Shi}
  received the B.Sc. degree in Department of Electronic Engineering and Information Science from University of Science and Technology of China, Hefei, China, in 2020. a Ph.D. candidate in the Department of Electronic Engineering and Information Science at University of Science and Technology of China, Hefei, China. His research interests include graph representation learning and natural language processing.
\end{IEEEbiography}

\begin{IEEEbiography}[{\includegraphics[width=1in,height=1.25in,clip,keepaspectratio]{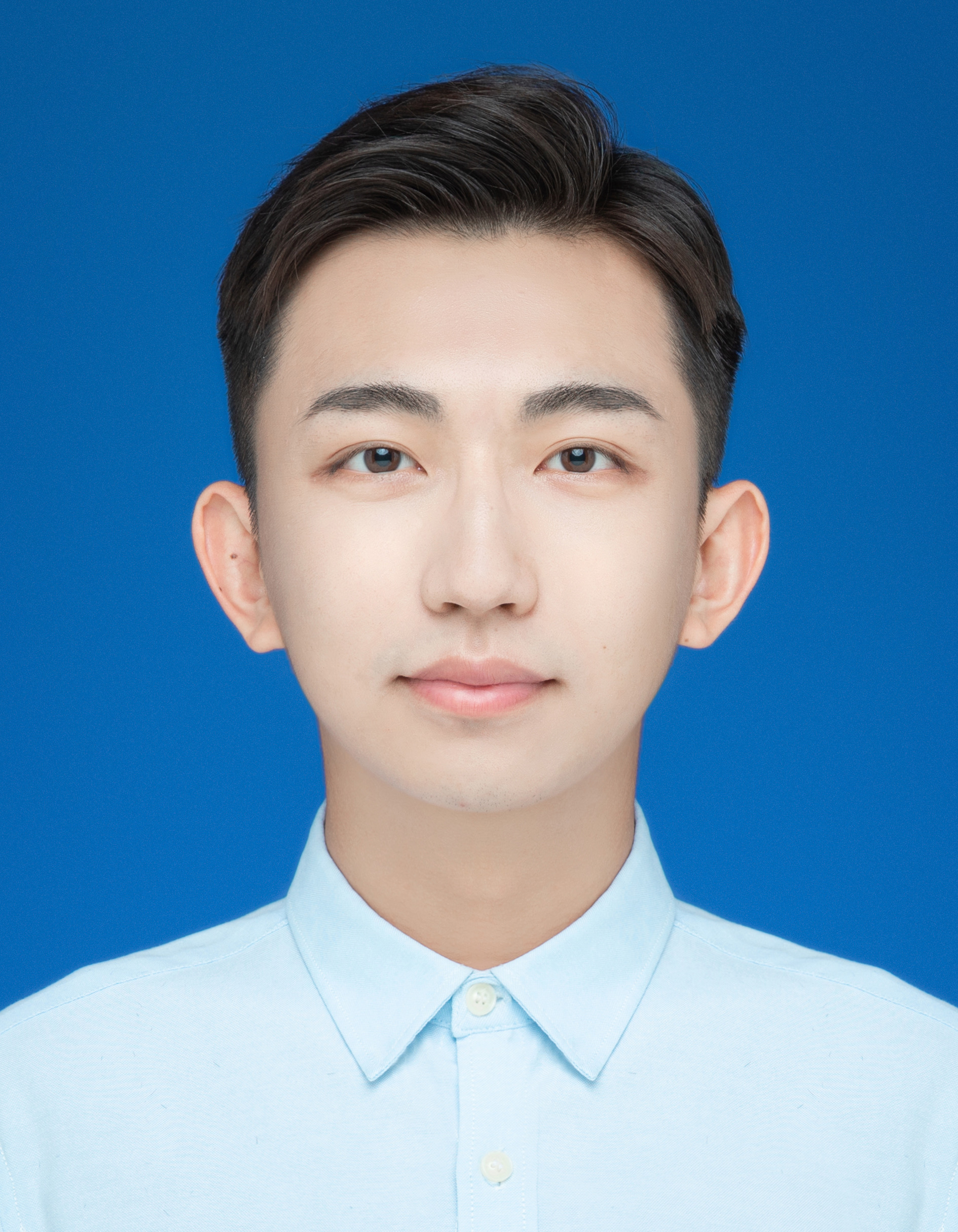}}]{Xize Liang}
  received the B.Sc degree in information and computing sciences from the University of Science and Technology of China, Hefei, China, in 2022. He is currently a graduate student in the Department of Electronic Engineering and Information Science at the University of Science and Technology of China. His research interests include graph representation learning and AI for Science.
\end{IEEEbiography}

\begin{IEEEbiography}[{\includegraphics[width=1in,height=1.25in,clip,keepaspectratio]{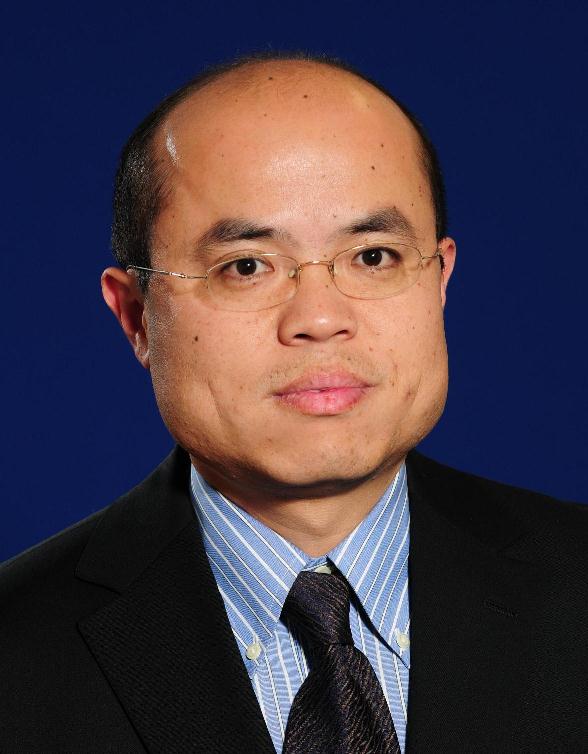}}]{Shuiwang Ji}
  received the PhD degree in computer science from Arizona State University, Tempe, Arizona, in 2010. Currently, he is an Associate Professor in the Department of Computer Science and Engineering, Texas A\&M University, College Station, Texas. His research interests include machine learning, deep learning, data mining, and computational biology. He received the National Science Foundation CAREER Award in 2014. He is currently an Associate Editor for IEEE Transactions on Pattern Analysis and Machine Intelligence, ACM Transactions on Knowledge Discovery from Data, and ACM Computing Surveys. He regularly serves as an Area Chair or equivalent roles for data mining and machine learning conferences, including AAAI, ICLR, ICML, IJCAI, KDD, and NeurIPS. He is a Fellow of IEEE.
\end{IEEEbiography}

\begin{IEEEbiography}[{\includegraphics[width=1in,height=1.5in,clip,keepaspectratio]{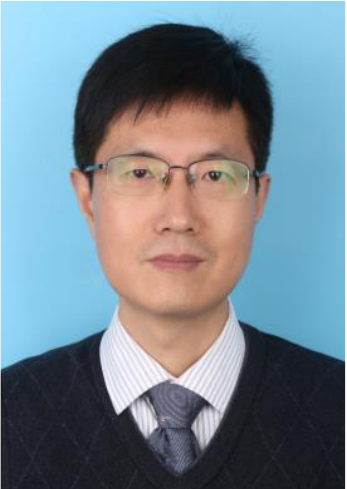}}]{Bin Li}
received the B.Sc. degree in electrical engineering from Hefei University of Technology, Hefei, China, in 1992, the M.Sc. degree from the Institute of Plasma Physics, Chinese Academy of Sciences, Hefei, in 1995, and the Ph.D. degree in Electronic Science and Technology from the University of Science and Technology of China (USTC), Hefei, in 2001. He is currently a Professor at the School of Information Science and Technology, USTC. He has authored or co-authored over 60 refereed publications. His current research interests include evolutionary computation, pattern recognition, and human-computer interaction. Dr. Li is the Founding Chair of IEEE Computational Intelligence Society Hefei Chapter, a Counselor of IEEE USTC Student Branch, a Senior Member of Chinese Institute of Electronics (CIE), and a member of Technical Committee of the Electronic Circuits and Systems Section of CIE. He is a Member of IEEE.
\end{IEEEbiography}

\begin{IEEEbiography}[{\includegraphics[width=1in,height=1.25in,clip,keepaspectratio]{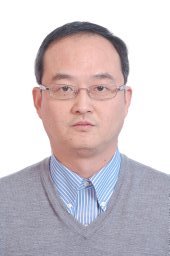}}]{Feng Wu}
received the B.S. degree in electrical engineering from Xidian University in 1992, and the M.S. and Ph.D. degrees in computer science from the Harbin Institute of Technology in 1996 and 1999, respectively. He is currently a Professor with the University of Science and Technology of China, where he is also the Dean of the School of Information Science and Technology. Before that, he was a Principal Researcher and the Research Manager with Microsoft Research Asia. His research interests include image and video compression, media communication, and media analysis and synthesis. He has authored or coauthored over 200 high quality articles (including several dozens of IEEE Transaction papers) and top conference papers on MOBICOM, SIGIR, CVPR, and ACM MM. He has 77 granted U.S. patents. His 15 techniques have been adopted into international video coding standards. As a coauthor, he received the Best Paper Award at 2009 IEEE Transactions on Circuits and Systems for Video Technology, PCM 2008, and SPIE VCIP 2007. He also received the Best Associate Editor Award from IEEE Circuits and Systems Society in 2012. He also serves as the TPC Chair for MMSP 2011, VCIP 2010, and PCM 2009, and the Special Sessions Chair for ICME 2010 and ISCAS 2013. He serves as an Associate Editor for IEEE Transactions on Circuits and Systems for Video Technology, IEEE Transactions ON Multimedia, and several other international journals.
\end{IEEEbiography}




\clearpage
\appendices

\section{More Details about Experiments} \label{appendix:details_experiments}
In this section, we introduce more details about our experiments, including datasets, training and evaluation protocols, and implementations.

\subsection{Datasets} \label{sec:dataset}
We evaluate LMC on Cora, Citeseer, PubMed \cite{planetoid}, PPI, REDDIT, FLICKR \cite{graphsage}, AMAZON \cite{amazon}, Ogbn-arxiv \cite{ogb}, and Ogbn-products \cite{ogb}.

All of the datasets do not contain personally identifiable information or offensive content.
Table \ref{tab:datasets} shows the summary statistics of the datasets.
Details about the datasets are as follows.
\begin{itemize}[leftmargin=5mm]
    \item Cora, Citeseer, and PubMed are directed citation networks. Each node indicates a paper with the corresponding bag-of-words features and each directed edge indicates that one paper cites another one. The task is to classify academic papers into different subjects.
    \item PPI contains 24 protein-protein interaction graphs. Each graph corresponds to a human tissue. Each node indicates a protein with positional gene sets, motif gene sets and immunological signatures as node features. Edges represent interactions between proteins. The task is to classify protein functions.
    (proteins) and edges (interactions).
    \item REDDIT is a post-to-post graph constructed from REDDIT. Each node indicates a post and each edge between posts indicates that the same user comments on both. The task is to classify REDDIT posts into different communities based on (1) the GloVe CommonCrawl word vectors \cite{glove} of the post titles and comments, (2) the post’s scores, and (3) the number of comments made on the posts.
    \item AMAZON is an Amazon product co-purchasing network. Each node indicates a product and each edge between two products indicates that the products are purchased together. The task is to predict product types without node features based on rare labeled nodes. We set the training set fraction be 0.06\% in experiments.
    \item Ogbn-arxiv is a directed citation network between all Computer Science (CS) arXiv papers indexed by MAG \cite{mag}. Each node is an arXiv paper and each directed edge indicates that one paper cites another one. The task is to classify unlabeled arXiv papers into different primary categories based on labeled papers and node features, which are computed by averaging word2vec \cite{word2vec} embeddings of words in papers' title and abstract.
    \item FLICKR categorizes types of images based on their descriptions and properties \cite{gas, graphsaint}.
    \item Ogbn-product is a large Amazon product co-purchasing network with rich node features. Each node indicates a product and each edge between two products indicates that the products are purchased together. The task is to predict product types based on low-dimensional bag-of-words features of product descriptions processed by Principal Component Analysis.
\end{itemize}

\begin{table}[htbp]
  \begin{center}
    \caption{Statistics of the datasets used in our experiments.
    }\label{tab:datasets}
    \vspace{5pt}
    \scalebox{0.9}{
    \begin{tabular}{ccccc}
    \toprule
    \textbf{Dataset} & \textbf{\#Graphs} & \textbf{\#Classes} &\textbf{Total \#Nodes} & \textbf{Total \#Edges}  \\
      \midrule
      \midrule
      Cora & 1 & 7  & 2,708 & 5,278\\
      Citeseer & 1 & 6  & 3,327 & 4,552 \\
      PubMed & 1 & 3  & 19,717 & 44,324 \\
      \midrule
      PPI & 24 & 121 & 56,944 & 793,632 \\
      REDDIT & 1 & 41 & 232,965 & 11,606,919   \\
      AMAZON & 1 & 58 & 334,863 & 925,872   \\
      Ogbn-arxiv & 1 & 40  & 169,343 & 1,157,799  \\
      FLICKR & 1 & 7  & 89,250 & 449,878  \\
       Ogbn-product & 1 & 47  & 2,449,029 & 61,859,076 \\
      \bottomrule
    \end{tabular}
    }
  \end{center}
\end{table}

\subsection{Training and Evaluation Protocols}\label{sec:training_evaluation}

We run all the experiments on a single GeForce RTX 2080 Ti (11 GB). All the models are implemented in Pytorch \cite{pytorch} and PyTorch Geometric \cite{pyg} based on the official implementation of \cite{gas}\footnote{\url{https://github.com/rusty1s/pyg_autoscale}. The owner does not mention the license.}. We use the data splitting strategies following previous works \cite{gas, ignn}.

\subsection{Implementation Details and Hyperparameters}\label{sec:implementation}

\subsubsection{Normalization Technique} \label{sec:normalization}

In Section \ref{sec:naive_sgd} in the main text, we assume that the subgraph $\inbatch$ is uniformly sampled from $\mathcal{V}$ and the corresponding set of labeled nodes {\small$\mathcal{V}_{L_{\mathcal{B}}} = \inbatch \cap \mathcal{V}_{L}$} is uniformly sampled from {\small$\mathcal{V}_{L}$}. To enforce the assumption, we use the normalization technique to reweight Eqs. (5) and (6) in the main text.

Suppose we partition the whole graph $\mathcal{V}$ into $b$ parts $\{\mathcal{V}_{\mathcal{B}_i}\}_{i=1}^b$ and then uniformly sample $c$ clusters without replacement to construct subgraph $\inbatch$. By the normalization technique, Eq. (5) becomes
\begin{align}
    \mathbf{g}_w(\inbatch) = \frac{b|\mathcal{V}_{L_{\mathcal{B}}}|}{c|\mathcal{V}_L|} \frac{1}{|\mathcal{V}_{L_{\mathcal{B}}}|} \sum_{v_j \in \mathcal{V}_{L_{\mathcal{B}}}}\nabla_{w} \ell_w(\embh_j ,y_j),
\end{align}
where $\frac{b|\mathcal{V}_{L_{\mathcal{B}}}|}{c|\mathcal{V}_L|}$ is the corresponding weight. Similarly, Eq. (6) becomes
\begin{align}
    \mathbf{g}_{\theta}(\inbatch) = \frac{b|\inbatch|}{c|\mathcal{V}|} \frac{|\mathcal{V}|}{|\inbatch|} \sum_{v_j \in \inbatch} \nabla_{\theta} \update(\embh_j ,\embm_{\neighbor{v_j}}  ,\embx_j) \embV_{j} \label{eqn:sgd_theta},
\end{align}
where $\frac{b|\inbatch|}{c|\mathcal{V}|}$ is the corresponding weight.

\begin{table*}[ht]
  \centering
  \caption{%
  \textbf{The hyperprameters used in the experiments for LMC4Rec.}
  }\label{tab:hyperparameters}
  \setlength{\tabcolsep}{5pt}
  \resizebox{1.0\linewidth}{!}{%
  \begin{tabular}{lcccccc}
    \toprule
    \textbf{Dataset}  & \textbf{Dimensions} & \textbf{Partitions} & \textbf{Clusters} & \textbf{LR (METIS partition)} & \textbf{LR (random partition)} \\
    \midrule
    \textsc{Cora} & 128 & 10 & 2 & 0.003 & 0.001\\
    \textsc{CiteSeer} & 128 & 10 & 2 & 0.01 & 0.001 \\
    \textsc{PubMed} & 128 & 10 & 2 & 0.003 & 0.0003 \\
    \textsc{PPI} & 1024 & 10 & 2 & 0.01 & - \\
    \textsc{REDDIT} & 256 & 200 & 100 & 0.003 & - \\
    \textsc{AMAZON} & 64 & 40 & 1 & 0.003 & - \\
    \textsc{Ogbn-arxiv} & 256 & 80 & 10 & 0.003 & - \\
    \textsc{Ogbn-products} & 256 & 1500 & 10 & 0.001 & - \\
    \bottomrule
  \end{tabular}
  }
\end{table*}

\subsubsection{Implement of Baslines}\label{sec:implement_baseline}

For a fair comparison, we implement GAS by setting the gradient compensation $\mathbf{C}_b$ to be zero. We implement Cluster-GCN by removing edges between partitioned subgraphs and then running GAS based on them.

\subsubsection{Other Implementation Details for ConvGNN}

\udfsection{Incorporating Batch Normalization for ConvGNNs.} We uniformly sample a mini-batch of nodes $\inbatch$ and generate the induced subgraph of $\neighbor{\inbatch}$.
If we directly feed the $\embH_{\neighbor{\inbatch}}^{(l)}$ to a batch normalization layer, the learned mean and standard deviation of the batch normalization layer may be biased.
Thus, LMC first feeds the embeddings of the mini-batch $\embH_{\inbatch}^{(l)}$ to a batch normalization layer and then feeds the embeddings outside the mini-batch $\embH_{\neighbor{\inbatch}\backslash\inbatch}^{(l)}$ to another batch normalization layer.

\udfsection{Selection of $\beta_{i}$.} \label{sec:selection_beta} We select $\beta_i = score(i) \alpha $ for each node $v_i$, where $\alpha \in [0,1]$ is a hyperparameter and $score$ is a function to measure the quality of the incomplete up-to-date messages.
We search $score$ in a $\{f(x)=x^2,f(x)=2x-x^2,f(x)=x,f(x)=1;x= deg_{local}(i)/deg_{global}(i)\}$, where $deg_{global}(i)$ is the degree of node $i$ in the whole graph and $deg_{local}(i)$ is the degree of node $i$ in the subgraph induced by $\neighbor{\inbatch}$.

\subsubsection{Hyperparameters for RecGNNs} \label{sec:hyperparameters_recgcn}

We report the embedding dimensions, the number of partitions, sampled clusters per mini-batch, and learning rates (LR) for each dataset in Table \ref{tab:hyperparameters}.
For the small datasets, we partition each graph into ten subgraphs. For the large datasets, we select the number of partitions and sampled clusters per mini-batch to avoid running out of memory on GPUs.
For a fair comparison, we use the same hyperparameters except for the learning rate. We search the best learning rate in $\{0.01, 0.003, 0.001\}$ for each training method on the large datasets and use the same learning rate on the small datasets.

\subsubsection{Other Implementation Details for RecGNNs} 
\label{sec:recgcn}

\udfsection{An Efficient Implement of LMC for RecGCN.} The equilibrium equations of RecGCN (see Equation \eqref{eqn:transformation_rec} in the main text) are
\begin{align*}
    \embH  = \sigma(  \mathbf{W} \embH   \hat{\mathbf{A}} + \mathbf{b}(\embX) ),
\end{align*}
where the nonlinear activation function $\sigma(\cdot)$ is the ReLU activation $\sigma(\cdot)=\max(\cdot,0)$, the function $\mathbf{b}(\embX) = \mathbf{P}\embX +\mathbf{c}$ is an affine function with parameters $\mathbf{P} \in \mathbb{R}^{d \times d_{x}},\mathbf{c} \in \mathbb{R}^{d}$
, the matrix $\mathbf{\hat{A}} = (\mathbf{D}+\mathbf{I})^{-1/2} (\mathbf{A}+\mathbf{I})(\mathbf{D}+\mathbf{I})^{-1/2}$ is the normalized adjacency matrix with self-loops, and $\mathbf{D}$ is the degree matrix. Let $\mathbf{Z} = \mathbf{W} \embH   \hat{\mathbf{A}} + \mathbf{b}(\embX)$.
The Jacobian vector-product is $\langle \vec{\embV}, \nabla_{\embH} \vec{f}_{\theta} \rangle$ \cite{ignn} is
\begin{align*}
     \mathbf{W}^{\top} (\sigma'(\mathbf{Z}) \odot \embV) \hat{\mathbf{A}}^{\top}.
\end{align*}

For two nodes $v_i,v_j$ such that $v_i \in \mathcal{N}(v_j)$, we have
\begin{align*}
    \vec{\embV}_i^{\top} \frac{\partial [f_{\theta}]_i}{\partial \embh_{j}} = \mathbf{W}^{\top} \left(\sigma({\mathbf{Z}}_i)' \odot {\embV}_i\right) \mathbf{\hat{A}}_{ij},
\end{align*}
which is only depend on the nodes $v_i,v_j$ rather than the 2-hop neighbors of $v_j$. Therefore, by additionally storing the historical auxiliary variables $\mathbf{Z}$, we implement LMC based on 1-hop neighbors, which employs $\mathcal{O}(n_{\max} |\inbatch| d)$ GPU memory in backward passes.

\udfsection{Randomly Using Dropout at Different Training Steps for RecGNNs.} We propose a trick to handle the randomness introduced by dropout \cite{dropout}. Let $\dropout(\embX) = \frac{1}{1-p} \mathbf{M} \circ \embX$ be the dropout operation, where $\mathbf{M}_{ij} \sim \text{Bern}(1-p)$ are i.i.d Bernoulli random variables, and $\circ$ is the element-wise product. As shown by \cite{vrgcn}, with dropout \cite{dropout} of input features $\mathbf{X}$,  historical embeddings and auxiliary variables become random variables, leading to inaccurate compensation messages.
Specifically, the solutions to Equations (1) and (3) in the main text are the function of the dropout features $\frac{1}{1-p} \mathbf{M}^{(k)} \circ \embX$ at the training step $k$, while the dropout features become $\frac{1}{1-p} \mathbf{M}^{(k+1)} \circ \embX$ at the training step $k+1$ (we assume that the learning rate $\eta=0$ to simplify the analysis). As the historical information under the dropout operation at the training step $k$ may be very inaccurate at the training step $k+1$, \cite{vrgcn} propose to first compute the random embeddings with dropout and the mean embeddings without dropout, and then use the mean embeddings to update historical information.
However, simultaneously computing two versions of embeddings in RecGNNs leads to double computational costs of solving Equations (1) and (3).
We thus propose to randomly use the dropout operation at different training steps and update the historical information when the dropout operation is invalid. Specifically, at each training step $k$, we either update the historical information without the dropout operation with probability $q$ or use the dropout operation without updating the historical information. We set $q=0.5$ in all experiments. Due to the trick, the historical embeddings and auxiliary variables depend on the stable input features $\embX$ rather than the random dropout features $\frac{1}{1-p} \mathbf{M} \circ \embX$. The trick can reduce overfitting and control variate for dropout efficiently.

\section{Well-posedness Conditions of RecGNNs}\label{sec:well-posedness}

In this section, we provide tractable well-posedness conditions \cite{ignn} of RecGNNs to ensure the existence and uniqueness of the solution to Eq. (1) in the main text.

The Eq. (1)  in the main text of RecGNNs with the message passing functions in GCN can be formulated as
\begin{align*}
    \embH   = f_{\theta}(\embH  ;\embX) = \sigma(  \mathbf{W} \embH   \hat{\mathbf{A}} + \mathbf{b}(\embX) ),
\end{align*}
where $\mathbf{\hat{A}}$ is the aggregation matrix,  $\sigma(\cdot)$ is the nonlinear activation function, and $\mathbf{b}(\embX) $ is an affine function to encode the input features.
The Perron-Frobenius (PF) sufficient condition for well-posedness \cite{idl} requires the activation function $\sigma$ is component-wise non-expansive and $\lambda_{pf}(|\hat{\mathbf{A}}^{\top} \otimes \mathbf{W}|) = \lambda_{pf}(\hat{\mathbf{A}})\lambda_{pf}(|\mathbf{W}|) < 1$
, where $\otimes$ is the Kronecker product.
As pointed out in \cite{ignn}, if the Perron-Frobenius (PF) sufficient condition holds, the solution $\embH  $ can be achieved by iterating Eq. (1) in the main text to convergence.

We further discuss how to enforce the non-convex constraint $\lambda_{pf}(|\mathbf{W}|) < (\lambda_{pf}(\hat{\mathbf{A}}))^{-1}$.
As $\lambda_{pf}(|\mathbf{W}|) \leq \| \mathbf{W} \|_p$ holds for induced norms $\| \cdot \|_p$, we enforce the stricter condition $\| \mathbf{W} \|_p < \lambda_{pf}(\hat{\mathbf{A}})^{-1}$. As pointed out in \cite{ignn}, if $p=1$ or $p=\infty$, we efficiently implement $\| \mathbf{W} \|_p < 1$ by projection. We follow the implementation in experiments and view $p$ as a hyperparameter.

\section{Computational Complexity} \label{appendix:complexity}

We summarize the computational complexity in Tables \ref{tab:complexity_conv} and \ref{tab:complexity_rec}, where $n_{\max}$ is the maximum of neighborhoods, $L$ is the number of message passing layers/iterations, $\mathcal{V}_{\mathcal{B}}$ is a set of nodes in a sampled mini-batch, $d$ is the embedding dimension, $\mathcal{V}$ is the set of nodes in the whole graph, and $\mathcal{E}$ is the set of edges in the whole graph.
As GD, backward SGD, Cluster-GCN, GAS, and LMC share the same memory complexity of parameters $\theta^{l}$ and $\theta^{\Diamond}$, we omit them in Tables \ref{tab:complexity_conv} and \ref{tab:complexity_rec}.

\begin{table}[htbp]
    \centering
    \caption{
    Time and memory complexity per gradient update for ConvGNNs (e.g. GCN \cite{gcn} and GCNII \cite{gcnii}).
    }
    \label{tab:complexity_conv}
    \scalebox{0.9}{
    \begin{tabular}{ccc}
    \toprule
        \textbf{Method} & \textbf{Time}  &\textbf{Memory}   \\
        \midrule
        GD and backward SGD  & $\mathcal{O}(L(|\mathcal{E}|d+|\mathcal{V}| d^2))$ & $\mathcal{O}(L|\mathcal{V}| d)$ \\
        Cluster-GCN \cite{cluster_gcn}  & $\mathcal{O}( L(n_{\max}|\inbatch|d+|\inbatch| d^2) )$ & $\mathcal{O}(L|\inbatch| d)$ \\
        GAS \cite{gas} & $\mathcal{O}( L(n_{\max}|\inbatch|d+|\inbatch| d^2) )$ & $\mathcal{O}( n_{\max} L|\inbatch| d)$ \\
        \midrule
        LMC4Conv & $\mathcal{O}( L(n_{\max}|\inbatch|d+|\inbatch| d^2) )$ & $\mathcal{O}(n_{\max} L|\inbatch| d)$ \\
    \bottomrule
    \end{tabular}
    }
\end{table}

\begin{table}[htbp]
    \centering
    \caption{
    Time and memory complexity for RecGCN.
    }
    \label{tab:complexity_rec}
    \scalebox{0.9}{
    \begin{tabular}{ccc}
    \toprule
        \textbf{Method} & \textbf{Time}  &\textbf{Memory}   \\
        \midrule
        GD and backward SGD  & $\mathcal{O}(L(|\mathcal{E}|d+|\mathcal{V}| d^2))$ & $\mathcal{O}(|\mathcal{V}| d)$ \\
        Cluster-GCN \cite{cluster_gcn}  & $\mathcal{O}( L(n_{\max}|\inbatch|d+|\inbatch| d^2) )$ & $\mathcal{O}(|\inbatch| d)$ \\
        GAS \cite{gas} & $\mathcal{O}( L(n_{\max}|\inbatch|d+|\inbatch| d^2) )$ & $\mathcal{O}( n_{\max} |\inbatch| d)$ \\
        \midrule
        LMC4Rec & $\mathcal{O}( L(n_{\max}|\inbatch|d+|\inbatch| d^2) )$ & $\mathcal{O}(n_{\max} |\inbatch| d)$ \\
    \bottomrule
    \end{tabular}
    }
\end{table}

\section{Additional Related Work}\label{appendix:add_related}

\subsection{Graph Neural Networks}\label{sec:related_gnn}

Graph neural networks (GNNs) aim to learn node embeddings by iteratively aggregating features and structure information of neighborhoods. Most graph neural networks for node property prediction on static graphs are categorized into convolutional graph neural networks (ConvGNNs) and recurrent graph neural networks (RecGNNs) \cite{comprehensive}. 

\subsubsection{Convolutional Graph Neural Networks}

\modifyok{}{ConvGNNs \cite{gcn, gat, gcnii} use $T$ different graph convolutional layers to learn node embeddings, where $T\in\mathbb{N}^*$ is a hyperparameter. 
Many ConvGNNs focus on the design of the message passing layer, i.e., aggregation and update operations.
For example, GCN \cite{gcn} proposes to aggregate the neighbor information by normalized averaging and GAT \cite{gat} introduces the attention mechanism into aggregation.
However, these models achieve the best performance with shallow architectures due to the \textit{over-smoothing} issues, i.e., the node embeddings of ConvGNNs may tend to be indistinguishable as $T$ increases.
To alleviate over-smoothing, GCNII \cite{gcnii} proposes the initial residual and identity mapping to develop deep GNNs.
}

\subsubsection{Recurrent Graph Neural Networks}

Inspired by \cite{idl, deq}, many researchers focus on RecGNNs recently, which approximate infinite MP layers using a shared MP layer until convergence.
Many works \cite{recgnn1, contraction, fdgnn, ignn} demonstrate that RecGNNs can effectively capture long-range dependencies. However, designing robust and scalable training methods for RecGNNs is challenging, limiting the real-world applications of RecGNNs. To improve the robustness, implicit graph neural networks (IGNNs) establish tractable well-posedness conditions and use projected gradient descent to guarantee well-posedness \cite{ignn}. Our proposed LMC4Rec focuses on the scalable training of RecGNNs, orthogonal to IGNNs.
In Appendix \ref{sec:sse}, we also discuss stochastic steady-state embedding (SSE), a scalable training algorithm for RecGNNs.

\subsection{Main differences between LMC and GraphFM}\label{sec:diff_graphfm}
First, LMC focuses on the convergence of subgraph-wise sampling methods, which is orthogonal to the idea of GraphFM-OB to alleviate the staleness problem of historical values. The advanced approach to alleviating the staleness problem of historical values can further improve the performance of LMC and is easy to establish provable convergence by the extension of LMC.

Second, LMC uses nodes in both mini-batches and their $1$-hop neighbors to compute incomplete up-to-date messages. In contrast, GraphFM-OB only uses nodes in the mini-batches. For the nodes whose neighbors are contained in the union of the nodes in mini-batches and their $1$-hop neighbors, the aggregation results of LMC are exact, while those of GraphFM-OB are not.

Third, by noticing that aggregation results are biased and the I/O bottleneck for the history access, LMC does not update the historical values in the storage outside the mini-batches. However, GraphFM-OB updates them based on the aggregation results.

\subsection{Stochastic steady-state embedding for RecGNNs}\label{sec:sse}
Stochastic Steady-state Embedding (SSE) \cite{sse} proposes a scalable training algorithm for RecGNNs, which uses a sampling method, namely stochastic fixed-point iteration, to reach global equilibrium points at each forward pass. The main differences between our proposed LMC and SSE are as follows. First, SSE performs message passing once in a backward pass, while LMC performs message passing many times until the iteration converges to a stable solution. Thus, the gradients computed by LMC are more accurate than SSE. Second, we show that LMC converges to first-order stationary points of RecGNNs, while SSE does not provide convergence analysis.

\section{Detailed proofs} \label{appendix:proof}

\subsection{Proof of Theorem \ref{thm:unbiased}: unbiased mini-batch gradients of backward SGD} \label{appendix:proof_thm1}

In this section, we give the proof of Theorem \ref{thm:unbiased}, which shows that the mini-batch gradients computed by backward SGD are unbiased.

\begin{proof}
    As $\mathcal{V}_{L_{\mathcal{B}}} = \inbatch \cap \mathcal{V}_{L}$ is uniformly sampled from $\mathcal{V}_{L}$, the expectation of $ \mathbf{g}_w(\inbatch)$ is
    \begin{align*}
        \mathbb{E}[\mathbf{g}_w(\inbatch)] &=\mathbb{E}[ \frac{1}{|\mathcal{V}_{L_{\mathcal{B}}}|} \sum_{v_j \in \mathcal{V}_{L_{\mathcal{B}}}}\nabla_{w} \ell_w(\embh_j ,y_j)]\\
        &= \nabla_{w} \mathbb{E}[ \ell_w(\embh_j ,y_j)]\\
        &= \nabla_{w} \loss.
     \end{align*}
     As the subgraph $\inbatch$ is uniformly sampled from $\mathcal{V}$, the expectation of $\mathbf{g}_{\theta^{l}}(\inbatch)$ is
    \begin{align*}
        \mathbb{E}[\mathbf{g}_{\theta^{l}}(\inbatch)] &= \mathbb{E}[\frac{|\mathcal{V}|}{|\mathcal{V}_\mathcal{B}|}\sum_{v_j\in\mathcal{V}_\mathcal{B}}\left(\nabla_{\theta^{l}}u_{\theta^l}(\embh^{l-1}_j, \embm^{l-1}_{\neighbor{v_j}}, \embx_j)\right) \embV_{j}^{l}] \\
        &= |\mathcal{V}| \mathbb{E}[ \nabla_{\theta^{l}}u_{\theta^l}(\embh^{l-1}_j, \embm^{l-1}_{\neighbor{v_j}}, \embx_j)  \embV_{j}^{l} ]\\
        &=|\mathcal{V}| \frac{1}{|\mathcal{V}|} \sum_{v_j \in \mathcal{V}} \nabla_{\theta^{l}}u_{\theta^l}(\embh^{l-1}_j, \embm^{l-1}_{\neighbor{v_j}}, \embx_j) \embV_{j}^{l}\\
        &= \sum_{v_j \in \mathcal{V}} \nabla_{\theta^{l}}u_{\theta^l}(\embh^{l-1}_j, \embm^{l-1}_{\neighbor{v_j}}, \embx_j) \embV_{j}^{l}\\
        &= \nabla_{\theta^{l}} \loss,\,\,\forall\,l\in[L].
    \end{align*}
    Similarly, we can show that $\mathbf{g}_{\theta^{\Diamond}}(\mathcal{V}_{\mathcal{B}})$ is unbiased.
\end{proof}

\subsection{Proofs for LMC4Conv}\label{appendix:proof_conv}

{\bf Notations.} Unless otherwise specified, $C$ and $C'$ with any superscript or subscript denotes constants. We denote the learning rate by $\eta$.

In this section, we suppose that Assumption \ref{assmp:proof} holds.

\subsubsection{Differences between exact values at adjacent iterations}

We first show that the differences between the exact values of the same layer in two adjacent iterations can be bounded by setting a proper learning rate.

\begin{lemma}\label{prop:exact_difference_h}
    Suppose that Assumption \ref{assmp:proof} holds. Given an $L$-layer ConvGNN, for any $\varepsilon>0$, by letting
    \begin{align*}
        \eta\leq\frac{\varepsilon}{(2\gamma)^L G}<\varepsilon,
    \end{align*}
    we have
    \begin{align*}
        \|\embH^{l,k+1}-\embH^{l,k}\|_F < \varepsilon,\,\,\forall\,l\in[L],\,k\in\mathbb{N}^*.
    \end{align*}
\end{lemma}
\begin{proof}
    Since $\eta\leq\frac{\varepsilon}{(2\gamma)^L G}<\frac{\varepsilon}{\gamma(2\gamma)^{L-1}G}$, we have
    \begin{align*}
        \|\embH^{1,k+1} - \embH^{1,k}\|_F ={}& \|f_{\theta^{1,k+1}}(\embX) - f_{\theta^{1,k}}(\embX)\|_F\\
        \leq{}& \gamma \|\theta^{1,k+1} - \theta^{1,k}\|\\
        \leq{}& \gamma \|\widetilde{\mathbf{g}}_{\theta^1}\|\eta\\
        <{}& \frac{\gamma G \varepsilon}{\gamma(2\gamma)^{L-1} G}\\
        ={}& \frac{\varepsilon}{(2\gamma)^{L-1}}.
    \end{align*}
    Then, because $\eta\leq\frac{\varepsilon}{(2\gamma)^L G}<\frac{\varepsilon}{(2\gamma)^{L-1}G}$, we have
    \begin{align*}
        \|\embH^{2,k+1} - \embH^{2,k}\|_F ={}& \|f_{\theta^{2,k+1}}(\embH^{1,k+1}) - f_{\theta^{2,k}}(\embH^{1,k})\|_F\\
        \leq{}& \|f_{\theta^{2,k+1}}(\embH^{1,k+1}) - f_{\theta^{2,k}}(\embH^{1,k+1})\|_F \\ 
        &+\|f_{\theta^{2,k}}(\embH^{1,k+1}) - f_{\theta^{2,k}}(\embH^{1,k})\|_F\\
        \leq{}& \gamma \|\theta^{2,k+1} - \theta^{2,k}\|\\
        &+ \gamma \|\embH^{1,k+1} - \embH^{1,k}\|_F\\
        \leq{}& \gamma G \eta + \frac{\varepsilon}{2(2\gamma)^{L-2}}\\
        <{}& \frac{\varepsilon}{2(2\gamma)^{L-2}}+\frac{\varepsilon}{2(2\gamma)^{L-2}}\\
        ={}& \frac{\varepsilon}{(2\gamma)^{L-2}}.
    \end{align*}
    And so on, we have
    \begin{align*}
        \|\embH^{l,k+1} - \embH^{l,k}\|_F < \frac{\varepsilon}{(2\gamma)^{L-l}},\,\,\forall\,l\in[L],\,k\in\mathbb{N}^*.
    \end{align*}
    Since $(2\gamma)^{L-l}>1$, we have
    \begin{align*}
        \|\embH^{l,k+1} - \embH^{l,k}\|_F < \varepsilon,\,\,\forall\,l\in[L],\,k\in\mathbb{N}^*.
    \end{align*}
\end{proof}

\begin{lemma}\label{prop:exact_difference_v}
    Suppose that Assumption \ref{assmp:proof} holds. Given an $L$-layer ConvGNN, for any $\varepsilon>0$, by letting
    \begin{align*}
        \eta\leq\frac{\varepsilon}{(2\gamma)^{L-1} G}<\varepsilon,
    \end{align*}
    we have
    \begin{align*}
        \|\embV^{l,k+1}-\embV^{l,k}\|_F < \varepsilon,\,\,\forall\,l\in[L],\,k\in\mathbb{N}^*.
    \end{align*}
\end{lemma}
\begin{proof}
    Since $\eta\leq\frac{\varepsilon}{(2\gamma)^{L-1} G}<\frac{\varepsilon}{\gamma(2\gamma)^{L-2}G}$, we have
    \begin{align*}
        &\|\embV^{L-1,k+1} - \embV^{L-1,k}\|_F\\
        ={}& \|\phi_{\theta^{L,k+1}}(\nabla_{\embH}\loss) - \phi_{\theta^{L,k}}(\nabla_{\embH}\loss)\|_F\\
        \leq{}& \gamma \|\theta^{L,k+1} - \theta^{L,k}\|\\
        \leq{}& \gamma \|\widetilde{\mathbf{g}}_{\theta^L}\|\eta\\
        <{}& \frac{\gamma G \varepsilon}{\gamma(2\gamma)^{L-2} G}\\
        ={}& \frac{\varepsilon}{(2\gamma)^{L-2}}.
    \end{align*}
    Then, because $\eta\leq\frac{\varepsilon}{(2\gamma)^{L-1} G}<\frac{\varepsilon}{(2\gamma)^{L-2}G}$, we have
    \begin{align*}
        &\|\embV^{L-2,k+1} - \embV^{L-2,k}\|_F\\
        ={}& \|\phi_{\theta^{L-1,k+1}}(\embV^{L-1,k+1}) - \phi_{\theta^{L-1,k}}(\embV^{L-1,k})\|_F\\
        \leq{}& \|\phi_{\theta^{L-1,k+1}}(\embV^{L-1,k+1}) - \phi_{\theta^{L-1,k}}(\embV^{L-1,k+1})\|_F\\
        &+ \|\phi_{\theta^{L-1,k}}(\embV^{L-1,k+1}) - \phi_{\theta^{L-1,k}}(\embV^{L-1,k})\|_F\\
        \leq{}& \gamma \|\theta^{L-1,k+1} - \theta^{L-1,k}\| + \gamma \|\embV^{L-1,k+1} - \embV^{L-1,k}\|_F\\
        \leq{}& \gamma G \eta + \frac{\varepsilon}{2(2\gamma)^{L-3}}\\
        <{}& \frac{\varepsilon}{2(2\gamma)^{L-3}}+\frac{\varepsilon}{2(2\gamma)^{L-3}}\\
        ={}& \frac{\varepsilon}{(2\gamma)^{L-3}}.
    \end{align*}
    And so on, we have
    \begin{align*}
        \|\embV^{l,k+1} - \embV^{l,k}\|_F < \frac{\varepsilon}{(2\gamma)^{l-1}},\,\,\forall\,l\in[L],\,k\in\mathbb{N}^*.
    \end{align*}
    Since $(2\gamma)^{l-1}>1$, we have
    \begin{align*}
        \|\embV^{l,k+1} - \embV^{l,k}\|_F < \varepsilon,\,\,\forall\,l\in[L],\,k\in\mathbb{N}^*.
    \end{align*}
\end{proof}

\subsubsection{Historical values and temporary values}

Suppose that we uniformly sample a mini-batch $\mathcal{V}_{\mathcal{B}}^k\subset \mathcal{V}$ at the $k$-th iteration and $|\mathcal{V}_{\mathcal{B}}^k| = S$. For the simplicity of notations, we denote the temporary node embeddings and auxiliary variables in the $l$-th layer by $\temH^{l,k}$ and $\temV^{l,k}$, respectively, where
\begin{align*}
    \temH^{l,k}_i = 
    \begin{cases}
        \temh^{l,k}_i, &v_i\in\neighbor{\mathcal{V}_{\mathcal{B}}^{k}}\setminus \mathcal{V}_{\mathcal{B}}^{k},\\
        \hish^{l,k}_i, &{\rm otherwise,} 
    \end{cases}
\end{align*}
and
\begin{align*}
    \temV^{l,k}_i = 
    \begin{cases}
        \temv^{l,k}_i, &v_i\in\neighbor{\mathcal{V}_{\mathcal{B}}^{k}}\setminus \mathcal{V}_{\mathcal{B}}^{k},\\
        \hisv^{l,k}_i, &{\rm otherwise.} 
    \end{cases}
\end{align*}
We abbreviate the process that LMC updates the node embeddings and auxiliary variables of $\mathcal{V}_{\mathcal{B}}^k$ in the $l$-th layer at the $k$-th iteration as
    \begin{align*}
        &\hisH^{l,k}_{\mathcal{V}_{\mathcal{B}}^k} = [f_{\theta^{l,k}}(\temH^{l-1,k})]_{\mathcal{V}_{\mathcal{B}}^k},\\
        &\hisV^{l,k}_{\mathcal{V}_{\mathcal{B}}^k} = [\phi_{\theta^{l+1,k}}(\temH^{l+1,k})]_{\mathcal{V}_{\mathcal{B}}^k}.
    \end{align*}

For each $v_i\in\mathcal{V}_{\mathcal{B}}^k$, the update process of $v_i$ in the $l$-th layer at the $k$-th iteration can be expressed by
\begin{align*}
    &\hish^{l,k}_i = f_{\theta^{l,k},i}(\temH^{l-1,k}),\\
    &\hisV^{l,k}_i = \phi_{\theta^{l+1,k},i}(\temV^{l+1,k}),
\end{align*}
where $f_{\theta^{l,k},i}$ and $\phi_{\theta^{l+1,k},i}$ are the components for node $v_i$ of $f_{\theta^{l,k}}$ and $\phi_{\theta^{l+1,k}}$, respectively.

We first focus on convex combination coefficients $\beta_i$, $i\in[n]$. For the simplicity of analysis, we assume $\beta_i=\beta$ for $i\in[i]$. The analysis of the case where $(\beta_i)_{i=1}^n$ are different from each other is the same.

\begin{lemma}\label{prop:tem_epsilon_h}
    Suppose that Assumption \ref{assmp:proof} holds. For any $\varepsilon>0$, by letting
    \begin{align*}
        \beta \leq \frac{\varepsilon}{2G},\,\,\forall\,l\in[L], \,i\in[n],
    \end{align*}
    we have
    \begin{align*}
        \|\temH^{l,k} - \embH^{l,k}\|_F \leq \|\hisH^{l,k} - \embH^{l,k}\|_F + \varepsilon,\,\,\forall\,l\in[L],\,k\in\mathbb{N}^*.
    \end{align*}
\end{lemma}
\begin{proof}
    Since $\temH^{l,k} = (1-\beta)\hisH^{l,k} + \beta \widetilde{\embH}^{l,k}$, we have
    \begin{align*}
        &\|\temH^{l,k} - \embH^{l,k}\|_F\\
        ={}& \|(1-\beta)\hisH^{l,k} + \beta \widetilde{\embH}^{l,k} - (1-\beta)\embH^{l,k}+\beta \embH^{l,k}\|_F\\
        \leq{}& (1-\beta) \|\hisH^{l,k} - \embH^{l,k}\|_F + \beta \|\widetilde{\embH}^{l,k}-\embH^{l,k}\|_F\\
        \leq{}& \|\hisH^{l,k} - \embH^{l,k}\|_F + 2\beta G.
    \end{align*}
    Hence letting $\beta \leq \frac{\varepsilon}{2G}$ leads to
    \begin{align*}
        \|\temH^{l,k} - \embH^{l,k}\|_F \leq \|\hisH^{l,k} - \embH^{l,k}\|_F + \varepsilon.
    \end{align*}
\end{proof}

\begin{lemma}\label{prop:tem_epsilon_v}
    Suppose that Assumption \ref{assmp:proof} holds. For any $\varepsilon>0$, by letting
    \begin{align*}
        \beta \leq \frac{\varepsilon}{2G},\,\, \forall\,l\in[L], \,i\in[n],
    \end{align*}
    we have
    \begin{align*}
        \|\temV^{l,k} - \embV^{l,k}\|_F \leq \|\hisV^{l,k} - \embV^{l,k}\|_F + \varepsilon.
    \end{align*}
\end{lemma}
\begin{proof}
    Since $\temH^{l,k} = (1-\beta)\hisH^{l,k} + \beta \widetilde{\embH}^{l,k}$, we have
    \begin{align*}
        &\|\temH^{l,k} - \embH^{l,k}\|\\
        ={}& \|(1-\beta)\hisH^{l,k} + \beta \widetilde{\embH}^{l,k} - (1-\beta)\embH^{l,k}+\beta \embH^{l,k}\|_F\\
        \leq{}& (1-\beta) \|\hisH^{l,k} - \embH^{l,k}\|_F + \beta \|\widetilde{\embH}^{l,k}-\embH^{l,k}\|_F\\
        \leq{}& \|\hisH^{l,k} - \embH^{l,k}\|_F + 2\beta G.
    \end{align*}
    Hence letting $\beta \leq \frac{\varepsilon}{2G}$ leads to
    \begin{align*}
        \|\temH^{l,k} - \embH^{l,k}\|_F \leq \|\hisH^{l,k} - \embH^{l,k}\|_F + \varepsilon.
    \end{align*}
\end{proof}
Next, we focus on the approximation errors of historical node embeddings and auxiliary variables
\begin{align*}
    &d_{h}^{l,k} := \left( \mathbb{E}[\|\hisH^{l,k} - \embH^{l,k}\|_F^2] \right)^{\frac{1}{2}},\,\,l\in[L],\\
    &d_{v}^{l,k} := \left( \mathbb{E}[\|\hisV^{l,k} - \embV^{l,k}\|_F^2] \right)^{\frac{1}{2}},\,\,l\in[L-1].
\end{align*}

\begin{lemma}\label{prop:approx_h_new}
    For an $L$-layer ConvGNN, suppose that Assumption \ref{assmp:proof} holds. Besides, we suppose that
    \begin{enumerate}
        \item $(d_h^{l,1})^2$ is bounded by $G>1$, $\forall\, l\in[L]$,

        \item there exists $N\in\mathbb{N}^*$ such that for any $k\in\mathbb{N}^*$ and $l\in[L]$ we have
        \begin{align*}
            &\|\temH^{l,k} - \embH^{l,k}\|_F \leq \|\hisH^{l,k} - \embH^{l,k}\|_F + \frac{1}{N^{\frac{2}{3}}},\\
            &\|\embH^{l,k} - \embH^{l,k-1}\|_F \leq \frac{1}{N^{\frac{2}{3}}}, 
        \end{align*}
    \end{enumerate}
    then there exist constants $C'_{*,1}$, $C'_{*,2}$, and $C'_{*,3}$ that do not depend on $k, l, N$, and $\eta$, such that 
    \begin{align*}
     (d^{l,k+1}_h)^2 \leq C'_{*,1} \eta + C'_{*,2} \rho^k +  \frac{C'_{*,3}}{N^{\frac{2}{3}}},\,\,\forall\, l\in[L],\,k\in\mathbb{N}^*,
    \end{align*}
    where $\rho = \frac{n-S}{n}<1$, $n=|\mathcal{V}|$, and $S$ is number of sampled nodes at each iteration.
\end{lemma}

\begin{proof}
    We have
\begin{align*}
    &(d^{l+1,k+1}_h)^2\\
    ={}& \mathbb{E}[\|\hisH^{l+1,k+1} - \embH^{l+1,k+1}\|_F^2] \\
    ={}&\mathbb{E}[\sum_{i=1}^n \| \hish^{l+1,k+1}_i - \embh^{l+1,k+1}_i\|_F^2]\\
    ={}&\mathbb{E}[\sum_{v_i\in\mathcal{V}_{\mathcal{B}}^k} \|f_{\theta^{l+1,k+1},i}(\temH^{l,k+1}) - f_{\theta^{l+1,k+1},i} (\embH^{l,k+1})\|_F^2\\
    &\quad\quad+ \sum_{v_i\not\in\mathcal{V}_{\mathcal{B}}^k} \|\hish^{l+1,k}_i - f_{\theta^{l+1,k+1},i} (\embH^{l,k+1})\|_F^2]\\
    ={}& \mathbb{E}[\frac{S}{n}\sum_{i=1}^n \|f_{\theta^{l+1,k+1},i}(\temH^{l,k+1}) - f_{\theta^{l+1,k+1},i} (\embH^{l,k+1})\|_F^2\\
    &\quad\quad+ \frac{n-S}{n} \sum_{i=1}^n \|\hish^{l+1,k}_i - f_{\theta^{l+1,k+1},i} (\embH^{l,k+1})\|_F^2]\\
    \leq{}& \frac{S}{n} \sum_{i=1}^n \mathbb{E}[\|f_{\theta^{l+1,k+1},i}(\temH^{l,k+1}) - f_{\theta^{l+1,k+1},i} (\embH^{l,k+1})\|_F^2]\\
    &+ \frac{n-S}{n} \sum_{i=1}^n \mathbb{E}[\|\hish^{l+1,k}_i - f_{\theta^{l+1,k+1},i} (\embH^{l,k+1})\|_F^2].\\
\end{align*}
About the first term, for $l\geq 1$, we have
\begin{align*}
    &\mathbb{E}[\|f_{\theta^{l+1,k+1},i}(\temH^{l,k+1}) - f_{\theta^{l+1,k+1},i} (\embH^{l,k+1})\|_F^2]\\
    \leq{}& \gamma^2 \mathbb{E}[\|\temH^{l,k+1} - \embH^{l,k+1}\|_F^2]\\
    \leq{}& \gamma^2 \mathbb{E}[(\|\hisH^{l,k+1} - \embH^{l,k+1}\|_F+\frac{1}{N^{\frac{2}{3}}})^2]\\
    \leq{}& 2\gamma^2 \mathbb{E}[\|\hisH^{l,k+1} - \embH^{l,k+1}\|_F^2] +  \frac{2\gamma^2}{N^{\frac{4}{3}}}\\
    ={}& 2\gamma^2 (d^{l,k+1}_h)^2+\frac{2\gamma^2}{N^{\frac{4}{3}}}.
\end{align*}

For $l=0$, we have
\begin{align*}
     &\mathbb{E}[\|f_{\theta^{l+1,k+1},i}(\temH^{l,k+1}) - f_{\theta^{l+1,k+1},i} (\embH^{l,k+1})\|_F^2]\\
     ={}& \mathbb{E}[\|f_{\theta^{1,k+1},i}(\temH^{0,k+1}) - f_{\theta^{1,k+1},i} (\embH^{0,k+1})\|_F^2]\\
     ={}& \mathbb{E}[\|f_{\theta^{1,k+1},i}(\embX) - f_{\theta^{1,k+1},i} (\embX)\|_F^2]\\
     ={}& 0.
\end{align*}
About the second term, for $l\geq 1$, we have
\begin{align*}
    &\mathbb{E}[\|\hish^{l+1,k}_i - f_{\theta^{l+1,k+1},i} (\embH^{l,k+1})\|_F^2]\\
    \leq{}& \mathbb{E}[\|\hish_i^{l+1,k} - \embh_i^{l+1,k} + \embh_i^{l+1,k} - f_{\theta^{l+1,k+1},i} (\embH^{l,k+1})\|_F^2]\\
    \leq{}& \mathbb{E}[\|\hish_i^{l+1,k} - \embh_i^{l+1,k}\|_F^2]\\
    &+ \mathbb{E}[\|\embh_i^{l+1,k} - f_{\theta^{l+1,k+1},i} (\embH^{l,k+1})\|_F^2]\\
    &+ 2\mathbb{E}[\langle  \hish_i^{l+1,k} - \embh_i^{l+1,k}, \embh_i^{l+1,k} - f_{\theta^{l+1,k+1},i} (\embH^{l,k+1})\rangle]\\
    \leq{}& \mathbb{E}[\|\hish_i^{l+1,k} - \embh_i^{l+1,k}\|_F^2]\\
    &+\mathbb{E}[\|\embh^{l+1,k}_i - f_{\theta^{l+1,k+1},i} (\embH^{l,k})\\
    &\quad\quad+ f_{\theta^{l+1,k+1},i} (\embH^{l,k}) - f_{\theta^{l+1,k+1},i} (\embH^{l,k+1})\|_F^2]\\
    &+2\mathbb{E}[\langle  \hish_i^{l+1,k} - \embh_i^{l+1,k}, \embh_i^{l+1,k} - f_{\theta^{l+1,k+1},i} (\embH^{l,k+1})\rangle]\\
    \leq{}& \mathbb{E}[\|\hish_i^{l+1,k} - \embh_i^{l+1,k}\|_F^2]\\
    &+ 2\mathbb{E}[\|\embh^{l+1,k}_i - f_{\theta^{l+1,k+1},i} (\embH^{l,k})\|_F^2]\\
    &+ 2\mathbb{E}[\|f_{\theta^{l+1,k+1},i} (\embH^{l,k}) - f_{\theta^{l+1,k+1},i} (\embH^{l,k+1})\|_F^2]\\
    &+ 4G\mathbb{E}[\|\embh_i^{l+1,k} - f_{\theta^{l+1,k+1},i} (\embH^{l,k+1})\|_F]\\
    \leq{}& \mathbb{E}[\|\hish_i^{l+1,k} - \embh_i^{l+1,k}\|_F^2]\\
    &+2\gamma^2\mathbb{E}[\|\theta^{l+1,k} - \theta^{l+1,k+1}\|^2]\\
    &+ 2\gamma^2\mathbb{E}[\|\embH^{l,k} - \embH^{l,k+1}\|_F^2]\\
    &+ 4G\gamma\mathbb{E}[\|\theta^{l+1,k} - \theta^{l+1,k+1}\|+\|\embH^{l,k} - \embH^{l,k+1}\|_F]\\
    \leq{}& \mathbb{E}[\|\hish_i^{l+1,k} - \embh_i^{l+1,k}\|_F^2]+ 2\gamma^2G^2\eta^2\\
    &+ 4G^2\gamma\eta + \frac{2\gamma^2}{N^{\frac{4}{3}}}+\frac{4G\gamma}{N^{\frac{2}{3}}}\\
    \leq{}& \mathbb{E}[\|\hish_i^{l+1,k} - \embh_i^{l+1,k}\|_F^2]+ 2G^2\gamma(\gamma+2)\eta\\
    &+ \frac{2\gamma(\gamma+2G)}{N^{\frac{2}{3}}}.
\end{align*}
For $l=0$, we have
\begin{align*}
    &\mathbb{E}[\|\hish^{l+1,k}_i - f_{\theta^{l+1,k+1},i} (\embH^{l,k+1})\|_F^2]\\
    \leq{}& \mathbb{E}[\|\hish_i^{l+1,k} - \embh_i^{l+1,k} + \embh_i^{l+1,k} - f_{\theta^{l+1,k+1},i} (\embH^{l,k+1})\|_F^2]\\
    \leq{}& \mathbb{E}[\|\hish_i^{l+1,k} - \embh_i^{l+1,k}\|_F^2]\\
    &+ \mathbb{E}[\|\embh_i^{l+1,k} - f_{\theta^{l+1,k+1},i} (\embH^{l,k+1})\|_F^2]\\
    &+ 2\mathbb{E}[\langle  \hish_i^{l+1,k} - \embh_i^{l+1,k}, \embh_i^{l+1,k} - f_{\theta^{l+1,k+1},i} (\embH^{l,k+1})\rangle]\\
    \leq{}& \mathbb{E}[\|\hish_i^{l+1,k} - \embh_i^{l+1,k}\|_F^2]\\
    &+\mathbb{E}[\|\embh^{l+1,k}_i - f_{\theta^{l+1,k+1},i} (\embH^{l,k})\\
    &\quad\quad+ f_{\theta^{l+1,k+1},i} (\embH^{l,k}) - f_{\theta^{l+1,k+1},i} (\embH^{l,k+1})\|_F^2]\\
    &+2\mathbb{E}[\langle  \hish_i^{l+1,k} - \embh_i^{l+1,k}, \embh_i^{l+1,k} - f_{\theta^{l+1,k+1},i} (\embH^{l,k+1})\rangle]\\
    \leq{}& \mathbb{E}[\|\hish_i^{l+1,k} - \embh_i^{l+1,k}\|_F^2]\\
    &+ 2\mathbb{E}[\|\embh^{l+1,k}_i - f_{\theta^{l+1,k+1},i} (\embH^{l,k})\|_F^2]\\
    &+ 2\mathbb{E}[\|f_{\theta^{l+1,k+1},i} (\embH^{l,k}) - f_{\theta^{l+1,k+1},i} (\embH^{l,k+1})\|_F^2]\\
    &+ 4G\mathbb{E}[\|\embh_i^{l+1,k} - f_{\theta^{l+1,k+1},i} (\embH^{l,k+1})\|_F]\\
    \leq{}& \mathbb{E}[\|\hish_i^{l+1,k} - \embh_i^{l+1,k}\|_F^2]+2\gamma^2\mathbb{E}[\|\theta^{l+1,k} - \theta^{l+1,k+1}\|^2]\\
    &+ 4G\gamma\mathbb{E}[\|\theta^{l+1,k} - \theta^{l+1,k+1}\|]\\
    \leq{}& \mathbb{E}[\|\hish_i^{l+1,k} - \embh_i^{l+1,k}\|_F^2]+ 2\gamma^2G^2\eta^2 + 4G^2\gamma\eta,\\
    \leq{}& \mathbb{E}[\|\hish_i^{l+1,k} - \embh_i^{l+1,k}\|_F^2]+ 2G^2\gamma(\gamma+2)\eta + 4G^2\gamma\eta.
\end{align*}
Hence we have
\begin{align*}
    &(d^{l+1,k+1}_h)^2\\
    \leq{}& \frac{(n-S)}{n} (d^{l+1,k}_h)^2+2(n-S)\gamma(\gamma+2)G^2\eta\\
    &+\begin{cases}
        0, &l=0,\\
        2\gamma^2 S (d^{l,k+1}_h)^2 + \frac{4n\gamma(\gamma+G)}{N^{\frac{2}{3}}}, &l\geq 1.
    \end{cases}
\end{align*}
Let $\rho = \frac{n-S}{n}<1$. For $l=0$, we have
\begin{align*}
    &(d^{1,k+1}_h)^2 - \frac{2(n-S)\gamma(\gamma+2)G^2\eta}{1-\rho}\\ 
    \leq{}& \rho((d^{1,k}_h)^2 - \frac{2(n-S)\gamma(\gamma+2)G^2\eta}{1-\rho})\\
    \leq{}& \rho^2 ((d^{1,k-1}_h)^2 - \frac{2(n-S)\gamma(\gamma+2)G^2\eta}{1-\rho})\\
    \leq{}&\cdots\\
    \leq{}& \rho^k ((d^{1,1}_h)^2 - \frac{2(n-S)\gamma(\gamma+2)G^2\eta}{1-\rho})\\
    \leq{}& \rho^k G,
\end{align*}
which leads to
\begin{align*}
    (d^{1,k+1}_h)^2 &\leq \frac{2(n-S)\gamma(\gamma+2)G^2}{1-\rho}\eta + \rho^k G\\
    &= C_{1,1}' \eta + \rho^k G.
\end{align*}
Then, for $l=1$ we have
\begin{align*}
    (d^{2,k+1}_h)^2 \leq \rho (d_h^{2,k})^2 + C_{2,1}\eta + C_{2,2}\rho^k + \frac{C_{2,3}}{N^{\frac{2}{3}}},
\end{align*}
where $C_{2,1}, C_{2,2}$, and $C_{2,3}$ are all constants. Hence we have
\begin{align*}
    &(d_h^{2,k+1})^2 - \frac{C_{2,1}\eta + C_{2,2}\rho^k + \frac{C_{2,3}}{N^{\frac{2}{3}}}}{1-\rho}\\
    \leq{}& \rho((d_h^{2,k})^2 - \frac{C_{2,1}\eta + C_{2,2}\rho^k + \frac{C_{2,3}}{N^{\frac{2}{3}}}}{1-\rho})\\
    \leq{}& \cdots\\
    \leq{}& \rho^k ((d_h^{2,1})^2 - \frac{C_{2,1}\eta + C_{2,2}\rho^k + \frac{C_{2,3}}{N^{\frac{2}{3}}}}{1-\rho})\\
    \leq{}& \rho^k G,
\end{align*}
which leads to
\begin{align*}
    (d_h^{2,k+1})^2 \leq C_{2,1}' \eta + C_{2,2}' \rho^k +\frac{C_{2,3}'}{N^{\frac{2}{3}}}.
\end{align*}
And so on, there exist constants $C'_{*,1}$, $C'_{*,2}$, and $C'_{*,3}$ that are independent with $\eta, k, l, N$ such that
\begin{align*}
    (d_h^{l,k+1})^2 \leq C'_{*,1} \eta + C'_{*,2} \rho^k + \frac{C'_{*,3}}{N^{\frac{2}{3}}},\,\,\forall\, l\in[L],\,k\in\mathbb{N}^*.
\end{align*}
\end{proof}

\begin{lemma}\label{prop:approx_v_new}
    For an $L$-layer ConvGNN, suppose that Assumption \ref{assmp:proof} holds. Besides, we suppose that
    \begin{enumerate}
        \item $(d_h^{l,1})^2$ is bounded by $G>1$, $\forall\, l\in[L]$,

        \item there exists $N\in\mathbb{N}^*$ such that for any $k\in\mathbb{N}^*$ and $l\in[L]$ we have
        \begin{align*}
            &\|\temV^{l,k} - \embV^{l,k}\|_F \leq \|\hisV^{l,k} - \embV^{l,k}\|_F + \frac{1}{N^{\frac{2}{3}}},\\
            &\|\embV^{l,k} - \embV^{l,k-1}\|_F \leq \frac{1}{N^{\frac{2}{3}}}, 
        \end{align*}
    \end{enumerate}
    then there exist constants $C'_{*,1}$, $C'_{*,2}$, and $C'_{*,3}$ that are independent with $k, l, \varepsilon^*$, and $\eta$, such that
    \begin{align*}
     (d^{l,k+1}_v)^2 \leq C'_{*,1} \eta + C'_{*,2} \rho^k + \frac{C'_{*,3}}{N^{\frac{2}{3}}},\,\,\forall\, l\in[L],\,k\in\mathbb{N}^*,
    \end{align*}
    where $\rho = \frac{n-S}{n}<1$, $n=|\mathcal{V}|$, and $S$ is number of sampled nodes at each iteration.
\end{lemma}

\begin{proof}
    Similar to the proof of Lemma \ref{prop:approx_h_new}.
\end{proof}

\subsubsection{Proof of Theorem \ref{thm:grad_error_conv}: approximation errors of mini-batch gradients}

In this subsection, we focus on the mini-batch gradients computed by LMC, i.e.,
\begin{align*}
    \widetilde{\embg}_w(w^k)=\frac{1}{|\mathcal{V}_{L}^k|} \sum_{v_j\in\mathcal{V}_{L}^k} \nabla_w \ell_{w^k}(\hish^{k}_j, y_j)
\end{align*}
and
\begin{align*}
    \widetilde{\embg}_{\theta^{l}}(\theta^{l,k}) = \frac{|\mathcal{V}|}{|\mathcal{V}_{\mathcal{B}}^k|}\sum_{v_j\in\mathcal{V}_{\mathcal{B}}^k} (\nabla_{\theta^{l}} u_{\theta^{l,k}}(\hish_j^{l-1,k}, \overline{\mathbf{m}}_{\neighbor{v_j}}^{l-1,k}, \embx_j))\hisV^{l,k}_j,
\end{align*}
where $\mathcal{V}_{\mathcal{B}}^k$ is the sampled mini-batch and $\mathcal{V}_{L_\mathcal{B}}^k$ is the corresponding labeled node set at the $k$-th iteration. We denote the mini-batch gradients computed by backward SGD by
\begin{align*}
    \embg_w(w^k)=\frac{1}{|\mathcal{V}_{L}^k|} \sum_{v_j\in\mathcal{V}_{L}^k} \nabla_w \ell_{w^k}(\embh^{k}_j, y_j)
\end{align*}
and
\begin{align*}
    \embg_{\theta^{l}}(\theta^{l,k}) = \frac{|\mathcal{V}|}{|\mathcal{V}_{\mathcal{B}}^k|}\sum_{v_j\in\mathcal{V}_{\mathcal{B}}^k} \left(\nabla_{\theta^{l}} u_{\theta^{l,k}}(\embh_j^{l-1,k}, \embm_{\neighbor{v_j}}^{l-1,k}, \embx_j)\right)\embV^{l,k}_j.
\end{align*}
The approximation errors of gradients are denoted by
\begin{align*}
    \Delta_{w}^{k} \triangleq \widetilde{\embg}_w(w^k) - \nabla_w \loss(w^k)
\end{align*}
and
\begin{align*}
    \Delta_{\theta^{l}}^{k} \triangleq \widetilde{\embg}_{\theta^{l}}(\theta^{l,k}) - \nabla_{\theta^{l}} \loss(\theta^{l,k}).
\end{align*}

\begin{lemma}\label{lemma:grad_error_bound_w_conv}
    Suppose that Assumption \ref{assmp:proof} holds. For any $k\in\mathbb{N}^*$, the difference between $\widetilde{\mathbf{g}}_w(w^k)$ and $\mathbf{g}_w(w^k)$ can be bounded as
    \begin{align*}
        \|\widetilde{\mathbf{g}}_w(w^k) - \mathbf{g}_w(w^k)\|_2 \leq \gamma \| \hisH^{L,k} - \embH^{L,k} \|_F.
    \end{align*}
\end{lemma}
\begin{proof}
    We have
    \begin{align*}
        \|\widetilde{\mathbf{g}}_w(w^k) - \mathbf{g}_w(w^k)\|_2 ={}& \frac{1}{|\mathcal{V}_L^k|}\|\sum_{v_j \in \mathcal{V}_{L}^k} \nabla_{w}\ell_{w^k}(\hish^{L,k}_j,y_j)\nonumber\\
        &\quad\quad\quad\quad\quad- \nabla_{w}\ell_{w^k}(\embh^{L,k}_j,y_j)\|_2\nonumber\\
        \leq{}& \frac{1}{|\mathcal{V}_L^k|}\sum_{v_j \in \mathcal{V}_{L}^k} \|\nabla_{w}\ell_{w^k}(\hish^{L,k}_j,y_j)\nonumber\\
        &\quad\quad\quad\quad\quad- \nabla_{w}\ell_{w^k}(\embh^{L,k}_j,y_j)\|_2\nonumber\\
        \leq{}& \frac{\gamma}{|\mathcal{V}_L^k|} \sum_{v_j \in \mathcal{V}_{L}^k}\|\hish^{L,k}_j-\embh^{L,k}_j\|_2\nonumber\\
        \leq{}& \frac{\gamma}{|\mathcal{V}_L^k|} \sum_{v_j \in \mathcal{V}_{L}^k}\|\hisH^{L,k}-\embH^{L,k}\|_F\nonumber\\
        ={}& \frac{\gamma}{|\mathcal{V}_L^k|} \cdot |\mathcal{V}_L^k|\cdot\|\hisH^{L,k}-\embH^{L,k}\|_F\nonumber\\
        ={}& \gamma \| \hisH^{L,k} - \embH^{L,k} \|_F\nonumber
    \end{align*}
\end{proof}

\begin{lemma}\label{lemma:grad_error_bound_theta_conv}
    Suppose that Assumption \ref{assmp:proof} holds. For any $k\in\mathbb{N}^*$ and $l\in[L]$, the difference between $\widetilde{\mathbf{g}}_{\theta^l}(\theta^{l,k})$ and $\mathbf{g}_{\theta^l}(\theta^{l,k})$ can be bounded as
    \begin{align*}
        &\|\widetilde{\mathbf{g}}_{\theta^l}(\theta^{l,k}) - \mathbf{g}_{\theta^l}(\theta^{l,k})\|_2\\
        \leq{}& |\mathcal{V}| G \|\hisV^{l,k}-\embV^{l,k}\|_F + |\mathcal{V}|G\gamma \|\hisH^{l,k}-\embH^{l,k}\|_F.
    \end{align*}
\end{lemma}
\begin{proof}
    As $\|\mathbf{A}\mathbf{a}-\mathbf{B}\mathbf{b}\|_2\leq \|\mathbf{A}\|_F\|\mathbf{a}-\mathbf{b}\|_2 + \|\mathbf{A}-\mathbf{B}\|_F\|\mathbf{b}\|_2$, we can bound $\|\widetilde{\mathbf{g}}_{\theta^l}(\theta^{l,k}) - \mathbf{g}_{\theta^l}(\theta^{l,k})\|_2$ by
    \begin{align*}
        &\|\widetilde{\mathbf{g}}_{\theta^l}(\theta^{l,k}) - \mathbf{g}_{\theta^l}(\theta^{l,k})\|_2\\
        \leq{}& \frac{|\mathcal{V}|}{|\mathcal{V}_{\mathcal{B}}^k|}\sum_{v_i\in\mathcal{V}_{\mathcal{B}}^k}\|\left(\nabla_{\theta^{l}} u_{\theta^{l,k}}(\hish_j^{l-1,k}, \overline{\mathbf{m}}_{\neighbor{v_j}}^{l-1,k}, \embx_j)\right)\hisV_j^{l,k}\\
        &\quad\quad\quad\quad\quad-\left(\nabla_{\theta^{l}} u_{\theta^{l,k}}(\embh_j^{l-1,k}, \mathbf{m}_{\neighbor{v_j}}^{l-1,k}, \embx_j)\right)\embV_j^{l,k}\|_2\\
        \leq{}& |\mathcal{V}|\max_{v_i\in\mathcal{V}_{\mathcal{B}}^k} \|\left(\nabla_{\theta^{l}} u_{\theta^{l,k}}(\hish_j^{l-1,k}, \overline{\mathbf{m}}_{\neighbor{v_j}}^{l-1,k}, \embx_j)\right)\hisV_j^{l,k}\\
        &\quad\quad\quad\quad\quad-\left(\nabla_{\theta^{l}} u_{\theta^{l,k}}(\embh_j^{l-1,k}, \mathbf{m}_{\neighbor{v_j}}^{l-1,k}, \embx_j)\right)\embV_j^{l,k}\|_2\\
        \leq{}& |\mathcal{V}|\max_{v_i\in\mathcal{V}_{\mathcal{B}}^k} \|\nabla_{\theta^{l}} u_{\theta^{l,k}}(\hish_j^{l-1,k}, \overline{\mathbf{m}}_{\neighbor{v_j}}^{l-1,k}, \embx_j)\|_F\|\hisV_j^{l,k}-\embV_j^{l,k}\|_2\\
        &\quad\quad\quad\quad+ \|\nabla_{\theta^{l}} u_{\theta^{l,k}}(\hish_j^{l-1,k}, \overline{\mathbf{m}}_{\neighbor{v_j}}^{l-1,k}, \embx_j)\\
        &\quad\quad\quad\quad\quad\quad-\nabla_{\theta^{l}} u_{\theta^{l,k}}(\embh_j^{l-1,k}, \mathbf{m}_{\neighbor{v_j}}^{l-1,k}, \embx_j)\|_F\|\embV_j^{l,k}\|_2\\
        \leq{}& |\mathcal{V}| G \|\hisV^{l,k}-\embV^{l,k}\|_F + |\mathcal{V}|G\gamma \|\hisH^{l,k}-\embH^{l,k}\|_F.
    \end{align*}
\end{proof}

\begin{lemma}\label{prop:grad_error}
    For an $L$-layer ConvGNN, suppose that Assumption \ref{assmp:proof} holds. For any $N\in\mathbb{N}^*$, by letting
    \begin{align*}
        \eta \leq \frac{1}{(2\gamma)^L G}\frac{1}{N^{\frac{2}{3}}} = \mathcal{O}(\frac{1}{N^{\frac{2}{3}}})
    \end{align*}
    and
    \begin{align*}
        \beta_i \leq \frac{1}{2G}\frac{1}{N^{\frac{2}{3}}}=\mathcal{O}(\frac{1}{N^{\frac{2}{3}}}),\,\, i\in[n],
    \end{align*}
    there exists $G_{2,*}>0$ and $\rho\in(0,1)$ such that for any $k\in\mathbb{N}^*$ we have
    \begin{align*}
        &\mathbb{E}[\|\Delta^k_w\|_2^2] = ({\rm Bias}(\widetilde{\mathbf{g}}_w(w^k)))^2+ {\rm Var}(\widetilde{\mathbf{g}}_w(w^k)),\\
        &\mathbb{E}[\|\Delta^k_{\theta^l}\|_2^2] = ({\rm Bias}(\widetilde{\mathbf{g}}_{\theta^l}(\theta^{l,k})))^2+ {\rm Var}(\widetilde{\mathbf{g}}_{\theta^l}(\theta^{l,k})),
    \end{align*}
    where
    \begin{align*}
        &{\rm Var}(\widetilde{\mathbf{g}}_w(w^k)) = \mathbb{E}[\|\mathbb{E}[\widetilde{\mathbf{g}}_w(w^k)] - \widetilde{\mathbf{g}}_w(w^k)\|_2^2],\\
        &{\rm Bias}(\widetilde{\mathbf{g}}_w(w^k)) = \|\mathbb{E}[\widetilde{\mathbf{g}}_w(w^k)] - \nabla_w\mathcal{L}(w^k)\|_2,\\
        &{\rm Var}(\widetilde{\mathbf{g}}_{\theta^l}(\theta^{l,k})) = \mathbb{E}[\|\mathbb{E}[\widetilde{\mathbf{g}}_{\theta^l}(\theta^{l,k})] - \widetilde{\mathbf{g}}_{\theta^l}(\theta^{l,k})\|_2^2],\\
        &{\rm Bias}(\widetilde{\mathbf{g}}_{\theta^l}(\theta^{l,k})) = \|\mathbb{E}[\widetilde{\mathbf{g}}_{\theta^l}(\theta^{l,k})] - \nabla_{\theta^l}\mathcal{L}(\theta^{l,k})\|_2,
    \end{align*}
    and
    \begin{align*}
        &{\rm Bias}(\widetilde{\mathbf{g}}_w(w^k))\leq G_{2,*}(\eta^{\frac{1}{2}} + \rho^{\frac{k-1}{2}}+\frac{1}{N^{\frac{1}{3}}}),\\
        &{\rm Bias}(\widetilde{\mathbf{g}}_{\theta^l}(\theta^{l,k}))\leq G_{2,*}(\eta^{\frac{1}{2}} + \rho^{\frac{k-1}{2}}+\frac{1}{N^{\frac{1}{3}}}).
    \end{align*}
\end{lemma}
\begin{proof}
    By Lemmas \ref{prop:exact_difference_h} and \ref{prop:exact_difference_v} we know that
    \begin{align*}
        &\|\embH^{l,k+1}-\embH^{l,k}\|_F < \frac{1}{N^{\frac{2}{3}}},\,\,\forall\,l\in[L],\,k\in\mathbb{N}^*,\\
        &\|\embV^{l,k+1}-\embV^{l,k}\|_F < \frac{1}{N^{\frac{2}{3}}},\,\,\forall\,l\in[L],\,k\in\mathbb{N}^*.
    \end{align*}
    By Lemmas \ref{prop:tem_epsilon_h} and \ref{prop:tem_epsilon_v} we know that for any $k\in\mathbb{N}^*$ and $l\in[L]$ we have
    \begin{align*}
        &\|\temH^{l,k} - \embH^{l,k}\|_F \leq \|\hisH^{l,k} - \embH^{l,k}\|_F + \frac{1}{N^{\frac{2}{3}}}
    \end{align*}
    and
    \begin{align*}
        &\|\temV^{l,k} - \embV^{l,k}\|_F \leq \|\hisV^{l,k} - \embV^{l,k}\|_F + \frac{1}{N^{\frac{2}{3}}}.
    \end{align*}
    Thus, by Lemmas \ref{prop:approx_h_new} and \ref{prop:approx_v_new} we know that there exist $C_{*,1}'$, $C_{*,2}'$, and $C_{*,3}'$ that do not depend on $k,l,\eta,N$ such that for $\forall\,l\in[L]$ and $k\in\mathbb{N}^*$ hold
    \begin{align*}
        d_{h}^{l,k} \leq{}& \sqrt{C_{*,1}'\eta + C_{*,2}'\rho^{k-1} + \frac{C_{*,3}'}{N^{\frac{2}{3}}}}\\
        \leq{} &\sqrt{C_{*,1}'} \eta^{\frac{1}{2}} + \sqrt{C_{*,2}'} \rho^{\frac{k-1}{2}} + \sqrt{C_{*,3}'} \frac{1}{N^{\frac{1}{3}}}
    \end{align*}
    and
    \begin{align*}
        d_{v}^{l,k} \leq{} &\sqrt{C_{*,1}'\eta + C_{*,2}'\rho^{k-1} + \frac{C_{*,3}'}{N^{\frac{2}{3}}}}\\
        \leq{} &\sqrt{C_{*,1}'} \eta^{\frac{1}{2}} + \sqrt{C_{*,2}'} \rho^{\frac{k-1}{2}} + \sqrt{C_{*,3}'} \frac{1}{N^{\frac{1}{3}}}.
    \end{align*}
    We can decompose $\|\Delta^k_w\|_2^2$ as
    \begin{align*}
        &\|\Delta^k_w\|_2^2\\
        ={}& \|\widetilde{\mathbf{g}}_w(w^k) - \nabla_w \loss(w^k)\|_2^2\\
        ={}& \|\widetilde{\mathbf{g}}_w(w^k) - \mathbb{E}[\widetilde{\mathbf{g}}_w(w^k)] + \mathbb{E}[\widetilde{\mathbf{g}}_w(w^k)] - \nabla_w \loss(w^k)\|_2^2\\
        ={}& \|\widetilde{\mathbf{g}}_w(w^k) - \mathbb{E}[\widetilde{\mathbf{g}}_w(w^k)] \|_2^2 + \|\mathbb{E}[\widetilde{\mathbf{g}}_w(w^k)]  - \nabla_w\loss(w^k)\|_2^2\\
        &+ 2\langle \|\widetilde{\mathbf{g}}_w(w^k) - \mathbb{E}[\widetilde{\mathbf{g}}_w(w^k)] , \mathbb{E}[\widetilde{\mathbf{g}}_w(w^k)]  - \nabla_w \loss(w^k) \rangle.
    \end{align*}
    We take expectation of both sides of the above expression, leading to
    \begin{align}
        \mathbb{E}[\|\Delta^k_w\|_2^2] = ({\rm Bias}(\widetilde{\mathbf{g}}_w(w^k)))^2+ {\rm Var}(\widetilde{\mathbf{g}}_w(w^k)), \label{eqn:bias-variance_w_conv}
    \end{align}
    where
    \begin{align*}
        &{\rm Var}(\widetilde{\mathbf{g}}_w(w^k)) = \mathbb{E}[\|\mathbb{E}[\widetilde{\mathbf{g}}_w(w^k)] - \widetilde{\mathbf{g}}_w(w^k)\|_2^2],\\
        &{\rm Bias}(\widetilde{\mathbf{g}}_w(w^k)) = \|\mathbb{E}[\widetilde{\mathbf{g}}_w(w^k)] - \nabla_w\mathcal{L}(w^k)\|_2
    \end{align*}
    as
    \begin{align*}
        \mathbb{E}[\langle \|\widetilde{\mathbf{g}}_w(w^k) - \mathbb{E}[\widetilde{\mathbf{g}}_w(w^k)] , \mathbb{E}[\widetilde{\mathbf{g}}_w(w^k)]  - \nabla_w \loss(w^k) \rangle] = 0.
    \end{align*}
    By Lemma \ref{lemma:grad_error_bound_w_conv}, we can bound the bias term as
    \begin{align}
        {\rm Bias}(\widetilde{\mathbf{g}}_w(w^k)) ={}& \|\mathbb{E}[\widetilde{\mathbf{g}}_w(w^k)] - \nabla_w\mathcal{L}(w^k)\|_2\nonumber\\
        ={}& \|\mathbb{E}[\widetilde{\mathbf{g}}_w(w^k) - \mathbf{g}_w(w^k)]\|_2\nonumber\\
        \leq{}& \mathbb{E}[\|\widetilde{\mathbf{g}}_w(w^k) - \mathbf{g}_w(w^k)\|_2]\nonumber\\
        \leq{}& \gamma \mathbb{E}[\|\hisH^{L,k}-\embH^{L,k}\|_F]\nonumber\\
        \leq{}& \gamma \left(\mathbb{E}[\|\hisH^{L,k}-\embH^{L,k}\|_F^2]\right)^{\frac{1}{2}}\nonumber\\
        ={}& \gamma \cdot d_h^{L,k}\nonumber\\
        \leq{}& \gamma(\sqrt{C_{*,1}'} \eta^{\frac{1}{2}} + \sqrt{C_{*,2}'} \rho^{\frac{k-1}{2}} + \sqrt{C_{*,3}'} \frac{1}{N^{\frac{1}{3}}})\nonumber\\
        \leq{}& G_{2,1}(\eta^{\frac{1}{2}} + \rho^{\frac{k-1}{2}}+\frac{1}{N^{\frac{1}{3}}}),\label{eqn:bias_bound_w_conv}
    \end{align}
    where $G_{2,1} = \gamma \max\{ \sqrt{C_{*,1}'},\sqrt{C_{*,2}'},\sqrt{C_{*,3}'} \}$.
    
    Similar to Eq. \eqref{eqn:bias-variance_w_conv}, we can decompose $\mathbb{E}[\|\Delta^k_{\theta^l}\|_2^2]$ as
    \begin{align*}
        \mathbb{E}[\|\Delta^k_{\theta^l}\|_2^2] = ({\rm Bias}(\widetilde{\mathbf{g}}_{\theta^l}(\theta^{l,k})))^2+ {\rm Var}(\widetilde{\mathbf{g}}_{\theta^l}(\theta^{l,k})),
    \end{align*}
    where
    \begin{align*}
        &{\rm Var}(\widetilde{\mathbf{g}}_{\theta^l}(\theta^{l,k})) = \mathbb{E}[\|\mathbb{E}[\widetilde{\mathbf{g}}_{\theta^l}(\theta^{l,k})] - \widetilde{\mathbf{g}}_{\theta^l}(\theta^{l,k})\|_2^2],\\
        &{\rm Bias}(\widetilde{\mathbf{g}}_{\theta^l}(\theta^{l,k})) = \|\mathbb{E}[\widetilde{\mathbf{g}}_{\theta^l}(\theta^{l,k})] - \nabla_{\theta^l}\mathcal{L}(\theta^{l,k})\|_2.
    \end{align*}
    By Lemma \ref{lemma:grad_error_bound_theta_conv}, we can bound the bias term as
    \begin{align}
        {\rm Bias}(\widetilde{\mathbf{g}}_{\theta^l}(\theta^{l,k}))={}& \|\mathbb{E}[\widetilde{\mathbf{g}}_{\theta^l}(\theta^{l,k})] - \nabla_{\theta^l}\mathcal{L}(\theta^{l,k})\|_2 \nonumber\\
        ={}& \|\mathbb{E}[\widetilde{\mathbf{g}}_{\theta^l}(\theta^{l,k}) - \mathbf{g}_{\theta^l}(\theta^{l,k})]\|_2 \nonumber \\
        \leq{}& \mathbb{E}[\|\widetilde{\mathbf{g}}_{\theta^l}(\theta^{l,k}) - \mathbf{g}_{\theta^l}(\theta^{l,k})\|_2] \nonumber \\
        \leq{}& |\mathcal{V}|G\mathbb{E}[\|\hisV^{l,k}-\embV^{l,k}\|_F]\nonumber\\ 
        &+ |\mathcal{V}|G\gamma \mathbb{E}[\|\hisH^{l,k}-\embH^{l,k}\|_F]\nonumber\\
        \leq{}& |\mathcal{V}|G\left(\mathbb{E}[\|\hisV^{l,k}-\embV^{l,k}\|^2_F]\right)^{\frac{1}{2}}\nonumber\\ 
        &+ |\mathcal{V}|G\gamma \left(\mathbb{E}[\|\hisH^{l,k}-\embH^{l,k}\|^2_F]\right)^{\frac{1}{2}}\nonumber\\
        ={}&|\mathcal{V}|G d_v^{l,k} + |\mathcal{V}|G\gamma d_h^{l,k}\nonumber\\
        \leq{}& G_{2,2}(\eta^{\frac{1}{2}} + \rho^{\frac{k-1}{2}}+\frac{1}{N^{\frac{1}{3}}}),\label{eqn:bias_bound_theta_conv}
    \end{align}
    where $G_{2,2}=|\mathcal{V}|G(1+\gamma)\max\{\sqrt{C_{*,1}'},\sqrt{C_{*,2}'},\sqrt{C_{*,3}'}\}$. 
    
    Let $G_{2,*}=\max\{G_{2,1}, G_{2,2}\}$, then we have
    \begin{align*}
        &{\rm Bias}(\widetilde{\mathbf{g}}_{w}(w^{k})) \leq G_{2,*}(\eta^{\frac{1}{2}} + \rho^{\frac{k-1}{2}}+\frac{1}{N^{\frac{1}{3}}}),\\
        &{\rm Bias}(\widetilde{\mathbf{g}}_{\theta^l}(\theta^{l,k})) \leq G_{2,*}(\eta^{\frac{1}{2}} + \rho^{\frac{k-1}{2}}+\frac{1}{N^{\frac{1}{3}}}).
    \end{align*}
\end{proof}

By letting $\varepsilon = \frac{1}{N^{\frac{1}{3}}}$ and $C=G_{2,*}$, Theorem \ref{thm:grad_error_conv} follows immediately.

\subsubsection{Proof of Theorem \ref{thm:convergence_conv}: convergence guarantees}

In this subsection, we give the convergence guarantees of LMC. We first give sufficient conditions for the convergence.
\begin{lemma}\label{prop:suff}
    Suppose that function $f:\mathbb{R}^{n} \to \mathbb{R}$ is continuously differentiable. Consider an optimization algorithm with any bounded initialization $\vecx^1$ and an update rule in the form of
    \begin{align*}
        \vecx^{k+1} = \vecx^{k} - \eta \vecd(\vecx^{k}),
    \end{align*}
    where $\eta>0$ is the learning rate and $\vecd(\vecx^{k})$ is the estimated gradient that can be seen as a stochastic vector depending on $\vecx^{k}$. Let the estimation error of the gradient be $\Delta^{k} = \vecd(\vecx^{k}) - \nabla f(\vecx^{k}) $. Suppose that
    \begin{enumerate}
        \item the optimal value $f^*  = \inf_{\vecx} f(\vecx)$ is bounded; \label{con:1_conv}
        
        \item the gradient of $f$ is $\gamma$-Lipschitz, i.e., \label{con:2_conv}
        \begin{align*}
            \|\nabla f(\vecy) - \nabla f(\vecx)\|_2 \leq \gamma\|\vecy - \vecx\|_2,\,\forall\,\vecx,\vecy \in \mathbb{R}^{n};
        \end{align*}
        
        \item there exists $G_0>0$ that does not depend on $\eta$ such that \label{cond:suff_3}
        \begin{align*}
            \mathbb{E}[\|\Delta^{k}\|_2^2] \leq G_0,\,\forall\, k\in\mathbb{N}^*;
        \end{align*}

        \item there exists $N\in\mathbb{N}^*$ and $\rho\in(0,1)$ that do not depend on $\eta$ such that for any $k\in\mathbb{N}^*$ we have
        \begin{align*}
            |\mathbb{E}[\langle \nabla f(\vecx^{k}),\Delta^{k} \rangle]| \leq G_0(\eta^{\frac{1}{2}} +\rho^{\frac{k-1}{2}}+\frac{1}{N^{\frac{1}{3}}}),
        \end{align*}
        where $G_0$ is the same constant as that in Condition \ref{cond:suff_3},
    \end{enumerate}
    then by letting $\eta=\min\{\frac{1}{\gamma},\frac{1}{N^{\frac{2}{3}}}\}$, we have
    \begin{align*}
        &\mathbb{E}[ \|\nabla f(\vecx^{R})\|_2^2]\\
        \leq{}&\frac{2(f(\vecx^{1})-f^*+G_0)}{N^{\frac{1}{3}}} +\frac{\gamma G_0}{N^{\frac{2}{3}}}+ \frac{G_0}{N(1-\sqrt{\rho})}\\
        ={}& \mathcal{O}(\frac{1}{N^{\frac{1}{3}}}),
    \end{align*}
     where $R$ is chosen uniformly from $[N]$.
\end{lemma}

\begin{proof}
    As the gradient of $f$ is $\gamma$-Lipschitz, we have
    \begin{align*}
    	f(\vecy)={}&f(\vecx)+\int_{\vecx}^{\vecy}\nabla f(\mathbf{z})\rmd\mathbf{z}\\
    	={}&f(\vecx)+\int_0^1\langle\nabla f(\vecx+t(\vecy-\vecx)), \vecy-\vecx\rangle \rmd t\\
    	={}&f(\vecx)+\langle\nabla f(\vecx),\vecy-\vecx\rangle\\
    	&+\int_0^1\langle\nabla f(\vecx+t(\vecy-\vecx))-\nabla f(\vecx), \vecy-\vecx\rangle \rmd t\\
    	\leq{}&f(\vecx)+\langle\nabla f(\vecx),\vecy-\vecx\rangle\\
    	&+\int_0^1\|\nabla f(\vecx+t(\vecy-\vecx))-\nabla f(\vecx)\|_2\| \vecy-\vecx\|_2 \rmd t\\
    	\leq{}&f(\vecx)+\langle\nabla f(\vecx),\vecy-\vecx\rangle+\int_0^1\gamma t\|\vecy-\vecx\|^2_2\rmd t\\
    	\leq{}&f(\vecx)+\langle\nabla f(\vecx), \vecy-\vecx\rangle+\frac{\gamma}{2}\|\vecy-\vecx\|_2^2,
    \end{align*}
    Then, we have
    \begin{align*}
        &f(\vecx^{k+1})\\ \leq{}&f(\vecx^{k}) + \langle  \nabla f(\vecx^{k}), \vecx^{k+1} - \vecx^{k} \rangle+\frac{\gamma}{2}\|\vecx^{k+1}-\vecx^{k}\|_2^2 \\
        ={}& f(\vecx^{k}) - \eta \langle \nabla f(\vecx^{k}), \vecd(\vecx^{k}) \rangle+ \frac{\eta^2 \gamma}{2}\|\vecd(\vecx^{k})\|_2^2 \\
        ={}& f(\vecx^{k}) - \eta \langle \nabla f(\vecx^{k}), \Delta^{k} \rangle- \eta \|\nabla f(\vecx^{k})\|_2^2\\
        &+  \frac{\eta^2 \gamma}{2}(\|\Delta^{k}\|_2^2+\|\nabla f(\vecx^{k})\|_2^2 +2\langle \Delta^{k}, \nabla f(\vecx^{k}) \rangle)\\
        ={}& f(\vecx^{k})  - \eta (1-\eta \gamma) \langle \nabla f(\vecx^{k}), \Delta^{k} \rangle\\
        &- \eta (1-\frac{\eta \gamma}{2}) \|\nabla f(\vecx^{k})\|_2^2 + \frac{\eta^2 \gamma}{2}\|\Delta^{k}\|_2^2.
    \end{align*}
    By taking expectation of both sides, we have
    \begin{align*}
        &\mathbb{E}[f(\vecx^{k+1})]\\
        \leq{}&\mathbb{E}[f(\vecx^{k})] - \eta (1-\eta \gamma)  \mathbb{E}[\langle \nabla f(\vecx^{k}), \Delta^{k} \rangle]\\
        &- \eta (1-\frac{\eta \gamma}{2}) \mathbb{E}[ \|\nabla f(\vecx^{k})\|_2^2]+ \frac{\eta^2 \gamma}{2}\mathbb{E}[\|\Delta^{k}\|_2^2].
    \end{align*}
    By summing up the above inequalities for $k\in[N]$ and dividing both sides by $N \eta(1-\frac{\eta \gamma}{2})$, we have
    \begin{align*}
        &\frac{\sum_{k=1}^{N} \mathbb{E}[ \|\nabla f(\vecx^{k})\|_2^2]}{N}\\
        \leq{}& \frac{f(\vecx^{1}) - \mathbb{E}[f(\vecx^{N})]}{N \eta(1-\frac{\eta \gamma}{2}) } + \frac{\eta \gamma}{2-\eta \gamma} \frac{\sum_{k=1}^N \mathbb{E}[\|\Delta^{k}\|_2^2]}{N}\\
        &- \frac{(1-\eta \gamma)}{(1-\frac{\eta \gamma}{2})} \frac{\sum_{k=1}^{N} \mathbb{E}[\langle \nabla f(\vecx^{k}), \Delta^{k} \rangle]}{N} \\
        \leq{}& \frac{f(\vecx^{1}) - f^* }{N \eta(1-\frac{\eta \gamma}{2}) } + \frac{\eta \gamma}{2-\eta \gamma} \frac{\sum_{k=1}^N \mathbb{E}[\|\Delta^{k}\|_2^2]}{N}\\
        &+ \frac{\sum_{k=1}^{N} |\mathbb{E}[\langle \nabla f(\vecx^{k}), \Delta^{k} \rangle]|}{N},
    \end{align*}
    where the second inequality comes from $\eta \gamma>0$ and $f(\vecx^{k}) \geq f^* $. According to the above conditions, we have
    \begin{align*}
        &\frac{\sum_{k=1}^{N} \mathbb{E}[ \|\nabla f(\vecx^{k})\|_2^2]}{N}\\
        \leq{}&  \frac{f(\vecx^{1}) - f^* }{N \eta(1-\frac{\eta \gamma}{2}) } + \frac{\eta \gamma}{2 - \eta \gamma}G_0 + G_0\sum_{k=1}^N \frac{\eta^{\frac{1}{2}}  + \rho^{\frac{k-1}{2}}}{N} + \frac{G_0}{N^{\frac{1}{3}}}\\
        \leq{}& \frac{f(\vecx^{1}) - f^* }{N \eta(1-\frac{\eta \gamma}{2}) } + \frac{\eta \gamma}{2 - \eta \gamma}G_0+\eta^{\frac{1}{2}} G_0 + \frac{G_0}{N}\sum_{k=1}^\infty \rho^{\frac{k-1}{2}}  + \frac{G_0}{N^{\frac{1}{3}}}\\
        ={}& \frac{f(\vecx^{1}) - f^* }{N \eta(1-\frac{\eta \gamma}{2}) } + \frac{\eta \gamma}{2 - \eta \gamma}G_0+\eta^{\frac{1}{2}} G_0 + \frac{G_0}{N(1-\sqrt{\rho})}+ \frac{G_0}{N^{\frac{1}{3}}}.
    \end{align*}
    Notice that
    \begin{align*}
        \mathbb{E}[ \|\nabla f(\vecx^{R})\|_2^2] ={}& \mathbb{E}_R[\mathbb{E}\|[\nabla f(\vecx^{R})\|_2^2\mid R]]\\
        ={}&\frac{\sum_{k=1}^{N} \mathbb{E}[ \|\nabla f(\vecx^{k})\|_2^2]}{ N},
    \end{align*}
    where $R$ is uniformly chosen from $[N]$, hence we have
    \begin{align*}
        \mathbb{E}[ \|\nabla f(\vecx^{R})\|_2^2] \leq{}& \frac{f(\vecx^{1}) - f^* }{N \eta(1-\frac{\eta \gamma}{2}) } + \frac{\eta \gamma}{2 - \eta \gamma}G_0\\
        &+\eta^{\frac{1}{2}} G_0 + \frac{G_0}{N(1-\sqrt{\rho})}+ \frac{G_0}{N^{\frac{1}{3}}}.
    \end{align*}
    By letting $\eta=\min\{\frac{1}{\gamma}, \frac{1}{N^{\frac{2}{3}}}\}$, we have
    \begin{align*}
        &\mathbb{E}[ \|\nabla f(\vecx^{R})\|_2^2]\\
        \leq{}& \frac{2(f(\vecx^{1})-f^*)}{N^{\frac{1}{3}}}+\frac{\gamma G_0}{N^{\frac{2}{3}}} + \frac{G_0}{N^{\frac{1}{3}}}+\frac{G_0}{N(1-\sqrt{\rho})}+\frac{G_0}{N^{\frac{1}{3}}}\\
        \leq{}& \frac{2(f(\vecx^{1})-f^*+G_0)}{N^{\frac{1}{3}}} +\frac{\gamma G_0}{N^{\frac{2}{3}}}+ \frac{G_0}{N(1-\sqrt{\rho})}\\
        ={}& \mathcal{O}(\frac{1}{N^{\frac{1}{3}}}).
    \end{align*}
\end{proof}

Given an $L$-layer ConvGNN, following \cite{vrgcn}, we directly assume that:
\begin{enumerate}
    \item the optimal value
    \begin{align*}
        \mathcal{L}^*=\inf\limits_{w,\theta^{1},\ldots,\theta^{L}} \mathcal{L}
    \end{align*}
    is bounded by $G>1$;

    \item the gradients of $\mathcal{L}$ with respect to parameters $w$ and $\theta^{l}$, i.e.,
    \begin{align*}
        \nabla_{w}\mathcal{L},\,\nabla_{\theta^{l}}\mathcal{L}
    \end{align*}
    are $\gamma$-Lipschitz for $\forall l\in[L]$.
\end{enumerate}

To show the convergence of LMC by Lemma \ref{prop:suff}, it suffices to show that
\begin{enumerate}[resume]
    \item there exists $G_1>0$ that does not depend on $\eta$ such that
    \begin{align*}
        &\mathbb{E}[\|\Delta_{w}^{k}\|_2^2]\leq G_1,\,\,\forall\,k\in\mathbb{N}^*,\\
        &\mathbb{E}[\|\Delta_{\theta^{l}}^{k}\|_2^2]\leq G_1,\,\,\forall\,l\in[L],\,k\in\mathbb{N}^*;
    \end{align*}

    \item for any $N\in\mathbb{N}^*$, there exist $G_2>0$ and $\rho\in(0,1)$ such that for any $k\in\mathbb{N}^*$ and $l\in[L]$ we have
    \begin{align*}
        &|\mathbb{E}[\langle \nabla_{w}\loss, \Delta_{w}^{k} \rangle]| \leq G_2(\eta^{\frac{1}{2}} + \rho^{\frac{k-1}{2}}+\frac{1}{N^{\frac{1}{3}}}),\\
        &|\mathbb{E}[\langle \nabla_{\theta^{l}}\loss, \Delta_{\theta^{l}}^{k} \rangle]| \leq G_2(\eta^{\frac{1}{2}} + \rho^{\frac{k-1}{2}}+\frac{1}{N^{\frac{1}{3}}})
    \end{align*}
    by letting
    \begin{align*}
        \eta \leq \frac{1}{(2\gamma)^L G}\frac{1}{N^{\frac{2}{3}}} = \mathcal{O}(\frac{1}{N^{\frac{2}{3}}})
    \end{align*}
    and
    \begin{align*}
        \beta_i \leq \frac{1}{2G}\frac{1}{N^{\frac{2}{3}}}=\mathcal{O}(\frac{1}{N^{\frac{2}{3}}}),\,\,i\in[n].
    \end{align*}
\end{enumerate}

\begin{lemma}\label{prop:cond3}
    Suppose that Assumption \ref{assmp:proof} holds, then
    \begin{align*}
        &\mathbb{E}[\|\Delta_{w}^{k}\|_2^2] \leq G_1 \triangleq 4G^2,\,\,\forall\,k\in\mathbb{N}^*,\\
        &\mathbb{E}[\|\Delta_{\theta^{l}}^{k}\|_2^2] \leq G_1 \triangleq 4G^2, \,\,\forall\,l\in[L],\,k\in\mathbb{N}^*.
    \end{align*}
\end{lemma}
\begin{proof}
    We have
    \begin{align*}
        \mathbb{E}[\|\Delta_{w}^{k}\|_2^2] ={}& \mathbb{E}[\|\widetilde{\mathbf{g}}_w(w^k) - \nabla_{w}\loss(w^k)\|_2^2]\\
        \leq{}& 2(\mathbb{E}[\|\widetilde{\mathbf{g}}_w(w^k)\|_2^2] + \mathbb{E}[\|\nabla_{w}\loss(w^k)\|_2^2])\\
        \leq{}& 4G^2
    \end{align*}
    and
    \begin{align*}
        \mathbb{E}[\|\Delta_{\theta^{l}}^{k}\|_2^2]={}& \mathbb{E}[\|\widetilde{\mathbf{g}}_{\theta^{l}}(\theta^{l,k}) - \nabla_{\theta^{l}}\loss(\theta^{l,k})\|_2^2]\\
        \leq{}& 2(\mathbb{E}[\|\widetilde{\mathbf{g}}_{\theta^{l}}(\theta^{l,k})\|_2^2] + \mathbb{E}[\|\nabla_{\theta^{l}}\loss(\theta^{l,k})\|_2^2])\\
        \leq{}& 4G^2.
    \end{align*}
\end{proof}

\begin{lemma}\label{prop:cond4} \label{prop:inner_product_w}
    Suppose that Assumption \ref{assmp:proof} holds. For any $N\in\mathbb{N}^*$, there exist $G_2>0$ and $\rho\in(0,1)$ such that
    \begin{align*}
        &|\mathbb{E}[\langle \nabla_{w}\loss, \Delta_{w}^{k} \rangle]| \leq G_2(\eta^{\frac{1}{2}} + \rho^{\frac{k-1}{2}}+\frac{1}{N^{\frac{1}{3}}}),\,\,\forall\,k\in\mathbb{N}^*,\\
        &|\mathbb{E}[\langle \nabla_{\theta^{l}}\loss, \Delta_{\theta^{l}}^{k} \rangle]| \leq G_2(\eta^{\frac{1}{2}} + \rho^{\frac{k-1}{2}}+\frac{1}{N^{\frac{1}{3}}}),\,\,\forall\,l\in[L],\,k\in\mathbb{N}^*
    \end{align*}
    by letting
    \begin{align*}
        \eta \leq \frac{1}{(2\gamma)^L G}\frac{1}{N^{\frac{2}{3}}} = \mathcal{O}(\frac{1}{N^{\frac{2}{3}}})
    \end{align*}
    and
    \begin{align*}
        \beta_i \leq \frac{1}{2G}\frac{1}{N^{\frac{2}{3}}}=\mathcal{O}(\frac{1}{N^{\frac{2}{3}}}),\,\,i\in[n].
    \end{align*}
\end{lemma}
\begin{proof}
    By Eqs. \eqref{eqn:bias_bound_w_conv} and \eqref{eqn:bias_bound_theta_conv} we know that there exists $G_{2,*}$ such that for any $k\in\mathbb{N}^*$ we have
    \begin{align*}
        &\mathbb{E}[\|\widetilde{\mathbf{g}}_w(w^k) - \mathbf{g}_w(w^k)\|_2]\leq G_{2,*}(\eta^{\frac{1}{2}} + \rho^{\frac{k-1}{2}}+\frac{1}{N^{\frac{1}{3}}})
    \end{align*}
    and
    \begin{align*}
        &\mathbb{E}[\|\widetilde{\mathbf{g}}_{\theta^l}(\theta^{l,k}) - \mathbf{g}_{\theta^l}(\theta^{l,k})\|_2]\leq G_{2,*}(\eta^{\frac{1}{2}} + \rho^{\frac{k-1}{2}}+\frac{1}{N^{\frac{1}{3}}}),
    \end{align*}
    where $\rho=\frac{n-S}{n}<1$ is a constant. Hence
    \begin{align*}
        |\mathbb{E}[\langle \nabla_{w}\loss, \Delta_{w}^{k} \rangle]|={}& |\mathbb{E}[\langle \nabla_{w}\loss, \widetilde{\mathbf{g}}_w(w^k) - \nabla_w\loss(w^k) \rangle]|\\
        ={}& |\mathbb{E}[\langle \nabla_{w}\loss, \widetilde{\mathbf{g}}_w(w^k) - \mathbf{g}_w(w^k) \rangle]|\\
        \leq{}& \mathbb{E}[\|\nabla_{w}\loss\|_2 \|\widetilde{\mathbf{g}}_w(w^k) - \mathbf{g}_w(w^k) \|_2]\\
        \leq{}& G\mathbb{E}[\|\widetilde{\mathbf{g}}_w(w^k) - \mathbf{g}_w(w^k)\|_2],\\
        \leq{}&G_{2}(\eta^{\frac{1}{2}} + \rho^{\frac{k-1}{2}}+\frac{1}{N^{\frac{1}{3}}})
    \end{align*}
    and
    \begin{align*}
        |\mathbb{E}[\langle \nabla_{\theta^l}\loss, \Delta_{\theta^l}^{k} \rangle]| ={}& |\mathbb{E}[\langle \nabla_{\theta^l}\loss, \widetilde{\mathbf{g}}_{\theta^l}(\theta^{l,k}) - \nabla_{\theta^l} \loss(\theta^{l,k}) \rangle]|\\
        ={}& |\mathbb{E}[\langle \nabla_{\theta^l}\loss, \widetilde{\mathbf{g}}_{\theta^l}(\theta^{l,k}) - \mathbf{g}_{\theta^l}(\theta^{l,k}) \rangle]|\\
        \leq{}& \mathbb{E}[\|\nabla_{\theta^l}\loss\|_2 \|\widetilde{\mathbf{g}}_{\theta^l}(\theta^{l,k}) - \mathbf{g}_{\theta^l}(\theta^{l,k}) \|_2]\\
        \leq{}& G\mathbb{E}[\|\widetilde{\mathbf{g}}_{\theta^l}(\theta^{l,k}) - \mathbf{g}_{\theta^l}(\theta^{l,k})\|_2]\\
        \leq{}&G_{2}(\eta^{\frac{1}{2}} + \rho^{\frac{k-1}{2}}+\frac{1}{N^{\frac{1}{3}}}),
    \end{align*}
    where $G_{2}=GG_{2,*}$. 
\end{proof}

According to Lemmas \ref{prop:cond3} and \ref{prop:cond4}, the conditions in Lemma \ref{prop:suff} hold. By letting 
\begin{align*}
    \varepsilon ={}& \left(\frac{2(f(\vecx^{1})-f^*+G_0)}{N^{\frac{1}{3}}} +\frac{\gamma G_0}{N^{\frac{2}{3}}}+ \frac{G_0}{N(1-\sqrt{\rho})}\right)^{\frac{1}{2}}\\
    ={}& \mathcal{O}(\frac{1}{N^{\frac{1}{6}}}),
\end{align*}
we know that Theorem \ref{thm:convergence_rec} follows immediately.

\subsection{Proofs for LMC4Rec}\label{appendix:proof_rec}

In this section, we suppose that Assumption \ref{assmp:proof_rec} holds.

\subsubsection{Approximation errors of historical node embeddings and auxiliary variables}

In this subsection, we show that the approximation errors of historical node embeddings and auxiliary variables
\begin{align*}
    &d^{\Diamond,k}_h :=\left(\mathbb{E}[\|\hisH^{\Diamond,k} - \embH^{\Diamond, k}\|_F^2]\right)^{\frac{1}{2}},\\
    &d^{\Diamond,k}_v :=\left(\mathbb{E}[\|\hisV^{\Diamond,k} - \embV^{\Diamond, k}\|_F^2]\right)^{\frac{1}{2}}
\end{align*}
converge to $0$ during the training of LMC4Rec by the following theorems.

\begin{theorem}\label{thm:convergence_h_rec}
    Suppose that Assumption \ref{assmp:proof_rec} holds. Besides, we suppose that $d_h^{\Diamond,1}$ is bounded by $G$, then we have
    \begin{align*}
        d_h^{\Diamond, k+1} \leq \rho^kG + \frac{KG}{1-\rho} \eta,
    \end{align*}
    where {$\rho = \sqrt{(1-(1-\gamma^2)/B)}<1$}, {$K = \frac{2\gamma}{1-\gamma}$}, and $B$ is the number of partition subgraphs.
\end{theorem}

\begin{theorem}\label{thm:convergence_v_rec}
    Suppose that Assumption \ref{assmp:proof_rec} holds. Besides, we suppose that $d_v^{\Diamond,1}$ is bounded by $G$, then we have
    \begin{align*}
        d_v^{\Diamond, k+1} \leq \rho^kG + \frac{KG}{1-\rho} \eta,
    \end{align*}
    where {$\rho = \sqrt{(1-(1-\gamma^2)/B)}<1$}, {$K = \frac{2\gamma}{1-\gamma}$}, and $B$ is the number of partition subgraphs.
\end{theorem}

To show Theorems \ref{thm:convergence_h_rec} and \ref{thm:convergence_v_rec}, we introduce some useful lemmas.

\begin{lemma}\label{lemma:cfpi}
    Let the transition function $f:\mathbb{R}^d \rightarrow \mathbb{R}^d$ be $\gamma$-contraction. Consider the block-coordinate fixed-point algorithm with the partition of coordinates to $n$ blocks $\mathbf{I} = (S_1,S_2,\dots,S_n) \in \mathbb{R}^{d \times d}$ and an update rule
    \begin{align*}
        \vecx^{k+1} = (\mathbf{I}- \alpha \vecM^{k}) \vecx^{k} + \alpha \vecM^{k} f(\vecx^{k}),
    \end{align*}
    where $ \alpha \in (0,1]$ and the stochastic matrix $\vecM^{k}$ chosen independently and uniformly from $\{(0,\dots,S_j,\dots,0)|j=1,\dots n\}$ indicates the updated coordinates at the $k$-th iteration. Then, we have
    \begin{align*}
        \mathbb{E}[\| \vecx^{k+1} - \vecx   \|_2^2] \leq \left( 1 - \frac{\alpha}{n}(1-\gamma^2) \right)\mathbb{E}[\| \vecx^{k} - \vecx   \|_2^2].
    \end{align*}
    Moreover,
    \begin{align*}
        \mathbb{E}[\| \vecx^{k+1} - \vecx   \|_2^2] \leq \left( 1 - \frac{\alpha}{n}(1-\gamma^2) \right)^{k}\| \vecx^{1} - \vecx   \|_2^2.
    \end{align*}
\end{lemma}

\begin{proof}
    As $\vecM^{k}$ is chosen uniformly from the set
    \begin{align*}
        \{(0,\dots,S_i,\dots,0)|i=1,\dots n\},
    \end{align*}
    we have
    \begin{align*}
        \mathbb{E}[\vecM^{k}] = \frac{\mathbf{I}}{n}.
    \end{align*}
    We first compute the conditional expectation
    \begin{align*}
        &\mathbb{E}[ \| \vecx^{k+1} - \vecx   \|_2^2 | \vecx^{k}]\\
        ={}& \mathbb{E}[ \| \vecx^{k+1} - \vecx^{k} + \vecx^{k} - \vecx   \|_2^2 | \vecx^{k}] \\
        ={}& \mathbb{E}[ \| \vecx^{k+1} - \vecx^{k} \|_2^2 | \vecx^{k}] + \| \vecx^{k} - \vecx   \|_2^2 \\
        &+ 2 \langle \mathbb{E}[\vecx^{k+1}| \vecx^{k}]  - \vecx^{k} ,  \vecx^{k} - \vecx   \rangle \\
        ={}& \| \vecx^{k} - \vecx   \|_2^2 + \alpha^2 \mathbb{E}[ \| \vecM^{k}(\vecx^{k} - f(\vecx^{k})) \|_2^2 | \vecx^{k}] \\
        &+ 2 \frac{\alpha}{n} \langle f(\vecx^{k})  - \vecx^{k} ,  \vecx^{k} - \vecx   \rangle.
    \end{align*}
    Notice that 
    \begin{align*}
        &\mathbb{E}[ \| \vecM^{k}(\vecx^{k} - f(\vecx^{k})) \|_2^2 | \vecx^{k}]\\
        &{}= \frac{1}{n} \sum_{j=1}^n \| (0,\dots,S_i,\dots,0)( \vecx^{k} - f(\vecx^{k})) \|_2^2\\
        &{}=\frac{1}{n} \| \vecx^{k} - f(\vecx^{k})\|_2^2
    \end{align*}
    and
    \begin{align*}
        &2 \langle f(\vecx^{k})  - \vecx^{k} ,  \vecx^{k} - \vecx   \rangle\\
        ={}&\|f(\vecx^{k}) - \vecx   \|_2^2 - \|f(\vecx^{k})  - \vecx^{k}\|_2^2 - \|\vecx^{k} - \vecx  \|_2^2.
    \end{align*}
    Combining the above equalities and the contraction property of $f$, we have
    \begin{align*}
        &\mathbb{E}[ \| \vecx^{k+1} - \vecx   \|_2^2 | \vecx^{k}]\\
        ={}& (1-\frac{\alpha}{n})\| \vecx^{k} - \vecx   \|_2^2 - \alpha \frac{1-\alpha}{n} \| \vecx^{k} - f(\vecx^{k})\|_2^2\\
        &+\frac{\alpha}{n} (\|f(\vecx^{k}) - \vecx   \|_2^2) \\
        \leq{}& (1-\frac{\alpha}{n})\| \vecx^{k} - \vecx   \|_2^2 - \alpha \frac{1-\alpha}{n} \| \vecx^{k} - f(\vecx^{k})\|_2^2\\
        &+\frac{\alpha \gamma^2}{n} (\|\vecx^{k} - \vecx   \|_2^2) \\
        \leq{} &(1-\frac{\alpha}{n}(1-\gamma^2)) \| \vecx^{k} - f(\vecx^{k})\|_2^2.
    \end{align*}
    By the law of total expectation, we have
    \begin{align*}
        \mathbb{E}[ \| \vecx^{k+1} - \vecx   \|_2^2] \leq (1-\frac{\alpha}{n}(1-\gamma^2) \mathbb{E}[ \| \vecx^{k} - \vecx   \|_2^2 ].
    \end{align*}
    Finally, we recursively deduce that
    \begin{align*}
        \mathbb{E}[\| \vecx^{k+1} - \vecx   \|_2^2] \leq \left( 1 - \frac{\alpha}{n}(1-\gamma^2) \right)^{k} \| \vecx^{1} - \vecx   \|_2^2.
    \end{align*}
\end{proof}

By Lemma \ref{lemma:cfpi}, we can deduce that $d^{\Diamond,k}_h$ and $d^{\Diamond,k}_v$ decrease if the parameters $\theta^{\Diamond},w$ change slow. Moreover, we can set the learning rate $\eta$ to be a such small value that the parameters $\thetarec,w$ change slowly.

\begin{lemma}\label{lemma:amp_cfpi}
    Suppose that Assumption \ref{assmp:proof_rec} holds. If
    \begin{align*}
        \|\theta^{\Diamond,k+1}-\theta^{\Diamond,k}\|_F \leq \varepsilon
    \end{align*}
    at the $k$-th iteration, then we have
    \begin{align*}
        d^{\Diamond,k+1}_{h} \leq \rho d^{\Diamond,k}_h + K \varepsilon,
    \end{align*}
    where $\rho = \sqrt{(1-\frac{1}{B}(1-\gamma^2))}$ and $K = \frac{\gamma}{1-\gamma}$.
\end{lemma}

\begin{proof}
    According to Lemma \ref{lemma:cfpi}, we have
    \begin{align*}
        &\mathbb{E}[\|\hisH^{\Diamond, k+1} - \embH^{\Diamond,k+1}\|_F^2]\\
        \leq{}& \rho^2\mathbb{E}[\|\hisH^{\Diamond, k} - \embH^{\Diamond,k+1}\|_F^2] \\
        \leq{}& \rho^2\mathbb{E}[(\|\hisH^{\Diamond, k} - \embH^{\Diamond,k}\|_F + \|\embH^{\Diamond,k} - \embH^{\Diamond,k+1}\|_F )^2].
    \end{align*}
    As the transition function $f_{\theta}$ is $\gamma$-Lipschitz, we have
    \begin{align*}
        &\|\embH^{\Diamond,k} - \embH^{\Diamond,k+1}\|_F\\
        ={}& \|f_{\theta^{\Diamond,k}}(\embH^{\Diamond,k}) - f_{\theta^{\Diamond,k+1}}\embH^{\Diamond,k+1}\|_F \\
        \leq{}& \|f_{\theta^{\Diamond,k}}(\embH^{\Diamond,k}) - f_{\theta^{\Diamond,k}}\embH^{\Diamond,k+1}\|_F\\
        &+ \|f_{\theta^{\Diamond,k}}(\embH^{\Diamond,k+1}) - f_{\theta^{\Diamond,k+1}}\embH^{\Diamond,k+1}\|_F\\
        \leq{}& \gamma \|\embH^{\Diamond,k} - \embH^{\Diamond,k+1}\|_F + L\|\theta^{\Diamond,k+1}-\theta^{\Diamond,k}\|_F.
    \end{align*}
    By rearranging the terms, we have
    \begin{align*}
        \|\embH^{\Diamond,k} - \embH^{\Diamond,k+1}\|_F &\leq \frac{\gamma}{1-\gamma} \|\theta^{\Diamond,k+1}-\theta^{\Diamond,k}\|_F\\
        &\leq K \varepsilon.
    \end{align*}
    Combining the above inequalities, we have
    \begin{align*}
        &\mathbb{E}[\|\hisH^{\Diamond, k+1} - \embH^{\Diamond,k+1}\|_F^2]\\
        \leq{}& \rho^2 \mathbb{E}[(\|\hisH^{\Diamond, k} - \embH^{\Diamond,k}\|_F + K \varepsilon)^2]\\
        ={}& \rho^2 (\mathbb{E}[\|\hisH^{\Diamond, k} - \embH^{\Diamond,k}\|_F^2] + 2 \mathbb{E}[\|\hisH^{\Diamond, k} - \embH^{\Diamond,k}\|_F] K \varepsilon\\
        &+ (K \varepsilon)^2)\\
        \leq{} & \rho^2 (\mathbb{E}[\|\hisH^{\Diamond, k} - \embH^{\Diamond,k}\|_F^2] + 2 \sqrt{\mathbb{E}[\|\hisH^{\Diamond, k} - \embH^{\Diamond,k}\|_F^2} K \varepsilon\\
        &+ (K \varepsilon)^2)\\
        ={}& \rho^2 (\sqrt{\mathbb{E}[\|\hisH^{\Diamond, k} - \embH^{\Diamond,k}\|_F^2} + K\varepsilon)^2.
    \end{align*}
    The claims follows immediately.
\end{proof}

\begin{lemma}\label{lemma:amp_cfpi_b}
    Suppose that Assumption \ref{assmp:proof_rec} holds. If
    \begin{align*}
        &\|\theta^{\Diamond,k+1}-\theta^{\Diamond,k}\|_F \leq \varepsilon,\\
        &\|w^{k+1}-w^{k}\|_F \leq \varepsilon
    \end{align*}
    at the $k$-th iteration, then we have
    \begin{align*}
        d^{\Diamond,k+1}_{v} \leq \rho d^{\Diamond,k}_v + K \varepsilon,
    \end{align*}
    where $\rho = \sqrt{(1-\frac{1}{B}(1-\gamma^2))}$ and $K = \frac{2\gamma}{1-\gamma}$.
\end{lemma}

\begin{proof}
    According to Lemma \ref{lemma:cfpi}, we have
    \begin{align*}
        &\mathbb{E}[\|\hisV^{\Diamond, k+1} - \embV^{\Diamond, k+1}\|_F^2] \\
        \leq{}& \rho^2\mathbb{E}[\|\hisV^{\Diamond, k} - \embV^{\Diamond, k+1}\|_F^2] \\
        \leq{}& \rho^2\mathbb{E}[(\|\hisV^{\Diamond, k} - \embV ^{\Diamond, k}\|_F + \|\embV ^{\Diamond, k} - \embV^{\Diamond, k+1}\|_F )^2].
    \end{align*}
    As the transition function $f_{\theta}$ is  $L$-Lipschitz, we have
    \begin{align*}
        &\|\embV ^{\Diamond, k} - \embV^{\Diamond, k+1}\|_F \\
        ={}& \|\phi_{\theta^{\Diamond,k},w^k}(\embV ^{\Diamond, k}) - \phi_{\theta^{\Diamond,k+1},w^{k+1}}\embV^{\Diamond, k+1}\|_F \\
        \leq{}& \|\phi_{\theta^{\Diamond,k},w^k}(\embV ^{\Diamond, k}) - \phi_{\theta^{\Diamond,k},w^k}\embV^{\Diamond, k+1}\|_F \\
        &+ \|\phi_{\theta^{\Diamond,k},w^k}(\embV^{\Diamond, k+1}) - \phi_{\theta^{\Diamond,k+1},w^{k+1}}\embV^{\Diamond, k+1}\|_F\\
        \leq{}& \gamma \|\embV ^{\Diamond, k} - \embV^{\Diamond, k+1}\|_F\\ 
        &+ \gamma\sqrt{\|\theta^{\Diamond,k+1}-\theta^{\Diamond,k}\|_F^2+ \|w^{k+1}-w^k\|_F^2}\\
        \leq{}& \gamma \|\embV ^{\Diamond, k} - \embV^{\Diamond, k+1}\|_F\\
        &+ \gamma(\|\theta^{\Diamond,k+1}-\theta^{\Diamond,k}\|_F+ \|w^{k+1}-w^k\|_F).
    \end{align*}
    By rearranging the terms, we have
    \begin{align*}
        &\|\embV ^{\Diamond, k} - \embV^{\Diamond, k+1}\|_F\\
        \leq{}& \frac{\gamma}{1-\gamma} (\|\theta^{\Diamond,k+1}-\theta^{\Diamond,k}\|_F+ \|w^{k+1}-w^k\|_F)\\
        \leq{} &K \varepsilon.
    \end{align*}
    Combining the above inequalities, we have
    \begin{align*}
        &\mathbb{E}[\|\hisV^{\Diamond, k+1} - \embV^{\Diamond, k+1}\|_F^2]\\
        \leq{}& \rho^2 \mathbb{E}[(\|\hisV^{\Diamond, k} - \embV ^{\Diamond, k}\|_F + K \varepsilon)^2]\\
        ={}& \rho^2 (\mathbb{E}[\|\hisV^{\Diamond, k} - \embV ^{\Diamond, k}\|_F^2] + 2 \mathbb{E}[\|\hisV^{\Diamond, k} - \embV ^{\Diamond, k}\|_F] K \varepsilon \\ 
        &+ (K \varepsilon)^2)\\
        \leq{} & \rho^2 (\mathbb{E}[\|\hisV^{\Diamond, k} - \embV ^{\Diamond, k}\|_F^2] + 2 \sqrt{\mathbb{E}[\|\hisV^{\Diamond, k} - \embV ^{\Diamond, k}\|_F^2} K \varepsilon  \\
        &+ (K \varepsilon)^2)\\
        ={}& \rho^2 (\sqrt{\mathbb{E}[\|\hisV^{\Diamond, k} - \embV ^{\Diamond, k}\|_F^2} + K\varepsilon)^2.
    \end{align*}
    The claims follows immediately.
\end{proof}

Next, we give the proofs of Theorems \ref{thm:convergence_h_rec} and \ref{thm:convergence_v_rec}.

\begin{proof}
    The update rule in LMC4Rec implies that
    \begin{align*}
        &\|\theta^{\Diamond,k+1} - \theta^{\Diamond,k}\|_F =  \eta \|\widetilde{\mathbf{g}}_{\thetarec}(\theta^{\Diamond,k})\|_F \leq \eta G,\\
        &\|w^{k+1} - w^{k}\|_F =  \eta \|\widetilde{\mathbf{g}}_{w}(w^{k})\|_F \leq \eta G.
    \end{align*}
    According to Lemmas \ref{lemma:amp_cfpi} and \ref{lemma:amp_cfpi_b}, we have
    \begin{align*}
        d_h^{\Diamond, k+1} - \frac{KG}{1-\rho} \eta \leq{}& \rho( d_h^{\Diamond,k} - \frac{KG}{1-\rho} \eta)\\
        \leq{} &\cdots\\
        \leq{} & \rho^k (  d_h^{\Diamond,1} - \frac{KG}{1-\rho} \eta)\\
        \leq{} &\rho^k G
    \end{align*}
    and
    \begin{align*}
        d_v^{\Diamond, k+1} - \frac{KG}{1-\rho} \eta \leq{}& \rho( d_v^{\Diamond,k} - \frac{KG}{1-\rho} \eta)\\
        \leq{} &\cdots\\
        \leq{} & \rho^k (  d_v^{\Diamond,1} - \frac{KG}{1-\rho} \eta)\\
        \leq{} &\rho^k G.
    \end{align*}
    Then the claims follow immediately.
\end{proof}

\subsubsection{Proof of Theorem \ref{thm:grad_error_rec}: approximation errors of mini-batch gradients}

In this subsection, we focus on the mini-batch gradients computed by LMC4Rec, i.e.,
\begin{align*}
    \widetilde{\embg}_w(w^k;\mathcal{V}_{\mathcal{B}}^k)=\frac{1}{|\mathcal{V}_{L}^k|} \sum_{v_j\in\mathcal{V}_{L}^k} \nabla_w \ell_{w^k}(\temh^{k}_j, y_j)
\end{align*}
and
\begin{align*}
    &\widetilde{\embg}_{\theta^{\Diamond}}(\theta^{\Diamond,k};\mathcal{V}_{\mathcal{B}}^k)\\
    ={}& \frac{|\mathcal{V}|}{|\mathcal{V}_{\mathcal{B}}^k|}\sum_{v_j\in\mathcal{V}_{\mathcal{B}}^k} (\nabla_{\theta^{\Diamond}} u_{\theta^{\Diamond,k}}(\temh_j^{\Diamond,k}, \widehat{\mathbf{m}}_{\neighbor{v_j}}^{\Diamond,k}, \embx_j))\temV^{\Diamond,k}_j,
\end{align*}
where $\mathcal{V}_{\mathcal{B}}^k$ is the sampled mini-batch and $\mathcal{V}_{L_\mathcal{B}}^k$ is the corresponding labeled node set at the $k$-th iteration. We denote the mini-batch gradients computed by backward SGD by
\begin{align*}
    \embg_w(w^k;\mathcal{V}_{\mathcal{B}}^k)=\frac{1}{|\mathcal{V}_{L}^k|} \sum_{v_j\in\mathcal{V}_{L}^k} \nabla_w \ell_{w^k}(\embh^{k}_j, y_j)
\end{align*}
and
\begin{align*}
    &\embg_{\theta^{\Diamond}}(\theta^{\Diamond,k};\mathcal{V}_{\mathcal{B}}^k)\\
    ={}& \frac{|\mathcal{V}|}{|\mathcal{V}_{\mathcal{B}}^k|}\sum_{v_j\in\mathcal{V}_{\mathcal{B}}^k} (\nabla_{\theta^{\Diamond}} u_{\theta^{\Diamond,k}}(\embh_j^{\Diamond,k}, \embm_{\neighbor{v_j}}^{\Diamond,k}, \embx_j))\embV^{\Diamond,k}_j.
\end{align*}
In this subsection, we omit the sampled subgraph $\mathcal{V}_{\mathcal{B}}^k$ and simply write the mini-batch gradients as $\widetilde{\embg}_w(w^k)$, $\widetilde{\embg}_{\theta^{\Diamond}}(\theta^{\Diamond,k})$, $\embg_w(w^k)$, and $\embg_{\theta^\Diamond}(\theta^{\Diamond,k})$.
The approximation errors of gradients are denoted by
\begin{align*}
    \Delta_{w}^{k} \triangleq \widetilde{\embg}_w(w^k) - \nabla_w \loss(w^k)
\end{align*}
and
\begin{align*}
    \Delta_{\theta^{\Diamond}}^{k} \triangleq \widetilde{\embg}_{\theta^{\Diamond}}(\theta^{\Diamond,k}) - \nabla_{\theta^{\Diamond}} \loss(\theta^{\Diamond,k}).
\end{align*}

\begin{lemma}\label{lemma:grad_error_bound_w_rec}
    Suppose that Assumption \ref{assmp:proof_rec} holds. For any $k\in\mathbb{N}^*$, the difference between $\widetilde{\mathbf{g}}_w(w^k)$ and $\mathbf{g}_w(w^k)$ can be bounded as
    \begin{align*}
        \|\widetilde{\mathbf{g}}_w(w^k) - \mathbf{g}_w(w^k)\|_2 \leq \frac{\gamma^2}{1-\gamma} \|\hisH^{\Diamond,k} - \embH^{\Diamond,k}\|_F.
    \end{align*}
\end{lemma}
\begin{proof}
    We have
    \begin{align*}
        \|\widetilde{\mathbf{g}}_w(w^k) - \mathbf{g}_w(w^k)\|_2 ={}& \frac{1}{|\mathcal{V}_L^k|}\|\sum_{v_j \in \mathcal{V}_L^k} \nabla_w \ell_{w^k}(\temh_j^{\Diamond,k}, y_j)\nonumber\\
        &\quad\quad\quad\quad\quad\quad- \nabla_w \ell_{w^k}(\embh_j^{\Diamond,k}, y_j)\|_2\nonumber\\
        \leq{}& \frac{1}{|\mathcal{V}_L^k|}\sum_{v_j \in \mathcal{V}_L^k} \|\nabla_w \ell_{w^k}(\temh_j^{\Diamond,k}, y_j)\nonumber\\
        &\quad\quad\quad\quad\quad\quad- \nabla_w \ell_{w^k}(\embh_j^{\Diamond,k}, y_j)\|_2\nonumber\\
        \leq{}& \frac{\gamma}{|\mathcal{V}_L^k|} \sum_{v_j \in \mathcal{V}_L^k} \|\temh^{\Diamond,k}_j - \embh^{\Diamond, k}_j\|_2\nonumber\\
        \leq{}& \frac{\gamma}{|\mathcal{V}_L^k|} \sum_{v_j \in \mathcal{V}_L^k} \|\temH^{\Diamond,k}_{\mathcal{V}_L^k} - \embH^{\Diamond, k}_{\mathcal{V}_L^k}\|_F\nonumber\\
        \leq{}& \frac{\gamma}{|\mathcal{V}_L^k|}\cdot |\mathcal{V}_L^k|\cdot  \|\temH^{\Diamond,k}_{\mathcal{V}_L^k} - \embH^{\Diamond, k}_{\mathcal{V}_L^k}\|_F\nonumber\\
        ={}& \gamma \cdot  \|\temH^{\Diamond,k}_{\mathcal{V}_L^k} - \embH^{\Diamond, k}_{\mathcal{V}_L^k}\|_F\nonumber\\
        \leq{}&\gamma \cdot  \|\temH^{\Diamond,k}_{\mathcal{V}^k} - \embH^{\Diamond, k}_{\mathcal{V}^k}\|_F\nonumber
    \end{align*}
    Besides, we have
    \begin{align*}
        \|\temH^{\Diamond,k}_{\mathcal{V}^k} - \embH^{\Diamond, k}_{\mathcal{V}^k}\|_F ={} &  \|f_{\theta^{\Diamond,k}}(\temH^{\Diamond,k}_{\mathcal{V}^k}, \hisH^{\Diamond,k}_{\mathcal{V}\setminus\mathcal{V}^k})\\
        &\quad\quad- f_{\theta^{\Diamond,k}}(\embH^{\Diamond,k}_{\mathcal{V}^k}, \embH^{\Diamond,k}_{\mathcal{V}\setminus\mathcal{V}^k})\|_F\\
        ={}&  \|f_{\theta^{\Diamond,k}}(\temH^{\Diamond,k}_{\mathcal{V}^k}, \hisH^{\Diamond,k}_{\mathcal{V}\setminus\mathcal{V}^k})\\
        &\quad\quad- f_{\theta^{\Diamond,k}}(\embH^{\Diamond,k}_{\mathcal{V}^k}, \hisH^{\Diamond,k}_{\mathcal{V}\setminus\mathcal{V}^k})\\
        &\quad\quad+ f_{\theta^{\Diamond,k}}(\embH^{\Diamond,k}_{\mathcal{V}^k}, \hisH^{\Diamond,k}_{\mathcal{V}\setminus\mathcal{V}^k})\\
        &\quad\quad- f_{\theta^{\Diamond,k}}(\embH^{\Diamond,k}_{\mathcal{V}^k}, \embH^{\Diamond,k}_{\mathcal{V}\setminus\mathcal{V}^k})\|_F\\
        \leq{}& \gamma \cdot \|\temH^{\Diamond,k}_{\mathcal{V}^k} - \embH^{\Diamond, k}_{\mathcal{V}^k}\|_F\\
        &+ \gamma \cdot \|\hisH^{\Diamond,k}_{\mathcal{V}\setminus\mathcal{V}^k} - \embH^{\Diamond,k}_{\mathcal{V}\setminus\mathcal{V}^k}\|_F,
    \end{align*}
    which leads to
    \begin{align*}
        \|\temH^{\Diamond,k}_{\mathcal{V}^k} - \embH^{\Diamond, k}_{\mathcal{V}^k}\|_F \leq{}& \frac{\gamma}{1-\gamma}\|\hisH^{\Diamond,k}_{\mathcal{V}\setminus\mathcal{V}^k} - \embH^{\Diamond,k}_{\mathcal{V}\setminus\mathcal{V}^k}\|_F\\
        \leq{}& \frac{\gamma}{1-\gamma}\|\hisH^{\Diamond,k} - \embH^{\Diamond,k}\|_F.
    \end{align*}
    Hence, we have
    \begin{align*}
        \|\widetilde{\mathbf{g}}_w(w^k) - \mathbf{g}_w(w^k)\|_2 \leq{} \frac{\gamma^2}{1-\gamma} \|\hisH^{\Diamond,k} - \embH^{\Diamond,k}\|_F.
    \end{align*}
\end{proof}

\begin{lemma}\label{lemma:grad_error_bound_theta_rec}
    Suppose that Assumption \ref{assmp:proof_rec} holds. For any $k\in\mathbb{N}^*$, the difference between $\widetilde{\mathbf{g}}_{\theta^{\Diamond}}(\theta^{{\Diamond},k})$ and $\mathbf{g}_{\theta^{\Diamond}}(\theta^{{\Diamond},k})$ can be bounded as
    \begin{align*}
        &\|\widetilde{\mathbf{g}}_{\theta^{\Diamond}}(\theta^{{\Diamond},k}) - \mathbf{g}_{\theta^{\Diamond}}(\theta^{{\Diamond},k})\|_2\\
        \leq{}& \frac{|\mathcal{V}| G\gamma}{1-\gamma} \|\hisV^{\Diamond,k}-\embV^{\Diamond,k}\|_F + \frac{|\mathcal{V}|G\gamma^2}{1-\gamma} \|\hisH^{\Diamond,k}-\embH^{\Diamond,k}\|_F.
    \end{align*}
\end{lemma}
\begin{proof}
    Similar to the proof of Lemma \ref{lemma:grad_error_bound_w_rec}, we have
    \begin{align*}
        \|\temV^{\Diamond,k}_{\mathcal{V}^k} - \embV^{\Diamond,k}_{\mathcal{V}^k}\|_F \leq \frac{\gamma}{1-\gamma} \|\hisV^{\Diamond,k}_{\mathcal{V}^k} - \embV^{\Diamond,k}_{\mathcal{V}^k}\|_F.
    \end{align*}
    As $\|\mathbf{A}\mathbf{a}-\mathbf{B}\mathbf{b}\|_2\leq \|\mathbf{A}\|_F\|\mathbf{a}-\mathbf{b}\|_2 + \|\mathbf{A}-\mathbf{B}\|_F\|\mathbf{b}\|_2$, we can bound $\|\widetilde{\mathbf{g}}_{\theta^{\Diamond}}(\theta^{\Diamond,k}) - \mathbf{g}_{\theta^\Diamond}(\theta^{\Diamond,k})\|_2$ by
    \begin{align*}
        &\|\widetilde{\mathbf{g}}_{\theta^{\Diamond}}(\theta^{\Diamond,k}) - \mathbf{g}_{\theta^\Diamond}(\theta^{\Diamond,k})\|_2\nonumber\\
        \leq{}& \frac{|\mathcal{V}|}{|\mathcal{V}_{\mathcal{B}}^k|}\sum_{v_i\in\mathcal{V}_{\mathcal{B}}^k}\|\left(\nabla_{\theta^{\Diamond}} u_{\theta^{\Diamond,k}}(\temh_j^{\Diamond,k}, \widehat{\mathbf{m}}_{\neighbor{v_j}}^{\Diamond,k}, \embx_j)\right)\temV_j^{\Diamond,k}\nonumber\\
        &\quad\quad\quad\quad\quad-\left(\nabla_{\theta^{\Diamond}} u_{\theta^{\Diamond,k}}(\embh_j^{\Diamond,k}, \mathbf{m}_{\neighbor{v_j}}^{\Diamond,k}, \embx_j)\right)\embV_j^{\Diamond,k}\|_2\nonumber\\
        \leq{}& |\mathcal{V}|\max_{v_i\in\mathcal{V}_{\mathcal{B}}^k} \|\left(\nabla_{\theta^{\Diamond}} u_{\theta^{\Diamond,k}}(\temh_j^{\Diamond,k}, \widehat{\mathbf{m}}_{\neighbor{v_j}}^{\Diamond,k}, \embx_j)\right)\temV_j^{\Diamond,k}\nonumber\\
        &\quad\quad\quad\quad\quad-\left(\nabla_{\theta^{\Diamond}} u_{\theta^{\Diamond,k}}(\embh_j^{\Diamond,k}, \mathbf{m}_{\neighbor{v_j}}^{\Diamond,k}, \embx_j)\right)\embV_j^{\Diamond,k}\|_2\nonumber\\
        \leq{}& |\mathcal{V}|\max_{v_i\in\mathcal{V}_{\mathcal{B}}^k} \|\nabla_{\theta^{\Diamond}} u_{\theta^{\Diamond,k}}(\temh_j^{\Diamond,k}, \widehat{\mathbf{m}}_{\neighbor{v_j}}^{\Diamond,k}, \embx_j)\|_F\|\temV_j^{\Diamond,k}-\embV_j^{\Diamond,k}\|_2\nonumber\\
        &\quad\quad\quad\quad+ \|\nabla_{\theta^{\Diamond}} u_{\theta^{\Diamond,k}}(\temh_j^{\Diamond,k}, \widehat{\mathbf{m}}_{\neighbor{v_j}}^{\Diamond,k}, \embx_j)\nonumber\\
        &\quad\quad\quad\quad\quad\quad-\nabla_{\theta^{\Diamond}} u_{\theta^{\Diamond,k}}(\embh_j^{\Diamond,k}, \mathbf{m}_{\neighbor{v_j}}^{\Diamond,k}, \embx_j)\|_F\|\embV_j^{\Diamond,k}\|_2\nonumber\\
        \leq{}& |\mathcal{V}| G \|\temV^{\Diamond,k}-\embV^{\Diamond,k}\|_F + |\mathcal{V}|G\gamma \|\temH^{\Diamond,k}-\embH^{\Diamond,k}\|_F\nonumber\\
        \leq{}& \frac{|\mathcal{V}| G\gamma}{1-\gamma} \|\hisV^{\Diamond,k}-\embV^{\Diamond,k}\|_F + \frac{|\mathcal{V}|G\gamma^2}{1-\gamma} \|\hisH^{\Diamond,k}-\embH^{\Diamond,k}\|_F.
    \end{align*}
\end{proof}

\begin{lemma}\label{lemma:grad_error}
    Suppose that Assumption \ref{assmp:proof_rec} holds, then there exists $C>0$ and $\rho\in(0,1)$ such that for any $k\in\mathbb{N}^*$ we have
    \begin{align*}
        &\mathbb{E}[\|\Delta_{w}^k\|_2^2] = ({\rm Bias}(\widetilde{\mathbf{g}}_w(w^k)))^2+{\rm Var}(\widetilde{\mathbf{g}}_w(w^k)),\\
        &\mathbb{E}[\|\Delta_{\theta^{\Diamond}}^k\|_2^2] = ({\rm Bias}(\widetilde{\mathbf{g}}_{\theta^{\Diamond}}(\theta^{\Diamond,k})))^2+{\rm Var}(\widetilde{\mathbf{g}}_{\theta^{\Diamond}}(\theta^{\Diamond,k})),
    \end{align*}
    where
    \begin{align*}
        &{\rm Var}(\widetilde{\mathbf{g}}_w(w^k)) = \mathbb{E}[\|\mathbb{E}[\widetilde{\mathbf{g}}_w(w^k)] - \widetilde{\mathbf{g}}_w(w^k)\|_2^2],\\
        &{\rm Bias}(\widetilde{\mathbf{g}}_w(w^k)) = \|\mathbb{E}[\widetilde{\mathbf{g}}_w(w^k)] - \nabla_w\mathcal{L}_w(w^k)\|_2,\\
        &{\rm Var}(\mathbf{g}_{\theta^{\Diamond}}(\theta^{\Diamond,k}))= \mathbb{E}[\|\mathbb{E}[\widetilde{\mathbf{g}}_{\theta^{\Diamond}}(\theta^{\Diamond,k})] - \widetilde{\mathbf{g}}_{\theta^{\Diamond}}(\theta^{\Diamond,k})\|_2^2],\\
        &{\rm Bias}(\widetilde{\mathbf{g}}_{\theta^{\Diamond}}(\theta^{\Diamond,k}))= \|\mathbb{E}[\widetilde{\mathbf{g}}_{\theta^{\Diamond}}(\theta^{\Diamond,k})] - \nabla_{\theta^{\Diamond}}\mathcal{L}_{\theta^{\Diamond}}(\theta^{\Diamond,k})\|_2
    \end{align*}
    and
    \begin{align*}
        &{\rm Bias}(\widetilde{\mathbf{g}}_w(w^k)) \leq C(\eta+\rho^{k-1}),\\
        &{\rm Bias}(\widetilde{\mathbf{g}}_{\theta^{\Diamond}}(\theta^{\Diamond,k})) \leq C(\eta+\rho^{k-1}).
    \end{align*}
\end{lemma}
\begin{proof}
    Similar to Lemma \ref{prop:grad_error}, the bias-variance decomposition equations of $\mathbb{E}[\|\Delta_{w}^k\|_2^2]$ and $\mathbb{E}[\|\Delta_{\theta^{\Diamond}}^k\|_2^2]$ are obvious. Next we give the upper bounds of the bias terms. By Lemma \ref{lemma:grad_error_bound_w_rec}, we have
    \begin{align}
        {\rm Bias}(\widetilde{\mathbf{g}}_w(w^k)) ={}& \|\mathbb{E}[\widetilde{\mathbf{g}}_w(w^k)] - \nabla_w\mathcal{L}_w(w^k)\|_2\nonumber\\
        ={}& \|\mathbb{E}[\widetilde{\mathbf{g}}_w(w^k) - \mathbf{g}_w(w^k)]\|_2\nonumber\\
        \leq{}& \frac{\gamma^2}{1-\gamma} \mathbb{E}[\|\hisH^{\Diamond,k} - \embH^{\Diamond,k}\|_F^2]\nonumber\\
        \leq{}& \frac{\gamma^2}{1-\gamma} \left(\mathbb{E}[\|\hisH^{\Diamond,k} - \embH^{\Diamond,k}\|_F^2]\right)^{\frac{1}{2}}\nonumber\\
        ={}& \frac{\gamma^2}{1-\gamma} d_h^{\Diamond, k}.\label{eqn:bias_bound_w_rec}
    \end{align}
    By Theorem \ref{thm:convergence_h_rec} we know that
    \begin{align*}
        d_h^{\Diamond, k} \leq \rho^{k-1}G + \frac{KG}{1-\rho}\eta.
    \end{align*}
    Hence we have
    \begin{align*}
        {\rm Bias}(\widetilde{\mathbf{g}}_w(w^k)) \leq \frac{\gamma^2G}{1-\gamma}\rho^{k-1} + \frac{KG\gamma^2}{(1-\rho)(1-\gamma)}\eta.
    \end{align*}
    By Lemma \ref{lemma:grad_error_bound_theta_rec}, we can bound ${\rm Bias}(\widetilde{\mathbf{g}}_{\theta^{\Diamond}}(\theta^{\Diamond,k}))$ as
    \begin{align}
        {\rm Bias}(\widetilde{\mathbf{g}}_{\theta^{\Diamond}}(\theta^{\Diamond,k}))={}& \|\mathbb{E}[\widetilde{\mathbf{g}}_{\theta^{\Diamond}}(\theta^{\Diamond,k})] - \nabla_{\theta^{\Diamond}}\mathcal{L}_{\theta^{\Diamond}}(\theta^{\Diamond,k})\|_2\nonumber\\
        ={}& \|\mathbb{E}[\widetilde{\mathbf{g}}_{\theta^{\Diamond}}(\theta^{\Diamond,k}) - \mathbf{g}_{\theta^{\Diamond}}(\theta^{\Diamond,k})]\|_2 \nonumber \\
        \leq{} & \frac{|\mathcal{V}| G\gamma}{1-\gamma} \mathbb{E}[\|\hisV^{\Diamond,k}-\embV^{\Diamond,k}\|_F]\nonumber\\
        &+ \frac{|\mathcal{V}|G\gamma^2}{1-\gamma} \mathbb{E}[\|\hisH^{\Diamond,k}-\embH^{\Diamond,k}\|_F]\nonumber\\
        \leq{} & \frac{|\mathcal{V}| G\gamma}{1-\gamma} \left(\mathbb{E}[\|\hisV^{\Diamond,k}-\embV^{\Diamond,k}\|_F^2]\right)^{\frac{1}{2}}\nonumber\\
        &+ \frac{|\mathcal{V}|G\gamma^2}{1-\gamma} \left(\mathbb{E}[\|\hisH^{\Diamond,k}-\embH^{\Diamond,k}\|_F^2]\right)^{\frac{1}{2}}\nonumber\\
        ={}& \frac{|\mathcal{V}| G\gamma}{1-\gamma} d_v^{\Diamond,k}+ \frac{|\mathcal{V}|G\gamma^2}{1-\gamma} d_h^{\Diamond,k} \label{eqn:bias_bound_theta_rec}.
    \end{align}
    By Theorems \ref{thm:convergence_h_rec} and \ref{thm:convergence_v_rec} we know that
    \begin{align*}
        &d_h^{\Diamond, k} \leq \rho^{k-1}G + \frac{KG}{1-\rho}\eta,\\
        &d_v^{\Diamond, k} \leq \rho^{k-1}G + \frac{KG}{1-\rho}\eta.
    \end{align*}
    Hence we have
    \begin{align*}
        {\rm Bias}(\widetilde{\mathbf{g}}_{\theta^{\Diamond}}(\theta^{\Diamond,k})) \leq{}& \frac{|\mathcal{V}|G^2\gamma(1+\gamma)}{1-\gamma} \rho^{k-1}\\
        &+ \frac{|\mathcal{V}|KG^2\gamma(1+\gamma)}{(1-\rho)(1-\gamma)} \eta.
    \end{align*}
    By letting
    \begin{align*}
        C = \max\{&\frac{\gamma^2G}{1-\gamma}, \frac{KG\gamma^2}{(1-\rho)(1-\gamma)},\\
        &\frac{|\mathcal{V}|G^2\gamma(1+\gamma)}{1-\gamma}, \frac{|\mathcal{V}|KG^2\gamma(1+\gamma)}{(1-\rho)(1-\gamma)}\}
    \end{align*}
    we have
    \begin{align*}
        &{\rm Bias}(\widetilde{\mathbf{g}}_w(w^k)) \leq C(\eta+\rho^{k-1}),\\
        &{\rm Bias}(\widetilde{\mathbf{g}}_{\theta^{\Diamond}}(\theta^{\Diamond,k})) \leq C(\eta+\rho^{k-1}).
    \end{align*}
\end{proof}
By letting $\eta = \mathcal{O}(\varepsilon)$, Theorem \ref{thm:grad_error_rec} follows immediately.

\subsubsection{Proof of Theorem \ref{thm:convergence_rec}: convergence guarantees}

In this subsection, we give the convergence guarantees of LMC4Rec. We first give sufficient conditions for the convergence.
\begin{lemma}\label{lemma:sgd}
    Suppose that function $f:\mathbb{R}^{n} \rightarrow \mathbb{R}$ is continuously differentiable. Consider a optimization algorithm with any bounded initialization $\vecx^{1}$ and an update rule in the form of
    \begin{align*}
        \vecx^{k+1} = \vecx^{k} - \eta \vecd(\vecx^{k}),
    \end{align*}
    where $\eta>0$ is the learning rate and $\vecd(\vecx^{k})$ is the estimated gradient that can be seen as a stochastic vector depending on $\vecx^{k}$. Let the estimation error of gradients be $\Delta^{k} = \vecd(\vecx^{k}) - \nabla f(\vecx^{k}) $.
    Suppose that
    \begin{enumerate}
        \item the optimal value $f^*  = \inf_{\vecx} f(\vecx)$ is bounded; \label{con:1_rec}
        \item the gradient of $f$ is $\gamma$-Lipschitz, i.e., \label{con:2_rec}
        \begin{align*}
            \|\nabla f(\vecy) - \nabla f(\vecx)\|_2 \leq \gamma\|\vecy - \vecx\|_2,\,\forall\,\vecx,\vecy \in \mathbb{R}^{n};
        \end{align*}
        \item there exists $G_0>0$ that does not depend on $\eta$ such that \label{con:3}
        \begin{align*}
            \mathbb{E}[\|\Delta^k\|_2^2] \leq G_0,\,\,\forall\,k\in\mathbb{N}^*;
        \end{align*}

        \item there exists $\rho \in (0,1)$ that does not depend on $\eta$ such that for any $k\in\mathbb{N}^*$ we have \label{con:4}
        \begin{align*}
            |\mathbb{E}[\langle \nabla f(\vecx^{k}), \Delta^{k} \rangle]| \leq \eta G_0 + \rho^{k-1} G_0,
        \end{align*}
        where $G_0$ is the same constant as that in Condition \ref{con:3},
    \end{enumerate}
    then by letting $\eta = \min\{\frac{1}{\gamma},\frac{1}{N^{\frac{1}{2}}}\}$, we have
    \begin{align*}
        &\mathbb{E}[ \|\nabla f(\vecx^{(R)})\|_2^2]\\
        \leq{}&  \frac{2(f(\vecx^{1}) - f^*)+(\gamma+1)G_0 }{ N^{\frac{1}{2}} }+ \frac{G_0}{(1-\rho)N}\\
        ={}& \mathcal{O}(\frac{1}{N^{\frac{1}{2}}})
    \end{align*}
     where $R$ is chosen uniformly from $[N]$.
\end{lemma}
\begin{proof}
    As the gradient of $f$ is $\gamma$-Lipschitz, we have
    \begin{align*}
    	f(\vecy)={}&f(\vecx)+\int_{\vecx}^{\vecy}\nabla f(\mathbf{z})d\mathbf{z}\\
    	={}&f(\vecx)+\int_0^1\langle\nabla f(\vecx+t(\vecy-\vecx)), \vecy-\vecx\rangle \rmd t\\
    	={}&f(\vecx)+\langle\nabla f(\vecx),\vecy-\vecx\rangle\\
    	&+\int_0^1\langle\nabla f(\vecx+t(\vecy-\vecx))-\nabla f(\vecx), \vecy-\vecx\rangle \rmd t\\
    	\leq{}&f(\vecx)+\langle\nabla f(\vecx),\vecy-\vecx\rangle\\
    	&+\int_0^1\|\nabla f(\vecx+t(\vecy-\vecx))-\nabla f(\vecx)\|\| \vecy-\vecx\| \rmd t\\
    	\leq{}&f(\vecx)+\langle\nabla f(\vecx),\vecy-\vecx\rangle+\int_0^1Lt\|\vecy-\vecx\|^2\rmd t\\
    	\leq{}&f(\vecx)+\langle\nabla f(\vecx), \vecy-\vecx\rangle+\frac{\gamma}{2}\|\vecy-\vecx\|^2,
    \end{align*}
    Then, we have
    \begin{align*}
        &f(\vecx^{k+1})\\
        \leq{}& f(\vecx^{k}) + \langle  \nabla f(\vecx^{k}), \vecx^{k+1} - \vecx^{k} \rangle+ \frac{\gamma}{2}\|\vecx^{k+1}-\vecx^{k}\|_2^2 \\
        = {}& f(\vecx^{k}) - \eta \langle \nabla f(\vecx^{k}), \vecd(\vecx^{k}) \rangle + \frac{\eta^2 \gamma}{2}\|\vecd(\vecx^{k})\|_2^2 \\
        = {}& f(\vecx^{k}) - \eta \langle \nabla f(\vecx^{k}), \Delta^{k} \rangle- \eta \|\nabla f(\vecx^{k})\|_2^2\\
        &+ \frac{\eta^2 \gamma}{2}(\|\Delta^{k}\|_2^2 +\|\nabla f(\vecx^{k})\|_2^2 +2\langle \Delta^{k}, \nabla f(\vecx^{k}) \rangle)\\
        = {}& f(\vecx^{k})  - \eta (1-\eta \gamma) \langle \nabla f(\vecx^{k}), \Delta^{k} \rangle\\
        &- \eta (1-\frac{\eta \gamma}{2}) \|\nabla f(\vecx^{k})\|_2^2 + \frac{\eta^2 \gamma}{2}\|\Delta^{k}\|_2^2.
    \end{align*}
    By taking expectation of both sides, we have
    \begin{align*}
        \mathbb{E}[f(\vecx^{k+1})] \leq{}& \mathbb{E}[f(\vecx^{k})] - \eta (1-\eta \gamma)  \mathbb{E}[\langle \nabla f(\vecx^{k}), \Delta^{k} \rangle]\\
        &- \eta (1-\frac{\eta \gamma}{2}) \mathbb{E}[ \|\nabla f(\vecx^{k})\|_2^2] + \frac{\eta^2 \gamma}{2}\mathbb{E}[\|\Delta^{k}\|_2^2].
    \end{align*}
    By summing up the above inequalities for $k=1,2,\ldots,N$ and dividing both sides by $N \eta(1-\frac{\eta \gamma}{2})$, we have
    \begin{align*}
        &\frac{\sum_{k=1}^{N} \mathbb{E}[ \|\nabla f(\vecx^{k})\|_2^2]}{N}\\
        \leq{}& \frac{f(\vecx^{1}) - \mathbb{E}[f(\vecx^{k})]}{N \eta(1-\frac{\eta \gamma}{2}) } - \frac{(1-\eta \gamma)}{(1-\frac{\eta \gamma}{2})} \frac{\sum_{k=1}^{N} \mathbb{E}[\langle \nabla f(\vecx^{k}), \Delta^{k} \rangle]}{N} \\
        &+ \frac{\eta \gamma}{2-\eta \gamma } \frac{\sum_{k=1}^{N} \mathbb{E}[\|\Delta^{k}\|_2^2]}{k} \\
        \leq{}& \frac{f(\vecx^{1}) - f^* }{N \eta(1-\frac{\eta \gamma}{2}) }  + \frac{\sum_{k=1}^{N} |\mathbb{E}[\langle \nabla f(\vecx^{k}), \Delta^{k} \rangle]|}{N}\\
        &+ \frac{\eta \gamma}{2-\eta \gamma } \frac{\sum_{k=1}^{N} \mathbb{E}[\|\Delta^{k}\|_2^2]}{N},
    \end{align*}
    where the second inequality comes from $\eta \gamma>0$ and $f(\vecx^{k}) \geq f^* $. According to the above conditions, we have
    \begin{align*}
        &\frac{\sum_{k=1}^{N} \mathbb{E}[ \|\nabla f(\vecx^{k})\|_2^2]}{N}\\
        \leq{}& \frac{f(\vecx^{1}) - f^* }{N \eta(1-\frac{\eta \gamma}{2}) } + \eta G_0 + \frac{G_0\sum_{k=1}^{N} \rho^{k-1}}{N} + \frac{\eta \gamma}{2-\eta \gamma} G_0\\
        \leq{}& \frac{f(\vecx^{1}) - f^* }{N \eta(1-\frac{\eta \gamma}{2}) } + \eta G_0 + \frac{G_0\sum_{k=1}^{\infty} \rho^{k-1}}{N} + \frac{\eta \gamma}{2-\eta \gamma } G_0\\
        \leq{}& \frac{f(\vecx^{1}) - f^* }{N \eta(1-\frac{\eta \gamma}{2}) } + \eta G_0 + \frac{G_0}{(1-\rho)N} + \frac{\eta \gamma}{2-\eta \gamma } G_0
    \end{align*}
    Notice that $\mathbb{E}[ \|\nabla f(\vecx^{(R)})\|_2^2] = \frac{\sum_{k=1}^{N} \mathbb{E}[ \|\nabla f(\vecx^{k})\|_2^2]}{N}$. By taking $\eta = \min\{\frac{1}{\gamma},\frac{1}{N^{\frac{1}{2}}}\}$ and rearranging the terms, we have
    \begin{align*}
         &\mathbb{E}[ \|\nabla f(\vecx^{(R)})\|_2^2]\\
         ={}& \frac{\sum_{k=1}^{N} \mathbb{E}[ \|\nabla f(\vecx^{k})\|_2^2]}{N} \\
         \leq{}& \frac{f(\vecx^{1}) - f^* }{N \eta(1-\frac{\eta \gamma}{2}) } + \frac{\eta \gamma G_0}{(2-\eta \gamma) } + \eta G_0 + \frac{G_0}{(1-\rho)N} \\
         \leq{}&  \frac{2(f(\vecx^{1}) - f^*)+(\gamma+1)G_0 }{ N^{\frac{1}{2}} } + \frac{G_0}{(1-\rho)N}\\
         ={}& \mathcal{O}(\frac{1}{N^{\frac{1}{2}}}).
    \end{align*}
\end{proof}

Given an RecGNN, following \cite{vrgcn}, we directly assume that:
\begin{enumerate}
    \item the optimal value
    \begin{align*}
        \loss^* = \inf_{w,\theta^{\Diamond}} \loss 
    \end{align*}
    is bounded by $G>0$;
    \item the gradients of $\loss$ with respect to parameters $w$ and $\theta^{\Diamond}$, i.e.,
    \begin{align*}
        \nabla_w \loss,\, \nabla_{\theta^{\Diamond}} \loss
    \end{align*}
    are $\gamma$-Lipschitz.
\end{enumerate}
To show the convergence of LMC4Rec by Lemma \ref{lemma:sgd}, it suffices to show that
\begin{enumerate}[resume]
    \item there exists $G_1>0$ that does not depend on $\eta$ such that
    \begin{align*}
        &\mathbb{E}[\|\Delta^k_w\|_2^2] \leq G_1,\,\,\forall\,k\in\mathbb{N}^*,\\
        &\mathbb{E}[\|\Delta^k_{\theta^{\Diamond}}\|_2^2] \leq G_1,\,\,\forall\,k\in\mathbb{N}^*;
    \end{align*}
    
    \item there exist $G_2>0$ and $\rho\in(0,1)$ that do not depend on $\eta$ such that for any $k\in\mathbb{N}^*$ we have
    \begin{align*}
        &|\mathbb{E}[\langle \nabla_w\loss, \Delta^k_w \rangle]| \leq G_2(\eta + \rho^{k-1}),\\
        &|\mathbb{E}[\langle \nabla_{\theta^{\Diamond}}\loss, \Delta^k_{\theta^{\Diamond}} \rangle]| \leq G_2(\eta + \rho^{k-1}).
    \end{align*}
\end{enumerate}

\begin{lemma}\label{lemma:cond3}
    Suppose that Assumption \ref{assmp:proof_rec} holds, then
    \begin{align*}
        &\mathbb{E}[\|\Delta^k_w\|_2^2] \leq G_1 \triangleq 4G^2,\,\,\forall\,k\in\mathbb{N}^*,\\
        &\mathbb{E}[\|\Delta^k_{\theta^{\Diamond}}\|_2^2] \leq G_1 \triangleq 4G^2,\,\,\forall\,k\in\mathbb{N}^*.
    \end{align*}
\end{lemma}
\begin{proof}
    We have
    \begin{align*}
        \mathbb{E}[\|\Delta^k_w\|_2^2]={}&\mathbb{E}[\| \widetilde{\mathbf{g}}_w(w^k) - \nabla_{w} \mathcal{L}(w^k) \|_2^2]\\
        \leq{}& 2( \mathbb{E}[ \| \widetilde{\mathbf{g}}_w(w^k)\|_2^2] +  \mathbb{E}[ \|\nabla_{w} \mathcal{L}(w^k)\|_2^2])\\
        \leq{}& 4 G^2
    \end{align*}
    and
    \begin{align*}
        \mathbb{E}[\|\Delta^k_{\theta^{\Diamond}}\|_2^2]={}&\mathbb{E}[\| \widetilde{\mathbf{g}}_{\theta^{\Diamond}}({\theta^{\Diamond,k}}) - \nabla_{\theta^{\Diamond}} \mathcal{L}({\theta^{\Diamond,k}}) \|_2^2]\\
        \leq{}& 2( \mathbb{E}[ \| \widetilde{\mathbf{g}}_{\theta^{\Diamond}}({\theta^{\Diamond,k}})\|_2^2] +  \mathbb{E}[ \|\nabla_{\theta^{\Diamond}} \mathcal{L}({\theta^{\Diamond,k}})\|_2^2])\\
        \leq{}& 4 G^2
    \end{align*}
\end{proof}

\begin{lemma}\label{lemma:cond4}
    Suppose that Assumption \ref{assmp:proof_rec} holds, then there exist $G_2>0$ and $\rho\in(0,1)$ that do not depend on $\eta$ such that for any $k\in\mathbb{N}^*$ we have
    \begin{align*}
        &|\mathbb{E}[\langle \nabla_w\loss, \Delta^k_w \rangle]| \leq G_2(\eta + \rho^{k-1}),\\
        &|\mathbb{E}[\langle \nabla_{\theta^{\Diamond}}\loss, \Delta^k_{\theta^{\Diamond}} \rangle]| \leq G_2(\eta + \rho^{k-1}).
    \end{align*}
\end{lemma}
\begin{proof}
    By Eqs. \eqref{eqn:bias_bound_w_rec} and \eqref{eqn:bias_bound_theta_rec} we know that there exist $C>0$ and $\rho\in(0,1)$ such that for any $k\in\mathbb{N}^*$ we have
    \begin{align*}
        &\mathbb{E}[\| \widetilde{\mathbf{g}}_w(w^k) - \mathbf{g}_w(w^k)\|_2]\leq C(\eta+\rho^{k-1}),
    \end{align*}
    and
    \begin{align*}
        &\mathbb{E}[\|\widetilde{\mathbf{g}}_{\theta^{\Diamond}}(\theta^{\Diamond,k}) - \mathbf{g}_{\theta^\Diamond}(\theta^{\Diamond,k})\|_2]\leq C(\eta+\rho^{k-1}),
    \end{align*}
    Hence we have
    \begin{align*}
        |\mathbb{E}[\langle \nabla_w\loss, \Delta^k_w \rangle]|={}& |\mathbb{E}[\langle \nabla_w\loss, \widetilde{\mathbf{g}}_w(w^k) - \nabla_w\loss(w^k) \rangle]|\\
        ={} &|\mathbb{E}[\langle \nabla_w\loss, \widetilde{\mathbf{g}}_w(w^k) - \mathbf{g}_w(w^k) \rangle]|\\
        \leq{}& \mathbb{E}[\|\nabla_w\loss\|_2\|\widetilde{\mathbf{g}}_w(w^k) - \mathbf{g}_w(w^k)\|_2]\\
        \leq{}& G \mathbb{E}[\| \widetilde{\mathbf{g}}_w(w^k) - \mathbf{g}_w(w^k)\|_2]\\
        \leq {}& G_2(\eta + \rho^{k-1})
    \end{align*}
    and
    \begin{align*}
        |\mathbb{E}[\langle \nabla_{\theta^{\Diamond}}\loss, \Delta^k_{\theta^{\Diamond}} \rangle]|={}& |\mathbb{E}[\langle \nabla_{\theta^{\Diamond}}\loss, \widetilde{\mathbf{g}}_{\theta^{\Diamond}}(\theta^{\Diamond,k}) - \nabla_{\theta^{\Diamond}}\loss(\theta^{\Diamond,k}) \rangle]|\\
        ={} &|\mathbb{E}[\langle \nabla_{\theta^{\Diamond}}\loss, \widetilde{\mathbf{g}}_{\theta^{\Diamond}}({\theta^{\Diamond,k}}) - \mathbf{g}_{\theta^{\Diamond}}({\theta^{\Diamond,k}}) \rangle]|\\
        \leq{}& \mathbb{E}[\|\nabla_{\theta^{\Diamond}}\loss\|_2\|\widetilde{\mathbf{g}}_{\theta^{\Diamond}}({\theta^{\Diamond,k}}) - \mathbf{g}_{\theta^{\Diamond}}({\theta^{\Diamond,k}})\|_2]\\
        \leq{}& G \mathbb{E}[\|\widetilde{\mathbf{g}}_{\theta^{\Diamond}}(\theta^{\Diamond,k}) - \mathbf{g}_{\theta^\Diamond}(\theta^{\Diamond,k})\|_2]\\
        \leq {}& G_2(\eta + \rho^{k-1}),
    \end{align*}
    where $G_2 = GC$.
\end{proof}

According to Lemmas \ref{lemma:cond3} and \ref{lemma:cond4}, the conditions in Lemma \ref{lemma:sgd} hold. By letting
\begin{align*}
    \varepsilon ={}& \left( \frac{2(f(\vecx^{1}) - f^*)+(\gamma+1)G_0 }{ N^{\frac{1}{2}} }+ \frac{G_0}{(1-\rho)N} \right)^{\frac{1}{2}}\\
    ={}& \mathcal{O}(\frac{1}{N^{\frac{1}{4}}}),
\end{align*}
we know that Theorem \ref{thm:convergence_rec} follows immediately.

\section{More Experiments}

\subsection{Experiments with ConvGNNs}

\subsubsection{Performance on Small Datasets}

Figure \ref{fig:runtime_small} reports the convergence curves GD, GAS, and LMC4Conv for GCN \cite{gcn} on three small datasets, i.e., Cora, Citeseer, and PubMed from Planetoid \cite{planetoid}.
LMC4Conv is faster than GAS, especially on the CiteSeer and PubMed datasets.
Notably, the key bottleneck on the small datasets is graph sampling rather than forward and backward passes. Thus, GD is faster than GAS and LMC4Conv, as it avoids graph sampling by directly using the whole graph. 

\begin{figure*}[t]
\centering 
\includegraphics[width=\linewidth]{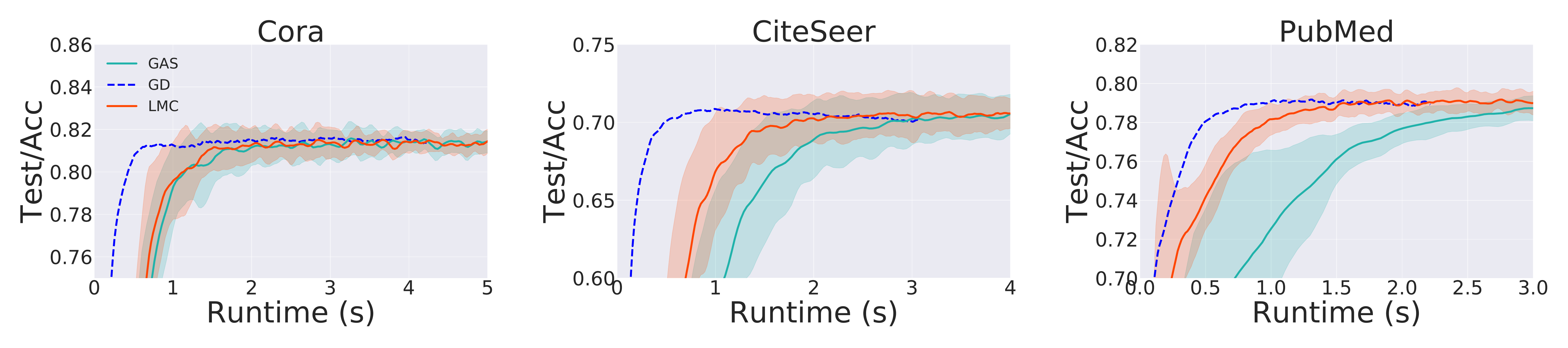}
\caption{
Testing accuracy w.r.t. runtimes (s).} \label{fig:runtime_small}
\end{figure*}

\subsubsection{Comparison in terms of Training Time per Epoch}\label{sec:exp_epochtime}
We evaluate the training time per epoch of Cluster-GCN, GAS, FM, and LMC4Conv in Table \ref{tab:epochtime}. 
Compared with GAS, LMC4Conv additionally accesses historical auxiliary variables.
Inspired by GAS \cite{gas}, we use the concurrent mini-batch execution to asynchronously access historical auxiliary variables.
Moreover, from the convergence analysis of LMC4Conv, we can sample clusters to construct fixed subgraphs at preprocessing step (Line 2 in Algorithm \ref{alg:lmc4conv}) rather than sample clusters to construct various subgraphs at each training step\footnote{Cluster-GCN proposes to sample clusters to construct various subgraphs at each training step and LMC4Conv follows it. If a subgraph-wise sampling method prunes an edge at the current step, the GNN may observe the pruned edge at the next step by resampling subgraphs.
This avoids GNN overfitting the graph which drops some important edges as shown in Section 3.2 in \cite{cluster_gcn} (we also observe that GAS achieves the accuracy of 71.5\% and 71.1\% under stochastic subgraph partition and fixed subgraph partition respectively on the Ogbn-arxiv dataset).
}.
This further avoids sampling costs.
Finally, the training time per epoch of LMC4Conv is comparable with GAS.
Cluster-GCN is slower than GAS and LMC4Conv, as it prunes edges in forward passes, introducing additional normalization operation for the adjacency matrix of the sampled subgraph by $[\mathbf{A}_{\inbatch}]_{i,j}/\sqrt{deg_{\inbatch}(i)deg_{\inbatch}(j)}$, where $deg_{\inbatch}(i)$ is the degree in the sampled subgraph rather than the whole graph.
The normalized adjacency matrix is difficult to store and reuse, as the sampled subgraph may be different.
FM is slower than other methods, as they additionally update historical embeddings in the storage for the nodes outside the mini-batches.

\begin{table}[h]\centering
      \caption{%
      Training time (s) per epoch of Cluster-GCN, GAS, FM, and LMC4Conv.
      }\label{tab:epochtime}
    \setlength{\tabcolsep}{1.9mm}
    \scalebox{1}{
    \begin{tabular}{lcccc}
    \toprule
    \textbf{Dataset} \& \textbf{GNN}
    & {\small Cluster-GCN} & {\small GAS} & {\small FM} & {\small LMC4Conv} \\
    \midrule
    Ogbn-arxiv \& GCN   & 0.51 & 0.46 & 0.75 & \textbf{0.45}\\
    FLICKR \& GCN       & 0.33 & 0.29 & 0.45 & \textbf{0.26}\\
    REDDIT \& GCN       & 2.16 & \textbf{2.11} & 5.67 & 2.28\\
    PPI \& GCN          & 0.84 & \textbf{0.61} & 0.78 & 0.62\\
    \midrule
    Ogbn-arxiv \& GCNII & --- & 0.92  & 1.02 & \textbf{0.91}\\
    FLICKR \& GCNII     & --- & \textbf{1.32}  & 1.44 & 1.33\\
    REDDIT \& GCNII     & --- & \textbf{4.47}  & 7.89 & 4.71\\
    PPI \& GCNII        & --- & \textbf{4.61}  & 5.38 & 4.77\\
    \bottomrule
  \end{tabular}
    }
\end{table}

\subsubsection{Ablation about $\beta_i$}\label{sec:ablation_beta}

As shown in Section \ref{sec:selection_beta}, $\beta_i = score(i)\alpha$ in LMC4Conv. We report the prediction performance under $\alpha \in \{0.0, 0.2, 0.4, 0.6, 0.8, 1.0\}$ and $score \in \{f(x)=x^2,f(x)=2x-x^2,f(x)=x,f(x)=1;x= deg_{local}(i)/deg_{global}(i)\}$ in Tables \ref{tab:ablation_alpha} and \ref{tab:ablation_score} respectively. When exploring the effect of a specific hyperparameter, we fix the other hyperparameters as their best values. Notably, $\alpha=0$ implies that LMC4Conv directly uses the historical values as affordable without alleviating their staleness, which is the same as that in GAS.
Under large batch sizes, LMC4Conv achieves the best performance with large $\beta_i=1$, as large batch sizes improve the quality of the incomplete up-to-date messages.
Under small batch sizes, LMC4Conv achieves the best performance with small $\beta_i=0.4 score_{2x-x^2}(i)$, as small learning rates alleviate the staleness of the historical values.

\begin{table}[h]\centering
      \caption{%
      Prediction performance under different $\alpha$ on the Ogbn-arxiv dataset.
      }\label{tab:ablation_alpha}
    \setlength{\tabcolsep}{1.1mm}
    \scalebox{1}{
    \begin{tabular}{cc|cccccc}
    \toprule
    \mr{2}{Batch Sizes} & \mr{2}{learning rates} & \mc{6}{c}{$\alpha$}  \\
    & & 0.0 & 0.2 & 0.4 & 0.6 & 0.8 & 1.0 \\
    \midrule
    1   & 1e-4 & 71.34 & 71.39 & \textbf{71.65} & 71.31 & 70.86 & 70.57\\
    40  & 1e-2 & 69.85 & 69.12 & 69.89 & 69.61 & 69.82 & \textbf{71.44}\\
    \bottomrule
  \end{tabular}
    }
\end{table}

\begin{table}[h]\centering
      \caption{%
      Prediction performance under different $score$ on the Ogbn-arxiv dataset.
      }\label{tab:ablation_score}
    \setlength{\tabcolsep}{0.5mm}
    \scalebox{0.75}{
    \begin{tabular}{cc|ccccc}
    \toprule
    \mr{2}{Batch Sizes} & \mr{2}{learning rates} & \mc{5}{c}{$score$}  \\
    & & $f(x)=2x-x^2$ & $f(x)=1$ & $f(x)=x^2$ & $f(x)=x$ & $f(x)=sin(x)$  \\
    \midrule
    1   & 1e-4 & \textbf{71.35} & 70.84 & 71.32 & 71.30 & 71.13 \\
    40  & 1e-2 & 67.59 & \textbf{71.44} & 69.91 & 70.03 & 70.32\\
    \bottomrule
  \end{tabular}
    }
\end{table}

\begin{figure*}[t]
\centering 
\begin{subfigure}{1.0\linewidth}
  \includegraphics[width=\linewidth]{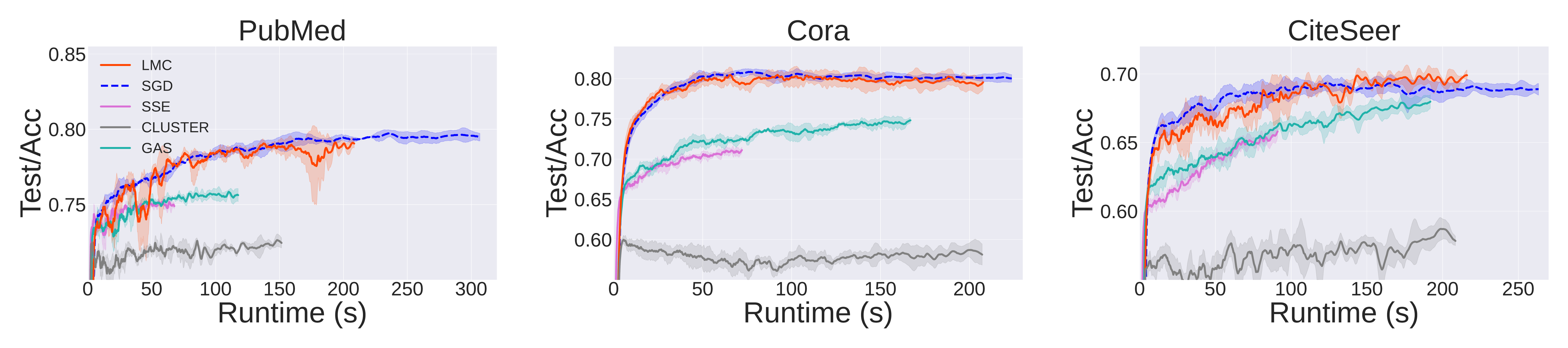}
    \caption{Random partitions}\label{subfig:test_acc_random_runtime}
\end{subfigure}\hfil
\begin{subfigure}{1.0\linewidth}
  \includegraphics[width=\linewidth]{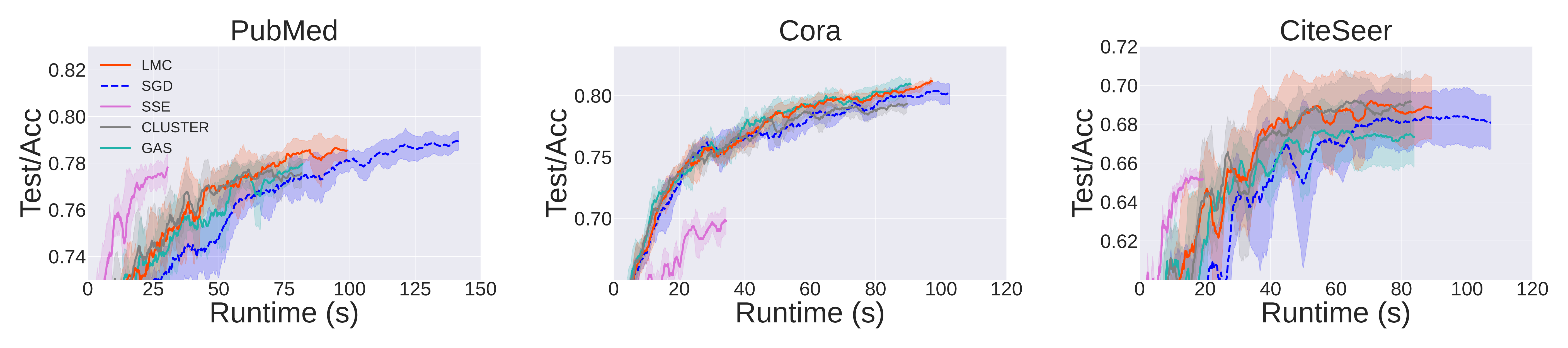}
  \caption{Metis partitions}\label{subfig:test_acc_metis_runtime}
\end{subfigure}\hfil 
\caption{
Testing accuracy w.r.t. runtime under random and METIS partitions for RecGCN.} \label{fig:convergence_random_runtime}
\end{figure*}

\begin{figure*}[b]
\centering 
\begin{subfigure}{0.5\linewidth}
  \includegraphics[width=\linewidth]{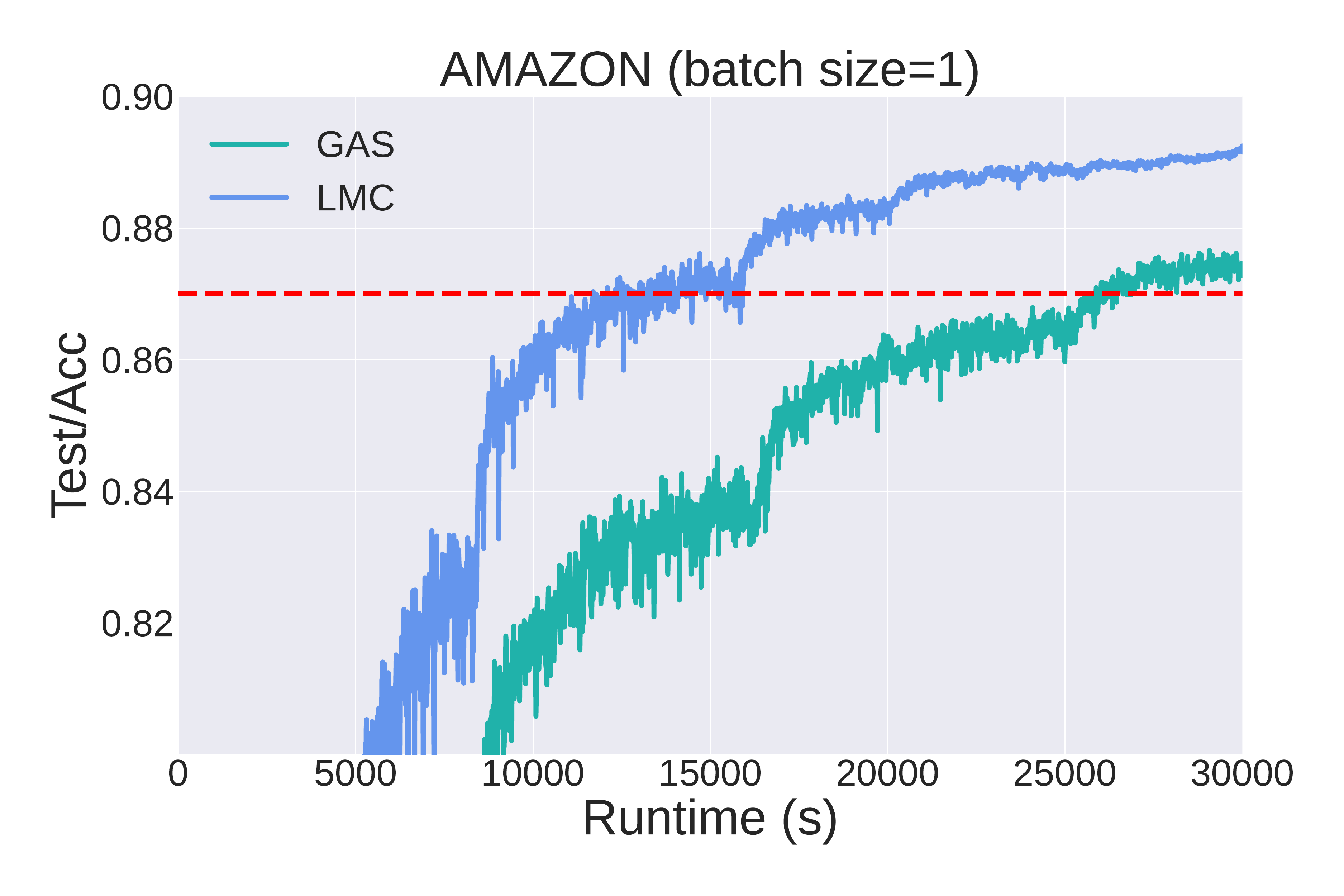}
    \caption{Batch size=1}\label{subfig:batchsize1}
\end{subfigure}\hfil
\begin{subfigure}{0.5\linewidth}
  \includegraphics[width=\linewidth]{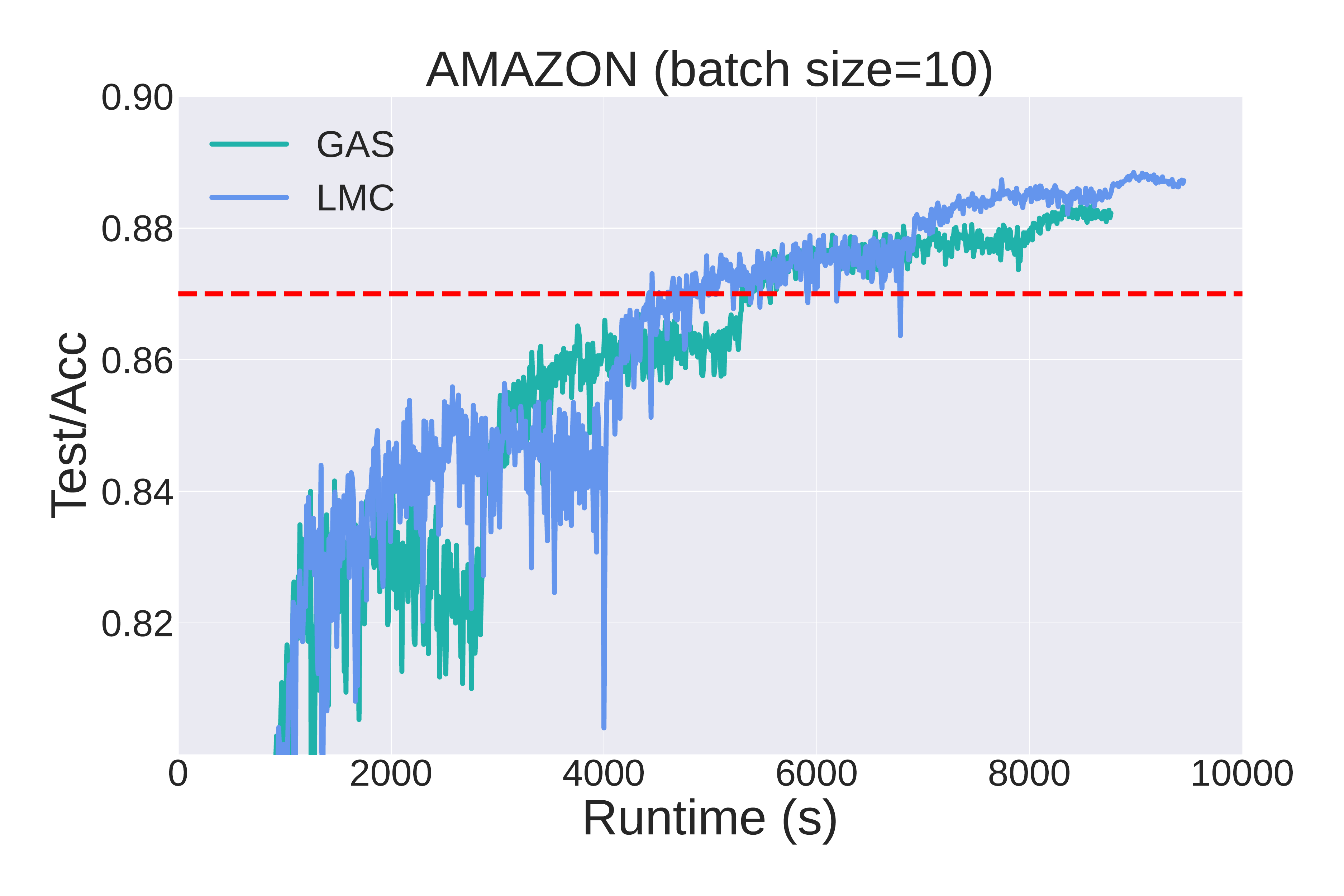}
  \caption{Batch size=10}\label{subfig:batchsize10}
\end{subfigure}\hfil 
\caption{
Testing accuracy w.r.t. runtime on the AMAZON dataset for RecGCN under different batch sizes (number of clusters).} \label{fig:bacthsize}
\end{figure*}

\subsection{Experiments with RecGNNs}

\subsubsection{Accuracy vs. Runtime on Small Graphs} \label{sec:acc_runtime_small_rec}

In this section, we report the accuracy vs. runtime on the small datasets in Figure \ref{fig:convergence_random_runtime}.
As SSE does not solve the fixed point equations during training, it converges faster at the start of the training. However, the final prediction performance of SSE does not achieve the state-of-the-art performance due to its inaccurate gradient approximations.
The efficiency of LMC4Rec achieves the-state-of-art performance on the small datasets under random and METIS partitions.

\subsubsection{Prediction Performance}\label{sec:prediction_large_rec}

In this section, we report more experiments on the prediction performance for more baselines including the state-of-the-art RecGNN (SSE \cite{sse} and IGNN \cite{ignn}) and our implemented RecGCNs trained by full-batch gradient descent (GD), Cluster-GCN \cite{deq_gcn}, GraphSAINT \cite{graphsaint}, and GAS.

\begin{table}[ht]
  \centering
  \caption{%
    \textbf{Performance on Large Graphs.}
    OOM denotes the out-of-memory issue.
  }\label{tab:largegraph_appen}
  \setlength{\tabcolsep}{2pt}
  \resizebox{\linewidth}{!}{%
  \begin{tabular}{llcccccc}
    \toprule
    \mc{2}{l}{\footnotesize{\textbf{\#\,nodes}}} & \footnotesize{57K}  & \footnotesize{230K}  & \footnotesize{335K} & \footnotesize{169K} & \footnotesize{2.4M} \\[-0.1cm]
    \mc{2}{l}{\footnotesize{\textbf{\#\,edges}}} & \footnotesize{794K}  &  \footnotesize{11.6M} & \footnotesize{926K} & \footnotesize{1.2M} & \footnotesize{61.9M} \\[-0.05cm]
    \mc{2}{l}{\textbf{Method}} & \textsc{PPI} & \textsc{REDDIT} & \textsc{Amazon}& \texttt{Ogbn-arxiv} & \texttt{Ogbn-products} \\
    \midrule
    &\textsc{SSE}         & 83.60          & ---              & 85.08                  & ---                  & --- \\
    &\textsc{IGNN}        & 97.56          & 94.95            & 85.31              & 70.88              & OOM \\
    \midrule
    & \textsc{GD}   & OOM   & 96.35   & 86.61 & 68.43   & OOM \\
    & \textsc{GraphSAINT}   & 61.14   & 96.60   & 43.53 & 70.93   & OOM \\
    & \textsc{Cluster-GCN}   & 97.37     & 94.49 & 85.69 & 69.23 & 76.32 \\ 
    & \textsc{GAS}   & 98.88   & 96.44 & 88.34 & 71.92   & 77.32 \\ 
    & \textsc{LMC4Rec}   & \textbf{98.99}     & \textbf{96.58} & \textbf{88.84} & \textbf{72.64}   & \textbf{77.73} \\
    \bottomrule
  \end{tabular}
 }
\end{table}

We evaluate these methods on the protein-protein interaction dataset (PPI), the Reddit dataset \cite{graphsage}, the Amazon product co-purchasing network dataset (AMAZON) \cite{amazon}, and open graph benchmarkings (Ogbn-arxiv and Ogbn-products) \cite{ogb} in Table \ref{tab:largegraph_appen}.
GraphSAINT suffers from the out-of-memory issue in the evaluation process on Ogbn-products.
LMC4Rec outperforms GD, GraphSAINT, and Cluster-GCN by a large margin and achieve the comparable prediction performance against GAS.
Moreover, Table 1 in the main text demonstrates that LMC4Rec enjoys much faster convergence than GAS, showing the effectiveness of LMC4Rec.
Therefore, LMC4Rec can accelerate the training of RecGNNs without sacrificing accuracy.

\subsubsection{Accuracy vs. Runtime on large datasets} \label{sec:acc_runtime_large_rec}

As shown in Table 2, LMC4Rec significantly improves the performance when long-range dependencies are important. Therefore, we first report accuracy vs. runtime on the AMAZON dataset to demonstrate the efficiency of LMC4Rec in Figure \ref{fig:bacthsize}. When the batch size (the number of sampled clusters) is equal to 1, the speedup of LMC4Rec is about 2x. On the other large datasets, Figure \ref{fig:large} demonstrates LMC4Rec is comparable to the state-of-the-art subgraph-wise sampling methods.

\begin{figure}[ht]
\centering 
\begin{subfigure}{0.5\linewidth}
  \includegraphics[width=\linewidth]{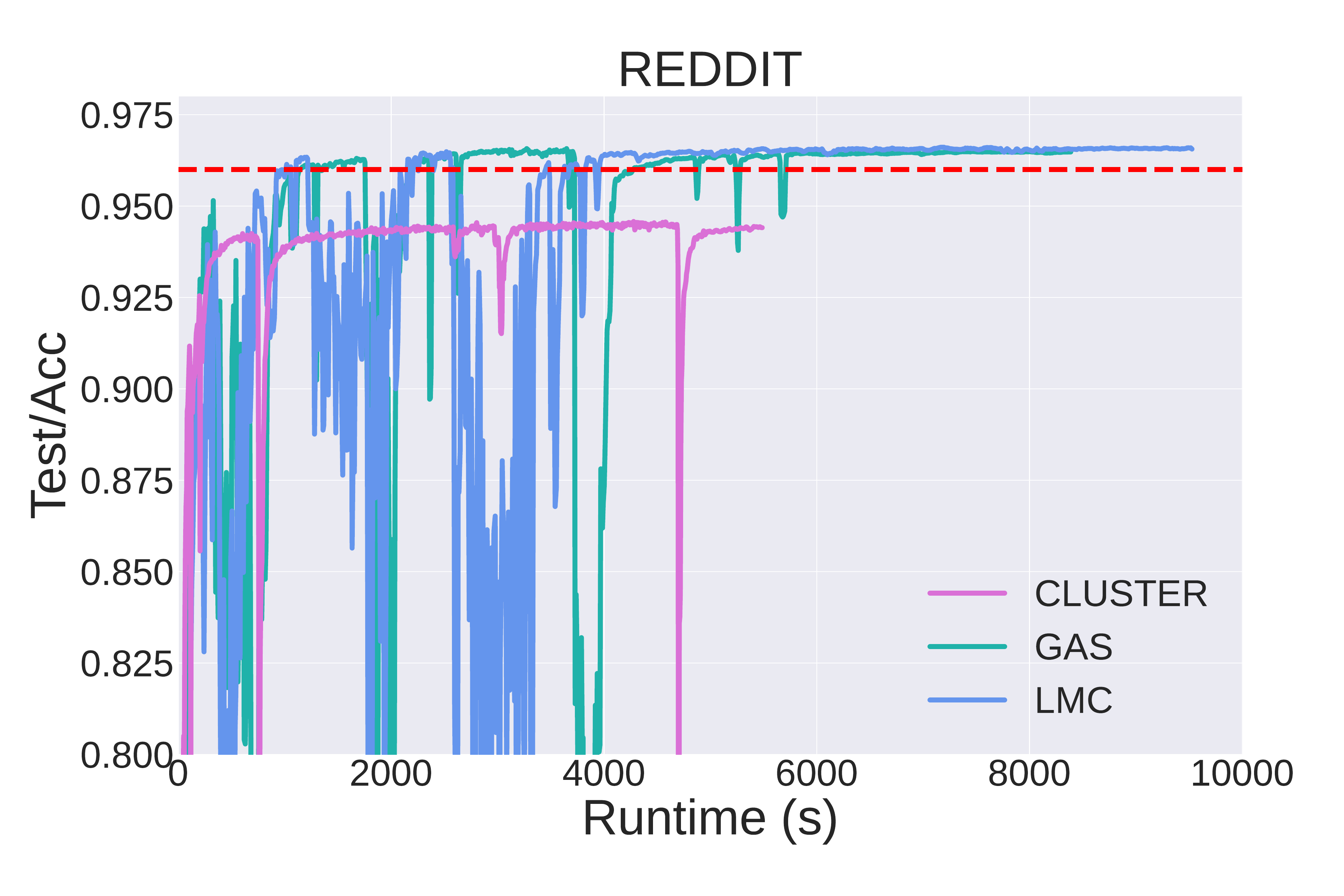}
  \caption{REDDIT}\label{subfig:reddit}
\end{subfigure}\hfil 
\begin{subfigure}{0.5\linewidth}
  \includegraphics[width=\linewidth]{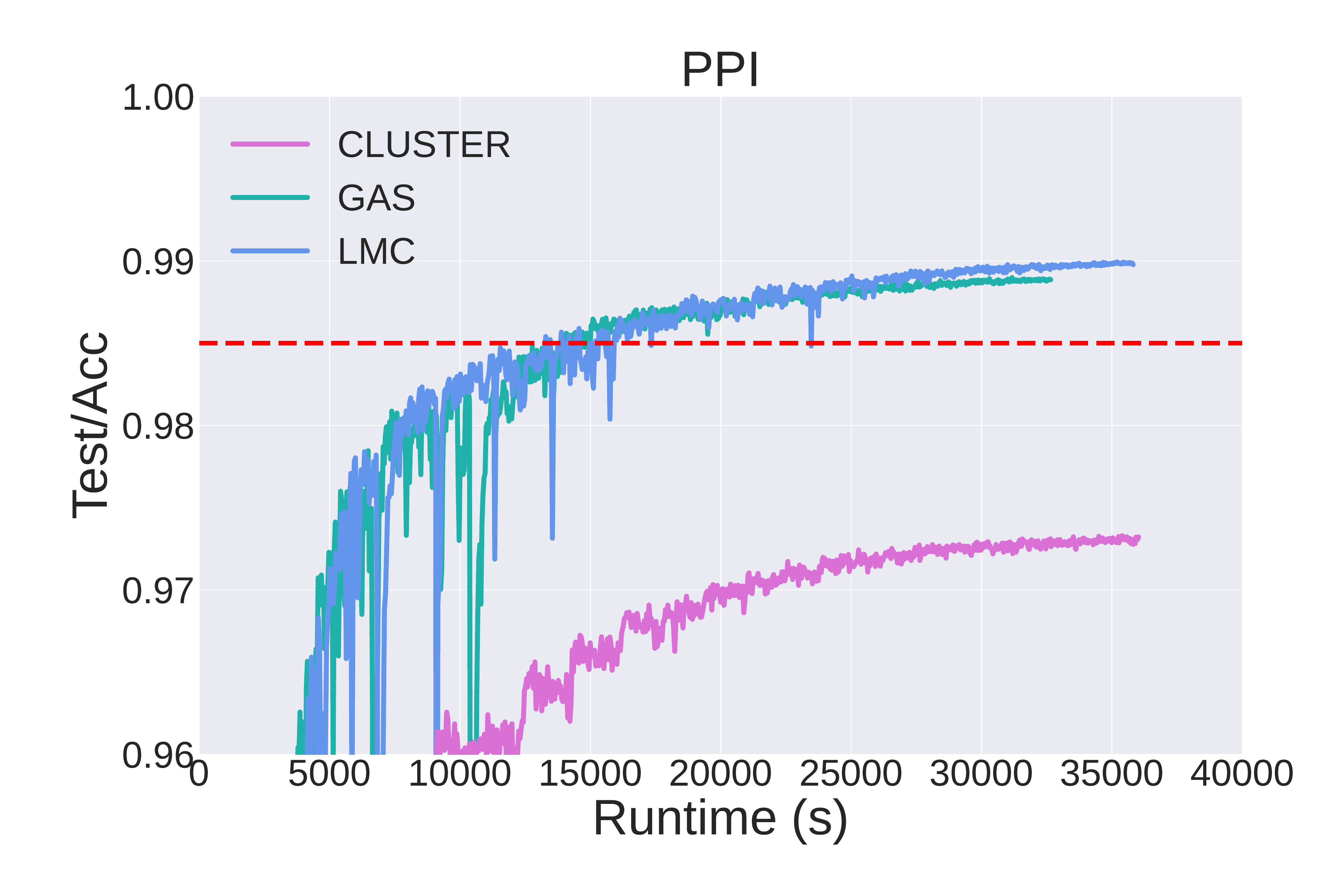}
  \caption{PPI}\label{subfig:ppi}
\end{subfigure}\hfil 
\begin{subfigure}{0.5\linewidth}
  \includegraphics[width=\linewidth]{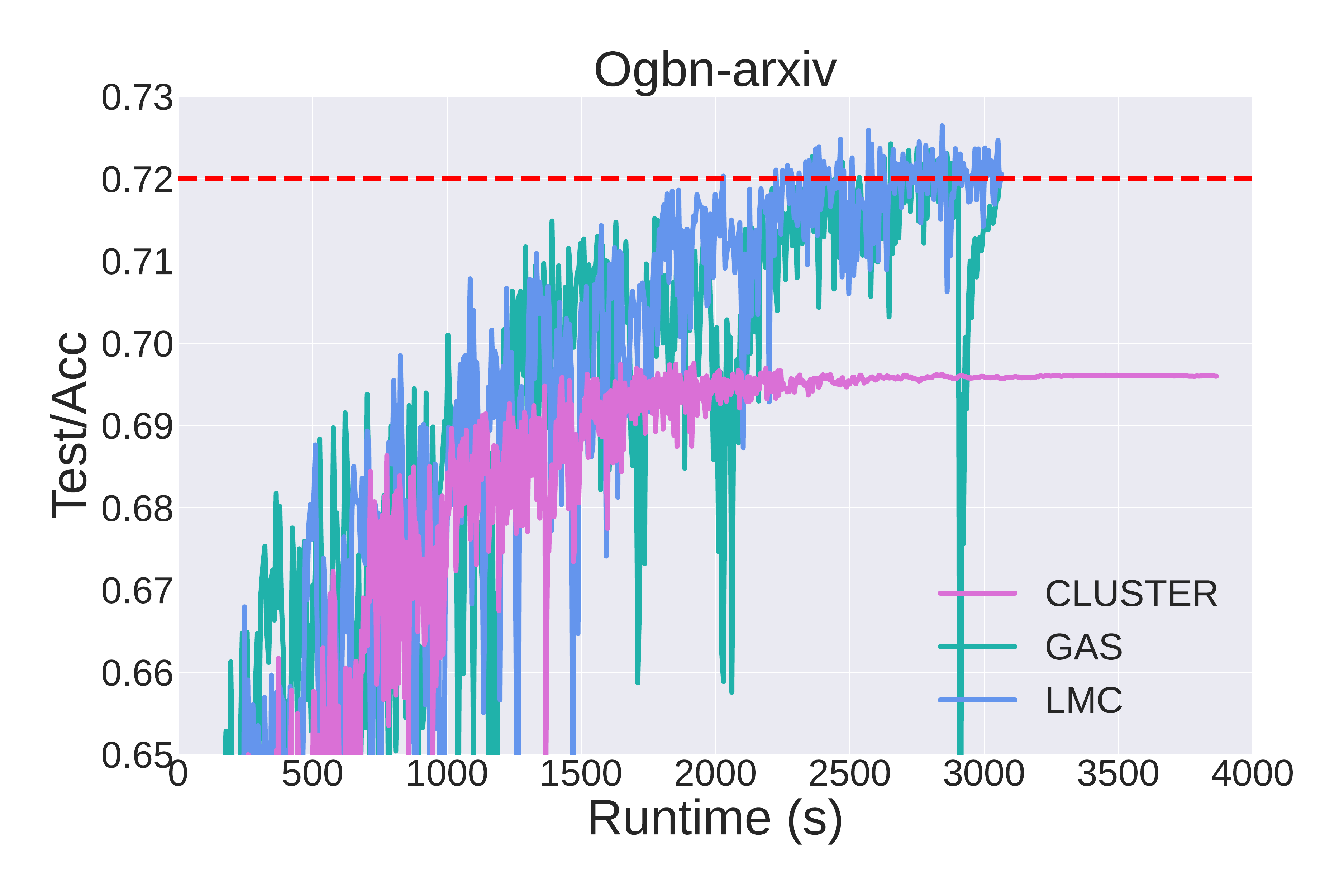}
  \caption{Ogbn-arxiv}\label{subfig:arxiv}
\end{subfigure}\hfil 
\begin{subfigure}{0.5\linewidth}
  \includegraphics[width=\linewidth]{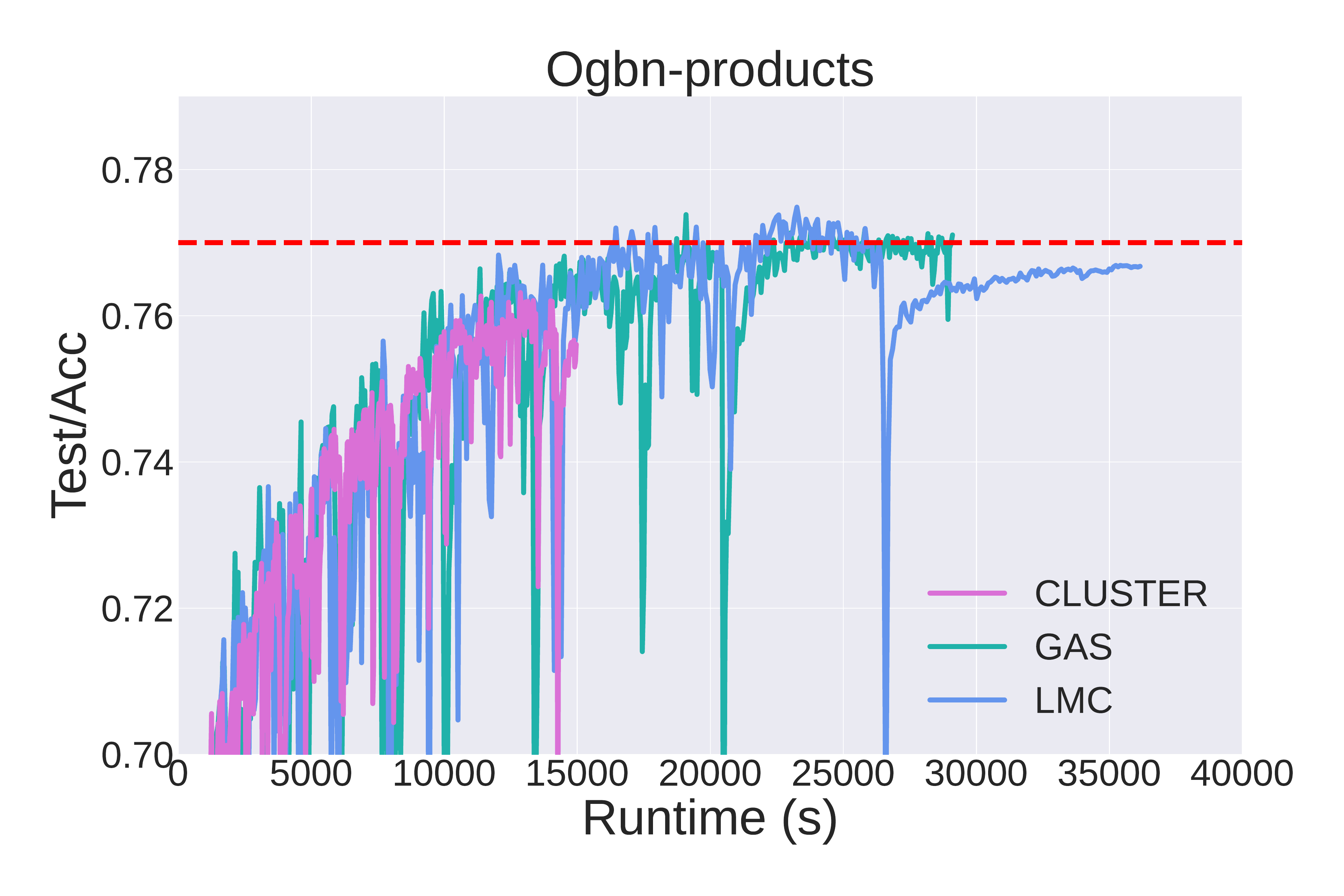}
  \caption{Ogbn-products}\label{subfig:products}
\end{subfigure}\hfil 
\caption{
Testing accuracy w.r.t. runtime (s) on Ogbn-arxiv, PPI, REDDIT, and Ogbn-products. The LMC in the figures refers to LMC4Rec.} \label{fig:large}
\end{figure}

\section{Potential Societal Impacts}

In this paper, we propose a novel and efficient subgraph-wise sampling method for the training of GNNs, i.e., LMC.
This work is promising in many practical and important scenarios such as search engine, recommendation systems, biological networks, and molecular property prediction.
Nonetheless, this work may have some potential risks. For example, using this work in search engine and recommendation systems to over-mine the behavior of users may cause undesirable privacy disclosure.

\end{document}


%
\title{LMC: Fast Training of Convolutional Graph Neural Networks and Recurrent Graph Neural Networks with Provable Convergence \\ Appendix}
%
%
%
%

\author{Jie~Wang,~\IEEEmembership{Senior Member,~IEEE,} Zhihao~Shi, Xize~Liang,\\Shuiwang~Ji,~\IEEEmembership{Senior Member,~IEEE,} and~Feng~Wu,~\IEEEmembership{Fellow,~IEEE}
\IEEEcompsocitemizethanks{\IEEEcompsocthanksitem M. Shell was with the Department
of Electrical and Computer Engineering, Georgia Institute of Technology, Atlanta,
GA, 30332.\protect\\
E-mail: see http://www.michaelshell.org/contact.html
\IEEEcompsocthanksitem J. Doe and J. Doe are with Anonymous University.}
\thanks{Manuscript received April 19, 2005; revised August 26, 2015.}}

%
%

\markboth{Journal of \LaTeX\ Class Files,~Vol.~14, No.~8, August~2015}%
{Shell \MakeLowercase{\textit{et al.}}: Bare Advanced Demo of IEEEtran.cls for IEEE Computer Society Journals}
\maketitle

\IEEEdisplaynontitleabstractindextext

%
\IEEEpeerreviewmaketitle

\appendices
\section{More Details about Experiments}
In this section, we introduce more details about our experiments, including datasets, training and evaluation protocols, and implementations.

\subsection{Datasets}
We evaluate LMC on three small datasets, including Cora, Citeseer, and PubMed from Planetoid \cite{planetoid}, and five large datasets, PPI, Reddit \cite{graphsage}, AMAZON \cite{amazon}, Ogbn-arxiv, and Ogbn-products \cite{ogb}.
All of the datasets do not contain personally identifiable information or offensive content.
Table \ref{tab:datasets} shows the summary statistics of the datasets.
Details about the datasets are as follows.
\begin{itemize}[leftmargin=10mm]
    \item Cora, Citeseer, and PubMed are directed citation networks. Each node indicates a paper with the corresponding bag-of-words features and each directed edge indicates that one paper cites another one. The task is to classify academic papers into different subjects.
    \item PPI contains 24 protein-protein interaction graphs. Each graph corresponds to a human tissue. Each node indicates a protein with positional gene sets, motif gene sets and immunological signatures as node features. Edges represent interactions between proteins. The task is to classify protein functions.
    (proteins) and edges (interactions).
    \item REDDIT is a post-to-post graph constructed from Reddit. Each node indicates a post and each edge between posts indicates that the same user comments on both. The task is to classify Reddit posts into different communities based on (1) the GloVe CommonCrawl word vectors \cite{glove} of the post titles and comments, (2) the post’s scores, and (3) the numbers of comments made on the posts.
    \item AMAZON is an Amazon product co-purchasing network. Each node indicates a product and each edge between two products indicates that the products are purchased together. The task is to predict product types \modify{}{without node features based on rare labeled nodes}. We set the training set fraction be 0.06\% in experiments.
    \item Ogbn-arxiv is a directed citation network between all Computer Science (CS) arXiv papers indexed by MAG \cite{mag}. Each node is an arXiv paper and each directed edge indicates that one paper cites another one. The task is to classify unlabeled arXiv papers into different primary categories based on labeled papers and node features, which are computed by averaging word2vec \cite{word2vec} embeddings of words in papers' title and abstract.
    \item Ogbn-product is a large Amazon product co-purchasing network \modify{}{with rich node features}. Each node indicates a product and each edge between two products indicates that the products are purchased together. The task is to predict product types based on low-dimensional bag-of-words features of product descriptions processed by Principal Component Analysis.
\end{itemize}

\begin{table}[htbp]
  \begin{center}
    \caption{Statistics of the datasets used in our experiments.
    }\label{tab:datasets}
    \vspace{5pt}
    \scalebox{1}{
    \begin{tabular}{ccccc}
    \toprule
    \textbf{Dataset} & \textbf{\#Graphs} & \textbf{\#Classes} &\textbf{Total \#Nodes} & \textbf{Total \#Edges}  \\
      \midrule
      \midrule
      Cora & 1 & 7  & 2,708 & 5,278\\
      Citeseer & 1 & 6  & 3,327 & 4,552 \\
      PubMed & 1 & 3  & 19,717 & 44,324 \\
      \midrule
      PPI & 24 & 121 & 56,944 & 793,632 \\
      REDDIT & 1 & 41 & 232,965 & 11,606,919   \\
      AMAZON & 1 & 58 & 334,863 & 925,872   \\
      Ogbn-arxiv & 1 & 40  & 169,343 & 1,157,799  \\
      Ogbn-product & 1 & 47  & 2,449,029 & 61,859,076 \\ 
      \bottomrule
    \end{tabular}
    }
  \end{center}
\end{table}

\subsection{Training and Evaluation Protocols}

We run all the experiments on a single GeForce RTX 2080 Ti (11 GB). All the models are implemented in Pytorch \cite{pytorch} and PyTorch Geometric \cite{pyg} based on the official implementation of \cite{gas}\footnote{\url{https://github.com/rusty1s/pyg_autoscale}. The owner does not mention the license.} and \cite{ignn}\footnote{\url{https://github.com/SwiftieH/IGNN}, licensed under the MIT License.}.

\udfsection{Data Splitting.} We use the data splitting strategies following previous works \cite{gas, ignn}.


\subsection{Implementation Details and Hyperparameters}

\subsubsection{Hyperparameters} 

We report the embedding dimensions, the number of partitions, sampled clusters per mini-batch, and learning rates (LR) for each dataset in Table \ref{tab:hyperprameters}.
For the small datasets, we partition each graph into ten subgraphs. For the large datasets, we select the number of partitions and sampled clusters per mini-batch to avoid running out of memory on GPUs.
For a fair comparison, we use the hyperparameters for each training method.
We provide other hyperparameters in the \textit{json} file in our codes.

\begin{table}[h]
  \centering
  \caption{%
  \textbf{The hyperprameters used in the experiments.}
  }\label{tab:hyperprameters}
  \setlength{\tabcolsep}{5pt}
  \resizebox{1.0\linewidth}{!}{%
  \begin{tabular}{lcccccc}
    \toprule
    \textbf{Dataset}  & \textbf{Dimensions} & \textbf{Partitions} & \textbf{Clusters} & \textbf{LR (METIS partition)} & \textbf{LR (random partition)} \\
    \midrule
    \textsc{Cora} & 128 & 10 & 2 & 0.003 & 0.001\\
    \textsc{CiteSeer} & 128 & 10 & 2 & 0.01 & 0.001 \\
    \textsc{PubMed} & 128 & 10 & 2 & 0.003 & 0.0003 \\
    \textsc{PPI} & 1024 & 10 & 2 & 0.01 & - \\
    \textsc{REDDIT} & 256 & 200 & 100 & 0.003 & - \\
    \textsc{AMAZON} & 64 & 40 & 1 & 0.003 & - \\
    \textsc{Ogbn-arxiv} & 256 & 80 & 10 & 0.003 & - \\
    \textsc{Ogbn-products} & 256 & 1500 & 10 & 0.001 & - \\
    \bottomrule
  \end{tabular}
  }
\end{table}

\subsubsection{An Efficient Implement of LMC for RecGCN}

The equilibrium equations of RecGCN (see Equation (1) in the main text) are
\begin{align*}
    \embH  = \sigma(  \mathbf{W} \embH   \hat{\mathbf{A}} + \mathbf{b}(\embX) ),
\end{align*}
where the nonlinear activation function $\sigma(\cdot)$ is the ReLU activation $\sigma(\cdot)=\max(\cdot,0)$, the function $\mathbf{b}(\embX) = \mathbf{P}\embX +\mathbf{c}$ is an affine function with parameters $\mathbf{P} \in \mathbb{R}^{d \times d_{x}},\mathbf{c} \in \mathbb{R}^{d}$
, the matrix $\mathbf{\hat{A}} = (\mathbf{D}+\mathbf{I})^{-1/2} (\mathbf{A}+\mathbf{I})(\mathbf{D}+\mathbf{I})^{-1/2}$ is the normalized adjacency matrix with self-loops, and $\mathbf{D}$ is the degree matrix.

Let $\mathbf{Z} = \mathbf{W} \embH   \hat{\mathbf{A}} + \mathbf{b}(\embX)$.
The Jacobian vector-product is $\langle \vec{\embV}, \nabla_{\embH} \vec{f}_{\theta} \rangle$ \cite{ignn} is
\begin{align*}
     \mathbf{W}^{\top} (\sigma'(\mathbf{Z}) \odot \embV) \hat{\mathbf{A}}^{\top}.
\end{align*}

For two nodes $v_i,v_j$ such that $v_i \in \mathcal{N}(v_j)$, we have
\begin{align*}
    \vec{\embV}_i^{\top} \frac{\partial [f_{\theta}]_i}{\partial \embh_{j}} = \mathbf{W}^{\top} \left(\sigma({\mathbf{Z}}_i)' \odot {\embV}_i\right) \mathbf{\hat{A}}_{ij},
\end{align*}
which is only depend on the nodes $v_i,v_j$ rather than the 2-hop neighbors of $v_j$.

Therefore, by additionally storing the historical auxiliary variables $\mathbf{Z}$, we implement LMC based on 1-hop neighbors, which employs $\mathcal{O}(n_{\max} |\inbatch| d)$ GPU memory in backward passes.

\subsubsection{Implement of Baslines}

For a fair comparison, we implement GAS by setting the gradient compensation $C_b$ to be zero. We implement CLUSTER by removing edges between partitioned subgraphs and then running GAS based on them.

\subsubsection{Normalization Technique}

In Section 4 in the main text, we assume that the subgraph $\inbatch$ is uniformly sampled from $\mathcal{V}$ and the corresponding set of labeled nodes {\small$\mathcal{V}_{L_{\mathcal{B}}} = \inbatch \cap \mathcal{V}_{L}$} is uniformly sampled from {\small$\mathcal{V}_{L}$}. To enforce the assumption, we use the normalization technique to reweight Equations (5) and (6) in the main text.

Suppose we partition the whole graph $\mathcal{V}$ into $b$ parts $\{\mathcal{V}_{\mathcal{B}_i}\}_{i=1}^b$ and then uniformly sample $c$ clusters without replacement to construct subgraph $\inbatch$. By the normalization technique, Equation (5) becomes
\begin{align}
    \mathbf{g}_w(\inbatch) = \frac{b|\mathcal{V}_{L_{\mathcal{B}}}|}{c|\mathcal{V}_L|} \frac{1}{|\mathcal{V}_{L_{\mathcal{B}}}|} \sum_{v_j \in \mathcal{V}_{L_{\mathcal{B}}}}\nabla_{w} l_w(\embh_j ,y_j),
\end{align}
where $\frac{b|\mathcal{V}_{L_{\mathcal{B}}}|}{c|\mathcal{V}_L|}$ is the corresponding weight. Similarly, Equation (6) becomes
\begin{align}
    \mathbf{g}_{\theta}(\inbatch) = \frac{b|\inbatch|}{c|\mathcal{V}|} \frac{|\mathcal{V}|}{|\inbatch|} \sum_{v_j \in \inbatch} \nabla_{\theta} \update(\embh_j ,\embm_{\neighbor{v_j}}  ,\embx_j) \embV_{j} \label{eqn:sgd_theta},
\end{align}
where $\frac{b|\inbatch|}{c|\mathcal{V}|}$ is the corresponding weight.

\subsection{Randomly Using Dropout at Different Training Steps}

In this section, we propose a trick to handle the randomness introduced by dropout \cite{dropout}. Let $\dropout(\embX) = \frac{1}{1-p} \mathbf{M} \circ \embX$ be the dropout operation, where $\mathbf{M}_{ij} \sim \text{Bern}(1-p)$ are i.i.d Bernoulli random variables, and $\circ$ is the element-wise product.

As shown by \cite{vrgcn}, with dropout \cite{dropout} of input features $\mathbf{X}$, historical embeddings and auxiliary variables become random variables, leading to inaccurate compensation messages.
Specifically, the solutions to Equations (1) and (3) in the main text are the function of the dropout features $\frac{1}{1-p} \mathbf{M}^{(k)} \circ \embX$ at the training step $k$, while the dropout features become $\frac{1}{1-p} \mathbf{M}^{(k+1)} \circ \embX$ at the training step $k+1$ (we assume that the learning rate $\eta=0$ to simplify the analysis).

\modify{}{
As the historical information under the dropout operation at the training step $k$ may be very inaccurate at the training step $k+1$, \cite{vrgcn} propose to first compute the random embeddings with dropout and the mean embeddings without dropout, and then use the mean embeddings to update historical information.
However, simultaneously computing two versions of embeddings in RecGNNs leads to double computational costs of solving Equations (1) and (3).
We thus propose to randomly use the dropout operation at different training steps and update the historical information when the dropout operation is invalid. Specifically, at each training step $k$, we either update the historical information without the dropout operation with probability $q$ or use the dropout operation without updating the historical information. We set $q=0.5$ in all experiments. Due to the trick, the historical embeddings and auxiliary variables depend on the stable input features $\embX$ rather than the random dropout features $\frac{1}{1-p} \mathbf{M} \circ \embX$. The trick can reduce overfitting and control variate for dropout efficiently.
}

\section{Message Passing}

In the GNN literature, message passing usually refers to neural message passing \cite{mpnn}, a framework in which vector messages are exchanged between nodes and updated using neural networks.

Overall, message passing is an iterative method and follows an \textit{aggregate} and \textit{update} scheme. In each iteration, the representation of each node $v_i$ is \textit{updated} using information \textit{aggregated} from $v_i$'s neighborhood $\mathcal{N}(v_i)$. Formally, the $k$-th iteration process \cite{grl} can be written as
\begin{align}
    \embm_{\neighbor{v_i}}^{(l)}&=\aggregate^{(l)}\left( \left\{g^{(l)}(\embh_j^{(l-1)})|v_j \in\neighbor{v_i}\right\}\right)\label{eqn:aggregate},\\
    \embh_i^{(l)}&=\update^{(l)}\left(\embh_i^{(l-1)}, \embm_{\neighbor{v_i}}^{(l)},\embx_i\right)\label{eqn:update},
\end{align}
where {\small $g^{(l)}$} is the function generating \textit{individual messages} for each neighborhood of $i$, {\small $\aggregate^{(l)}$} is the aggregation function mapping from a set of messages to the final message {\small $\embm_{\neighbor{v_i}}^{(l)}$}, and {\small $\update^{(l)}$} is the update function combining the message {\small $\embm_{\neighbor{v_i}}^{(l)}$} with the previous embedding {\small $\embh_i^{(l-1)}$} to obtain a new embedding for $i$.
To decrease memory costs, RecGNNs use the same generation function {\small $g^{(l)}$}, the aggregation function {\small $\aggregate^{(l)}$}, and the update function {\small $\update^{(l)}$} in each iteration, so that the message passing iteration only saves the latest embeddings {\small $\embh_i^{(K)}$} rather than all historical embeddings {\small $\{\embh_i^{(l)}\}_{l=1}^{K}$} \cite{ignn, deq}.

\section{Well-posedness Conditions of RecGNNs}\label{sec:well-posedness}

In this section, we provide tractable well-posedness conditions \cite{ignn} of RecGNNs to ensure the existence and uniqueness of the solution to Equation (1) in the main text.

The Equation (1)  in the main text of RecGNNs with the message passing functions in GCN can be formulated as
\begin{align*}
    \embH   = f_{\theta}(\embH  ;\embX) = \sigma(  \mathbf{W} \embH   \hat{\mathbf{A}} + \mathbf{b}(\embX) ),
\end{align*}
where $\mathbf{\hat{A}}$ is the aggregation matrix,  $\sigma(\cdot)$ is the nonlinear activation function, and $\mathbf{b}(\embX) $ is an affine function to encode the input features.
The Perron-Frobenius (PF) sufficient condition for well-posedness \cite{idl} requires the activation function $\sigma$ is component-wise non-expansive and $\lambda_{pf}(|\hat{\mathbf{A}}^{\top} \otimes \mathbf{W}|) = \lambda_{pf}(\hat{\mathbf{A}})\lambda_{pf}(|\mathbf{W}|) < 1$
, where $\otimes$ is the Kronecker product.
As pointed out in \cite{ignn}, if the Perron-Frobenius (PF) sufficient condition holds, the solution $\embH  $ can be achieved by iterating Equation (1) in the main text to convergence.

We further discuss how to enforce the non-convex constraint $\lambda_{pf}(|\mathbf{W}|) < (\lambda_{pf}(\hat{\mathbf{A}}))^{-1}$.
As $\lambda_{pf}(|\mathbf{W}|) \leq \| \mathbf{W} \|_p$ holds for induced norms $\| \cdot \|_p$, we enforce the stricter condition $\| \mathbf{W} \|_p < \lambda_{pf}(\hat{\mathbf{A}})^{-1}$. As pointed out in \cite{ignn}, if $p=1$ or $p=\infty$, we efficiently implement $\| \mathbf{W} \|_p < 1$ by projection. We follow the implementation in experiments and view $p$ as a hyperparameter.

\section{Examples about Long-Range Dependencies}

We provide an example inspired by label propagation algorithm.
In a citation network, if a paper $A_1$ cites another paper $A_2$, then they are likely to belong to the same area.
Recursively following citation relations between papers $A_i$ and $A_{i+1}$, we deduce that $A_1$ and $A_{K}$ are dependent for any positive integer $K$ while the dependency decreases as $K$ grows (as it is a hidden Markov process).
For a large $K$, the long-range dependency implies that papers $A_1$ and $A_{K}$ are likely to belong to the same area, although the confident is usually lower than that provided by $A_{2}$.
Anyway, the long-range dependency is still helpful to denoise and predict the area of $A_1$, when the features of $A_1,A_2,\dots,A_{K-1}$ are noisy and even missing

\section{Computational Complexity}

We summarize the computational complexity in Table \ref{tab:complexity}, where $n_{\max}$ is the maximum of neighborhoods, $K$ is the number of message passing iterations, $\mathcal{V}_{\mathcal{B}}$ is a set of nodes in a sampled mini-batch, $d$ is the embedding dimension, $\mathcal{V}$ is the set of nodes in the whole graph, and $\mathcal{E}$ is the set of edges in the whole graph.

\begin{table}[htbp]
    \centering
    \caption{
    Time and space complexity of message passing based GNNs.
    }
    \label{tab:complexity}
    \begin{tabular}{ccc}
    \toprule
        \textbf{Method} & \textbf{Time complexity per gradient update}  &\textbf{Memory per gradient update}   \\
        \midrule
        GD and naive SGD  & $\mathcal{O}(K(|\mathcal{E}|d+|\mathcal{V}| d^2))$ & $\mathcal{O}(|\mathcal{V}| d)$ \\
        CLUSTER \cite{cluster_gcn}  & $\mathcal{O}( K(|\inbatch|^2d+|\inbatch| d^2) )$ & $\mathcal{O}(|\inbatch| d)$ \\
        GAS \cite{gas} & $\mathcal{O}( K(|\inbatch|^2d+|\inbatch| d^2) )$ & $\mathcal{O}( n_{\max} |\inbatch| d)$ \\
        \midrule
        LMC for RecGCN & $\mathcal{O}( K(|\inbatch|^2d+|\inbatch| d^2) )$ & $\mathcal{O}(n_{\max} |\inbatch| d)$ \\
        LMC for general RecGNNs & $\mathcal{O}( K(|\inbatch|^2d+|\inbatch| d^2) )$ & $\mathcal{O}(n_{\max}^2 |\inbatch| d)$ \\
    \bottomrule
    \end{tabular}
\end{table}

\section{Proof}

\subsection{Proof of Theorem 1}
\begin{proof}
    As the labeled nodes $\mathcal{V}_{L_{\mathcal{B}}} = \inbatch \cap \mathcal{V}_{L}$ is uniformly sampled from $\mathcal{V}_{L}$, the expectation of $ \mathbf{g}_w(\inbatch)$ is
    \begin{align*}
        \mathbb{E}[\mathbf{g}_w(\inbatch)] &=\mathbb{E}[ \frac{1}{|\mathcal{V}_{L_{\mathcal{B}}}|} \sum_{v_j \in \mathcal{V}_{L_{\mathcal{B}}}}\nabla_{w} l_w(\embh_j ,y_j)]\\
        &= \nabla_{w} \mathbb{E}[ l_w(\embh_j ,y_j)]\\
        &= \nabla_{w} \loss.
     \end{align*}
     As the subgraph $\inbatch$ is uniformly sampled from $\mathcal{V}$, the expectation of $\mathbf{g}_{\theta_l}(\inbatch)$ is
    \begin{align*}
        \mathbb{E}[\mathbf{g}_{\theta_l}(\inbatch)] &= \mathbb{E}[\frac{|\mathcal{V}|}{|\mathcal{V}_\mathcal{B}|}\sum_{v_j\in\mathcal{V}_\mathcal{B}}\left(\nabla_{\theta_l}u^{(l)}(\embh^{(l-1)}_j, \embm^{(l-1)}_{\neighbor{v_j}}, \embx_j)\right)] \embV_{j}^{(l)} \\
        &= |\mathcal{V}| \mathbb{E}[ \nabla_{\theta_l}u^{(l)}(\embh^{(l-1)}_j, \embm^{(l-1)}_{\neighbor{v_j}}, \embx_j) ] \embV_{j}^{(l)} ]\\
        &=|\mathcal{V}| \frac{1}{|\mathcal{V}|} \sum_{v_j \in \mathcal{V}} \nabla_{\theta_l}u^{(l)}(\embh^{(l-1)}_j, \embm^{(l-1)}_{\neighbor{v_j}}, \embx_j) \embV_{j}^{(l)}\\
        &= \sum_{v_j \in \mathcal{V}} \nabla_{\theta_l}u^{(l)}(\embh^{(l-1)}_j, \embm^{(l-1)}_{\neighbor{v_j}}, \embx_j) \embV_{j}^{(l)}\\
        &= \nabla_{\theta_l} \loss,\,\forall\,l=1,\ldots,L.
    \end{align*}
\end{proof}

\subsection{Proof of Theorem 2}

\begin{proof}
    As the labeled nodes $\mathcal{V}_{L_{\mathcal{B}}} = \inbatch \cap \mathcal{V}_{L}$ is uniformly sampled from $\mathcal{V}_{L}$, the expectation of $ \mathbf{g}_w(\inbatch)$ is
    \begin{align*}
        \mathbb{E}[\mathbf{g}_w(\inbatch)] &=\mathbb{E}[ \frac{1}{|\mathcal{V}_{L_{\mathcal{B}}}|} \sum_{v_j \in \mathcal{V}_{L_{\mathcal{B}}}}\nabla_{w} l_w(\embh_j ,y_j)]\\
        &= \nabla_{w} \mathbb{E}[ l_w(\embh_j ,y_j)]\\
        &= \nabla_{w} \loss.
     \end{align*}
 As the subgraph $\inbatch$ is uniformly sampled from $\mathcal{V}$, the expectation of $\mathbf{g}_{\theta}(\inbatch)$ is
 \begin{align*}
     \mathbb{E}[\mathbf{g}_{\theta}(\inbatch)] &= \mathbb{E}[\frac{|\mathcal{V}|}{|\inbatch|} \sum_{v_j \in \inbatch} \nabla_{\theta} \update(\embh_j ,\embm_{\neighbor{v_j}}  ,\embx_j) ] \embV_{j} \\
     &= |\mathcal{V}| \mathbb{E}[ \nabla_{\theta} \update(\embh_j ,\embm_{\neighbor{v_j}}  ,\embx_j) ] \embV_{j} ]\\
     &=|\mathcal{V}| \frac{1}{|\mathcal{V}|} \sum_{v_j \in \mathcal{V}} \nabla_{\theta} \update(\embh_j ,\embm_{\neighbor{v_j}}  ,\embx_j) \embV_{j}\\
     &= \sum_{v_j \in \mathcal{V}} \nabla_{\theta} \update(\embh_j ,\embm_{\neighbor{v_j}}  ,\embx_j) \embV_{j}\\
     &= \nabla_{\theta} \loss.
 \end{align*}
\end{proof}

\subsection{Proof of Lemma 1 in the main text}
We first show that for a single layer ConvGNN, if the inputs of LMC and the exact forward pass are close to each other, then the approximation error is close to zero.

\begin{lemma}\label{prop:singlelayer}
    For a single layer ConvGNN, we denote the inputs we feed to the forward passing model of LMC and the exact forward passing model at the $k$-th iteration by $\hisH_k^{(0)}$ and $\embH_k^{(0)}$, respectively. Suppose that
    \begin{enumerate}
        \item the message passing function $f_{\theta^{(1)}}(\embH^{(0)})$ is $\gamma$-Lipschitz for parameter $\theta^{(1)}$ and $\embH^{(0)}$,

        \item the gradients $\widetilde{\embg}_{\theta^{(1)}}(\theta_k^{(1)})$ are bounded by $G$ for $k\in\mathbb{N}^*$,

        \item the differences are bounded by $\varepsilon_0>0$, i.e.,
        \begin{align*}
            &\|\hisH_k^{(0)}-\embH_k^{(0)}\|_F < \varepsilon_0,\,\,\forall\, k\in\mathbb{N}^*,\\
            & \|\embH_{k_1}^{(0)} - \embH_{k_2}^{(0)}\|_F < \varepsilon_0,\,\,\forall\, k_1,k_2\in\mathbb{N}^*,
        \end{align*}

        \item the differences between temporary embeddings and exact embeddings are dominated by those between historical embeddings and exact embeddings, i.e.,
        \begin{align*}
            \|\widehat{\embH}_{k}^{(0)} - \embH_{k}^{(0)}\|_F \leq C \|\overline{\embH}_{k}^{(0)} - \embH_{k}^{(0)}\|_F< C\varepsilon_0,\,\,\forall\,k\in\mathbb{N}^*,
        \end{align*}
    \end{enumerate}
    then for any $0<\delta<1$, by letting 
    \begin{align*}
        \eta \leq \frac{\varepsilon_0}{G(1-\log_{\frac{B}{B-1}} (\frac{1-\delta}{B-1}))},    
    \end{align*}
    we have
    \begin{align*}
        P\left(d_{h,k}^{(1)}>(C+3)\varepsilon_0\gamma\sqrt{B}\right)<\delta,\,\,\forall k\in\mathbb{N}^*,
    \end{align*}
    where $B$ is the number of subgraphs.
\end{lemma}

\begin{proof}
    Without loss of generality, we assume that the sampled subgraph at the $k$-th iteration is $\mathcal{V}_1$ and the most recent sampling to $\mathcal{V}_{j}$ is in the $\alpha_{k,j}$-th iteration, $j=1,\ldots, B$. Then we have
    \begin{align*}
        &[\hisH_k^{(1)}]_{\mathcal{V}_{1}} = [f_{\theta_k^{(1)}}(\temH_k^{(0)})]_{\mathcal{V}_{1}}, \,\, [\embH_k^{(1)}]_{\mathcal{V}_{1}}=[f_{\theta_k^{(1)}}(\embH_k^{(0)})]_{\mathcal{V}_{1}},\\
        &[\hisH_k^{(1)}]_{\mathcal{V}_{j}} = [f_{\theta_{\alpha_{k,j}}^{(1)}}(\temH_{\alpha_{k,j}}^{(0)})]_{\mathcal{V}_{j}}, \,\,[\embH_k^{(1)}]_{\mathcal{V}_{j}} = [f_{\theta_k^{(1)}}(\embH_k^{(0)})]_{\mathcal{V}_{j}},\\ 
        &j=2,\ldots, B.
    \end{align*}
    Hence we have
    \begin{align*}
        &\|\hisH_k^{(1)} - \embH_k^{(1)}\|_F^2\\
        ={}&\sum_{j=1}^B\|[\hisH_k^{(1)}]_{\mathcal{V}_{j}}-[\embH_k^{(1)}]_{\mathcal{V}_{j}}\|_F^2\\
        ={}& \|[f_{\theta_k^{(1)}}(\temH_k^{(0)})]_{\mathcal{V}_{1}} - [f_{\theta_k^{(1)}}(\embH_k^{(0)})]_{\mathcal{V}_{1}}\|_F^2\\
        &+\sum_{j=2}^B\| [f_{\theta_{\alpha_{k,j}}^{(1)}}(\temH_{\alpha_{k,j}}^{(0)})]_{\mathcal{V}_{j}} - [f_{\theta_k^{(1)}}(\embH_k^{(0)})]_{\mathcal{V}_{j}} \|_F^2\\
        \leq{}& \|f_{\theta_k^{(1)}}(\temH_k^{(0)}) - f_{\theta_k^{(1)}}(\embH_k^{(0)})\|_F^2\\
        &+\sum_{j=2}^B\| f_{\theta_{\alpha_{k,j}}^{(1)}}(\temH_{\alpha_{k,j}}^{(0)}) - f_{\theta_k^{(1)}}(\embH_k^{(0)}) \|_F^2\\
        \leq{}& \gamma^2C^2\varepsilon_0^2+\sum_{j=2}^B (\|f_{\theta_{\alpha_{k,j}}^{(1)}}(\temH_k^{(0)}) - f_{\theta_{\alpha_{k,j}}^{(1)}}(\embH_{\alpha_{k,j}}^{(0)})\|_F \\
        &\quad\quad\quad\quad\quad +\|f_{\theta_{\alpha_{k,j}}^{(1)}}(\embH_{\alpha_{k,j}}^{(0)}) - f_{\theta_{\alpha_{k,j}}^{(1)}}(\embH_k^{(0)})\|_F\\
        &\quad\quad\quad\quad\quad+\|f_{\theta_{\alpha_{k,j}}^{(1)}}(\embH_{k}^{(0)}) - f_{\theta_k^{(1)}}(\embH_k^{(0)})\|_F)^2\\
        \leq{}& \gamma^2C^2\varepsilon_0^2 + \gamma^2\sum_{j=2}^{B} (\|\temH_k^{(0)} - \embH_{\alpha_{k,j}}^{(0)}\|_F + \|\embH_{\alpha_{k,j}}^{(0)} - \embH_k^{(0)}\|_F\\
        &\quad\quad\quad\quad\quad\quad\quad +\|\theta_{\alpha_{k,j}}^{(1)} - \theta_k^{(1)}\|)^2\\
        \leq{}& \gamma^2C^2\varepsilon_0^2 + \gamma^2 \sum_{j=2}^B ((C+2)\varepsilon_0 + \|\theta_{\alpha_{k,j}}^{(1)} - \theta_k^{(1)}\|)^2\\
        \leq{}& \gamma^2 \sum_{j=1}^B ((C+2)\varepsilon_0+\|\theta_{\alpha_{k,j}}^{(1)} - \theta_k^{(1)}\|)^2
    \end{align*}
    Note that $\alpha_{k,1}=k$. Thus, for $\varepsilon=(C+3)\varepsilon_0\gamma\sqrt{B}$, we have
    \begin{align*}
        &P\left(\|\hisH_k^{(1)} - \embH_k^{(1)}\|_F^2>\varepsilon^2\right)\\
        ={}& 1-P\left(\|\hisH_k^{(1)} - \embH_k^{(1)}\|_F^2 \leq \varepsilon^2\right)\\
        \leq{}& 1-P\left(\sum_{j=1}^B ((C+2)\varepsilon_0+\|\theta_{\alpha_{k,j}}^{(1)} - \theta_k^{(1)}\|)^2 \leq (C+3)^2\varepsilon_0^2B\right).
    \end{align*}
    Next we show that
    \begin{align*}
        P\left(\sum_{j=1}^B ((C+2)\varepsilon_0+\|\theta_{\alpha_{k,j}}^{(1)} - \theta_k^{(1)}\|)^2 \leq (C+3)^2\varepsilon_0^2B\right) > \delta.
    \end{align*}
    Since the gradients $\widetilde{\embg}_{\theta^{(1)}}(\theta_k^{(1)})$ are bounded by $G>0$, then we have $\|\theta_{\alpha_{k,j}}^{(1)} - \theta_k^{(1)}\| \leq (k - \alpha_{k,j}) \eta G$, hence
    \begin{align*}
        &P\left(\sum_{j=1}^B ((C+2)\varepsilon_0+\|\theta_{\alpha_{k,j}}^{(1)} - \theta_k^{(1)}\|)^2 \leq (C+3)^2\varepsilon_0^2B\right)\\
        \geq{}& P\left(\sum_{j=1}^B((C+2)\varepsilon_0 + (k - \alpha_{k,j}) \eta G)^2 \leq (C+3)^2\varepsilon_0^2B\right)\\
        \geq{}& P\left(\bigcap_{j=1}^B \left\{((C+2)\varepsilon_0+(k - \alpha_{k,j})\eta G)^2 \leq (C+3)^2\varepsilon_0^2\right\}\right)\\
        ={}& P\left(\bigcap_{j=2}^B \left\{\alpha_{k,j} \geq k-\frac{\varepsilon_0}{\eta G} \right\}\right)\\
        ={}& 1 - P\left(\bigcup_{j=2}^B\left\{\alpha_{k,j} < k-\frac{\varepsilon_0}{\eta G} \right\}\right)\\
        \geq{}& 1-\sum_{j=2}^B P\left(\alpha_{k,j} < k-\frac{\varepsilon_0}{\eta G}\right).\\
    \end{align*}
    Since we sample a subgraph from $\{\mathcal{V}_j\}_{j=1}^B$ randomly at each iteration, $\alpha_{k,j} < k-\frac{\varepsilon_0}{\eta G}$ means that all the sampled subgraphs from the $\lceil k-\frac{\varepsilon_0}{\eta G} \rceil$-th to the $k$-th iterations are not $\mathcal{V}_j$. Thus, we have
    \begin{align*}
        &P\left(\alpha_{k,j} < k-\frac{\varepsilon_0}{\eta G}\right)\\
        ={}& (1-\frac{1}{B})^{\lfloor\frac{\varepsilon_0}{\eta G}\rfloor},\,\,\forall j=2,\ldots,L.
    \end{align*}
    Hence by letting
    \begin{align*}
        \eta \leq \frac{\varepsilon_0}{G(1-\log_{\frac{B}{B-1}} (\frac{1-\delta}{B-1}))}
    \end{align*}
    we have
    \begin{align*}
        &P\left(\sum_{j=1}^B ((C+2)\varepsilon_0+\|\theta_{\alpha_{k,j}}^{(1)} - \theta_k^{(1)}\|)^2 \leq (C+3)^2\varepsilon_0^2B\right)\\
        \geq{}& 1 - \sum_{j=2}^B (1-\frac{1}{B})^{\lfloor\frac{\varepsilon_0}{\eta G}\rfloor}\\
        ={}& 1-(B-1)(1-\frac{1}{B})^{\lfloor\frac{\varepsilon_0}{\eta G}\rfloor}\\
        >{}& \delta,
    \end{align*}
    which leads to
    \begin{align*}
        P\left(d_{h,k}^{(1)}>(C+3)\varepsilon_0\gamma\sqrt{B}\right) < \delta,\,\,\forall\,k\in\mathbb{N}^*.
    \end{align*}
\end{proof}

By Lemma \ref{prop:singlelayer} we show that for a single layer ConvGNN, if the inputs of the forward passing model of LMC and those of the exact forward passing model are the same, the approximation error of node embeddings $d_{h,k}^{(1)}$ converges to zero in probability since $\varepsilon_0=0$ in this case.




\subsubsection{Proof of Lemma 1 in the main text}

\begin{proof}
    For any $\varepsilon>0$ and $\delta>0$, since the inputs we feed to the forward passing model of LMC and the exact forward passing model are both the node features $\embX$, the distance between the inputs are $d_{h,k}^{(0)}=0$. Thus, we have
    \begin{align*}
        P\left( d_{h,k}^{(0)} > \frac{\varepsilon}{((C+3)\gamma\sqrt{B})^L} \right) = 0 < \delta.
    \end{align*}
    By Lemma \ref{prop:singlelayer} we know that by letting
    \begin{align*}
        \eta \leq \frac{\frac{\varepsilon}{((C+3)\gamma\sqrt{B})^L}}{G(1-\log_{\frac{B}{B-1}} (\frac{1-\delta}{B-1}))}
    \end{align*}
    we have
    \begin{align*}
        P\left( d_{h,k}^{(1)} > \frac{\varepsilon}{((C+3)\gamma\sqrt{B})^{L-1}} \right)< \delta.
    \end{align*}
    By Lemma \ref{prop:singlelayer} again, letting
    \begin{align*}
        \eta \leq \min\{\frac{\frac{\varepsilon}{((C+3)\gamma\sqrt{B})^{L-1}}}{G(1-\log_{\frac{B}{B-1}} (\frac{1-\delta}{B-1}))},\frac{\frac{\varepsilon}{((C+3)\gamma\sqrt{B})^L}}{G(1-\log_{\frac{B}{B-1}} (\frac{1-\delta}{B-1}))}\}
    \end{align*}
    leads to
    \begin{align*}
        P\left( d_{h,k}^{(2)} > \frac{\varepsilon}{((C+3)\gamma\sqrt{B})^{L-2}} \right)< \delta.
    \end{align*}
    And so on, we know that by letting
    \begin{align*}
        \eta \leq \min_{l\in\{1,L\}} \frac{\frac{\varepsilon}{((C+3)\gamma\sqrt{B})^{l}}}{G(1-\log_{\frac{B}{B-1}} (\frac{1-\delta}{B-1}))}
    \end{align*}
    we have
    \begin{align*}
        P\left( d_{h,k}^{(L)} > \varepsilon \right)< \delta.
    \end{align*}
\end{proof}

\subsection{Proof of Lemma 2 in the main text}
We first show that for a two-layer ConvGNN, if the inputs of the backward passing model of LMC and those of the exact backward passing model are close to each other, then the approximation error is close to zero.

\begin{lemma}\label{prop:singlelayer_v}
    For a two-layer ConvGNN, we denote the inputs we feed to the backward passing model of LMC and the exact backward passing model by $\hisV_k^{(2)}$ and $\embV_k^{(2)}$, respectively. Suppose that
    \begin{enumerate}
        \item the mapping $\phi_{\theta^{(2)}}(\embV^{(2)})$ is $\gamma$-Lipschitz for parameter $\theta^{(2)}$ and $\embV^{(2)}$,

        \item the gradients $\widetilde{\embg}_{\theta^{(2)}}(\theta_k^{(2)})$ are bounded by $G$ for $k\in\mathbb{N}^*$,

        \item the differences are bounded by $\varepsilon_0>0$, i.e.,
        \begin{align*}
            &\|\hisV_k^{(2)}-\embV_k^{(2)}\|_F < \varepsilon_0,\,\,\forall\, k\in\mathbb{N}^*,\\
            &\|\embV_{k_1}^{(2)} - \embV_{k_2}^{(2)}\| < \varepsilon_0,\,\,\forall\, k_1,k_2\in\mathbb{N}^*,
        \end{align*}

        \item the differences between temporary embeddings and exact embeddings are dominated by those between historical embeddings and exact embeddings, i.e.,
        \begin{align*}
            \|\temV_k^{(2)} - \embV_k^{(2)}\|_F \leq C\|\hisV_k^{(2)} - \embV_k^{(2)}\|_F < C\varepsilon_0,\,\,\forall\, k\in\mathbb{N}^*,
        \end{align*}
    \end{enumerate}
    then for any $0<\delta<1$, by letting
    \begin{align*}
        \eta \leq \frac{\varepsilon_0}{G(1-\log_{\frac{B}{B-1}} (\frac{1-\delta}{B-1}))},
    \end{align*}
    we have
    \begin{align*}
        P\left(d_{v,k}^{(1)} \geq (C+3)\varepsilon_0\gamma\sqrt{B} \right)<\delta,\,\,\forall k>K,
    \end{align*}
    where $B$ is the number of subgraphs.
\end{lemma}

\begin{proof}
    Without loss of generality, we assume that the sampled subgraph at the $k$-th iteration is $\mathcal{V}_1$ and the most recent sampling to $\mathcal{V}_{j}$ is in the $\alpha_{k,j}$-th iteration, $j=1,\ldots, B$. Then we have
    \begin{align*}
        &[\hisV_k^{(1)}]_{\mathcal{V}_{1}} = [\phi_{\theta_k^{(2)}}(\temV_k^{(2)})]_{\mathcal{V}_{1}}, \,\, [\embV_k^{(1)}]_{\mathcal{V}_{1}}=[\phi_{\theta_k^{(2)}}(\embV_k^{(2)})]_{\mathcal{V}_{1}},\\
        &[\hisV_k^{(1)}]_{\mathcal{V}_{j}} = [\phi_{\theta_{\alpha_{k,j}}^{(2)}}(\temV_{\alpha_{k,j}}^{(2)})]_{\mathcal{V}_{j}}, \,\,[\embV_k^{(1)}]_{\mathcal{V}_{j}} = [\phi_{\theta_k^{(2)}}(\embV_k^{(2)})]_{\mathcal{V}_{j}},\\ 
        &j=2,\ldots, B.
    \end{align*}
    Hence we have
    \begin{align*}
        &\|\hisV_k^{(1)} - \embV_k^{(1)}\|_F^2\\
        ={}&\sum_{j=1}^B\|[\hisV_k^{(1)}]_{\mathcal{V}_{j}}-[\embV_k^{(1)}]_{\mathcal{V}_{j}}\|_F^2\\
        ={}& \|[\phi_{\theta_k^{(2)}}(\temV_k^{(2)})]_{\mathcal{V}_{1}} - [\phi_{\theta_k^{(2)}}(\embV_k^{(2)})]_{\mathcal{V}_{1}}\|_F^2\\
        &+\sum_{j=2}^B\| [\phi_{\theta_{\alpha_{k,j}}^{(2)}}(\temV_{\alpha_{k,j}}^{(2)})]_{\mathcal{V}_{j}} - [\phi_{\theta_k^{(2)}}(\embV_k^{(2)})]_{\mathcal{V}_{j}} \|_F^2\\
        \leq{}& \|\phi_{\theta_k^{(2)}}(\temV_k^{(2)}) - \phi_{\theta_k^{(2)}}(\embV_k^{(2)})\|_F^2\\
        &+\sum_{j=2}^B\| \phi_{\theta_{\alpha_{k,j}}^{(2)}}(\temV_{\alpha_{k,j}}^{(2)}) - \phi_{\theta_k^{(2)}}(\embV_k^{(2)}) \|_F^2\\
        \leq{}& \gamma^2C^2\varepsilon_0^2+\sum_{j=2}^B (\|\phi_{\theta_{\alpha_{k,j}}^{(2)}}(\temV_k^{(2)}) - \phi_{\theta_{\alpha_{k,j}}^{(2)}}(\embV_{\alpha_{k,j}}^{(2)})\|_F \\
        &\quad\quad\quad\quad\quad +\|\phi_{\theta_{\alpha_{k,j}}^{(2)}}(\embV_{\alpha_{k,j}}^{(2)}) - \phi_{\theta_{\alpha_{k,j}}^{(2)}}(\embV_k^{(2)})\|_F\\
        &\quad\quad\quad\quad\quad+\|\phi_{\theta_{\alpha_{k,j}}^{(2)}}(\embV_{k}^{(2)}) - \phi_{\theta_k^{(2)}}(\embV_k^{(2)})\|_F)^2\\
        \leq{}& \gamma^2C^2\varepsilon_0^2 + \gamma^2\sum_{j=2}^{B} (\|\temV_k^{(2)} - \embV_{\alpha_{k,j}}^{(2)}\|_F + \|\embV_{\alpha_{k,j}}^{(2)} - \embV_k^{(2)}\|_F\\
        &\quad\quad\quad\quad\quad\quad\quad +\|\theta_{\alpha_{k,j}}^{(2)} - \theta_k^{(2)}\|)^2\\
        \leq{}& \gamma^2C^2\varepsilon_0^2 + \gamma^2 \sum_{j=2}^B ((C+2)\varepsilon_0 + \|\theta_{\alpha_{k,j}}^{(2)} - \theta_k^{(2)}\|)^2\\
        \leq{}& \gamma^2 \sum_{j=1}^B ((C+2)\varepsilon_0+\|\theta_{\alpha_{k,j}}^{(2)} - \theta_k^{(2)}\|)^2
    \end{align*}
    Note that $\alpha_{k,1}=k$. Thus, for $\varepsilon = (C+3)\varepsilon_0\gamma\sqrt{B}$, we have
    \begin{align*}
        &P\left(\|\hisV_k^{(1)} - \embV_k^{(1)}\|_F^2>\varepsilon^2\right)\\
        ={}& 1-P\left(\|\hisV_k^{(1)} - \embV_k^{(1)}\|_F^2 \leq \varepsilon^2\right)\\
        \leq{}& 1-P\left(\sum_{j=1}^B ((C+2)\varepsilon_0+\|\theta_{\alpha_{k,j}}^{(2)} - \theta_k^{(2)}\|)^2 \leq (C+3)^2\varepsilon_0^2 B \right).
    \end{align*}
    Next we show that
    \begin{align*}
        P\left(\sum_{j=1}^B ((C+2)\varepsilon_0+\|\theta_{\alpha_{k,j}}^{(2)} - \theta_k^{(2)}\|)^2 \leq (C+3)^2\varepsilon_0^2 B\right) > 1 - \delta.
    \end{align*}
    Since the gradients $\widetilde{\embg}_{\theta^{(2)}}(\theta_k^{(2)})$ are bounded by $G>0$, then we have $\|\theta_{\alpha_{k,j}}^{(2)} - \theta_k^{(2)}\| \leq (k - \alpha_{k,j}) \eta G$, hence
    \begin{align*}
        &P\left(\sum_{j=1}^B ((C+2)\varepsilon_0+\|\theta_{\alpha_{k,j}}^{(2)} - \theta_k^{(2)}\|)^2 \leq (C+3)^2\varepsilon_0^2 B\right)\\
        \geq{}& P\left(\sum_{j=1}^B((C+2)\varepsilon_0 + (k - \alpha_{k,j}) \eta G)^2 \leq (C+3)^2\varepsilon_0^2 B\right)\\
        \geq{}& P\left(\bigcap_{j=1}^B \left\{((C+2)\varepsilon_0+(k - \alpha_{k,j})\eta G)^2 \leq (C+3)^2\varepsilon_0^2\right\}\right)\\
        ={}& P\left(\bigcap_{j=2}^B \left\{\alpha_{k,j} \geq k-\frac{\varepsilon_0}{\eta G} \right\}\right)\\
        ={}& 1 - P\left(\bigcup_{j=2}^B\left\{\alpha_{k,j} < k-\frac{\varepsilon_0}{\eta G} \right\}\right)\\
        \geq{}& 1-\sum_{j=2}^B P\left(\alpha_{k,j} < k-\frac{\varepsilon_0}{\eta G}\right).\\
    \end{align*}
    Since we sample a subgraph from $\{\mathcal{V}_j\}_{j=1}^B$ randomly at each iteration, $\alpha_{k,j} < k-\frac{\varepsilon_0}{\eta G}$ means that all the sampled subgraphs from the $\lceil k-\frac{\varepsilon_0}{\eta G} \rceil$-th to the $k$-th iterations are not $\mathcal{V}_j$. Thus, we have
    \begin{align*}
        &P\left(\alpha_{k,j} < k-\frac{\varepsilon_0}{\eta G}\right)\\
        ={}& (1-\frac{1}{B})^{\lfloor\frac{\varepsilon_0}{\eta G}\rfloor},\,\,\forall j=2,\ldots,L.
    \end{align*}
    Hence by letting
    \begin{align*}
        \eta \leq \frac{\varepsilon_0}{G(1-\log_{\frac{B}{B-1}} (\frac{1-\delta}{B-1}))}
    \end{align*}
    we have
    \begin{align*}
        &P\left(\sum_{j=1}^B ((C+2)\varepsilon_0+\|\theta_{\alpha_{k,j}}^{(2)} - \theta_k^{(2)}\|)^2 \leq (C+3)^2\varepsilon_0^2B\right)\\
        \geq{}& 1 - \sum_{j=2}^B (1-\frac{1}{B})^{\lfloor\frac{\varepsilon_0}{\eta G}\rfloor}\\
        ={}& 1-(B-1)(1-\frac{1}{B})^{\lfloor\frac{\varepsilon_0}{\eta G}\rfloor}\\
        >{}& \delta,
    \end{align*}
    which leads to
    \begin{align*}
        P\left(d_{v,k}^{(1)}\geq (C+3) \varepsilon_0\gamma\sqrt{B}\right)<\delta.
    \end{align*}
\end{proof}

By Lemma \ref{prop:singlelayer_v} we show that for a two-layer ConvGNN, if the inputs of the backward passing model of LMC and those of the exact backward passing model are the same, the approximation error of auxiliary variables $d_{v,k}^{(1)}$ converges to zero in probability since $\varepsilon_0=0$ in this case.

\subsubsection{Proof of Lemma 2 in the main text}

\begin{proof}
    For any $\varepsilon>0$ and $\delta>0$, since the inputs we feed to the backward passing model of LMC and the exact backward passing model are both $\embV^{(L)}=\frac{\partial \loss}{\partial \embH}$, the distance between the inputs are $d_{v,k}^{(L)}=0$. Thus, we have
    \begin{align*}
        P\left( d_{v,k}^{(L)} > \frac{\varepsilon}{((C+3)\gamma\sqrt{B})^{L-1}} \right) = 0 < \delta.
    \end{align*}
    By Lemma \ref{prop:singlelayer} we know that by letting
    \begin{align*}
        \eta \leq \frac{\frac{\varepsilon}{((C+3)\gamma\sqrt{B})^{L-1}}}{G(1-\log_{\frac{B}{B-1}} (\frac{1-\delta}{B-1}))}
    \end{align*}
    we have
    \begin{align*}
        P\left( d_{v,k}^{(L-1)} > \frac{\varepsilon}{((C+3)\gamma\sqrt{B})^{L-2}} \right)< \delta.
    \end{align*}
    By Lemma \ref{prop:singlelayer} again, letting
    \begin{align*}
        \eta \leq \min\{\frac{\frac{\varepsilon}{((C+3)\gamma\sqrt{B})^{L-2}}}{G(1-\log_{\frac{B}{B-1}} (\frac{1-\delta}{B-1}))},\frac{\frac{\varepsilon}{((C+3)\gamma\sqrt{B})^{L-1}}}{G(1-\log_{\frac{B}{B-1}} (\frac{1-\delta}{B-1}))}\}
    \end{align*}
    leads to
    \begin{align*}
        P\left( d_{h,k}^{(2)} > \frac{\varepsilon}{((C+3)\gamma\sqrt{B})^{L-2}} \right)< \delta.
    \end{align*}
    And so on, we know that by letting
    \begin{align*}
        \eta \leq \min_{l\in\{1,L-1\}} \frac{\frac{\varepsilon}{((C+3)\gamma\sqrt{B})^{l}}}{G(1-\log_{\frac{B}{B-1}} (\frac{1-\delta}{B-1}))}
    \end{align*}
    we have
    \begin{align*}
        P\left( d_{v,k}^{(1)} > \varepsilon \right)< \delta.
    \end{align*}
\end{proof}

\subsection{Proof of Lemmas 3 and 4}

\setcounter{lemma}{2}
We first introduce some useful lemmas.

\begin{lemma}\label{lemma:cfpi}
    Let the transition function $f:\mathbb{R}^d \rightarrow \mathbb{R}^d$ be $\gamma$-contraction. Consider the block-coordinate fixed-point algorithm with the partition of coordinates to $n$ blocks $\mathbf{I} = (S_1,S_2,\dots,S_n) \in \mathbb{R}^{d \times d}$ and an update rule
    \begin{align*}
        \vecx_{k+1} = (\mathbf{I}- \alpha \vecM^{(k)}) \vecx_k + \alpha \vecM^{(k)} f(\vecx_k)),
    \end{align*}
    where $ \alpha \in (0,1]$ and the stochastic matrix $\vecM^{(k)}$ chosen independently and uniformly from $\{(0,\dots,S_j,\dots,0)|j=1,\dots n\}$ indicates the updated coordinates at the $k$-th iteration. Then, we have
    \begin{align*}
        \mathbb{E}[\| \vecx_{k+1} - \vecx   \|_2^2] \leq \left( 1 - \frac{\alpha}{n}(1-\gamma^2) \right)\| \vecx_k - \vecx   \|_2^2.
    \end{align*}
    Moreover,
    \begin{align*}
        \mathbb{E}[\| \vecx_{k+1} - \vecx   \|_2^2] \leq \left( 1 - \frac{\alpha}{n}(1-\gamma^2) \right)^{k}\| \vecx_{1} - \vecx   \|_2^2.
    \end{align*}
\end{lemma}

\begin{proof}
    As $\vecM^{(k)}$ is chosen uniformly from the set $\{(0,\dots,S_i,\dots,0)|i=1,\dots n\}$, we have
    \begin{align*}
        \mathbb{E}[\vecM^{(k)}] = \frac{\mathbf{I}}{n}.
    \end{align*}
    We first compute the conditional expectation
    \begin{align*}
        &\mathbb{E}[ \| \vecx_{k+1} - \vecx   \|_2^2 | \vecx_k]\\ 
        ={}& \mathbb{E}[ \| \vecx_{k+1} - \vecx_k + \vecx_k - \vecx   \|_2^2 | \vecx_k] \\
        ={}& \mathbb{E}[ \| \vecx_{k+1} - \vecx_k \|_2^2 | \vecx_k] + \| \vecx_k - \vecx   \|_2^2 \\
        &+ 2 \langle \mathbb{E}[\vecx_{k+1}| \vecx_k]  - \vecx_k ,  \vecx_k - \vecx   \rangle \\
        ={}& \| \vecx_k - \vecx   \|_2^2 + \alpha^2 \mathbb{E}[ \| \vecM^{(k)}(\vecx_k - f(\vecx_k)) \|_2^2 | \vecx_k] \\
        &+ 2 \frac{\alpha}{n} \langle f(\vecx_k)  - \vecx_k ,  \vecx_k - \vecx   \rangle.
    \end{align*}
    Notice that 
    \begin{align*}
        &\mathbb{E}[ \| \vecM^{(k)}(\vecx_k - f(\vecx_k)) \|_2^2 | \vecx_k]\\
        ={}& \frac{1}{n} \sum_{j=1}^n \| (0,\dots,S_i,\dots,0)( \vecx_k - f(\vecx_k)) \|_2^2\\
        ={}&\frac{1}{n} \| \vecx_k - f(\vecx_k)\|_2^2
    \end{align*}
    and
    \begin{align*}
        &2 \langle f(\vecx_k)  - \vecx_k ,  \vecx_k - \vecx   \rangle\\
        ={}&\|f(\vecx_k) - \vecx   \|_2^2 - \|f(\vecx_k)  - \vecx_k\|_2^2 - \|\vecx_k - \vecx  \|_2^2.
    \end{align*}
    Combining the above equalities and the contraction property of $f$, we have
    \begin{align*}
        &\mathbb{E}[ \| \vecx_{k+1} - \vecx   \|_2^2 | \vecx_k]\\
        ={}& (1-\frac{\alpha}{n})\| \vecx_k - \vecx   \|_2^2 - \alpha \frac{1-\alpha}{n} \| \vecx_k - f(\vecx_k)\|_2^2 \\
        &+\frac{\alpha}{n} (\|f(\vecx_k) - \vecx   \|_2^2) \\
        \leq{}& (1-\frac{\alpha}{n})\| \vecx_k - \vecx   \|_2^2 - \alpha \frac{1-\alpha}{n} \| \vecx_k - f(\vecx_k)\|_2^2\\ 
        &+\frac{\alpha \gamma^2}{n} (\|\vecx_k - \vecx   \|_2^2) \\
        \leq{}& (1-\frac{\alpha}{n}(1-\gamma^2)) \| \vecx_k - f(\vecx_k)\|_2^2.
    \end{align*}
    By the law of total expectation, we have
    \begin{align*}
        \mathbb{E}[ \| \vecx_{k+1} - \vecx   \|_2^2] \leq (1-\frac{\alpha}{n}(1-\gamma^2) \mathbb{E}[ \| \vecx_k - \vecx   \|_2^2 ].
    \end{align*}
    Finally, we recursively deduce that
    \begin{align*}
        \mathbb{E}[\| \vecx_{k+1} - \vecx   \|_2^2] \leq \left( 1 - \frac{\alpha}{n}(1-\gamma^2) \right)^{k} \| \vecx_{1} - \vecx   \|_2^2.
    \end{align*}
\end{proof}

By Lemma \ref{lemma:cfpi}, we can deduce that the expectation of the approximate errors for the representations $d^{(k)}_h = \sqrt{ \mathbb{E}[\|\overline{\embH}^{(k)} - [\embH ]^{(k)}\|_F^2]}$ and the auxiliary variables $d^{(k)}_v = \sqrt{ \mathbb{E}[\|\overline{\embV}^{(k)} - \embV ^{(k)}\|_F^2]}$ decrease if the parameters $\theta,w$ change slow. Moreover, we can set the learning rate $\eta$ to be a such small value that the parameters $\theta,w$ change slowly.

\begin{lemma}\label{lemma:amp_cfpi}
    Suppose that the message passing function $f_{\theta}$ is $L$-Lipschitz for parameter $\theta$ and $\gamma$-contraction for $\embH$. Let $d^{(k)}_h = \sqrt{ \mathbb{E}[\|\overline{\embH}^{(k)} - [\embH ]^{(k)}\|_F^2]}$. If $\|\theta^{(k+1)}-\theta^{(k)}\|_F \leq \varepsilon$ at the $k$-th iteration, then we have
    \begin{align*}
        d^{(k+1)}_{h} \leq \rho d^{(k)}_h + K \varepsilon,
    \end{align*}
    where $\rho = \sqrt{(1-\frac{\alpha}{I}(1-\gamma^2))}$ and $K = \frac{L}{1-\gamma}$.
\end{lemma}

\begin{proof}
    According to Lemma \ref{lemma:cfpi}, we have
    \begin{align*}
        &\mathbb{E}[\|\overline{\embH}^{(k+1)} - [\embH ]^{(k+1)}\|_F^2]\\
        \leq{}& \rho^2\mathbb{E}[\|\overline{\embH}^{(k)} - [\embH ]^{(k+1)}\|_F^2] \\
        \leq{}& \rho^2\mathbb{E}[(\|\overline{\embH}^{(k)} - [\embH ]^{(k)}\|_F + \|[\embH ]^{(k)} - [\embH ]^{(k+1)}\|_F )^2].
    \end{align*}
    As the transition function $f_{\theta}$ is  $L$-Lipschitz, we have
    \begin{align*}
        &\|[\embH ]^{(k)} - [\embH ]^{(k+1)}\|_F \\
        ={}& \|f_{\theta^{(k)}}([\embH ]^{(k)}) - f_{\theta^{(k+1)}}[\embH ]^{(k+1)}\|_F \\
        \leq{}& \|f_{\theta^{(k)}}([\embH ]^{(k)}) - f_{\theta^{(k)}}[\embH ]^{(k+1)}\|_F\\
        &+\|f_{\theta^{(k)}}([\embH ]^{(k+1)}) - f_{\theta^{(k+1)}}[\embH ]^{(k+1)}\|_F\\
        \leq{}& \gamma \|[\embH ]^{(k)} - [\embH ]^{(k+1)}\|_F + L\|\theta^{(k+1)}-\theta^{(k)}\|_F.
    \end{align*}
    By rearranging the terms, we have
    \begin{align*}
        \|[\embH ]^{(k)} - [\embH ]^{(k+1)}\|_F \leq \frac{L}{1-\gamma} \|\theta^{(k+1)}-\theta^{(k)}\|_F\leq K \varepsilon.
    \end{align*}
    Combining the above inequalities, we have
    \begin{align*}
        &\mathbb{E}[\|\overline{\embH}^{(k+1)} - [\embH ]^{(k+1)}\|_F^2]\\
        \leq{}& \rho^2 \mathbb{E}[(\|\overline{\embH}^{(k)} - [\embH ]^{(k)}\|_F + K \varepsilon)^2]\\
        ={}& \rho^2 (\mathbb{E}[\|\overline{\embH}^{(k)} - [\embH ]^{(k)}\|_F^2]\\
        &+ 2 \mathbb{E}[\|\overline{\embH}^{(k)} - [\embH ]^{(k)}\|_F] K \varepsilon  + (K \varepsilon)^2)\\
        \leq{}& \rho^2 (\mathbb{E}[\|\overline{\embH}^{(k)} - [\embH ]^{(k)}\|_F^2]\\
        &+ 2 \sqrt{\mathbb{E}[\|\overline{\embH}^{(k)} - [\embH ]^{(k)}\|_F^2} K \varepsilon  + (K \varepsilon)^2)\\
        ={}& \rho^2 (\sqrt{\mathbb{E}[\|\overline{\embH}^{(k)} - [\embH ]^{(k)}\|_F^2} + K\varepsilon)^2.
    \end{align*}
    The claims follows immediately.
\end{proof}

\begin{lemma}\label{lemma:amp_cfpi_b}
    Suppose that the mapping $\phi_{\theta,w}(\embV) = \nabla_{\vec{\embH}} f_{\theta} \embV + \nabla_{\vec{\embH}} \loss$ is $L$-Lipschitz for parameters $\theta,w$ and $\gamma$-contraction for $\embV$. Let $d^{(k)}_v = \sqrt{ \mathbb{E}[\|\overline{\embV}^{(k)} - \embV ^{(k)}\|_F^2]}$. If $\|\theta^{(k+1)}-\theta^{(k)}\|_F \leq \varepsilon$ and $\|w^{(k+1)}-w^{(k)}\|_F \leq \varepsilon$ at the $k$-th iteration, then we have
    \begin{align*}
        d^{(k+1)}_{v} \leq \rho d^{(k)}_v + K \varepsilon,
    \end{align*}
    where $\rho = \sqrt{(1-\frac{\alpha}{I}(1-\gamma^2))}$ and $K = \frac{2L}{1-\gamma}$.
\end{lemma}

\begin{proof}
    According to Lemma \ref{lemma:cfpi}, we have
    \begin{align*}
        &\mathbb{E}[\|\overline{\embV}^{(k+1)} - [\embV ]^{(k+1)}\|_F^2]\\
        \leq{}& \rho^2\mathbb{E}[\|\overline{\embV}^{(k)} - [\embV ]^{(k+1)}\|_F^2] \\
        \leq{}& \rho^2\mathbb{E}[(\|\overline{\embV}^{(k)} - \embV ^{(k)}\|_F + \|\embV ^{(k)} - [\embV ]^{(k+1)}\|_F )^2].
    \end{align*}
    As the transition function $f_{\theta}$ is  $L$-Lipschitz, we have
    \begin{align*}
        &\|\embV ^{(k)} - [\embV ]^{(k+1)}\|_F\\
        ={}& \|\phi_{\theta^{(k)},w^{(k)}}(\embV ^{(k)}) - \phi_{\theta^{(k+1)},w^{(k+1)}}[\embV ]^{(k+1)}\|_F \\
        \leq{}& \|\phi_{\theta^{(k)},w^{(k)}}(\embV ^{(k)}) - \phi_{\theta^{(k)},w^{(k)}}[\embV ]^{(k+1)}\|_F \\
        &+ \|\phi_{\theta^{(k)},w^{(k)}}([\embV ]^{(k+1)}) - \phi_{\theta^{(k+1)},w^{(k+1)}}[\embV ]^{(k+1)}\|_F\\
        \leq{}& \gamma \|\embV ^{(k)} - [\embV ]^{(k+1)}\|_F\\
        &+ L\sqrt{\|\theta^{(k+1)}-\theta^{(k)}\|_F^2+ \|w-w^{(k)}\|_F^2}\\
        \leq{}& \gamma \|\embV ^{(k)} - [\embV ]^{(k+1)}\|_F\\
        &+ L(\|\theta^{(k+1)}-\theta^{(k)}\|_F+ \|w-w^{(k)}\|_F).
    \end{align*}
    By rearranging the terms, we have
    \begin{align*}
        &\|\embV ^{(k)} - [\embV ]^{(k+1)}\|_F \\
        \leq{}& \frac{L}{1-\gamma} (\|\theta^{(k+1)}-\theta^{(k)}\|_F+ \|w-w^{(k)}\|_F)\\
        \leq{}& K \varepsilon.
    \end{align*}
    Combining the above inequalities, we have
    \begin{align*}
        &\mathbb{E}[\|\overline{\embV}^{(k+1)} - [\embV ]^{(k+1)}\|_F^2]\\
        \leq{}& \rho^2 \mathbb{E}[(\|\overline{\embV}^{(k)} - \embV ^{(k)}\|_F + K \varepsilon)^2]\\
        ={}& \rho^2 (\mathbb{E}[\|\overline{\embV}^{(k)} - \embV ^{(k)}\|_F^2]\\
        &+ 2 \mathbb{E}[\|\overline{\embV}^{(k)} - \embV ^{(k)}\|_F] K \varepsilon + (K \varepsilon)^2)\\
        \leq{} & \rho^2 (\mathbb{E}[\|\overline{\embV}^{(k)} - \embV ^{(k)}\|_F^2]\\
        &+ 2 \sqrt{\mathbb{E}[\|\overline{\embV}^{(k)} - \embV ^{(k)}\|_F^2} K \varepsilon + (K \varepsilon)^2)\\
        ={}& \rho^2 (\sqrt{\mathbb{E}[\|\overline{\embV}^{(k)} - \embV ^{(k)}\|_F^2} + K\varepsilon)^2.
    \end{align*}
    The claims follows immediately.
\end{proof}

\subsubsection{Proof of Lemma 1}
\begin{proof}
    The update rule in LMC implies that
    \begin{align*}
        \|\theta^{(k+1)} - \theta^{(k)}\|_F &=  \eta \|\widetilde{\mathbf{g}}_{\theta}(\theta^{(k)})\|_F \leq \eta B.
    \end{align*}
    According to Lemma \ref{lemma:amp_cfpi}, we have
    \begin{align*}
        d_h^{(k+1)} - \frac{KB}{1-\rho} \eta &\leq \rho( d_h^{(k)} - \frac{KB}{1-\rho} \eta)\\
        &\leq \cdots\\
        &\leq  \rho^k (  d_h^{(1)} - \frac{KB}{1-\rho} \eta)\\
        &\leq \rho^k B
    \end{align*}
    The claim follows immediately.
\end{proof}

\subsubsection{Proof of Lemma 2}
\begin{proof}
    The update rule in LMC implies that
    \begin{align*}
        \|\theta^{(k+1)} - \theta^{(k)}\|_F &=  \eta \|\widetilde{\mathbf{g}}_{\theta}(\theta^{(k)})\|_F \leq \eta B
    \end{align*}
    and
    \begin{align*}
        \|w^{(k+1)} - w^{(k)}\|_F &=  \eta \|\widetilde{\mathbf{g}}_{w}(w^{(k)})\|_F \leq \eta B.
    \end{align*}
    According to Lemma \ref{lemma:amp_cfpi_b}, we have
    \begin{align*}
        d_v^{(k+1)} - \frac{KB}{1-\rho} \eta &\leq \rho( d_v^{(k)} - \frac{KB}{1-\rho} \eta)\\
        &\leq \cdots\\
        &\leq \rho^k (  d_v^{(1)} - \frac{KB}{1-\rho} \eta)\\
        &\leq \rho^k B
    \end{align*}
    The claim follows immediately.
\end{proof}

\subsection{Proof of Theorems 3 and 4}

\subsubsection{Sufficient conditions for convergence}

We first give sufficient conditions for convergence.
\begin{lemma}\label{lemma:suff_conv}
    Suppose that function $f:\mathbb{R}^{n} \to \mathbb{R}$ is continuously differentiable. Consider an optimization algorithm with any bounded initialization $\vecx_1$ and an update rule in the form of
    \begin{align*}
        \vecx_{k+1} = \vecx_{k} - \eta \vecd(\vecx_k),
    \end{align*}
    where $\eta>0$ is the learning rate and $\vecd(\vecx_k)$ is the estimated gradient that can be seen as a stochastic vector depending on $\vecx_k$. Let the estimation error of the gradient be $\Delta_{k} = \vecd(\vecx_k) - \nabla f(\vecx_k) $. Suppose that
    \begin{enumerate}
        \item the optimal value $f^*  = \inf_{\vecx} f(\vecx)$ is bounded; \label{con:1}
        
        \item the gradient of $f$ is $\gamma$-Lipschitz, i.e., \label{con:2}
        \begin{align*}
            \|\nabla f(\vecy) - \nabla f(\vecx)\|_2 \leq \gamma\|\vecy - \vecx\|_2,\,\forall\,\vecx,\vecy \in \mathbb{R}^{n};
        \end{align*}
        
        \item there exists $G_0>0$ that does not depend on $\eta$ such that
        \label{con:3}
        \begin{align*}
            \mathbb{E}[\|\Delta_{k}\|_2^2] \leq G_0,\,\forall\, k\in\mathbb{N}^*;
        \end{align*}

        \item there exists $N\in\mathbb{N}^*$ such that
        \begin{align*}
            |\mathbb{E}[\langle \nabla f(\vecx_k),\Delta_{k} \rangle]|<\frac{1}{\sqrt{N}},\,\,\forall\,k\in\mathbb{N}^*;
        \end{align*}
    \end{enumerate}
    then by letting $\eta=\min\{\frac{1}{\gamma},\frac{1}{\sqrt{N}}\}$, we have
    \begin{align*}
        \mathbb{E}[ \|\nabla f(\vecx_{R})\|_2^2] \leq \frac{2(f(\vecx_1) - f^*)+\gamma G_0 + 1}{\sqrt{N}}=O(\frac{1}{\sqrt{N}}),
    \end{align*}
     where $R$ is chosen uniformly from $[N]$.
\end{lemma}

\begin{proof}
    As the gradient of $f$ is $\gamma$-Lipschitz, we have
    \begin{align*}
    	f(\vecy)={}&f(\vecx)+\int_{\vecx}^{\vecy}\nabla f(\mathbf{z})\rmd\mathbf{z}\\
    	={}&f(\vecx)+\int_0^1\langle\nabla f(\vecx+t(\vecy-\vecx)), \vecy-\vecx\rangle \rmd t\\
    	={}&f(\vecx)+\langle\nabla f(\vecx),\vecy-\vecx\rangle\\
    	&+\int_0^1\langle\nabla f(\vecx+t(\vecy-\vecx))-\nabla f(\vecx), \vecy-\vecx\rangle \rmd t\\
    	\leq{}&f(\vecx)+\langle\nabla f(\vecx),\vecy-\vecx\rangle\\
    	&+\int_0^1\|\nabla f(\vecx+t(\vecy-\vecx))-\nabla f(\vecx)\|_2\| \vecy-\vecx\|_2 \rmd t\\
    	\leq{}&f(\vecx)+\langle\nabla f(\vecx),\vecy-\vecx\rangle+\int_0^1\gamma t\|\vecy-\vecx\|^2_2\rmd t\\
    	\leq{}&f(\vecx)+\langle\nabla f(\vecx), \vecy-\vecx\rangle+\frac{\gamma}{2}\|\vecy-\vecx\|_2^2,
    \end{align*}
    Then, we have
    \begin{align*}
        &f(\vecx_{k+1})\\ \leq{}&f(\vecx_k) + \langle  \nabla f(\vecx_k), \vecx_{k+1} - \vecx_k \rangle+\frac{\gamma}{2}\|\vecx_{k+1}-\vecx_k\|_2^2 \\
        ={}& f(\vecx_k) - \eta \langle \nabla f(\vecx_k), \vecd(\vecx_k) \rangle+ \frac{\eta^2 \gamma}{2}\|\vecd(\vecx_k)\|_2^2 \\
        ={}& f(\vecx_k) - \eta \langle \nabla f(\vecx_k), \Delta_{k} \rangle- \eta \|\nabla f(\vecx_k)\|_2^2\\
        &+ \frac{\eta^2 \gamma}{2}(\|\Delta_{k}\|_2^2+\|\nabla f(\vecx_k)\|_2^2 +2\langle \Delta_{k}, \nabla f(\vecx_k) \rangle)\\
        ={}& f(\vecx_k)  - \eta (1-\eta \gamma) \langle \nabla f(\vecx_k), \Delta_{k} \rangle\\
        &- \eta (1-\frac{\eta \gamma}{2}) \|\nabla f(\vecx_k)\|_2^2 + \frac{\eta^2 \gamma}{2}\|\Delta_{k}\|_2^2.
    \end{align*}
    By taking expectation of both sides, we have
    \begin{align*}
        \mathbb{E}[f(\vecx_{k+1})] \leq{}&\mathbb{E}[f(\vecx_k)] - \eta (1-\eta \gamma)  \mathbb{E}[\langle \nabla f(\vecx_k), \Delta_{k} \rangle]\\
        &- \eta (1-\frac{\eta \gamma}{2}) \mathbb{E}[ \|\nabla f(\vecx_k)\|_2^2]+ \frac{\eta^2 \gamma}{2}\mathbb{E}[\|\Delta_{k}\|_2^2].
    \end{align*}
    By summing up the above inequalities for $k\in[N]$ and dividing both sides by $N \eta(1-\frac{\eta \gamma}{2})$, we have
    \begin{align*}
        &\frac{\sum_{k=1}^{N} \mathbb{E}[ \|\nabla f(\vecx_{k})\|_2^2]}{N}\\
        \leq{}& \frac{f(\vecx_{1}) - \mathbb{E}[f(\vecx_{N})]}{N \eta(1-\frac{\eta \gamma}{2}) } + \frac{\eta \gamma}{2-\eta \gamma} \frac{\sum_{k=1}^N \mathbb{E}[\|\Delta_{k}\|_2^2]}{N}\\
        &- \frac{(1-\eta \gamma)}{(1-\frac{\eta \gamma}{2})} \frac{\sum_{k=1}^{N} \mathbb{E}[\langle \nabla f(\vecx_{k}), \Delta_{k} \rangle]}{N} \\
        \leq{}& \frac{f(\vecx_{1}) - f^* }{N \eta(1-\frac{\eta \gamma}{2}) } + \frac{\eta \gamma}{2-\eta \gamma} \frac{\sum_{k=1}^N \mathbb{E}[\|\Delta_{k}\|_2^2]}{N}\\
        &+ \frac{\sum_{k=1}^{N} |\mathbb{E}[\langle \nabla f(\vecx_{k}), \Delta_{k} \rangle]|}{N},
    \end{align*}
    where the second inequality comes from $\eta \gamma>0$ and $f(\vecx_k) \geq f^* $. According to the above conditions, we have
    \begin{align*}
        &\frac{\sum_{k=1}^{N} \mathbb{E}[ \|\nabla f(\vecx_{k})\|_2^2]}{N}\\
        \leq{}& \frac{f(\vecx_{1}) - f^* }{N \eta(1-\frac{\eta \gamma}{2}) } + \frac{\eta \gamma}{2 - \eta \gamma}G+\frac{1}{\sqrt{N}}.
    \end{align*}
    Notice that
    \begin{align*}
        \mathbb{E}[ \|\nabla f(\vecx_{R})\|_2^2] &= \mathbb{E}_R[\mathbb{E}\|[\nabla f(\vecx_{R})\|_2^2\mid R]]\\
        &=\frac{\sum_{k=1}^{N} \mathbb{E}[ \|\nabla f(\vecx_{k})\|_2^2]}{N},
    \end{align*}
    where $R$ is uniformly chosen from $[N]$, hence we have
    \begin{align*}
        &\mathbb{E}[ \|\nabla f(\vecx_{R})\|_2^2]\\
        \leq{}&\frac{f(\vecx_{1}) - f^* }{N \eta(1-\frac{\eta \gamma}{2}) } + \frac{\eta \gamma}{2 - \eta \gamma}G+\frac{1}{\sqrt{N}}.
    \end{align*}
    By letting $\eta=\min\{\frac{1}{\gamma}, \frac{1}{\sqrt{N}}\}$, we have
    \begin{align*}
        &\mathbb{E}[ \|\nabla f(\vecx_{R})\|_2^2]\\
        \leq{}& \frac{2(f(\vecx_{1})-f^*)}{\sqrt{N}}+\frac{\gamma G}{\sqrt{N}} + \frac{1}{\sqrt{N}}\\
        ={}& \frac{2(f(\vecx_{1})-f^*)+\gamma G + 1}{\sqrt{N}}\\
        ={}& O(\frac{1}{\sqrt{N}}).
    \end{align*}
\end{proof}

\subsubsection{Proof of Theorem 3}

Given an $L$-layer ConvGNN, forllowing \cite{vrgcn}, we directly assume that:
\begin{enumerate}
    \item the optimal value
    \begin{align*}
        \mathcal{L}^*=\inf\limits_{w,\theta^{(1)},\ldots,\theta^{(L)}} \mathcal{L}
    \end{align*}
    is bounded by $G>0$;

    \item the gradients of $\mathcal{L}$ with respect to parameters $w$ and $\theta^{(l)}$, i.e.,
    \begin{align*}
        \nabla_{w}\mathcal{L},\,\nabla_{\theta^{(l)}}\mathcal{L}
    \end{align*}
    are $\gamma$-Lipschitz for $\forall l\in[L]$.
\end{enumerate}
We denote the mini-batch gradients computed by LMC with normallization techniques by
\begin{align*}
    \widetilde{\embg}_w(w_k)=\frac{B}{|\mathcal{V}_{L}|} \sum_{v_j\in\mathcal{V}_{L_{\mathcal{B}_k}}} \nabla_w \ell_{w_k}([\hisH_k]_j, y_j)
\end{align*}
and
\begin{align*}
    \widetilde{\embg}_{\theta^{(l)}}(\theta^{(l)}_k) = B\sum_{v_j\in\mathcal{V}_{\mathcal{B}_k}} (\nabla_{\theta^{(l)}} u_j(\hisH_k^{(l-1)};\theta^{(l)}_k))[\hisV^{(l)}_k]_j,
\end{align*}
where $B$ is the number of subgraphs by partitions and we omit the sampled subgraph $\mathcal{V}_{\mathcal{B}_k}$ when writing $\widetilde{\embg}_w(w_k)$ and $\widetilde{\embg}_{\theta^{(l)}}(\theta^{(l)}_k)$. We denote the exact mini-batch gradients with normalization techniques by
\begin{align*}
    \embg_w(w_k)=\frac{B}{|\mathcal{V}_{L}|} \sum_{v_j\in\mathcal{V}_{L_{\mathcal{B}_k}}} \nabla_w \ell_{w_k}([\embH_k]_j, y_j)
\end{align*}
and
\begin{align*}
    \embg_{\theta^{(l)}}(\theta^{(l)}_k) = B\sum_{v_j\in\mathcal{V}_{\mathcal{B}_k}} (\nabla_{\theta^{(l)}} u_j(\embH_k^{(l-1)};\theta^{(l)}_k))[\embV^{(l)}_k]_j,
\end{align*}
The approximation error of gradients is denoted by
\begin{align*}
    \Delta_{w,k} \triangleq \widetilde{\embg}_w(w_k;\mathcal{V}_{\mathcal{B}_k}) - \nabla_w \loss(w_k)
\end{align*}
and
\begin{align*}
    \Delta_{\theta^{(l)},k} \triangleq \widetilde{\embg}_{\theta^{(l)}}(\theta_k^{(l)};\mathcal{V}_{\mathcal{B}_k}) - \nabla_{\theta^{(l)}} \loss(\theta_k^{(l)}).
\end{align*}
To show the convergence of LMC by Lemma \ref{lemma:suff_conv}, it suffices to show that
\begin{enumerate}[resume]
    \item there exists $G_0>0$ that does not depend on $\eta$ such that
    \begin{align*}
        &\mathbb{E}[\|\Delta_{w,k}\|_2^2]\leq G_0,\,\,\forall\,k\in\mathbb{N}^*,\\
        &\mathbb{E}[\|\Delta_{\theta^{(l)},k}\|_2^2]\leq G_0,\,\,\forall\,l\in[L],\,k\in\mathbb{N}^*;
    \end{align*}

    \item for any $N\in\mathbb{N}^*$, we have
    \begin{align*}
        &|\mathbb{E}[\langle \nabla_{w}\loss, \Delta_{w,k} \rangle]| \leq \frac{1}{\sqrt{N}},\,\,\forall\,k\in\mathbb{N}^*,\\
        &|\mathbb{E}[\langle \nabla_{\theta^{(l)}}\loss, \Delta_{\theta^{(l)},k} \rangle]| \leq \frac{1}{\sqrt{N}},\,\,\forall\,l\in[L],\,k\in\mathbb{N}^*
    \end{align*}
    by letting
    \begin{align*}
        \eta \leq \eta_0 \triangleq \min \{\eta_{w,0}, \eta_{\theta,0}\},
    \end{align*}
    where
    \begin{align*}
        &\eta_{w,0} \triangleq \min_{l\in\{1,L\}} \frac{\frac{1}{2G^2B\gamma\sqrt{N} ((C+3)\gamma \sqrt{B} )^l } }{1-\log_{\frac{B}{B-1}} (\frac{4G^2B\gamma \sqrt{N}-1}{4G^2B(B-1)\gamma \sqrt{N} }) },\\
        &\eta_{\theta,0} \triangleq \min\{ \eta_{h,0}, \eta_{v,0} \},\\
        &\eta_{h,0} \triangleq \min_{l\in\{1,L\}} \frac{\frac{1}{4G^3B|\mathcal{V}|\gamma\sqrt{N} ((C+3)\gamma \sqrt{B} )^l } }{1-\log_{\frac{B}{B-1}} (\frac{8G^3B|\mathcal{V}|\gamma \sqrt{N}-1}{8G^3B(B-1)|\mathcal{V}|\gamma \sqrt{N} }) },\\
        &\eta_{v,0} \triangleq \min_{l\in\{1,L-1\}} \frac{\frac{1}{4G^3B|\mathcal{V}|\sqrt{N} ((C+3)\gamma \sqrt{B} )^l } }{1-\log_{\frac{B}{B-1}} (\frac{8G^3B|\mathcal{V}| \sqrt{N}-1}{8G^3B(B-1)|\mathcal{V}| \sqrt{N} }) }.
    \end{align*}
\end{enumerate}

\begin{lemma}
    If $\|\widetilde{\mathbf{g}}_w(w_k)\|_2$ and $\|\nabla_{w}\loss(w_k)\|_2$ are bounded by $G>0$ for $k\in\mathbb{N}^*$, then we have
    \begin{align*}
        \mathbb{E}[\|\Delta_{w,k}\|_2^2] \leq G_0 \triangleq 4G^2,\,\,\forall\,k\in\mathbb{N}^*.
    \end{align*}
\end{lemma}
\begin{proof}
    We have
    \begin{align*}
        &\mathbb{E}[\|\Delta_{w,k}\|_2^2]\\
        ={}& \mathbb{E}[\|\widetilde{\mathbf{g}}_w(w_k) - \nabla_{w}\loss(w_k)\|_2^2]\\
        \leq{}& 2(\mathbb{E}[\|\widetilde{\mathbf{g}}_w(w_k)\|_2^2] + \mathbb{E}[\|\nabla_{w}\loss(w_k)\|_2^2])\\
        \leq{}& 4G^2.
    \end{align*}
\end{proof}

\begin{lemma}
    If $\|\widetilde{\mathbf{g}}_{\theta^{(l)}}(\theta^{(l)}_k)\|_2$ and $\|\nabla_{\theta^{(l)}}\loss(\theta^{(l)}_k)\|_2$ are bounded by $G>0$ for $k\in\mathbb{N}^*$ and $l\in[L]$, then we have
    \begin{align*}
        \mathbb{E}[\|\Delta_{\theta^{(l)},k}\|_2^2] \leq G_0 \triangleq 4G^2, \,\,\forall\,l\in[L],\,k\in\mathbb{N}^*.
    \end{align*}
\end{lemma}
\begin{proof}
    We have
    \begin{align*}
        &\mathbb{E}[\|\Delta_{\theta^{(l)},k}\|_2^2]\\
        ={}& \mathbb{E}[\|\widetilde{\mathbf{g}}_{\theta^{(l)}}(\theta^{(l)}_k) - \nabla_{\theta^{(l)}}\loss(\theta^{(l)}_k)\|_2^2]\\
        \leq{}& 2(\mathbb{E}[\|\widetilde{\mathbf{g}}_{\theta^{(l)}}(\theta^{(l)}_k)\|_2^2] + \mathbb{E}[\|\nabla_{\theta^{(l)}}\loss(\theta^{(l)}_k)\|_2^2])\\
        \leq{}& 4G^2.
    \end{align*}
\end{proof}

\begin{lemma}
    Suppose some conditions. For any $N\in\mathbb{N}^*$, we have
    \begin{align*}
        |\mathbb{E}[\langle \nabla_{w}\loss, \Delta_{w,k} \rangle]| \leq \frac{1}{\sqrt{N}},\,\,\forall\,k\in\mathbb{N}^*
    \end{align*}
    by letting
    \begin{align*}
        \eta \leq \eta_{w,0} \triangleq \min_{l\in\{1,L\}} \frac{\frac{1}{2G^2B\gamma\sqrt{N} ((C+3)\gamma \sqrt{B} )^l } }{1-\log_{\frac{B}{B-1}} (\frac{4G^2B\gamma \sqrt{N}-1}{4G^2B(B-1)\gamma \sqrt{N} }) }.
    \end{align*}
\end{lemma}
\begin{proof}
    We have
    \begin{align*}
        &|\mathbb{E}[\langle \nabla_{w}\loss, \Delta_{w,k} \rangle]|\\
        ={}& |\mathbb{E}[\langle \nabla_{w}\loss, \widetilde{\mathbf{g}}_w(w_k) - \mathbf{g}_w(w_k) \rangle]|\\
        \leq{}& \mathbb{E}[\|\nabla_{w}\loss\|_2 \|\widetilde{\mathbf{g}}_w(w_k) - \mathbf{g}_w(w_k) \|_2]\\
        \leq{}& \frac{GB}{|\mathcal{V}_L|}\mathbb{E}[\sum_{v_j \in \mathcal{V}_{L_{\mathcal{B}_k}]}} \|\nabla_{w}\ell_{w_k}([\hisH_k]_j,y_j) - \nabla_{w}\ell_{w_k}([\embH_k]_j,y_j)\|_2]\\
        \leq{}& \frac{GB\gamma}{|\mathcal{V}_L|} \mathbb{E}[\sum_{v_j \in \mathcal{V}_{L_{\mathcal{B}_k}]}}\|[\hisH_k]_j-[\embH_k]_j\|_2]\\
        \leq{}& \frac{GB\gamma}{|\mathcal{V}_L|} \mathbb{E}[\sum_{v_j \in \mathcal{V}_{L_{\mathcal{B}_k}]}}\|\hisH_k-\embH_k\|_F]\\
        \leq{}& \frac{GB\gamma}{|\mathcal{V}_L|} \cdot |\mathcal{V}_L|\cdot\mathbb{E}[\|\hisH_k-\embH_k\|_F]\\
        ={}& GB\gamma \cdot \mathbb{E}[\| \hisH_k - \embH_k \|_F].
    \end{align*}
    For any $N\in\mathbb{N}^*$, let $\varepsilon_0=\frac{1}{2GB\gamma\sqrt{N}}>0$ and $\delta_0=\frac{1}{4G^2B\gamma\sqrt{N}}>0$, by Lemma 1 in the main text we know that
    \begin{align*}
        P(\|\hisH_k - \embH_k\|_F > \varepsilon_0) < \delta_0,\,\,\forall k\in\mathbb{N}^*,\,\,{\rm if}\,\,\eta\leq\eta_{w,0}.
    \end{align*}
    Therefore, by letting $\eta\leq\eta_{w,0}$ we have
    \begin{align*}
        &\mathbb{E}[\| \hisH_k - \embH_k \|_F]\\
        ={}& \mathbb{E}[\| \hisH_k - \embH_k \|_F \cdot \mathbbm{1}_{\{\| \hisH_k - \embH_k \|_F<\varepsilon_0}\} ]\\
        &+ \mathbb{E}[\| \hisH_k - \embH_k \|_F \cdot \mathbbm{1}_{\{\| \hisH_k - \embH_k \|_F\geq\varepsilon_0\}} ]\\
        \leq{}& \varepsilon_0 + 2G\delta_0\\
        ={}& \frac{1}{GB\gamma\sqrt{N}},
    \end{align*}
    which leads to
    \begin{align*}
        |\mathbb{E}[\langle \nabla_{w}\loss, \Delta_{w,k} \rangle]| \leq \frac{1}{\sqrt{N}},\,\,\forall\,k\in\mathbb{N}^*.
    \end{align*}
\end{proof}

\begin{lemma}
    Suppose some conditions. For any $N\in\mathbb{N}^*$, we have
    \begin{align*}
        |\mathbb{E}[\langle \nabla_{\theta^{(l)}}\loss, \Delta_{\theta^{(l)},k} \rangle]| \leq \frac{1}{\sqrt{N}},\,\,\forall\, l\in[L],\,k\in\mathbb{N}^*
    \end{align*}
    by letting
    \begin{align*}
        \eta\leq \eta_{\theta,0} \triangleq \min\{ \eta_{h,0}, \eta_{v,0} \},
    \end{align*}
    where
    \begin{align*}
        \eta_{h,0} \triangleq \min_{l\in\{1,L\}} \frac{\frac{1}{4G^3B|\mathcal{V}|\gamma\sqrt{N} ((C+3)\gamma \sqrt{B} )^l } }{1-\log_{\frac{B}{B-1}} (\frac{8G^3B|\mathcal{V}|\gamma \sqrt{N}-1}{8G^3B(B-1)|\mathcal{V}|\gamma \sqrt{N} }) }
    \end{align*}
    and
    \begin{align*}
        \eta_{v,0} \triangleq \min_{l\in\{1,L-1\}} \frac{\frac{1}{4G^3B|\mathcal{V}|\sqrt{N} ((C+3)\gamma \sqrt{B} )^l } }{1-\log_{\frac{B}{B-1}} (\frac{8G^3B|\mathcal{V}| \sqrt{N}-1}{8G^3B(B-1)|\mathcal{V}| \sqrt{N} }) }
    \end{align*}
\end{lemma}
\begin{proof}
    We have
    \begin{align*}
        &|\mathbb{E}[\langle \nabla_{\theta^{(l)}}\loss, \Delta_{\theta^{(l)},k} \rangle]|\\
        ={}& |\mathbb{E}[\langle \nabla_{\theta^{(l)}}\loss, \widetilde{\mathbf{g}}_{\theta^{(l)}}(\theta^{(l)}_k) - \mathbf{g}_{\theta^{(l)}}(\theta^{(l)}_k) \rangle]|\\
        \leq{}& \mathbb{E}[\|\nabla_{\theta^{(l)}}\loss\|_2 \|\widetilde{\mathbf{g}}_{\theta^{(l)}}(\theta^{(l)}_k) - \mathbf{g}_{\theta^{(l)}}(\theta^{(l)}_k) \|_2]\\
        \leq{}& GB\mathbb{E}[\sum_{v_j\in\mathcal{V}_{\mathcal{B}_k}} \|(\nabla_{\theta^{(l)}} u_j(\hisH_k^{(l-1)};\theta^{(l)}_k))[\hisV^{(l)}_k]_j\\
        &\quad\quad\quad\quad\quad\quad-(\nabla_{\theta^{(l)}} u_j(\embH_k^{(l-1)};\theta^{(l)}_k))[\embV^{(l)}_k]_j\|_2]\\
        \leq{}& GB\mathbb{E}[\sum_{v_j\in\mathcal{V}_{\mathcal{B}_k}} \|(\nabla_{\theta^{(l)}} u_j(\hisH_k^{(l-1)};\theta^{(l)}_k))[\hisV^{(l)}_k]_j\\
        &\quad\quad\quad\quad\quad\quad-(\nabla_{\theta^{(l)}} u_j(\hisH_k^{(l-1)};\theta^{(l)}_k))[\embV^{(l)}_k]_j\\
        &\quad\quad\quad\quad\quad\quad+(\nabla_{\theta^{(l)}} u_j(\hisH_k^{(l-1)};\theta^{(l)}_k))[\embV^{(l)}_k]_j\\
        &\quad\quad\quad\quad\quad\quad-(\nabla_{\theta^{(l)}} u_j(\embH_k^{(l-1)};\theta^{(l)}_k))[\embV^{(l)}_k]_j\|_2]\\
        \leq{}& G^2B\mathbb{E}[\sum_{v_j\in\mathcal{V}_{\mathcal{B}_k}} \|[\hisV_k^{(l)}]_j-[\embV_k^{(l)}]_j\|_2 + \gamma \|[\hisH_k^{(l)}]_j-[\embH_k^{(l)}]_j\|_2]\\
        \leq{}& G^2B\mathbb{E}[\sum_{v_j\in\mathcal{V}_{\mathcal{B}_k}} \|\hisV_k^{(l)}-\embV_k^{(l)}\|_2 + \gamma \|\hisH_k^{(l)}-\embH_k^{(l)}\|_2]\\
        \leq{}& G^2B|\mathcal{V}|\cdot\mathbb{E}[\|\hisV_k^{(l)}-\embV_k^{(l)}\|_2] + G^2B|\mathcal{V}|\gamma\cdot \mathbb{E}[\|\hisH_k^{(l)}-\embH_k^{(l)}\|_2].
    \end{align*}
    For any $N\in\mathbb{N}^*$, let $\varepsilon_1=\frac{1}{4G^2B|\mathcal{V}|\gamma \sqrt{N} }$, $\varepsilon_2=\frac{1}{4G^2B|\mathcal{V}| \sqrt{N} }$, $\delta_1=\frac{1}{8G^3B|\mathcal{V}|\gamma \sqrt{N} }$, and $\delta_2=\frac{1}{8G^3B|\mathcal{V}|\sqrt{N} }$, by Lemma 1 in the main text we know that
    \begin{align*}
        &P(\|\hisH_k^{(l)} - \embH_k^{(l)}\|_F > \varepsilon_1) < \delta_1,\,\,\forall k\in\mathbb{N}^*,\,l\in[L],\\
        &P(\|\hisV_k^{(l)} - \embV_k^{(l)}\|_F > \varepsilon_2) < \delta_2,\,\,\forall k\in\mathbb{N}^*,\,l\in[L],\,\,{\rm if}\,\,\eta\leq\eta_{\theta,0}.
    \end{align*}
    Therefore, if we let $\eta\leq\eta_{\theta,0}$, then we have
    \begin{align*}
        &\mathbb{E}[\| \hisH^{(l)}_k - \embH^{(l)}_k \|_F]\\
        ={}& \mathbb{E}[\| \hisH^{(l)}_k - \embH^{(l)}_k \|_F \cdot \mathbbm{1}_{\{\| \hisH^{(l)}_k - \embH^{(l)}_k \|_F<\varepsilon_1}\} ]\\
        &+ \mathbb{E}[\| \hisH^{(l)}_k - \embH^{(l)}_k \|_F \cdot \mathbbm{1}_{\{\| \hisH^{(l)}_k - \embH^{(l)}_k \|_F\geq\varepsilon_1\}} ]\\
        \leq{}& \varepsilon_1 + 2G\delta_1\\
        = {}& \frac{\varepsilon}{2G^2B|\mathcal{V}|\gamma}
    \end{align*}
    and
    \begin{align*}
        &\mathbb{E}[\| \hisV^{(l)}_k - \embV^{(l)}_k \|_F]\\
        ={}& \mathbb{E}[\| \hisV^{(l)}_k - \embV^{(l)}_k \|_F \cdot \mathbbm{1}_{\{\| \hisV^{(l)}_k - \embV^{(l)}_k \|_F<\varepsilon_1}\} ]\\
        &+ \mathbb{E}[\| \hisV^{(l)}_k - \embV^{(l)}_k \|_F \cdot \mathbbm{1}_{\{\| \hisV^{(l)}_k - \embV^{(l)}_k \|_F\geq\varepsilon_1\}} ]\\
        \leq{}& \varepsilon_2 + 2G\delta_2\\
        \leq{}& \frac{\varepsilon}{2G^2B|\mathcal{V}|},
    \end{align*}
    which leads to
    \begin{align*}
        |\mathbb{E}[\langle \nabla_{\theta^{(l)}}\loss, \Delta_{\theta^{(l)},k} \rangle]| \leq \varepsilon,\,\,\forall\,k\in\mathbb{N}^*. 
    \end{align*}  
\end{proof}

\subsubsection{Proof of Theorem 4}

Following \cite{vrgcn}, we directly assume that conditions \ref{con:1} and \ref{con:2} hold. We then derive that the conditions \ref{con:3} and \ref{con:4} hold for the mini-batch gradients
\begin{align*}
    \widetilde{\mathbf{g}}_w(w^{(k)}) = \frac{1}{|\mathcal{V}_{L_{\mathcal{B}}}^{(k)}|} \sum_{v_j \in \mathcal{V}_{L_{\mathcal{B}}}^{(k)}}\nabla_{w} l_w(\widetilde{\embh}_j,y_j)
\end{align*}
and
\begin{align*}
    \widetilde{\mathbf{g}}_{\theta}(\theta^{(k)}) = \frac{|\mathcal{V}|}{|\inbatch^{(k)}|} \sum_{v_j \in \inbatch^{(k)}} \nabla_{\theta} \update(\widetilde{\embh}_j,\overline{\embm}_{\neighbor{v_j}},\embx_j) \widetilde{\embV}_j
\end{align*}
computed by LMC, where we omit the superscript $(k)$ of $\widetilde{\embh}$, $\overline{\embm}$, $\widetilde{\embV}$, and $\inbatch^{(k)}$ is the mini-batch of nodes at $k$-th iteration. Denote the exact mini-batch gradients by
\begin{align*}
    \mathbf{g}_w(w^{(k)}) = \frac{1}{|\mathcal{V}_{L_{\mathcal{B}}}^{(k)}|} \sum_{v_j \in \mathcal{V}_{L_{\mathcal{B}}}^{(k)}}\nabla_{w} l_w(\embh_j ,y_j)
\end{align*}
and
\begin{align*}
    \mathbf{g}_{\theta}(\theta^{(k)}) = \frac{|\mathcal{V}|}{|\inbatch^{(k)}|} \sum_{v_j \in \inbatch^{(k)}} \nabla_{\theta} \update(\embh_j ,\embm_{\neighbor{v_j}}  ,\embx_j) \embV_j.
\end{align*}

\begin{lemma}\label{lemma:w_amp_variance}
    If $\|\widetilde{\mathbf{g}}_w(w^{(k)})\|_2$ and $\|\nabla_{w} \mathcal{L}\|_2$ are bounded by $G>0$,
    then we have
    \begin{align*}
        \mathbb{E}[\| \Delta_w^{(k)} \|_2^2] \leq B_1,
    \end{align*}
    where $B_1=4G^2$ and $\Delta_w^{(k)}=\widetilde{\mathbf{g}}_w(w^{(k)}) - \nabla_{w} \loss$ is the estimation error of the gradient of $\loss$ with respect to $w$ at $k$-th iteration.
\end{lemma}
\begin{proof}
    By 
    \begin{align*}
        \mathbb{E}[(X_1+X_2)^2] \leq 2(\mathbb{E}[X_1^2]+\mathbb{E}[X_2^2]),
    \end{align*}
    we have
    \begin{align*}
        &\mathbb{E}[\| \Delta_w^{(k)} \|_2^2]\\
        ={}&\mathbb{E}[\| \widetilde{\mathbf{g}}_w(w^{(k)}) - \nabla_{w} \mathcal{L} \|_2^2]\\
        \leq{}& 2( \mathbb{E}[ \| \widetilde{\mathbf{g}}_w(w^{(k)})\|_2^2] +  \mathbb{E}[ \|\nabla_{w} \mathcal{L}\|_2^2])\\
        \leq{}& 4 G^2.
    \end{align*}
\end{proof}

\begin{lemma}\label{lemma:theta_amp_variance}
    If $\|\widetilde{\mathbf{g}}_{\theta}(\theta^{(k)})\|_2$ and $\|\nabla_{\theta} \mathcal{L}\|_2$ are bounded by $G>0$, then we have
    \begin{align*}
        \mathbb{E}[\| \Delta_\theta^{(k)} \|_2^2] \leq B_1,
    \end{align*}
    where $B_1=4G^2$ and $\Delta_\theta^{(k)}=\widetilde{\mathbf{g}}_\theta(\theta^{(k)}) - \nabla_{\theta} \loss$ is the estimation error of the gradient of $\loss$ with respect to $\theta$ at $k$-th iteration.
\end{lemma}
\begin{proof}
    By 
    \begin{align*}
        \mathbb{E}[(X_1+X_2)^2] \leq 2(\mathbb{E}[X_1^2]+\mathbb{E}[X_2^2]),
    \end{align*}
    we have
    \begin{align*}
        &\mathbb{E}[\| \Delta_\theta^{(k)} \|_2^2]\\
        ={}&\mathbb{E}[\| \widetilde{\mathbf{g}}_\theta(\theta^{(k)}) - \nabla_{\theta} \mathcal{L} \|_2^2]\\
        \leq{}& 2( \mathbb{E}[ \| \widetilde{\mathbf{g}}_\theta(\theta^{(k)})\|_2^2] +  \mathbb{E}[ \|\nabla_{\theta} \mathcal{L}\|_2^2])\\
        \leq{}& 4 G^2.
    \end{align*}
\end{proof}

\begin{lemma}\label{lemma:w_amp_inner_product}
    Suppose that the conditions in Lemma \ref{lemma:amp_cfpi} hold, i.e., $f_\theta$ is $L$-Lipschitz for parameter $\theta$ and $\gamma$-contraction for $\embH$. Suppose that $\nabla_{w} l_w(\embh,y)$ is $L$-Lipschitz, and $\|\widetilde{\mathbf{g}}_w(w^{(k)})\|_2, \|\nabla_{w} \mathcal{L}\|_2$ are bounded by $G>0$, then there exists $B_2 \geq \max\{ LG^2, \frac{KLG^2}{1-\rho}\}$ such that
    \begin{align*}
        |\mathbb{E}[\langle \nabla_{w} \mathcal{L}, \Delta_w^{(k)} \rangle]| \leq \eta B_2 + \rho^{k-1} B_2,
    \end{align*}
    where $\rho=\sqrt{(1-(1-\gamma^2)/b)}<1$, $K=\frac{2L}{1-\gamma}$, $b$ is the number of partition subgraphs, and $\Delta_w^{(k)}=\widetilde{\mathbf{g}}_w(w^{(k)}) - \nabla_{w} \loss$ is the estimation error of the gradient of $\loss$ with respect to $w$ at $k$-th iteration.
\end{lemma}
\begin{proof}
    We have
    \begin{align*}
        &|\mathbb{E}[\langle \nabla_{w} \mathcal{L}, \Delta_w^{(k)} \rangle]|\\
        ={}&|\mathbb{E}[\langle \nabla_{w} \mathcal{L}, \widetilde{\mathbf{g}}_w(w^{(k)}) - \nabla_{w} \mathcal{L} \rangle]|\\
        ={}&|\mathbb{E}[\langle \nabla_{w} \mathcal{L}, \widetilde{\mathbf{g}}_w(w^{(k)}) - \mathbf{g}_w(w^{(k)}) \rangle]|\\
        \leq{}& \mathbb{E}[\| \nabla_{w} \mathcal{L}\|_2 \| \widetilde{\mathbf{g}}_w(w^{(k)}) - \mathbf{g}_w(w^{(k)}) \|_2]\\
        \leq{}& G \mathbb{E}[ \frac{1}{|\mathcal{V}^{(k)}_{L_\mathcal{B}}|}\|\sum_{v_j \in \mathcal{V}^{(k)}_{L_\mathcal{B}}}\nabla_w l_w(\widetilde{\embh}_j,y_j)- \nabla_w l_w(\embh_j,y_j)\|_2]\\
        \leq{}& G L \mathbb{E}[ \|\overline{\embH}_{\mathcal{V}_{\mathcal{B}_j}}^{(k)}-\embH _{\mathcal{V}_{\mathcal{B}_j}}^{(k)}\|_F]\\
        \leq{}& G L \mathbb{E}[ \|\overline{\embH}^{(k)} - \embH^{(k)}\|_F]\\
        \leq{}& G L \sqrt{\mathbb{E}[ \|\overline{\embH}^{(k)}-\embH^{(k)}\|_F^2]}
    \end{align*}
    According to Lemma 1 in the main text, we have
    \begin{align*}
        |\mathbb{E}[\langle \nabla_{w} \mathcal{L}, \Delta_w^{(k)} \rangle]| &\leq G L (\frac{KG}{1-\rho}\eta + \rho^{k-1}G)\\
        &=\eta \frac{KLG^2}{1-\rho}+\rho^{k-1}LG^2\\
        &\leq \eta B_2 + \rho^{k-1} B_2.
    \end{align*}
\end{proof}

\begin{lemma}\label{lemma:theta_amp_inner_product}
    Suppose that the conditions in Lemmas \ref{lemma:amp_cfpi} and \ref{lemma:amp_cfpi_b} hold. Suppose that the gradient $\nabla_{\vec{\embH}} \update(\embh_j,\embm_{\neighbor{v_j}},\embx_j)$ is $L$-Lipschitz.
    The norms of gradients $\|\embV\|_F$ and $\|\nabla_{\theta}\update(\embh_j,\embm_{\neighbor{v_j}},\embx_j)\|_F$ are bounded by $G$. Then, there exist constants $\hat{B} \geq \max\{ (1+L|\mathcal{V}|)B^2, (1+L|\mathcal{V}|)\frac{KB^2}{1-\rho}\}$ such that
    \begin{align*}
        |\mathbb{E}[\langle \nabla_{\theta} \mathcal{L}, \widetilde{\mathbf{g}}_{\theta}(\theta^{(k)}) - \nabla_{\theta} \mathcal{L} \rangle]| \leq \hat{B} \rho^{k-1} + \hat{B} \eta,
    \end{align*}
    where $|\mathcal{V}|$ is the number of nodes in the graph.
\end{lemma}
\begin{proof}
    As $\|\mathbf{A}_1\mathbf{a}_1-\mathbf{A}_2\mathbf{a}_2\|_2 \leq \|\mathbf{A}_1\|_F\|\mathbf{a}_1-\mathbf{a}_2\|_2+\|\mathbf{A}_1-\mathbf{A}_2\|_F\|\mathbf{a}_2\|_2$,
    we can bound $\| \widetilde{\mathbf{g}}_{\theta}(\theta^{(k)}) - \mathbf{g}_{\theta}(\theta^{(k)}) \|_2$ by
    \begin{align*}
         &\| \widetilde{\mathbf{g}}_{\theta}(\theta^{(k)}) - \mathbf{g}_{\theta}(\theta^{(k)}) \|_2\\
         ={}& \frac{|\mathcal{V}|}{|\inbatch^{(k)}|} \| \sum_{v_j \in \inbatch^{(k)}} (\nabla_{\theta} \update(\widetilde{\embh}_j,\overline{\embm}_{\neighbor{v_j}},\embx_j) \widetilde{\embV}_j)\\
         &\quad\quad\quad\quad\quad\quad\quad-\nabla_{\theta} \update(\embh_j,\embm_{\neighbor{v_j}},\embx_j) \embV_j\|_2\\
         \leq{}& \frac{|\mathcal{V}|}{|\inbatch^{(k)}|} \sum_{v_j \in \inbatch^{(k)}} \|\nabla_{\theta} \update(\widetilde{\embh}_j,\overline{\embm}_{\neighbor{v_j}},\embx_j) \widetilde{\embV}_j\\
         &\quad\quad\quad\quad\quad\quad\quad-\nabla_{\theta} \update(\embh_j,\embm_{\neighbor{v_j}},\embx_j) \embV_j\|_2\\
         \leq{}& |\mathcal{V}|\max_{v_j \in \inbatch^{(k)}} \| \nabla_{\theta} \update(\widetilde{\embh}_j,\overline{\embm}_{\neighbor{v_j}},\embx_j) \widetilde{\embV}_j\\
         &\quad\,\,\,\quad\quad\quad\quad-\nabla_{\theta} \update(\embh_j,\embm_{\neighbor{v_j}},\embx_j) \embV_j\|_2\\
         \leq{}& |\mathcal{V}|\max_{v_j \in \inbatch}\{ \| \nabla_{\theta} \update(\widetilde{\embh}_j,\overline{\embm}_{\neighbor{v_j}},\embx_j)  \|_F\| \widetilde{\embV}_j - \embV_j\|_2 \\
         &\quad\quad\quad\quad\,+\|\embV_j\|_2\| \nabla_{\theta} \update(\widetilde{\embh}_j,\overline{\embm}_{\neighbor{v_j}},\embx_j)\\
         &\quad\,\,\,\,\quad\quad\quad\,\,\,\quad\quad\quad\quad- \nabla_{\theta} \update(\embh_j,\embm_{\neighbor{v_j}},\embx_j) \|_F\}\\
         \leq{}& \max_{v_j \in \inbatch^{(k)}} \{ \| \nabla_{\theta} \update(\widetilde{\embh}_j^{(k)},\overline{\embm}_{\neighbor{v_j}}^{(k)},\embx_j)  \|_F\| \overline{\embV}_j^{(k)} - \embV_j^{(k)} \|_2\\
         &+GL \max(\| \overline{\embh}_j - \embh_j \|_F,\| \widetilde{\embh}_j - \embh_j \|_F) \}\\
         \leq{}&  B \|\overline{\embV}^{(k)} - \embV^{(k)}\|_F + BL|\mathcal{V}|\|\overline{\embH}^{(k)} -  \embH^{(k)}\|_F.
    \end{align*}
    By Lemmas 1 and 2 in the main text, we have
    \begin{align*}
        \mathbb{E}[\|\overline{\embV}^{(k)} - \embV ^{(k)}\|_F] \leq d_v^{(k+1)} \leq \rho^{k-1}B + \frac{KB}{1-\rho} \eta,\\
        \mathbb{E}[\|\overline{\embH}^{(k)} -  \embH^{(k)}\|_F] \leq d_h^{(k+1)} \leq \rho^{k-1}B + \frac{KB}{1-\rho} \eta.
    \end{align*}
    By taking the expectation, we have
    \begin{align*}
        &\mathbb{E}[\| \widetilde{\mathbf{g}}_{\theta}(\theta^{(k)}) - \mathbf{g}_{\theta}(\theta^{(k)}) \|_2]\\
        \leq{}& B(\rho^{k-1}B+\frac{KB}{1-\rho}\eta)+BL|\mathcal{V}|(\rho^{k-1}B+\frac{KB}{1-\rho}\eta)\\
        ={}& \hat{B} \rho^{k-1} + \hat{B} \eta.
    \end{align*}
\end{proof}

According to Lemma \ref{lemma:w_amp_variance}, \ref{lemma:theta_amp_variance}, \ref{lemma:w_amp_inner_product}, and \ref{lemma:theta_amp_inner_product}, the conditions in Lemma \ref{lemma:sgd} hold. Theorem 2 in the main text follows immediately.

\section{More Experiments}

In this section, we report more experiments on the prediction performance for more baselines including the state-of-the-art RecGNNs (SSE \cite{sse} and IGNN \cite{ignn}) and \modify{}{our implemented RecGCNs trained by} full-batch gradient descent (GD), CLUSTER \cite{deq_gcn}, GraphSAINT \cite{graphsaint}, and GAS.

We evaluate these methods on the protein-protein interaction dataset (PPI), the Reddit dataset \cite{graphsage}, the Amazon product co-purchasing network dataset (AMAZON) \cite{amazon}, and open graph benchmarkings (Ogbn-arxiv and Ogbn-products) \cite{ogb} in Table \ref{tab:largegraph}.
GraphSAINT suffers from the out-of-memory issue in the evaluation process on Ogbn-products.
LMC outperform GD, GraphSAINT, and CLUSTER by a large margin and achieve the comparable prediction performance against GAS.
Moreover, Table 1 in the main text demonstrates that LMC enjoys much faster convergence than GAS, showing the effectiveness of LMC.
Therefore, LMC can accelerate the training of RecGNNs without sacrificing accuracy.


\begin{table}[h]
  \centering
  \caption{%
    \textbf{Performance on large graph datasets.}
    OOM denotes the out-of-memory issue.
  }\label{tab:largegraph}
  \setlength{\tabcolsep}{2pt}
  \resizebox{\linewidth}{!}{%
  \begin{tabular}{llcccccc}
    \toprule
    \mc{2}{l}{\footnotesize{\textbf{\#\,nodes}}} & \footnotesize{57K}  & \footnotesize{230K}  & \footnotesize{335K} & \footnotesize{169K} & \footnotesize{2.4M} \\[-0.1cm]
    \mc{2}{l}{\footnotesize{\textbf{\#\,edges}}} & \footnotesize{794K}  &  \footnotesize{11.6M} & \footnotesize{926K} & \footnotesize{1.2M} & \footnotesize{61.9M} \\[-0.05cm]
    \mc{2}{l}{\textbf{Method}} & \textsc{PPI} & \textsc{REDDIT} & \textsc{Amazon}& \texttt{Ogbn-arxiv} & \texttt{Ogbn-products} \\
    \midrule
    &\textsc{SSE}         & 83.60          & ---              & 85.08                  & ---                  & --- \\
    &\textsc{IGNN}        & 97.56          & 94.95            & 85.31              & 70.88              & OOM \\
    \midrule
    & \textsc{GD}   & OOM   & 96.35   & 86.61 & 68.43   & OOM \\
    & \textsc{GraphSAINT}   & 61.14   & 96.60   & 43.53 & 70.93   & OOM \\
    & \textsc{CLUSTER}   & 97.37     & 94.49 & 85.69 & 69.23 & 76.32 \\ 
    & \textsc{GAS}   & 98.88   & 96.44 & 88.34 & 71.92   & 77.32 \\ 
    & \textsc{LMC}   & \textbf{98.99}     & \textbf{96.58} & \textbf{88.84} & \textbf{72.64}   & \textbf{77.73} \\ 
    \bottomrule
  \end{tabular}
 }
\end{table}

\section{Potential Societal Impacts}

In this paper, we propose  novel and scalable training algorithm for RecGNNs.
This work is promising in many practical and important scenarios such as search engine, recommendation systems, biological networks, and molecular property prediction.
Nonetheless, this work may have some potential risks. For example, using this work in search engine and recommendation systems to over-mine the behavior of users may cause undesirable privacy disclosure. 

\bibliographystyle{IEEEtran}
\bibliography{LMC@TPAMI}